\documentclass[twoside,11pt]{article}

\usepackage{jmlr2earxiv}

\ShortHeadings{Implicit bias of gradient descent for regression}{Jin and Mont\'ufar}
\firstpageno{1}

\usepackage[utf8]{inputenc} 
\usepackage[T1]{fontenc}    
\usepackage{booktabs}       
\usepackage{amsfonts}       
\usepackage{nicefrac}       
\usepackage{microtype}      
\usepackage{amsmath}
\usepackage{tikz}
\usepackage[normalem]{ulem}
\usepackage{mathtools}
\usepackage{stmaryrd}
\usepackage{bbm}
\usepackage{bm}
\usepackage{enumitem}
\usepackage{verbatim}

\usepackage{textcomp}
\usepackage{color}
\usepackage{xcolor}
\newcommand{\gm}[1]{#1} 
\newcommand{\gmm}[1]{#1} 
\newcommand{\hj}[1]{#1} 

\newenvironment{highlight}{\par\color{black}}{\par}

\begin{document}

\title{Implicit Bias of Gradient Descent for Mean Squared Error Regression with \hj{Two-Layer} Wide Neural Networks}

\author{%
  \name Hui Jin \email huijin@ucla.edu\\
  \addr Department of Mathematics\\
  University of California, Los Angeles\\
  Los Angeles, CA 90095, USA 
 \AND
  \name Guido Mont\'ufar \email montufar@math.ucla.edu\\
  \addr Department of Mathematics and Department of Statistics\\
  University of California, Los Angeles\\
  Los Angeles, CA 90095, USA; and \\
  Max Planck Institute for Mathematics in the Sciences\\ 
  04103 Leipzig, Germany 
}

\editor{}

\maketitle

\begin{abstract}%
We investigate gradient descent training of wide neural networks and the corresponding implicit bias in function space. For univariate regression, we show that the solution of training a width-$n$ shallow ReLU network is within $n^{- 1/2}$ of the function which fits the training data and whose difference from the initial function has the smallest 2-norm of the second derivative weighted by a curvature penalty that depends on the probability distribution that is used to initialize the network parameters. We compute the curvature penalty function explicitly for various common initialization procedures. For instance, asymmetric initialization with a uniform distribution yields a constant curvature penalty, and thence the solution function is the natural cubic spline interpolation of the training data. \hj{For stochastic gradient descent we obtain the same implicit bias result.}  We obtain a similar result for different activation functions. For multivariate regression we show an analogous result, whereby the second derivative is replaced by the Radon transform of a fractional Laplacian. For initialization schemes that yield a constant penalty function, the solutions are polyharmonic splines. Moreover, we show that the training trajectories are captured by trajectories of smoothing splines with decreasing regularization strength. 
\end{abstract}


\section{Introduction} 
Understanding why artificial neural networks trained in the overparametrized regime and without explicit regularization generalize well in practice is one of the key challenges in contemporary deep learning \citep{zhang2016understanding}. 
A series of works have observed that this phenomenon must involve some form of capacity control beyond the network size \citep{neyshabur2014search} and, specifically, an implicit bias resulting from the parameter optimization procedures \citep{neyshabur2017geometry}. 
By implicit bias we mean that among the many candidate hypotheses that fit the training data, the optimization procedure selects one which satisfies additional properties benefitting its performance on new data. 
In this work we investigate the implicit bias of gradient descent parameter optimization for mean squared error regression with wide shallow ReLU networks. Our theory shows that gradient descent is biased towards smooth functions. More precisely, the trained functions are well captured by interpolating splines depending on the initial function and the probability distribution that is used to initialize the network parameters. 

Under appropriate conditions, we intuitively expect that gradient descent will be biased towards solutions close to the initial parameter. 
Indeed, considering overparametrized neural networks, 
\cite{oymak2018overparameterized} showed that gradient descent finds a global minimizer of the training objective which is close to the initialization. 
This intuition is spot-on for least squares regression with linearized models. In this case, \cite{zhang2019type} showed that gradient flow optimization converges to the global minimum which is closest to the initialization in parameter space. 
Although neural networks have a non-linear parametrization, 
\cite{jacot2018neural}, \cite{lee2019wide} and  \cite{lai2023generalization} 
showed that the training dynamics of wide neural networks is well approximated by the dynamics of the linearization at a suitable initialization. 
This is referred to as the kernel regime, in contrast to the adaptive regime where the models are not well approximated by their linearization. 
Also, \cite{chizat2019lazy} showed that, under appropriate scaling of the output weights, a model can converge to zero training loss while hardly varying its parameters. This phenomenon is referred to as ``lazy training''. 
On the other hand, it is also possible to relate properties of the parameters to properties of the represented functions. \cite{savarese2019infinite} studied infinite-width univariate (single input) neural networks 
and showed that, under a standard parametrization, the complexity of the represented functions, as measured by the $1$-norm of the second derivative, can be controlled by the $2$-norm of the parameters. 
\cite{ongie2019function} extended these results to the multivariate setting. 
Using these results, one can show that gradient descent with $\ell_2$ weight penalty leads to simple functions. 
We will pursue an approach following these ideas, where we first approximate the gradient dynamics of a wide network in terms of a linear model and then establish a function space description of the implicit bias in parameter space. 


The implicit bias of parameter optimization has also been investigated in terms of the properties of the loss function at the points reached by different optimization procedures~\citep{keskar2016large,Wu2017TowardsUG,pmlr-v70-dinh17b}. 
\cite{gunasekar2018characterizing} analyze the implicit bias of different optimization methods (natural gradient, steepest and mirror descent) for linear regression and separable linear classification problems, and obtain characterizations in terms of minimum norm or max-margin solutions. 
Several works have studied the implicit bias of optimization for classification tasks  %
in terms of margins. 
\cite{soudry2018implicit} showed that in classification problems with separable data, gradient descent with linear networks converges to %
a max-margin solution. 
\cite{gunasekar2018implicit} presented a result on implicit bias for deep linear convolutional networks, and \cite{pmlr-v99-ji19a} studied non-separable data. 
\cite{chizat2020implicit} showed that gradient flow for logistic regression with infinitely wide two-layer networks yields a max-margin classifier in a certain space. 
In the adaptive regime, \cite{maennel2018gradient} showed that gradient flow for shallow ReLU networks initialized close to zero quantizes features depending on the training data but not on the network size. 
\cite{pmlr-v130-baratin21a} showed the evolution of the tangent features during training which can be interpreted as feature selection and compression. 
\citet{williams2019gradient} obtained results for univariate regression contrasting the kernel regime and the adaptive regime. 
We will obtain a related result for univariate regression in the kernel regime and a corresponding result for the multivariate case. 

This article is organized as follows. 
In Section~\ref{sec:notation} we provide settings and notation. 
We present our main results in Section~\ref{sec:main}, along with a discussion. The main techniques pertaining wide networks and the infinite width limit are presented in Sections~\ref{sec:3} and~\ref{2.4}. 
In Sections \ref{sec:implicit_bias_univariate} and \ref{sec:implicit_bias_multivariate}, we present the main derivations for the implicit bias in function space for univariate and multivariate regression. 
In the interest of a concise presentation, technical proofs and extended discussions are deferred to appendices.

\section{Notation and Problem Setup}
\label{sec:notation} 
Consider a fully connected network with $d$ inputs, 
one hidden layer of width $n$, and a single output. 
For any given input $\mathbf{x}\in\mathbb{R}^d$, the output of the network is
\begin{equation}
f(\mathbf{x},\theta)=\sum_{i=1}^n W_i^{(2)}\phi(\langle \mathbf{W}_i^{(1)},\mathbf{x}\rangle +b_i^{(1)}) +b^{(2)},
\label{standard-parametrization}
\end{equation}
where $\phi$ is an entry-wise activation function, 
$\mathbf{W}^{(1)}=(\mathbf{W}^{(1)}_1,\ldots,\mathbf{W}^{(1)}_n)^T\in \mathbb{R}^{n\times d}$, $\mathbf{W}_i^{(1)}=(W^{(1)}_{i,1},\ldots,W^{(1)}_{i,d})^T\in \mathbb{R}^{d}$, $\mathbf{W}^{(2)}=(W^{(2)}_1,\ldots,W^{(2)}_n)^T\in \mathbb{R}^{n}$, $\mathbf{b}^{(1)}=(b^{(1)}_1,\ldots,b^{(1)}_n)^T\in \mathbb{R}^n$ and $b^{(2)}\in \mathbb{R}$ are the weights and biases of the first and second layer. 
We write $\theta=\mathrm{vec}(\mathbf{W}^{(1)},\mathbf{b}^{(1)},\mathbf{W}^{(2)},b^{(2)})$ for the vector of all network parameters. 
These parameters are initialized by independent samples of pre-specified random variables $\mathcal{W}$ and $\mathcal{B}$ as follows: 
\begin{equation}
  \begin{aligned}
    W_{i,j}^{(1)} \buildrel d \over = \sqrt{1/d} ~\mathcal{W}, \quad
    &b_{i}^{(1)} \buildrel d \over = \sqrt{1/d} ~\mathcal{B},\\
    W_{i}^{(2)} \buildrel d \over = \sqrt{1/n} ~\mathcal{W}, \quad
    &b^{(2)} \buildrel d \over = \sqrt{1/n} ~\mathcal{B}.
  \end{aligned}
  \label{initialization}
\end{equation}
In the analysis of \citet{jacot2018neural, lee2019wide}, $\mathcal{W}$ and $\mathcal{B}$ are Gaussian $\mathcal{N}(0,\sigma^2)$. 
In the default initialization of PyTorch \citep{NEURIPS2019_9015}, $\mathcal{W}$ and $\mathcal{B}$ have uniform distribution $\mathrm{Unif}(-\sigma,\sigma)$. More generally, we will also allow weight-bias pairs $(\mathbf{W}_i^{(1)}, b_i^{(1)})$ of units in the hidden layer to be sampled from the joint distribution of a sub-Gaussian $(\bm{\mathcal{W}},\mathcal{B})$, where $\bm{\mathcal{W}}$ is a $d$-dimensional random vector and $\mathcal{B}$ is a random variable. 
The parameters of the second layer are still sampled from random variables $\mathcal{W}^{(2)}$ and $\mathcal{B}^{(2)}$. 
Then the parameters of the network are initialized as follows:
\begin{equation}
  \begin{aligned}
    &(\mathbf{W}_{i}^{(1)},b_{i}^{(1)}) \buildrel d \over = ~(\bm{\mathcal{W}},\mathcal{B})%
    \\
    &W_{i}^{(2)} \buildrel d \over = \sqrt{1/n} ~\mathcal{W}^{(2)}, \quad
    b^{(2)} \buildrel d \over = \sqrt{1/n} ~\mathcal{B}^{(2)}.
  \end{aligned}
  \label{initialization2}
\end{equation}

The setting \eqref{standard-parametrization} is known as the standard parametrization. Some works \citep{jacot2018neural, lee2019wide} use the so-called NTK parametrization, where the factor $\sqrt{1/n}$ is carried outside of the trainable parameter (for details see Appendix~\ref{twoParametrization}). 
If we fix the learning rate for all parameters, gradient descent leads to different trajectories under these two parametrizations (for details see Appendix~\ref{twoParametrization}). 
Our results are presented for the standard parametrization. 

We consider a regression problem for data $\{(\mathbf{x}_j,y_j) \}_{j=1}^M$ with inputs $\mathcal{X}=\{\mathbf{x}_j\}_{j=1}^M$ and outputs $\mathcal{Y}=\{y_j\}_{j=1}^M$. 
For a loss function $\ell \colon \mathbb{R} \times\mathbb{R} \rightarrow\mathbb{R}$, the empirical risk (also called training error) is $L(\theta) = \frac{1}{M}\sum_{j=1}^M \ell(f(\mathbf{x}_j,\theta),y_j)$. 
We will mainly focus on the square loss $\ell(y,\hat y)=\frac{1}{2}\|y-\hat y \|^2$, in which case $L$ is the mean squared error. 
We use full batch gradient descent with a fixed learning rate $\eta$ to minimize $L(\theta)$. 
Writing $\theta_t$ for the parameter at time $t$, and $\theta_0$ for the initialization, this defines an iteration 
\begin{equation}
  \theta_{t+1} = \theta_{t}-\eta\nabla L(\theta) = \theta_{t}-\eta\nabla_\theta f(\mathcal{X},\theta_{t})^T \nabla_{f(\mathcal{X},\theta_{t})}L, 
  \label{gd-iteration}
\end{equation}
where $f(\mathcal{X},\theta_{t}) = [f(\mathbf{x}_1,\theta_{t}),\ldots,f(\mathbf{x}_M,\theta_{t})]^T$ is the vector of network outputs for all training inputs, 
and $\nabla_{f(\mathcal{X},\theta_{t})}L$ is the gradient of $L$ as a function of the network outputs $f(\mathcal{X},\theta_{t})$. 
We will use subscript $i$ to index neurons and subscript $t$ to index time. 
Furthermore, we denote by $\hat{\Theta}_n$ the empirical neural tangent kernel (NTK) of the standard parametrization \eqref{standard-parametrization} at time $0$, which is the matrix 
$\hat{\Theta}_n = \frac{1}{n}\nabla_\theta f(\mathcal{X},\theta_0) \nabla_\theta f(\mathcal{X},\theta_0)^T$. 
We write $C^k$ for the space of real valued functions with continuous $k$th derivatives and $\mathrm{Lip}$ for the space of Lipschitz continuous functions. We use the notations $O_p$ to denote the standard mathematical orders in probability.\footnote{$X_n=O_p(a_n)$ as $n\to\infty$ means that for any $\epsilon > 0$, there exists a finite $M_\epsilon > 0$ and a finite $N_\epsilon > 0$ such that $\mathbb{P}(|X_n/a_n|>M_\epsilon)<\epsilon, \forall n>N_\epsilon$.}

\section{Main Results}
\label{sec:main}
In this section we describe our main results for univariate and multivariate regression, followed by an interpretation and overview of the proof steps developed in the next sections.

\subsection{Univariate Regression} 
We have the following description of the implicit bias in function space when applying gradient descent to univariate least squares regression with wide ReLU neural networks. 
\begin{theorem}
[Implicit bias of gradient descent for univariate regression]
\label{thm:theorem1}
Consider a feedforward network with a single input unit, a hidden layer of $n$ rectified linear units, and a single linear output unit. Assume standard parametrization \eqref{standard-parametrization} and parameter initialization \eqref{initialization2}, which means for each hidden unit the input weight and bias are initialized from a sub-Gaussian $(\mathcal{W},\mathcal{B})$ with continuous joint density $p_{\mathcal{W},\mathcal{B}}$. %
Then, for any finite data set $\{(x_j,y_j)\}_{j=1}^M$ 
and sufficiently large $n$ there exist constants $u,v\in\mathbb{R}$ so that optimization of the mean squared error on the adjusted training data 
 $\{(x_j,y_j-ux_j-v)\}_{j=1}^M$ 
by full-batch gradient descent with sufficiently small step size converges to a parameter $\theta^\ast$ for which the output function $f(x,\theta^\ast)$ %
attains zero training error. 
Furthermore, letting $\zeta(x) = \int_\mathbb{R} |W|^3 p_{\mathcal{W},\mathcal{B}}(W,-Wx)~\mathrm{d}W$ and $S = \operatorname{supp}(\zeta) \cap [\min_j x_j, \max_j x_j]$, we have  $\sup_{x\in S}\|f(x,\theta^\ast) - g^\ast(x)\|_2 = \hj{O_p(n^{-\frac{1}{2}})}$%
over the random initialization $\theta_0$, 
where $g^\ast$ solves following variational problem:\footnote{The existence of the minimum of the variational problem is not obvious. We prove that the minimum exists and the solution of the variational problem is $g^*$.}
\begin{equation}
  \begin{aligned}
\min_{g\in C^2(S)}\quad & \int_S \frac{1}{\zeta(x)} (g''(x) - f''(x,\theta_0))^2~\mathrm{d}x\\
\textup{subject to}\quad & g(x_j) = y_j-ux_j-v ,\quad j=1,\ldots, M .
\end{aligned}
\label{main_result}
\end{equation}
\end{theorem}
The proof is provided in Appendix~\ref{app:proof-thm1}. \hj{Our main theorem also holds when the network parameters are trained by stochastic gradient descent. We provide details in Theorem~\ref{minimum_weight_sgd} and Remark~\ref{rk:sgd_gd} in Appendix~\ref{proof1}.
} 
\hj{In Appendix~\ref{app:skip_connection} we also present a corresponding result for networks with skip connections, which does not need a linear adjustment of the data.} 
We will give an interpretation of the result in Section~\ref{sec:discussion_main_results}. 
We first give the explicit form of $\zeta$ for several common parameter initialization procedures. 

\begin{theorem}[Explicit form of the curvature penalty for common initializations]
\label{proposition:explicit-rho}
\mbox{} 
\begin{enumerate}[label=(\alph*),leftmargin=*,noitemsep]
\item {Gaussian initialization.} Assume that $\mathcal{W}$ and $\mathcal{B}$ are independent, $\mathcal{W}\sim \mathcal{N}(0, \sigma_w^2)$ and $\mathcal{B}\sim \mathcal{N}(0, \sigma_b^2)$. 
Then $\zeta(x)=\frac{2\sigma_w^3\sigma_b^3}{\pi(\sigma_b^2+x^2\sigma_w^2)^2}$. 

\item {Binary-uniform initialization.}  
\label{pro9:bin}
Assume that $\mathcal{W}$ and $\mathcal{B}$ are independent, $\mathcal{W}\in\{-1,1\}$ and $\mathcal{B}\sim \mathrm{Unif}(-a_b, a_b)$ with $a_b\geq I$. 
Then $\zeta$ is constant on $[-I,I]$. 

\item {Uniform initialization.}  
\label{pro9:unif}
Assume that $\mathcal{W}$ and $\mathcal{B}$ are independent, $\mathcal{W}\sim \mathrm{Unif}(-a_w, a_w)$ and $\mathcal{B}\sim \mathrm{Unif}(-a_b, a_b)$ with $\frac{a_b}{a_w}\geq I$. 
Then $\zeta$ is constant on $[-I,I]$. 
\end{enumerate}
\end{theorem}
The proof is provided in Appendix~\ref{appendix:proof-explicit-rho}. 
\begin{remark} %
\label{remark_univarite_curvature}
Theorem~\ref{proposition:explicit-rho}\,\ref{pro9:bin} and \ref{pro9:unif} show that for certain parameter initialization distributions, %
the function $\zeta$ is constant on an interval. 
In this case, the solution $(g(x)-f(x,\theta_0))$ to the variational problem \eqref{main_result} in Theorem~\ref{thm:theorem1} corresponds to cubic spline interpolation with natural boundary conditions \citep[see, e.g.,][]{1967theory}. 
For general $\zeta$, the solution corresponds to a spatially adaptive natural cubic spline, which can be computed numerically by solving a linear system and theoretically in an RKHS formalism (see Appendix~\ref{appendix:splines} for details). 
\end{remark}
For different activation functions, we have the following corollary, proved in Appendix~\ref{Other_activation}. 
\begin{corollary}[Different activation functions]
\label{cor:diff_activation}
Use the same settings as in Theorem~\ref{thm:theorem1} except with activation function $\phi$ instead of ReLU. 
Suppose that $\phi$ is a Green’s function of a linear operator $\mathrm{L}$, i.e., $\mathrm{L}\phi=\delta$, where $\delta$ denotes the Dirac delta function. Assume that $\phi$ is homogeneous of degree $k$, i.e., \ $\phi(ax) = a^k
\phi(x)$ for all $a>0$. Then we can find a function $p$ satisfying $\mathrm{L}p\equiv 0$ and adjust the training data $\{(x_j,y_j)\}_{j=1}^M$ to 
 $\{(x_j,y_j-p(x_j)\}_{j=1}^M$. After that, the statement in Theorem \ref{thm:theorem1} holds with the variational problem \eqref{main_result} changed to 
\begin{equation}
  \begin{aligned}
\min_{g\in C^2(S)}\quad & \int_S \frac{1}{\zeta(x)} [\mathrm{L}(g(x) - f(x,\theta_0))]^2~\mathrm{d}x\\ 
\textup{subject to}\quad & g(x_j) = y_j-p(x_j) ,\quad j=1,\ldots, M , 
\end{aligned}
\label{gen_diff_activation}
\end{equation}
where $\zeta(x) = p_\mathcal{C}(x)\mathbb{E}(\mathcal{W}^{2k}|\mathcal{C}=x) $ and $S = \operatorname{supp}(\zeta) \cap [\min_j x_j, \max_j x_j]$. 
\end{corollary} 
Based on Theorem~\ref{thm:theorem1}, we can also give an approximate description of the optimization trajectory in function space. 
If we substitute the constraints $g(x_j)=y_j$ in \eqref{main_result} by a quadratic penalty $\frac{1}{\lambda} \frac{1}{M}\sum_{j=1}^M(g(x_j)-y_j)^2$, then we obtain the variational problem for a so-called spatially adaptive smoothing spline \citep[see][]{ABRAMOVICH1996327,10.1093/biomet/93.1.113}. This problem can be solved explicitly and can be shown to approximate early stopping. 
In Appendix~\ref{appendix:smoothingspline} we provide details for the following observation. 
\begin{remark}[Training trajectory]
The output function of the network after gradient descent training for $t$ steps with learning rate $\bar{\eta}/n$ is approximated by the solution to following optimization problem: 
\begin{equation}
  \min_{g\in C^2(S)}\quad
\sum_{j=1}^M\left[ g(x_j)-y_j\right]^2+\frac{1}{\bar{\eta} t}
\int_{S} \frac{1}{\zeta(x)}
(g''(x)- f''(x,\theta_0))^2~\mathrm{d}x . 
  \label{regularized_function_space_uniform}
\end{equation} 
\end{remark}

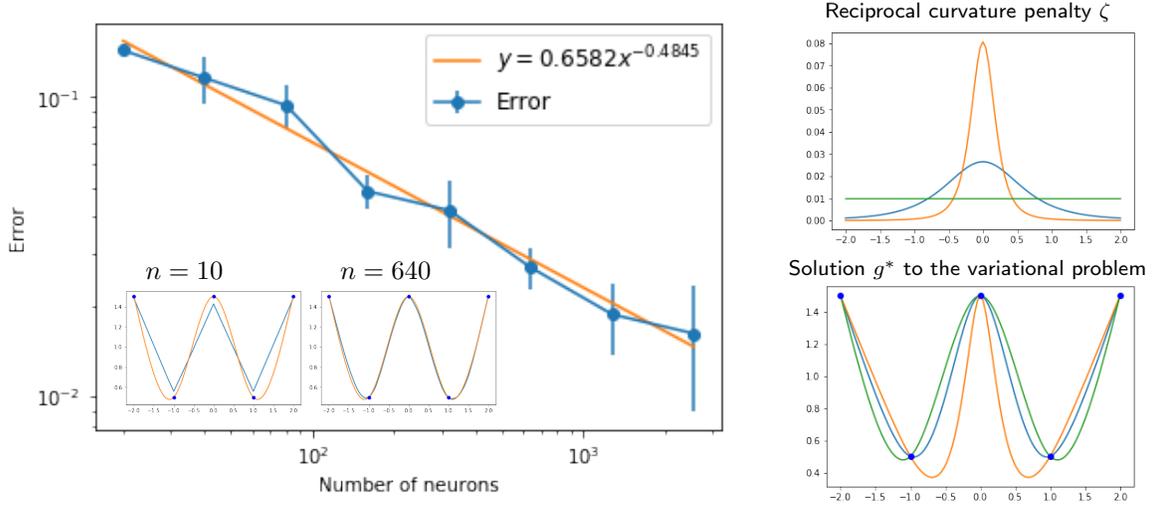
\begin{figure}[t]
  \centering
\begin{tikzpicture}[x=\textwidth,y=\textwidth, inner sep = 1pt]
\node[above right] at (0,0) {\includegraphics[width=0.65\textwidth]{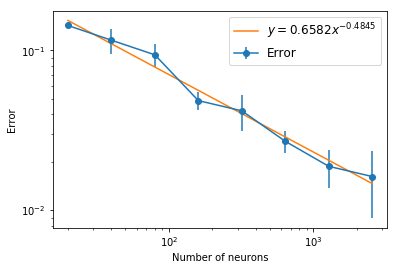}};
\node[above right] at (.1, .08) {\includegraphics[width=0.17\textwidth]{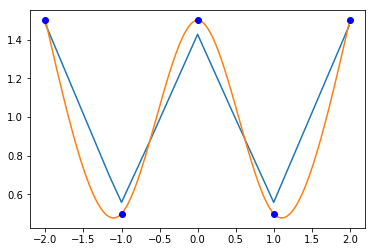}};
\node[above right] at (.27, .08) {\includegraphics[width=0.17\textwidth]{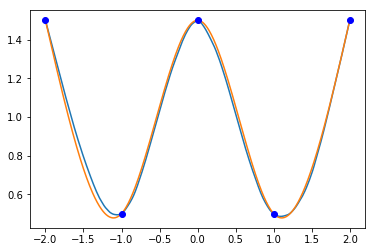}};

\node[above right,fill=white] at (.1+.03, .08+.12) {\small $n=10$};
\node[above right,fill=white] at (.27+.03, .08+.12) {\small $n=640$};

\node[above right] at (.7,0) {\includegraphics[width=0.3\textwidth]{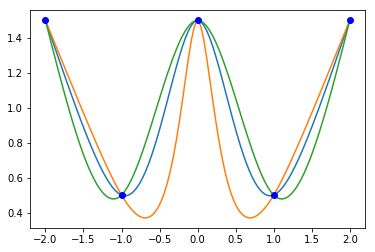}};
\node[above right] at (.7,.225) {\includegraphics[width=0.3\textwidth]{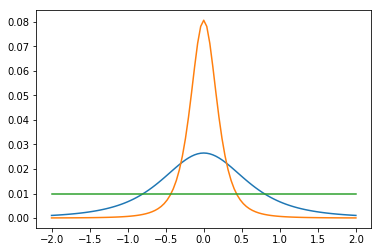}};
\node[above,fill=white] at (.7+.15,0+.2) {\scriptsize \textsf{Solution $g^\ast$ to the variational problem}};
\node[above,fill=white] at (.7+.15,.225+.2) {\scriptsize \textsf{Reciprocal curvature penalty $\zeta$}}; 

\end{tikzpicture}

\caption{Illustration of Theorem~\ref{thm:theorem1}. 
Left: Uniform error between the solution $g^\ast$ to the variational problem and 
the functions $f(\cdot,\theta^\ast)$ obtained by gradient descent training with uniform initialization $\mathcal{W}\sim \mathrm{Unif}(-1,1)$, $\mathcal{B}\sim \mathrm{Unif}(-2,2)$, against the number of neurons $n$. 
The inset shows the training data (dots), $g^\ast$ (orange), and $f(\cdot,\theta^\ast)$ (blue) for two values of $n$. 
Right: Effect of the curvature penalty function on the shape of the solution function. 
    The bottom shows $g^\ast$ for various $\zeta$ shown at the top. 
    The green curve is for $\zeta$ constant on $[-2,2]$, derived from $\mathcal{W}\sim \mathrm{Unif}(-1,1)$, $\mathcal{B}\sim \mathrm{Unif}(-2,2)$; 
    blue is for $\zeta(x)=1/(1+x^2)^2$, derived from $\mathcal{W} \sim \mathcal{N}(0,1)$, $\mathcal{B} \sim \mathcal{N}(0,1)$; 
    and orange for $\zeta(x)=1/(0.1+x^2)^2$, derived from $\mathcal{W} \sim \mathcal{N}(0,1)$, $\mathcal{B} \sim \mathcal{N}(0,0.1)$.  Theorem~\ref{proposition:explicit-rho} shows how to compute $\zeta$ for these distributions.}
\label{fig:nr_neurons}
\label{fig:different_rho}
\end{figure}

\subsection{Multivariate Regression}

For multivariate regression, we have the following generalization of Theorem~\ref{thm:theorem1}. 
\begin{theorem}
[Implicit bias of gradient descent for multivariate regression]
\label{thm:theorem_multi}
Consider the same network settings as in Theorem~\ref{thm:theorem1} except with $d$ input units instead of a single input unit. 
Assume that $\bm{\mathcal{W}}$ is a random vector with $\mathbb{P}(\|\bm{\mathcal{W}}\|=0)=0$ and $\mathcal{B}$ is a random variable; the distribution of $(\bm{\mathcal{W}},\mathcal{B})$ is symmetric, i.e., $(\bm{\mathcal{W}},\mathcal{B})$ and $(-\bm{\mathcal{W}},-\mathcal{B})$ have the same distribution; and $\|\bm{\mathcal{W}}\|_2$ and $\mathcal{B}$ are both sub-Gaussian.
Then, for any finite data set $\{(\mathbf{x}_j,y_j)\}_{i=1}^M$ 
and sufficiently large $n$ there exist a constant vector $\mathbf{u}$ and a constant $v$ so that optimization of the mean squared error on the adjusted training data 
 $\{(\mathbf{x}_j,y_j-\langle \mathbf{u},\mathbf{x}_j\rangle-v)\}_{j=1}^M$ 
by full-batch gradient descent with sufficiently small step size converges to a parameter $\theta^\ast$ for which $f(\mathbf{x},\theta^\ast)$ attains zero training error. 
Furthermore, let $\mathcal{U}=\|\bm{\mathcal{W}}\|_2$, $\bm{\mathcal{V}}=\bm{\mathcal{W}}/\|\bm{\mathcal{W}}\|_2$, $\mathcal{C}=-\mathcal{B}/\|\bm{\mathcal{W}}\|_2$ and $\zeta(\bm{V},c) = p_{\bm{\mathcal{V}},\mathcal{C}}(\bm{V},c)\mathbb{E}(\mathcal{U}^2|\bm{\mathcal{V}}=\bm{V},\mathcal{C}=c)$, where $p_{\bm{\mathcal{V}},\mathcal{C}}$ is the continuous joint density of $(\bm{\mathcal{V}},\mathcal{C})$. 
Then, for any compact set $D\subset \mathbb{R}^d$, we have  $\sup_{\mathbf{x}\in D}\|f(\mathbf{x},\theta^\ast) - g^\ast(\mathbf{x})\|_2 = \hj{O_p(n^{-\frac{1}{2}})}$
over the random initialization $\theta_0$, 
where $g^\ast$ solves following variational problem:  
\begin{equation}
\begin{aligned}
 \min_{g\in \operatorname{Lip}(\mathbb{R}^d)}\quad & \int_{\operatorname{supp}(\zeta)} \frac{\left({\mathcal{R}\{(-\Delta)^{(d+1)/2}(g-f(\cdot,\theta_0))\}(\bm{V},c)}\right)^2}{\zeta(\bm{V},c)}~\mathrm{d}\bm{V}\mathrm{d}c\\
 \textup{subject to}\quad & g(\mathbf{x}_j)=y_j-\langle \mathbf{u},\mathbf{x}_j\rangle-v,\quad j=1,\ldots,M \\
 & \mathcal{R}\{(-\Delta)^{(d+1)/2}(g-f(\cdot,\theta_0))\}(\bm{V},c)=0,\quad (\bm{V},c)\not\in\operatorname{supp}(\zeta) \\
 &(-\Delta)^{(d+1)/2}(g-f(\cdot,\theta_0)) \in L^p(\mathbb{R}^d),\ 1\leq p<d/(d-1). 
\end{aligned}
\label{gen_multi_dim}
\end{equation}
Here $\mathcal{R}$ is the Radon transform defined by $
    \mathcal{R}\{f\}(\bm{\omega},b)\coloneqq \int_{\langle{\bm\omega},\mathbf{x}\rangle=b}f(\mathbf{x})\mathrm{d}s(\mathbf{x})$, 
the fractional power of the negative Laplacian $(-\Delta)^{(d+1)/2}$ is defined in Fourier domain by $\widehat{(-\Delta)^{(d+1)/2}f}(\bm\xi)=\|\bm\xi\|^{d+1} \widehat f(\bm\xi)$, and $\operatorname{Lip}(\mathbb{R}^d)$ is the space of Lipschitz continuous functions on $\mathbb{R}^d$. 
\end{theorem} 
The proof is given in Appendix \ref{app:proof-thm1}. 
\hj{In Appendix~\ref{app:skip_connection} we also present a corresponding result for networks with skip connections, which does not need a linear adjustment of the data.} 
In Proposition~\ref{constant_zeta} we will show that for specific distributions of $(\bm{\mathcal{W}},\mathcal{B})$, the function $\zeta(\bm{V},c)$ is constant on $\operatorname{supp}(\zeta)$, which greatly simplifies the variational problem \eqref{gen_multi_dim}. We prove the following theorem in Appendix~\ref{proof_uniform_curvature}. 

\begin{theorem}[Variational problem for constant $\zeta$]
\label{uniform_curvature}
Suppose $\bm{\mathcal{W}}$ is uniformly distributed on $\mathbb{S}^{d-1}$ and $\mathcal{B}$ is uniformly distributed on $[-a_b,a_b]$. Assume that $a_b\geq\max_i \|\mathbf{x}_i\|_2$. Then the variational problem \eqref{gen_multi_dim} is equivalent to 
\begin{equation}
\begin{aligned}
 \min_{h\in \operatorname{Lip}(\mathbb{R}^d)\cap C(\mathbb{R}^d)}\quad & \int_{\mathbb{R}^{d}} \left((-\Delta)^{(d+3)/4}(h(\mathbf{x})-f(\mathbf{x},\theta_0))\right)^2~\mathrm{d}\mathbf{x}\\
 \textup{subject to}\quad & h(\mathbf{x}_j)=y_j,\quad j=1,\ldots,M\\
 &(-\Delta)^{(d+1)/2}(h(\cdot)-f(\cdot,\theta_0)) \in L^p(\mathbb{R}^d),\ 1\leq p<d/(d-1).
\end{aligned}
\label{function_space_multi_simple}
\end{equation}
\end{theorem}
We can solve the simplified variational problem \eqref{function_space_multi_simple} explicitly. We prove the following theorem in Appendix~\ref{proof_closed_form_sol}. 
\begin{theorem}[Closed form solution]
\label{closed_form_sol}
Suppose $h(\mathbf{x})$ solves the variational problem \eqref{function_space_multi_simple}. Then $h(\mathbf{x})$ is given by
 \begin{equation}
     h(\mathbf{x})-f(\mathbf{x},\theta_0)=\sum_{j=1}^M \lambda_j\|\mathbf{x}-\mathbf{x}_j\|^3+\langle \mathbf{u},\mathbf{x}_i \rangle +v, \label{exact_solution_h}
\end{equation}
where the coefficients $\lambda_j$, $\mathbf{u}$ and $v$ are determined by 
\begin{equation}
\begin{cases}
      \sum_{j=1}^M \lambda_j\|\mathbf{x}_i-\mathbf{x}_j\|^3+\langle \mathbf{u},\mathbf{x}_i \rangle +v= y_i- f(\mathbf{x}_i,\theta_0),\quad i=1,\ldots,M\\
      \sum_{j=1}^M \lambda_j=0\\
      \sum_{j=1}^M \lambda_j\mathbf{x}_j=\mathbf{0}\,.
    \end{cases}    
    \label{coefficients}
\end{equation}
\end{theorem}
\begin{remark} %
A function of the form \eqref{exact_solution_h}--\eqref{coefficients} is referred to as a polyharmonic spline \citep[see][]{potter1981multivariate}, which is a special type of radial basis function interpolation \citep{du2008radial}. 
When $d=1$ (i.e., the univariate case), this corresponds to the natural cubic spline interpolation described in Remark~\ref{remark_univarite_curvature}. 
Finally, we observe that the training trajectory of gradient descent for multivariate regression can be approximately described by a sequence of so-called polyharmonic smoothing splines \citep{segeth2019multivariate} with decreasing regularization parameter, similar to the description \eqref{regularized_function_space_uniform} for the univariate case. 
\end{remark}
\subsection{Discussion of the Main Results}
\label{sec:discussion_main_results}
\textbf{Interpretation}\quad 
An intuitive interpretation of Theorem~\ref{thm:theorem1} is that gradient descent optimization is biased towards smooth functions. 
At those regions of the input space where $\zeta$ is smaller, we can expect the difference between the functions after and before training to have a small curvature. 
We call $\rho = 1/\zeta$ a curvature penalty function. 
The theorem gives an explicit description of the bias in function space depending on the initialization. 
In Theorem~\ref{proposition:explicit-rho} we obtain the explicit form of $\zeta$ for various common parameter initialization procedures. 
In particular, when the parameters are initialized independently from a uniform distribution on a finite interval, $\zeta$ is constant and the problem is solved by the natural cubic spline interpolation of the data.  

We illustrate Theorem~\ref{thm:theorem1} numerically in Figure~\ref{fig:nr_neurons} and more extensively in Appendix~\ref{app:additional-experiments}.\footnote{\hj{The code of our experiments and the plots can be found in our GitHub repository: \href{https://github.com/huijin12/Implicit\_Bias\_Wide\_Neural\_Networks}{\texttt{https://github.com/huijin12/Implicit\_Bias\_Wide\_Neural\_Networks}}}}  
In close agreement with the theory, the solution to the variational problem captures the solution of gradient descent training uniformly with error of order $n^{-1/2}$. %
To illustrate the effect of the curvature penalty function, Figure~\ref{fig:different_rho} also shows the solutions to the variational problem for different values of $\zeta$ corresponding to different initialization distributions. 
We see that indeed at input points where $\zeta$ is small resp.\ peaks strongly, the solution function tends to have a lower curvature resp.\ use a higher curvature in order to fit the training data. 
This description could be used to formulate heuristics for parameter initialization either to ease optimization or to induce specific smoothness priors on the solutions. 
In particular, in Proposition~\ref{set_rho} we will show that any curvature penalty $1/\zeta$ can be implemented by an appropriate choice of the parameter initialization distribution. 

Similar to the univariate case, in the multivariate case gradient descent implicitly controls the complexity of the solution functions obtained upon training. In this case the complexity is measured by the weighted 2-norm of the Radon transform of the $(d+1)/2$ power of the negative Laplacian. The weight function $\zeta$ is again determined by the distribution used to initialize the parameters. 
Although the precise interpretation of these expressions is no longer as straightforward, intuitively the implicit bias corresponds to penalizing a global notion of overall curvature across hyperplanes in the input space. 
For certain parameter initialization distributions, Theorem~\ref{closed_form_sol} shows that the network output after training is a polyharmonic spline. 
We illustrate Theorem~\ref{thm:theorem_multi} numerically in Figure~\ref{error_against_n_2D} and more extensively in Appendix~\ref{app:additional-experiments}. Again in close agreement with the theory, the solution to the variational problem captures the solution returned by gradient descent training with a uniform error of order $n^{-1/2}$. 

These results show that the effective capacity of the network, understood as the set of possible output functions after training, is well captured by a space of cubic splines (polyharmonic splines for multivariate regression) relative to the initial function. This is a space with dimension of order $M$ (the number of training examples) independently of the number of parameters of the network. 

We note that under suitable asymmetric parameter initialization (see Appendix \ref{app:ASI}), it is possible to achieve $f(\cdot,\theta_0)\equiv 0$. 
Then in Theorem \ref{thm:theorem1} and Theorem \ref{thm:theorem_multi}, the regularization is on the curvature of the output function itself (rather than its difference to the initial function). 
Further, we note that although Theorem~\ref{thm:theorem1} and Theorem~\ref{thm:theorem_multi} describe gradient descent training with linearly adjusted data, they also approximately describe training with the original training data (see Appendix~\ref{Difference_between_solutions_of_variational_problems} for more details). 
The adjustment of the training data simply accounts for the fact that the second derivative and the Laplace operator are invariant to addition of linear terms. In practice we can use the coefficients $\mathbf{u}$ and $v$ of linear regression $y_j=\langle \mathbf{u},\mathbf{x}_j\rangle+v+\epsilon_j$, $j=1,\ldots,M$, and set the adjusted data as $\{(\mathbf{x}_j,\epsilon_j)\}_{j=1}^M$. 
\hj{Furthermore, if we consider a network architecture with skip connections from the inputs to the outputs, our result holds for the original training data without any adjustments. We present the details to this result in Appendix~\ref{app:skip_connection}.
}

\hj{%
\textbf{Generalization results}\quad
Theorem~\ref{thm:theorem1} allows us to show how gradient descent on wide neural networks learns a target function. 
In the following paragraphs, we show how the solution to the variational problems \eqref{main_result}, \eqref{regularized_function_space_uniform} and \eqref{gen_multi_dim} converges to a target function as the amount of data increases. 
}

\hj{%
In the so-called univariate noiseless model, the training outputs are given by $y_j=g_0(x_j)$, where $g_0\colon [a,b] \mapsto \mathbb{R}$ is the target function. Let $a=x_0<x_1<\cdots<x_M<x_{M+1}=b$ and $h=\max_i x_{i+1}-x_i$. If $\zeta$ is constant on $[a,b]$, the solution $g^*$ of  \eqref{main_result} is the cubic spline interpolation of the training data. \cite{hall1976optimal} showed in the context of splines that for a target function $g_0\in C^4([a,b])$ one has $\|g^*-g_0\|_\infty\leq C\|g_0^{(4)}\|_\infty h^4$, where $g_0^{(4)}$ is the fourth derivative of $g_0$. 
}

\hj{%
For univariate noisy models, the training outputs are given by $y_j=g_0(x_j)+\epsilon_j$, where $\epsilon_j$ are zero-mean independent random variables with a common variance $\sigma^2$. In this case we use early stopping to smooth out the noise and the training result is characterized by the solution of  \eqref{regularized_function_space_uniform}. If $\zeta$ is constant on $[a,b]$, the solution $g^*$ of  \eqref{regularized_function_space_uniform} is the cubic smoothing spline of the training data. \citet[Theorem 5.8]{ragozin1983error} showed that if $g_0\in C^2([a,b])$ and $\{x_j\}_{j=1}^M$ are the uniform partition of $[a,b]$, then $\mathbb{E}\|g^*-g_0\|_2^2\leq C\left((1/t+(1/M)^4)\|g_0''\|^2+t^{1/4}/M\right)$, where $t$ is the number of training steps. If we choose $t$ to be $\Theta(M^{4/5})$, then  $\mathbb{E}\|g^*-g_0\|_2^2=O(M^{-4/5})$.  This gives us some hints about how to choose the stopping time depending on the number of training samples. 
Similar observations can be obtained for more general settings.  
\citet{ragozin1983error} also gives an error bound for $g^*$ in the case of non-uniform training inputs. 
\citet{eggermont2006uniform} shows a similar result if $\{x_j\}_{j=1}^M$ are sampled independently from a distribution.} 

\hj{%
If $\zeta$ is non-constant on $[a,b]$, the solution $g^*$ of  \eqref{regularized_function_space_uniform} is called a spatially adaptive smoothing spline of the training data. \citet[Corollary 1]{wang2013smoothing} showed that if $g_0\in C^4([a,b])$, $\zeta\in C^3([a,b])$, $t=\Theta(M^{4/9})$ and $\{x_j\}_{j=1}^M$ are sampled from a distribution on $[a,b]$ with bounded positive density function $q\in C^3([a,b])$, then $|g^*(x)-g_0(x)|=O_p(M^{-4/9})$. 
If the curvature of the target function changes a lot on its domain, spatially adaptive smoothing splines with properly chosen $\zeta$ perform better than cubic smoothing splines. 
\citet[Corollary 1]{wang2013smoothing} showed that the optimal $\zeta$ is the solution of a variational problem if the target function is known. They approximate the optimal $\zeta$ by a piecewise constant function and estimate the target function from training data by interpolating splines. Then they numerically solve the variational problem and get a suitable $\zeta$ for the training data. \citet{ABRAMOVICH1996327} and \citet{storlie2010locally} proposed to choose $\zeta$ based on an estimation of the second derivative of $g_0$. 
\citet{liu2010data} used a piecewise constant $\zeta$ and proposed a search algorithm to find such $\zeta$. 
Proposition~\ref{set_rho} shows a way to choose the joint distribution of weight and bias parameters in order to have that $\zeta$ is proportional to a given function. Once we find an appropriate $\zeta$ according to the training data using the methods in the above literature, we can initialize the weight and bias parameters by the corresponding joint distribution and train the wide neural network by gradient descent. According to the theory, this parameter initialization should perform better than uniform or Gaussian initialization.} 

\hj{%
For multivariate noiseless models, if $\zeta$ is constant over its support, then the solution $g^*$ of variational problem \eqref{gen_multi_dim} is the polyharmonic spline. For this setting, \citet[Theorem 3.2]{potter1981multivariate} gave an error bound between $g^*$ and the target function $g_0$.} 

\textbf{Strategy of the proof}\quad
   In Section~\ref{sec:3} we observe that for a linearized model, gradient descent with sufficiently small step size finds the minimizer of the training objective which is closest to the initial parameter \citep[similar to a result by][]{zhang2019type}. 
   Then Theorem~\ref{th-lin-onlyout} shows that the training dynamics of a linearized  wide network is well approximated in parameter and in function space by that of a lower dimensional linear model which trains only the output weights. 
   This property is sometimes taken for granted in the literature. 
   We show that it holds for the standard parametrization, although it does not hold for the NTK parametrization, which leads to the adaptive regime. 
\gm{A similar result has been previously obtained by \citet{NIPS2017_489d0396}.} 
Under these settings, the implicit bias of gradient descent amounts to minimizing the distance from the initial parameter, subject to fitting the training data. 
In Section~\ref{2.4} we relate this description of the implicit bias in parameter space to an alternative optimization problem. 
In Theorem~\ref{theorem4} we show that the solution to this alternative problem has a well defined limit as the width of the network tends to infinity, which allows us to obtain a variational description. 
In Section \ref{sec:implicit_bias_univariate}, we focus on the case of univariate regression. In Theorem~\ref{theorem_func} we translate the description of the bias from parameter space to function space. 
In Section \ref{sec:implicit_bias_multivariate}, we turn to the case of multivariate regression and use the inversion formula of the dual Radon transform to analyze the optimization objective. 
Finally, we exploit recent results \citep[][Proposition 3.2]{lai2023generalization} bounding the difference in function space of the solutions obtained from training a wide network and its linearization to conclude the proof.

\textbf{Related works}\quad
\citet{zhang2019type} described the implicit bias of gradient descent %
in the kernel regime as minimizing %
a kernel norm from initialization, %
subject to fitting the training data. 
Our result can be regarded as making the kernel norm explicit, thus providing an interpretable description of the bias in function space and further illuminating the role of the parameter initialization procedure. 
We prove the equivalence in Appendix~\ref{appendix:relation_NTK}. 
\hj{\cite{cao2019generalization} derived generalization bounds for overparametrized deep neural networks under stochastic gradient descent training. They also approximated the neural network by a linearized model, which is called a neural tangent random feature (NTRF) model in their work. }

\citet{savarese2019infinite} showed that infinitely wide networks with $2$-norm weight regularization represent functions with smallest $1$-norm of the second derivative, an example of which are linear splines (see Appendix~\ref{Parameter_norm_min} for more details). 
A recent work by \citet{Parhi2019MinimumN} further develops this direction %
for two-layer networks with certain activation functions that interpolate data while minimizing a weight norm. 
In contrast, our result characterizes the solutions of training from a given initialization without explicit regularization, which turn out to minimize a weighted $2$-norm of the second derivative and hence correspond to cubic splines. %
Another recent work \citep{Heiss2019HowIR} discusses ridge weight penalty, adaptive splines, and early stopping for one-input ReLU networks training only the output layer. 
The spline perspective for univariate shallow ReLU networks has recently been also discussed by \citet{sahs2020shallow}. 
\hj{\citet{schmidt2020rejoinder} showed that a shallow ReLU network with one input and one output node approximately converges to the natural cubic spline interpolant under SGD training.} 
\citet{williams2019gradient} showed a similar result in the kernel regime for shallow ReLU networks training only the output layer from zero initialization. 
In contrast, we consider the initialization of the second layer and show that the difference from the initial output function is implicitly regularized by gradient descent. We show that the result of training both layers %
can be approximated by training only the second layer in Theorem \ref{th-lin-onlyout}. 
In addition, we give the explicit form of $\zeta$ in Theorem~\ref{proposition:explicit-rho}, while the description given by \citet{williams2019gradient} has a minor error because of a typo in their computation. Significantly, our results also cover multivariate regression, different activation functions, and training trajectories. 

In the multivariate case, \cite{ongie2019function} studied infinite-width neural networks with parameters having bounded norm. 
They showed that the complexity of the functions represented by the network, as measured by the $1$-norm of the Radon transform of the $(d+1)/2$-power of the negative Laplacian of the function, can be controlled by the $2$-norm of the parameters.  Rather than bounding the $2$-norm of the parameters, our result describes the implicit bias of gradient descent and in turn we obtain a weighted $2$-norm. 
A recent work by \cite{parhi2021kinds} considers adding an explicit regularization of $1$-norm of the Radon tranform in function space for multivariate regression, and uses the representer theorem to obtain the solution to the variational problem. In contrast, we consider gradient descent without explicit regularization and the implicit bias turns out to be a weighted $2$-norm. %

\begin{figure}
    \centering
\begin{tikzpicture}[x=\textwidth,y=\textwidth, inner sep = 1pt]
\node[above right] at (0,0) {\includegraphics[width=0.49\textwidth]{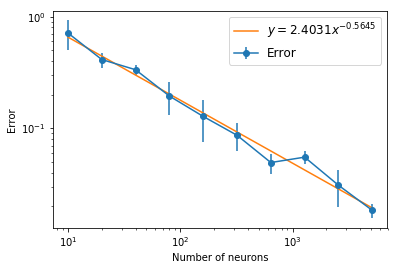}};
\node[above right] at (.5, .175) {\includegraphics[width=0.25\textwidth]{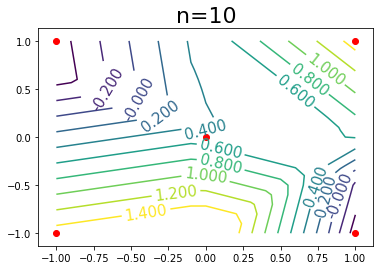}};
\node[above right] at (.75, .175) {\includegraphics[width=0.25\textwidth]{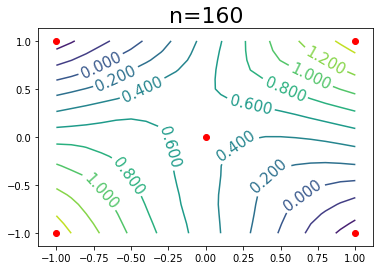}};
\node[above right] at (.5, 0) {\includegraphics[width=0.25\textwidth]{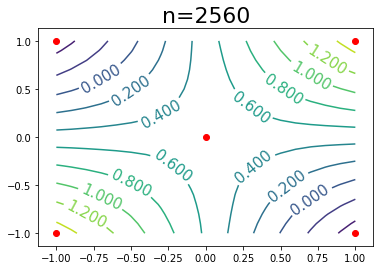}};
\node[above right] at (.75, 0) {\includegraphics[width=0.25\textwidth]{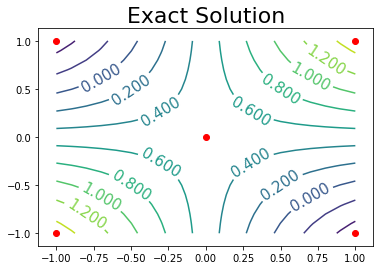}};
\end{tikzpicture}
\caption{Illustration of Theorem~\ref{thm:theorem_multi}. 
Left: Uniform error between the solution $g^\ast$ to the variational problem and 
the functions $f(\cdot,\theta^\ast)$ obtained by gradient descent training of a neural network (in this case with initialization $\bm{\mathcal{W}}\sim \mathrm{Unif}(\mathbb{S}^1)$, $\mathcal{B}\sim \mathrm{Unif}(-2,2)$), against the number of neurons. 
Right: The input training data (dots), the contour plots of trained network functions with $10$, $160$, $2560$ neurons, and the exact solution to the variational problem. 
}
\label{error_against_n_2D}
\end{figure}

\section{Wide Networks and Parameter Space}
\label{sec:3}

In this section, we characterize the implicit bias in parameter space and show that,  under our initialization and parametrization scheme, training only the output layer approximates training all parameters. 

\subsection{Implicit Bias in Parameter Space for a Linearized Model} 
\label{parameter}

In this section we describe how training a linearized network or a wide network by gradient descent leads to solutions having parameter values close to the initial parameter values. 
First, we consider the following linearized model:
\begin{equation}
  f^{\mathrm{lin}}(\mathbf{x},\omega) =  f(\mathbf{x},\theta_0)+\nabla_\theta f(\mathbf{x},\theta_0)(\omega-\theta_0).
  \label{linearized_model}
\end{equation}
We write $\omega$ for the parameter of the linearized model, in order to distinguish it from the parameter $\theta$ of the nonlinearized model. The empirical loss of the linearized model is defined by
\begin{equation}
L^{\mathrm{lin}}(\omega)
=\frac{1}{M}\sum_{j=1}^M \ell(f^{\mathrm{lin}}(\mathbf{x}_j,\omega),y_j).
\label{eq:loss_L}
\end{equation}
The gradient descent iteration for the linearized model is given by 
\begin{equation}
 \omega_0=\theta_0,\quad 
  \omega_{t+1}
  = \omega_t- \eta\nabla_\theta f(\mathcal{X},\theta_{0})^T \nabla_{f^{\mathrm{lin}}(\mathcal{X},\omega_{t})}L^{\mathrm{lin}}. 
  \label{linearized}
\end{equation}

Next, we consider wide neural networks. 
According to \citet[][Theorem H.1]{lee2019wide} and \citep[][Proposition 3.2]{lai2023generalization}, 
\[
\sup_t \|f^{\mathrm{lin}}(\mathbf{x},\omega_t)-f(\mathbf{x},\theta_t)\|_2=\hj{O_p(n^{-\frac{1}{2}})} . 
\]
This means that gradient descent training of a wide network or of the linearization of the network results in similar trajectories and solutions in function space. 
Both solution functions fit the training data perfectly, meaning $f^{\mathrm{lin}}(\mathcal{X},\omega_\infty) = f(\mathcal{X},\theta_\infty)=\mathcal{Y}$, and they are also approximately equal outside of the training data.  

Under the assumption that $\mathrm{rank}(\nabla_\theta f(\mathcal{X},\theta_0))=M$, the gradient descent iterations \eqref{linearized} of the linearized network converge to the unique global minimum that is closest to initialization \citep{gunasekar2018characterizing, zhang2019type}. More precisely, $\omega_\infty$ is the solution to following constrained optimization problem (further details are provided in Appendix~\ref{proof1}):
\begin{equation}
  \min_{\omega} \|\omega- \theta_0\|_2 \quad \text{s.t. } f^{\mathrm{lin}}(\mathcal{X},\omega)=\mathcal{Y}. 
  \label{problem}
\end{equation}

\subsection{Training Only the Output Layer Approximates Training All Parameters}
\label{Training_only_the_output_layer}

In the following we consider networks with a single hidden layer of $n$ ReLUs and a linear output, 
$  
f(\mathbf{x},\theta)=\sum_{i=1}^n W_i^{(2)}[\langle \mathbf{W}_i^{(1)},\mathbf{x}\rangle +b_i^{(1)}]_+ +b^{(2)}
$.
We show that the functions and parameter vectors obtained by training the linearized model are close to those obtained by training only the output layer. 
In view of the previous subsection, this implies that training all parameters of a wide network or training only the output layer results in similar functions. 

Let $\theta_0=\mathrm{vec}(\overline{\mathbf{W}}^{(1)},\overline{\mathbf{b}}^{(1)},\overline{\mathbf{W}}^{(2)},\overline{b}^{(2)})$ be the parameter at initialization so that $f^{\mathrm{lin}}(\cdot, \theta_0) = f(\cdot,\theta_0)$. 
Denote the trained parameter of the linearized network by $\omega_\infty=\mathrm{vec}(\mathbf{\widehat{W}}^{(1)},\widehat{\mathbf{b}}^{(1)},\widehat{\mathbf{W}}^{(2)},\widehat{b}^{(2)})$. 
Using initialization \eqref{initialization2}, 
given $1\leq i\leq n$, we have that $\|\overline{\mathbf{W}}_i^{(1)}\|, \overline{b}_i^{(1)}=\hj{O_p(1)}$ and $\overline{W}_i^{(2)}, \overline{b}^{(2)}=\hj{O_p(n^{-\frac{1}{2}})}$.%
\footnote{More precisely, given $1\leq i\leq n$, $\exists C$, for any $\delta>0$, %
s.t.\ with prob.\ $1-\delta$, $|\overline{W}_i^{(2)}|, |\overline{b}^{(2)}| \leq Cn^{-1/2}\sqrt{\log\frac{1}{\delta}}$ and $\|\overline{\mathbf{W}}_i^{(1)}\|, |\overline{b}_i^{(1)}| \leq C\sqrt{\log\frac{1}{\delta}}$ since the random variables are sub-Gaussian.} 
Therefore, writing $H$ for the Heaviside function, we have  
\begin{align}
\begin{split}
  \nabla_{\mathbf{W}_i^{(1)}, b^{(1)}_i} f(\mathbf{x},\theta_0)=&
  \left[\overline{W}_i^{(2)} H(\langle\overline{\mathbf{W}}_i^{(1)},\mathbf{x}\rangle+\overline{b}^{(1)}_i)\cdot \mathbf{x}
  \;,\;
  \overline{W}_i^{(2)} H(\langle\overline{\mathbf{W}}_i^{(1)},\mathbf{x}\rangle+\overline{b}_i^{(1)})
  \right]
  =\hj{O_p(n^{-\frac{1}{2}})}, \\
\nabla_{W_i^{(2)},b^{(2)}} f(\mathbf{x},\theta_0)=&
\left[
[\langle\overline{\mathbf{W}}_i^{(1)},\mathbf{x}\rangle+\overline{b}_i^{(1)}]_+
\;,\;
1
\right]
=\hj{O_p(1)}. 
\end{split}
\end{align}
This implies that when $n$ is large, if we use gradient descent with a constant learning rate for all parameters, then the changes of $\mathbf{\mathbf{W}}^{(1)}$, $\mathbf{b}^{(1)}$, $b^{(2)}$ are negligible compared with the changes of $\mathbf{W}^{(2)}$. 
In turn, approximately we can train just the output weights, $W^{(2)}_i, i=1,\ldots,n$, and fix all other parameters, which corresponds to training a smaller linear model. 
Let $\widetilde{\omega}_t=\mathrm{vec}(\overline{\mathbf{W}}^{(1)},\overline{\mathbf{b}}^{(1)},\widetilde{\mathbf{W}}_t^{(2)},\overline{b}^{(2)})$ be the parameter at time $t$ under the update rule where  $\overline{\mathbf{\mathbf{W}}}^{(1)},\overline{\mathbf{b}}^{(1)}$, $\overline{b}^{(2)}$ are kept fixed at their initial values, and 
\begin{equation}
  \widetilde{\mathbf{W}}^{(2)}_0=\overline{\mathbf{W}}^{(2)},\quad 
  \widetilde{\mathbf{W}}^{(2)}_{t+1}=\widetilde{\mathbf{W}}^{(2)}_{t}- \eta\nabla_{\mathbf{W}^{(2)}} L^{\mathrm{lin}}(\widetilde{\omega}_t). 
  \label{fix_first}
\end{equation}
Let $\widetilde{\omega}_\infty = \lim_{t\to\infty}\widetilde{\omega}_t$. 
By the above discussion, we expect that $f^{\mathrm{lin}}(\mathbf{x},\widetilde{\omega}_\infty)$ will be close to $f^{\mathrm{lin}}(\mathbf{x},\omega_\infty)$. 
We have the following formal result for mean squared error regression. 

\begin{theorem}[Training only output weights vs linearized network]  
\label{th-lin-onlyout}
Consider a finite data set $\{(\mathbf{x}_i,y_i)\}_{i=1}^M$. 
Assume that  we use the square loss $\ell(\widehat{y},y)=\frac{1}{2}\|\widehat{y}-y\|_2^2$ and $\inf_n \lambda_{\min}(\hat{\Theta}_n)>0$. 
Let $\omega_t$ denote the parameters of the linearized model at time $t$ when we train all parameters using \eqref{linearized}, and let $\widetilde{\omega}_t$ denote the parameters at time $t$ when we only train weights of the output layer using \eqref{fix_first}. 
If we use the same learning rate $\eta$ in these two training processes and $\eta < \frac{2}{n\lambda_{\max}(\hat{\Theta}_n)}$, then for 
any compact set $D\subset \mathbb{R}^d$, we have 
\begin{equation*}
    \sup_{\mathbf{x}\in D}\sup_t |f^{\mathrm{lin}}(\mathbf{x},\widetilde{\omega}_t)- f^{\mathrm{lin}}(\mathbf{x},\omega_t)|=\hj{O_p(n^{-1})}, \text{ as } n\to\infty. 
\end{equation*}
Moreover, in terms of the parameter trajectories we have     $\sup_t \|\overline{\mathbf{W}}^{(1)}- \widehat{\mathbf{W}}^{(1)}_t\|_2=\hj{O_p(n^{-1})}$, $\sup_t \|\overline{\mathbf{b}}^{(1)}- \widehat{\mathbf{b}}^{(1)}_t\|_2 =\hj{O_p(n^{-1})}$,  
    $\sup_t \|\widetilde{\mathbf{W}}^{(2)}_t- \widehat{\mathbf{W}}^{(2)}_t\|_2=\hj{O_p(n^{-3/2})}$,  
    $\sup_t \|\overline{b}^{(2)}- \widehat{b}^{(2)}_t\| =\hj{O_p(n^{-1})}$. 
\end{theorem}
The proof is provided in Appendix~\ref{Proof2}. 
By combining Theorem~\ref{th-lin-onlyout} and the fact that training a linearized model approximates training a wide network \citep[][Proposition 3.2]{lai2023generalization}, we obtain the following. 
\begin{corollary}[Training only output weights vs training all weights]
\label{cor:outlayer-only}
Consider the settings of Theorem~\ref{th-lin-onlyout}, and assume that the joint distribution of $(\mathcal{W},\mathcal{B})$ is sub-Gaussian. Given any compact set $D\subset \mathbb{R}^d$, 
$\sup_{\mathbf{x}\in D}\sup_t \|f^{\mathrm{lin}}(\mathbf{x},\tilde\omega_t)-f(\mathbf{x},\theta_t)\|_2=\hj{O_p(n^{-\frac{1}{2}})}$. 
\end{corollary} 
The proof is given in Appendix~\ref{appendix:Lee-generalization}. 
In view of the arguments in this section, in the next sections we will focus on training only the output weights and understanding the corresponding solution functions.

\section{Infinite Width Limit of Shallow Networks} 
\label{2.4}

According to \eqref{problem}, gradient descent training of the output weights \eqref{fix_first} achieves zero loss, 
$f^{\mathrm{lin}}(\mathbf{x}_{j}, \widetilde{\omega}_\infty)-f^{\mathrm{lin}}(\mathbf{x}_{j}, \theta_0) 
  =\sum_{i=1}^n (\widetilde{W}_i^{(2)}-\overline{W}_i^{(2)})[\langle \overline{\mathbf{W}}_i^{(1)}, \mathbf{x}_{j}\rangle+\overline{b}_i^{(1)}]_+
  =y_{j}-f(\mathbf{x}_{j}, \theta_0)$, $j=1,\ldots, M$, 
with minimum $\|\widetilde{\mathbf{W}}^{(2)}-\overline{\mathbf{W}}^{(2)}\|^2_2$. 
Hence gradient descent is actually solving 
\begin{equation}
 \min_{\mathbf{W}^{(2)}}  \|\mathbf{W}^{(2)}-\overline{\mathbf{W}}^{(2)}\|^2_2 \quad \text{s.t.}\quad \sum_{i=1}^n (W_i^{(2)}-\overline{W}_i^{(2)})[\langle \overline{\mathbf{W}}_i^{(1)}, \mathbf{x}_{j}\rangle+\overline{b}_i^{(1)}]_+=y_j-f(\mathbf{x}_{j}, \theta_0), \;  j=1,\ldots, M.
\label{direct_version_non_ASI}
\end{equation} 
To simplify the presentation, in the following we let $f^{\mathrm{lin}}(\mathbf{x}, \theta_0)\equiv 0$ by using the Anti-Symmetrical Initialization (ASI) trick (see Appendix~\ref{app:ASI}). The analysis still goes through without this simplification (see Appendix~\ref{app:univariate}). 

We reformulate problem %
\eqref{direct_version_non_ASI} 
in a way that allows us to consider the limit of infinitely wide networks, with $n\to\infty$, and obtain a deterministic counterpart, analogous to the convergence of the NTK. 
Let $\mu_n$ denote the empirical distribution of the samples $(\overline{\mathbf{W}}_i^{(1)},\overline{b}_i^{(1)})_{i=1}^n$, 
i.e., $\mu_n(A)=\frac{1}{n}\sum_{i=1}^n \mathbbm{1}_A\left((\overline{\mathbf{W}}_i^{(1)},\overline{b}_i^{(1)})\right)$, where $\mathbbm{1}_A$ denotes the indicator function for measurable subsets $A$ in $\mathbb{R}^{d+1}$. 
We further consider a function $\alpha_n\colon \mathbb{R}^{d+1}\to\mathbb{R}$ whose value encodes the difference of the output weight from its initialization for a hidden unit with input weight and bias given by the argument, i.e., $\alpha_n(\overline{\mathbf{W}}_i^{(1)},\overline{b}_i^{(1)})=n(W_i^{(2)}-\overline{W}_i^{(2)})$. 
Then \eqref{direct_version_non_ASI} with ASI can be rewritten as 
\begin{equation}
\begin{aligned}
 \min_{\alpha_n\in C(\mathbb{R}^{d+1})}\  \int_{\mathbb{R}^2} \alpha_n^2(\mathbf{W}^{(1)},b)~\mathrm{d}\mu_n(\mathbf{W}^{(1)},b)\
 \textup{ s.t.}\int_{\mathbb{R}^{d+1}} \alpha_n(\mathbf{W}^{(1)},b)[\langle \mathbf{W}^{(1)}, \mathbf{x}_{j}\rangle+b]_+~\mathrm{d}\mu_n(\mathbf{W}^{(1)},b)=y_j,
\end{aligned}
\label{probablity_version}
\end{equation}
where $j$ ranges from $1$ to $M$. Here we minimize over functions $\alpha_n$ in $C(\mathbb{R}^{d+1})$, but since only the values on $(\overline{\mathbf{W}}_i^{(1)},\overline{b}_i^{(1)})_{i=1}^n$ are taken into account, 
we can take any continuous interpolation of $\alpha_n(\overline{\mathbf{W}}_i^{(1)},\overline{b}_i^{(1)})$, $i=1,\ldots,n$. 

Now we can consider the infinite width limit. 
Let $\mu$ be the probability measure of $(\bm{\mathcal{W}},\mathcal{B})$. 
By substituting $\mu$ for $\mu_n$, we obtain a continuous version of problem~\eqref{probablity_version} as follows: 
\begin{equation}
\begin{aligned}
 \min_{\alpha\in C(\mathbb{R}^{d+1})}\quad & \int_{\mathbb{R}^{d+1}} \alpha^2(\mathbf{W}^{(1)},b)~\mathrm{d}\mu(\mathbf{W}^{(1)},b)\\
 \textup{subject to}\quad & \int_{\mathbb{R}^{d+1}} \alpha(\mathbf{W}^{(1)},b)[\langle \mathbf{W}^{(1)}, \mathbf{x}_{j}\rangle+b]_+~\mathrm{d}\mu(\mathbf{W}^{(1)},b)=y_j, \quad j=1,\ldots,M . 
\end{aligned}
\label{continuous_version}
\end{equation}
Using that $\mu_n$ weakly converges to $\mu$, 
the following theorem shows that in fact the solution of 
problem \eqref{probablity_version} converges to the solution of 
\eqref{continuous_version}. The proof is given in Appendix~\ref{Proof4}. 

\begin{theorem}[Infinite width limit]
\label{theorem4}
Let $(\overline{\mathbf{W}}_i^{(1)},\overline{b}_i^{(1)})_{i=1}^n$ be i.i.d.\ samples from a pair $(\bm{\mathcal{W}},\mathcal{B})$ with finite fourth moment. 
Suppose $\mu_n$ is the empirical distribution of $(\overline{\mathbf{W}}_i^{(1)},\overline{b}_i^{(1)})_{i=1}^n$ and $\overline{\alpha}_n(\mathbf{W}^{(1)},b)$ is the solution of \eqref{probablity_version}. 
Let $\overline{\alpha}(\mathbf{W}^{(1)},b)$ be the solution of \eqref{continuous_version}. 
Then, for any compact set $D\subset \mathbb{R}^d$, we have $\sup_{\mathbf{x}\in D}|g_n(\mathbf{x},\overline{\alpha}_n)- g(\mathbf{x},\overline{\alpha})|=O_p(n^{-1/2})$ 
, where $g_n(\mathbf{x},\alpha_n) = \int_{\mathbb{R}^{d+1}} \alpha_n(\mathbf{W}^{(1)},b)[\langle \mathbf{W}^{(1)}, \mathbf{x}\rangle+b]_+~\mathrm{d}\mu_n(\mathbf{W}^{(1)},b)$
is the function represented by a network with $n$ hidden neurons after training, 
and $g(\mathbf{x},\alpha) = \int_{\mathbb{R}^{d+1}} \alpha(\mathbf{W}^{(1)},b)[\langle \mathbf{W}^{(1)}, \mathbf{x}\rangle+b]_+~\mathrm{d}\mu(\mathbf{W}^{(1)},b)$ is
the function represented by the infinite-width network. 
\end{theorem}

\section{Implicit Bias for Univariate Regression} 
\label{sec:implicit_bias_univariate}
In this section we solve the optimization problem \eqref{continuous_version} in the univariate case, which provides a function space characterization of the implicit bias previously described in parameter space. 
First we rewrite the problem in terms of breakpoints. 
Consider the breakpoint $c = -b/W^{(1)}$ of a ReLU with weight $W^{(1)}$ and bias $b$. 
We define a corresponding random variable $\mathcal{C} = -\mathcal{B}/\mathcal{W}$ and let $\nu$ denote the distribution of $(\mathcal{W},\mathcal{C})$.\footnotemark 
\footnotetext{Here we assume that $\mathbb{P}(\mathcal{W}=0)=0$ so that the random variable $\mathcal{C}$ is well defined. This is not an important restriction, since neurons with weight $W^{(1)}=0$ have a constant output value that can be absorbed in the bias of the output layer.}  
Then, writing $\gamma(W^{(1)},c)=\alpha(W^{(1)},-cW^{(1)})$, 
the optimization problem \eqref{continuous_version} is equivalently given as 
\begin{equation}
\begin{aligned}
 \min_{\gamma\in C(\mathbb{R}^2)} \int_{\mathbb{R}^2} \gamma^2(W^{(1)},c)~\mathrm{d}\nu(W^{(1)},c)\
 \textup{ s.t.}\int_{\mathbb{R}^2} \gamma(W^{(1)},c)[W^{(1)}(x_j-c)]_+~\mathrm{d}\nu(W^{(1)},c)=y_j,
\end{aligned}
\label{continuous_new}
\end{equation}
where $j$ ranges from $1$ to $M$. Let $\nu_\mathcal{C}$ denote the  distribution of $\mathcal{C} = -\mathcal{B}/\mathcal{W}$, and $\nu_{\mathcal{W}|\mathcal{C}=c}$ the conditional distribution of $\mathcal{W}$ given $\mathcal{C}=c$. 
Suppose $\nu_\mathcal{C}$ has support $\mathrm{supp}(\nu_\mathcal{C})$ and a density function $p_\mathcal{C}(c)$. 
Let $g(x,\gamma) = \int_{\mathbb{R}^2}\gamma(W^{(1)},c)[W^{(1)}(x-c)]_+~\mathrm{d}\nu(W^{(1)},c)$, which again corresponds to the output function of the network. 
Then, the second derivative $g''$ with respect to $x$  %
satisfies $g''(x,\gamma)
=p_{\mathcal{C}}(x)\int_{\mathbb{R}}\gamma(W^{(1)},x)\big|W^{(1)}\big|~\mathrm{d}\nu_{\mathcal{W}|\mathcal{C}=x}(W^{(1)})$ (for details on this see Appendix~\ref{Proof5}). 
This shows that $\gamma(W^{(1)},c)$ is closely related to $g''(x,\gamma)$. 
In the following we seek to express \eqref{continuous_new} in terms of $g''(x,\gamma)$. 
Since $g''(x,\gamma)$ determines $g(x,\gamma)$ only up to linear functions, we consider the following problem:
\begin{equation}
\begin{aligned}
 \min_{\gamma\in C(\mathbb{R}^2), u\in\mathbb{R}, v\in\mathbb{R}}\quad & \int_{\mathbb{R}^2} \gamma^2(W^{(1)},c)~\mathrm{d}\nu(W^{(1)},c)\\
 \textup{subject to}\quad & ux_j+v+\int_{\mathbb{R}^2} \gamma(W^{(1)},c)[W^{(1)}(x_j-c)]_+~\mathrm{d}\nu(W^{(1)},c)=y_j,\quad j=1,\ldots,M. 
\end{aligned}
\label{continuous_add_linear}
\end{equation}
Here $u,v$ are not included in the cost. 
They add a linear function to the output of the neural network. If $u$ and $v$ in the solution of \eqref{continuous_add_linear} are small, then the solution is close to the solution of \eqref{continuous_new}. 
\cite{ongie2019function} also use this trick to simplify the characterization of neural networks in function space. 
Next we study the solution of \eqref{continuous_add_linear} in function space. 
This is our main technical result for univariate regression. 

\begin{theorem}[Implicit bias in function space for univariate regression]
  \label{theorem_func}
Assume $\mathcal{W}$ and $\mathcal{B}$ are random variables with $\mathbb{P}(\mathcal{W}=0)=0$, and let $\mathcal{C}=-\mathcal{B}/\mathcal{W}$. 
Let $\nu$ denote the probability distribution of
$(\mathcal{W}, \mathcal{C})$.  %
Suppose $(\overline{\gamma},\overline{u},\overline{v})$ is the solution of \eqref{continuous_add_linear}, 
and consider the corresponding output function 
\begin{equation}
  g(x,(\overline{\gamma},\overline{u},\overline{v}))=\overline{u} %
  x+\overline{v}+\int_{\mathbb{R}^2} \overline{\gamma}(W^{(1)},c)[W^{(1)}(%
  x-c)]_+~\mathrm{d}\nu(W^{(1)},c). 
  \label{func_g}
\end{equation}
Let $\nu_\mathcal{C}$ denote the marginal distribution of $\mathcal{C}$ and assume it has a density function $p_\mathcal{C}$. Assume that $\mathcal{W}$ has finite second moment. Let $\mathbb{E}(\mathcal{W}^2|\mathcal{C})$ denote the conditional expectation of $\mathcal{W}^2$ given $\mathcal{C}$. Consider the function 
$\zeta(x) = p_\mathcal{C}(x)\mathbb{E}(\mathcal{W}^2|\mathcal{C}=x)$, 
assume its support contains the input samples, $x_i\in\mathrm{supp}(\zeta)$, $i=1,\ldots,m$, and let $S = \operatorname{supp}(\zeta) \cap [\min_i x_i, \max_i x_i]$. 
Then $g(x,(\overline{\gamma},\overline{u},\overline{v}))$ 
satisfies $g''(x,(\overline{\gamma},\overline{u},\overline{v}))=0$ for 
$x\not\in S$ 
and for $x\in S$ it is the solution to the following problem: 
\begin{equation}
  \min_{h\in C^2(S)}
\int_{S}
\frac{(h''(x))^2}{\zeta(x)}~\mathrm{d}x %
\quad 
\text{s.t.}\quad h(x_j)=y_j,\quad j=1,\ldots,m. 
  \label{function_space}
\end{equation}
\end{theorem}
The proof is provided in Appendix~\ref{Proof5}, where we also present the corresponding statement without ASI.

Finally, we discuss the curvature penalty function. 
We provide the proof %
of following propositions in Appendix~\ref{app:set-rho}. 
\begin{proposition}[Curvature penalty function]
\label{proposition:form-rho}
Let $p_{\mathcal{W},\mathcal{B}}$ denote the joint density function of $(\mathcal{W},\mathcal{B})$ and let $\mathcal{C}=-\mathcal{B}/\mathcal{W}$ so that $p_{\mathcal{C}}$ is the breakpoint density. 
Then $\zeta(x) = \mathbb{E}(W^2|C=x)p_\mathcal{C}(x)
=\int_\mathbb{R}|W|^3p_{\mathcal{W},\mathcal{B}}(W,-Wx)~\mathrm{d}W$. 
\end{proposition}
We note that if we sample the initial weight and biases from a suitable joint distribution, we can make the curvature penalty $\rho=1/\zeta$ arbitrary:  

\begin{proposition}[Constructing any curvature penalty]
\label{set_rho}
Given any function $\varrho \colon \mathbb{R}\to\mathbb{R}_{>0}$, satisfying $Z=\int_\mathbb{R}\frac{1}{\varrho}<\infty$, if we set the density of $\mathcal{C}$ as $p_\mathcal{C}(x)=\frac{1}{Z}\frac{1}{\varrho(x)}$ and make $\mathcal{W}$ independent of $\mathcal{C}$ with non-vanishing second moment, then 
$(\mathbb{E}(W^2|C=x)p_\mathcal{C}(x))^{-1} = 
(\mathbb{E}(W^2)p_\mathcal{C}(x))^{-1}
\propto \varrho(x)$, $x\in\mathbb{R}$.
\end{proposition}

\section{Implicit Bias for Multivariate Regression} 
\label{sec:implicit_bias_multivariate}

In this section we solve the optimization problem \eqref{continuous_version} in the multivariate case. Similar to Section~\ref{sec:implicit_bias_univariate}, we can relax the optimization problem %
to 
\begin{equation}
\begin{aligned}
 \min_{\substack{\alpha\in C(\mathbb{R}^d\times \mathbb{R}),\\ \mathbf{u}\in\mathbb{R}^{d}, v\in\mathbb{R}}}\ & \int_{\mathbb{R}^d\times \mathbb{R}} \alpha^2(\mathbf{W}^{(1)},b)~\mathrm{d}\mu(\mathbf{W}^{(1)},b)\\
 \textup{subject to}\ & \int_{\mathbb{R}^d\times \mathbb{R}} \alpha(\mathbf{W}^{(1)},b)[\langle \mathbf{W}^{(1)},\mathbf{x}_j\rangle+b]_+~\mathrm{d}\mu(\mathbf{W}^{(1)},b)+\langle \mathbf{u},\mathbf{x}_j \rangle +v=y_j,\ j=1,\ldots,M.
\end{aligned}
\label{continuous_version_multi_relax}
\end{equation}
Let $\mathcal{U}=\|\bm{\mathcal{W}}\|_2$, $\bm{\mathcal{V}}=\bm{\mathcal{W}}/\|\bm{\mathcal{W}}\|_2$ and $\mathcal{C}=-\mathcal{B}/\|\bm{\mathcal{W}}\|_2$. 
Let $\nu$ denote the distribution of $(\mathcal{U},\bm{\mathcal{V}},\mathcal{C})$ and $\gamma(u,\bm{V},c)=\alpha(u\bm{V},-cu)$. Then, after the change of variables, the optimization problem \eqref{continuous_version_multi_relax} is equivalently expressed as  
\begin{equation}
\begin{aligned}
 \min_{\substack{\alpha\in C(\mathbb{R}^+\times\mathbb{S}^{d-1}\times\mathbb{R}),\\ \mathbf{u}\in\mathbb{R}^{d}, v\in\mathbb{R}}}\ & \int_{\mathbb{R}^+\times\mathbb{S}^{d-1}\times\mathbb{R}} \gamma^2(u,\bm{V},c)~\mathrm{d}\nu(u,\bm{V},c)\\
 \textup{subject to}\ & \int_{\mathbb{R}^+\times\mathbb{S}^{d-1}\times\mathbb{R}} \gamma(u,\bm{V},c)\cdot u\cdot[\langle \bm{V},\mathbf{x}_j\rangle-c]_+~\mathrm{d}\nu(u,\bm{V},c)+\langle \mathbf{u},\mathbf{x}_j \rangle +v=y_j,\ j=1,\ldots,M.
\end{aligned}
\label{continuous_version_multi_relax_gamma}
\end{equation}
Define the output of the infinite-width network by 
\begin{equation*}
    g(\mathbf{x},(\gamma,\mathbf{u},v)) = \int_{\mathbb{R}^+\times\mathbb{S}^{d-1}\times\mathbb{R}} \gamma(u,\bm{V},c)\cdot u\cdot[\langle \bm{V},\mathbf{x}\rangle-c]_+~\mathrm{d}\nu(u,\bm{V},c)+\langle \mathbf{u},\mathbf{x} \rangle +v.
\end{equation*} 
Then the Laplacian $\Delta g(\mathbf{x},(\gamma,\mathbf{u},v))=\sum_{i=1}^d \partial^2_{x_i}g(\mathbf{x},(\gamma,\mathbf{u},v))$ is given by
\begin{equation}
  \begin{aligned}
    \Delta g(\mathbf{x},(\gamma,\mathbf{u},v))
    &=\int_{\mathbb{R}^+\times\mathbb{S}^{d-1}\times\mathbb{R}} \gamma(u,\bm{V},c)\cdot u\cdot~\delta(\langle \bm{V},\mathbf{x}\rangle-c)~\mathrm{d}\nu(u,\bm{V},c)\\
    &=\int_{\mathbb{S}^{d-1}\times\mathbb{R}}\left(\int_{\mathbb{R}^+}\gamma(u,\bm{V},c)\cdot u~\mathrm{d}\nu_{\mathcal{U}|\bm{\mathcal{V}}=\bm{V},\mathcal{C}=c}(u)\right)\delta(\langle \bm{V},\mathbf{x}\rangle-c)~\mathrm{d}\nu_{\bm{\mathcal{V}},\mathcal{C}}(\bm{V},c),
  \end{aligned}
  \label{2nd_derivative-long_multi_partial}
\end{equation}
where $\nu_{\bm{\mathcal{V}},\mathcal{C}}$ denotes the joint distribution of $(\bm{\mathcal{V}},\mathcal{C})$, and $\nu_{\mathcal{U}|\bm{\mathcal{V}}=\bm{V},\mathcal{C}=c}$ the conditional distribution of $\mathcal{U}$ given $\bm{\mathcal{V}}=\bm{V}$ and $\mathcal{C}=c$. 
Let $\nu_{C|\bm{\mathcal{V}}=\bm{V}}$ denote the conditional distribution of $\mathcal{C}$ given $\bm{\mathcal{V}}=\bm{V}$. Suppose $\nu_{C|\bm{\mathcal{V}}=\bm{V}}$ has a density function $p_{\mathcal{C}|\bm{\mathcal{V}}=\bm{V}}(c)$. Define 
\begin{equation}
    \kappa(\bm{V},c)=\int_{\mathbb{R}^+}\gamma(u,\bm{V},c)\cdot u~\mathrm{d}\nu_{\mathcal{U}|\bm{\mathcal{V}}=\bm{V},\mathcal{C}=c}(u).
    \label{definition_of_kappa}
\end{equation} 
Then \eqref{2nd_derivative-long_multi_partial} becomes 
\begin{equation}
  \begin{aligned}
    \Delta g(\mathbf{x},(\alpha,\mathbf{u},v))
    &=\int_{\mathbb{S}^{d-1}\times\mathbb{R}}\kappa(\bm{V},c)~\delta(\langle \bm{V},\mathbf{x}\rangle-c)~\mathrm{d}\nu_{\bm{\mathcal{V}},\mathcal{C}}(\bm{V},c)\\
    &=\int_{\mathbb{S}^{d-1}}\left(\int_{\mathbb{R}}\kappa(\bm{V},c) ~  
    \delta(\langle \bm{V},\mathbf{x}\rangle-c) p_{\mathcal{C}|\bm{\mathcal{V}}=\bm{V}}(c)\mathrm{d}c\right)~\mathrm{d}\nu_{\bm{\mathcal{V}}}(\bm{V})\\
    &=\int_{\mathbb{S}^{d-1}}\kappa(\bm{V},\langle \bm{V},\mathbf{x}\rangle) ~  
     p_{\mathcal{C}|\bm{\mathcal{V}}=\bm{V}}(\langle \bm{V},\mathbf{x}\rangle)~\mathrm{d}\nu_{\bm{\mathcal{V}}}(\bm{V}),
    \end{aligned}
  \label{2nd_derivative-long_multi}
\end{equation}
where $\nu_{\bm{\mathcal{V}}}$ denotes the distribution of $\bm{\mathcal{V}}$.
Assume that $\nu_{\bm{\mathcal{V}}}$ has a density function $p_{\bm{\mathcal{V}}}(\bm{V})$ with respect to the spherical measure $\sigma^{d-1}$. Then \eqref{2nd_derivative-long_multi} becomes
\begin{equation}
  \begin{aligned}
    \Delta g(\mathbf{x},(\alpha,\mathbf{u},v))
    &=\int_{\mathbb{S}^{d-1}}\kappa(\bm{V},\langle \bm{V},\mathbf{x}\rangle) ~  
     p_{\mathcal{C}|\bm{\mathcal{V}}=\bm{V}}(\langle \bm{V},\mathbf{x}\rangle)p_{\bm{\mathcal{V}}}(\bm{V})~\mathrm{d}\sigma^{d-1}(\bm{V}). 
    \end{aligned}
  \label{2nd_derivative-long_multi_sphere}
\end{equation}
Now, defining 
\begin{equation}
    \beta(\bm{V},c)=\kappa(\bm{V},c)~p_{\mathcal{C}|\bm{\mathcal{V}}=\bm{V}}(c)~p_{\bm{\mathcal{V}}}(\bm{V}),
    \label{definition_of_beta}
\end{equation}
we observe that 
\begin{equation}
  \begin{aligned}
    \Delta g(\mathbf{x},(\alpha,\mathbf{u},v))
    &=\int_{\mathbb{S}^{d-1}}\beta(\bm{V},\langle \bm{V},\mathbf{x}\rangle)~\mathrm{d}\bm{V}\\
    &=\mathcal{R}^*\{\beta\}(\mathbf{x}).
    \end{aligned}
  \label{2nd_derivative-long_multi_sphere_beta}
\end{equation}
The right-hand side of \eqref{2nd_derivative-long_multi_sphere_beta} is precisely the dual Radon transform of $\beta$. 
Let $\beta=\beta^++\beta^-$ be the even–odd decomposition of $\beta$, where $\beta^+$ is even and $\beta^-$ is odd, i.e., $\beta^+(\bm{V},c)=\beta^+(-\bm{V},-c)$ and $\beta^-(\bm{V},c)=-\beta^-(-\bm{V},-c)$ for all $(\bm{V},c)\in\mathbb{S}^{d-1}\times \mathbb{R}$. 
Since the dual Radon transform annihilates odd functions, we have $\Delta g(\mathbf{x},(\alpha,\mathbf{u},v))
    =\int_{\mathbb{S}^{d-1}}\beta^+(\bm{V},\langle \bm{V},\mathbf{x}\rangle)~\mathrm{d}\bm{V}$. \cite{ongie2019function} observed that $\beta^+$ can be recovered from $\Delta g$ by using the inversion formula of the dual Radon transform. According to \citet[Lemma~3]{ongie2019function},
\begin{equation}
    \beta^+=-\frac{1}{2(2\pi)^{d-1}}\mathcal{R}\{(-\Delta)^{(d+1)/2}g(\cdot,\alpha)\},
    \label{beta}
\end{equation} 
where $\mathcal{R}$ is the Radon transform which is defined by
\begin{equation*}
    \mathcal{R}\{f\}(\bm{\omega},b)\coloneqq \int_{\langle{\bm\omega},\mathbf{x}\rangle=b}f(\mathbf{x})\mathrm{d}s(\mathbf{x}), \quad (\bm{\omega},b)\in\mathbb{S}^{d-1}\times \mathbb{R},
\end{equation*}
where $\mathrm{d}s(\mathbf{x})$ represents integration with respect to the $(d-1)$-dimensional surface measure on the hyperplane $\langle{\bm\omega},\mathbf{x}\rangle=b$. 
The fractional power of the negative Laplacian $(-\Delta)^{(d+1)/2}$ in \eqref{beta} is the operator defined in Fourier domain by 
\begin{equation*}
  \widehat{(-\Delta)^{(d+1)/2}f}(\bm\xi)=\|\bm\xi\|^{d+1} \widehat f(\bm\xi).
\end{equation*}
When $d+1$ is a even number, $(-\Delta)^{(d+1)/2}$ is the same as applying the negative Laplacian $(d+1)/2$ times. When $d+1$ is odd, it is a pseudo-differential operator given by convolution with a singular kernel \citep[see][]{kwasnicki2017ten}. %
Then according to \eqref{beta} and the definition of $\beta$, we have
\begin{equation}
\begin{aligned}
    &\mathcal{R}\{(-\Delta)^{(d+1)/2}g(\cdot,\alpha)\}(\bm{V},c)-2(2\pi)^{d-1}\beta^-\\
    =&-2(2\pi)^{d-1}\kappa(\bm{V},c)p_{\mathcal{C}|\bm{\mathcal{V}}=\bm{V}}(c)p_{\bm{\mathcal{V}}}(\bm{V})\\
    =&-2(2\pi)^{d-1}p_{\mathcal{C}|\bm{\mathcal{V}}=\bm{V}}(c)p_{\bm{\mathcal{V}}}(\bm{V})\int_{\mathbb{R}^+}\gamma(u,\bm{V},c)\cdot u~\mathrm{d}\nu_{\mathcal{U}|\bm{\mathcal{V}}=\bm{V},\mathcal{C}=c}(u). 
\end{aligned}
    \label{alpha_multi}
\end{equation} 
From the above equation, we show how $\gamma(u,\bm{V},c)$ is characterized by the network output function, which allows us to study the solution of \eqref{continuous_version_multi_relax_gamma} in function space. The following theorem generalizes Theorem~\ref{theorem_func} to the multivariate case. 

\begin{theorem}[Implicit bias in function space for multivariate regression]
  \label{theorem_func_multi_dim}
Assume that (1) $\bm{\mathcal{W}}$ is a random vector with $\mathbb{P}(\|\bm{\mathcal{W}}\|=0)=0$ and $\mathcal{B}$ is a random variable; (2) the distribution of $(\bm{\mathcal{W}},\mathcal{B})$ is symmetric, i.e., $(\bm{\mathcal{W}},\mathcal{B})$ and $(-\bm{\mathcal{W}},-\mathcal{B})$ have the same distribution; (3) $\|\bm{\mathcal{W}}\|_2$ and $\mathcal{B}$ both have finite second moments. Let $\mathcal{U}=\|\bm{\mathcal{W}}\|_2$, $\bm{\mathcal{V}}=\bm{\mathcal{W}}/\|\bm{\mathcal{W}}\|_2$ and $\mathcal{C}=-\mathcal{B}/\|\bm{\mathcal{W}}\|_2$.
Let $\nu$ denote the distribution of $(\mathcal{U},\bm{\mathcal{V}},\mathcal{C})$. %
Suppose $(\overline{\gamma},\overline{ \mathbf{u}},\overline{v})$ is the solution of \eqref{continuous_version_multi_relax_gamma}, 
and assume that \eqref{continuous_version_multi_relax_gamma} is feasible, which means 
\[\int_{\mathbb{R}^+\times\mathbb{S}^{d-1}\times\mathbb{R}} \overline{\gamma}^2(u,\bm{V},c)~\mathrm{d}\nu(u,\bm{V},c)<+\infty.\]
Consider the corresponding output function 
\begin{equation}
  g(\mathbf{x},(\overline{\gamma},\overline{ \mathbf{u}},\overline{v})) = \int_{\mathbb{R}^+\times\mathbb{S}^{d-1}\times\mathbb{R}} \overline{\gamma}(u,\bm{V},c)\cdot u\cdot[\langle \bm{V},\mathbf{x}\rangle-c]_+~\mathrm{d}\nu(u,\bm{V},c)+\langle \overline{ \mathbf{u}},\mathbf{x} \rangle +\overline{v}.
  \label{func_g_multi_dim}
\end{equation}
Let $\nu_{\bm{\mathcal{V}}}$ denote the marginal distribution of $\mathcal{C}$ and assume it has a density function $p_{\bm{\mathcal{V}}}(\bm{V})$. Let $\nu_{C|\bm{\mathcal{V}}=\bm{V}}$ denote the conditional distribution of $\mathcal{C}$ given $\bm{\mathcal{V}}=\bm{V}$ and assume it has a density function $p_{\mathcal{C}|\bm{\mathcal{V}}=\bm{V}}(c)$. Let $\mathbb{E}(\mathcal{U}^2|\bm{\mathcal{V}}=\bm{V},\mathcal{C}=c)$ denote the conditional expectation of $\mathcal{U}^2$ given $\mathcal{V}$ and $\mathcal{C}$. Consider the following function $\zeta\colon \mathbb{S}^{d-1}\times \mathbb{R} \to \mathbb{R}$, 
\begin{equation}
    \zeta(\bm{V},c) = p_{\mathcal{C}|\bm{\mathcal{V}}=\bm{V}}(c)~p_{\bm{\mathcal{V}}}(\bm{V})\mathbb{E}(\mathcal{U}^2|\bm{\mathcal{V}}=\bm{V},\mathcal{C}=c). 
    \label{definition_of_zeta}
\end{equation}
Then $g(\mathbf{x},(\overline{\gamma},\overline{ \mathbf{u}},\overline{v}))$ 
is the solution of the following problem: 
\begin{equation}
\begin{aligned}
 \min_{h\in \operatorname{Lip}(\mathbb{R}^d)\cap C(\mathbb{R}^d)}\quad & \int_{\operatorname{supp}(\zeta)} \frac{\left({\mathcal{R}\{(-\Delta)^{(d+1)/2}h\}(\bm{V},c)}\right)^2}{\zeta(\bm{V},c)}~\mathrm{d}\sigma^{d-1}(\bm{V})\mathrm{d}c\\
 \textup{subject to}\quad & h(\mathbf{x}_j)=y_j,\quad j=1,\ldots,M, \\
 & {\mathcal{R}\{(-\Delta)^{(d+1)/2}h\}(\bm{V},c)}=0, \ \forall(\bm{V},c)\not\in \operatorname{supp}(\zeta),\\
 &(-\Delta)^{(d+1)/2}h \in L^p(\mathbb{R}^d),\ 1\leq p<d/(d-1), 
\end{aligned}
\label{function_space_multi}
\end{equation}
where $\operatorname{Lip}(\mathbb{R}^d)$ is the space of Lipschitz continuous function on $\mathbb{R}^d$ and $\sigma^{d-1}$ is the spherical measure. 
\end{theorem}
The proof of Theorem \ref{theorem_func_multi_dim} is provided in Appendix \ref{proof_theorem_func_multi_dim}. The optimization problem \eqref{function_space_multi} characterizes the implicit bias of the gradient descent in function space for the multivariate setting. 
\cite{zhang2019type} obtained a characterization in terms of the minimization of a kernel norm in function space, which is also valid for multi-dimensional inputs. In Appendix~\ref{appendix:relation_NTK} we prove the equivalence between the kernel norm minimization and our result in the one-dimensional setting. 
In future work it will be interesting to show that in the multivariate setting, the kernel norm is equivalent to the objective in \eqref{function_space_multi} under appropriate conditions.

To conclude this section, we discuss the function $\zeta$ in the variational problem \eqref{function_space_multi}. The proofs of the following statements are presented in Appendix~\ref{proof_multi_curvature_penalty_function}. 
First we propose an initialization scheme such that $\zeta$ is constant over a bounded region. 
\begin{proposition}[Constant $\zeta$ over a bounded region]
\label{constant_zeta}
If $\bm{\mathcal{W}}$ is sampled uniformly from the unit sphere and $\mathcal{B}$ from a symmetric interval, i.e., $\bm{\mathcal{W}}\sim \mathrm{Unif}(\mathbb{S}^{d-1})$ and $\mathcal{B}\sim \mathrm{Unif}(-a,a)$, then $\zeta(\bm{V},c)$ is constant over $\{(\bm{V},c):|c|\leq a\}$ and $\zeta(\bm{V},c)=0$ for $|c|>a$.
\end{proposition}
Now we discuss the form of $\zeta$ under certain conditions. 
\begin{proposition}[Penalty function $\zeta$]
\label{proposition:form-rho-2D}
Let $p_{\bm{\mathcal{W}},\mathcal{B}}$ denote the joint density function of $(\bm{\mathcal{W}},\mathcal{B})$ and let $\mathcal{U}=\|\bm{\mathcal{W}}\|_2$, $\bm{\mathcal{V}}=\bm{\mathcal{W}}/\|\bm{\mathcal{W}}\|_2$ and $\mathcal{C}=-\mathcal{B}/\|\bm{\mathcal{W}}\|_2$.
Then $\zeta(\bm{V},c) = p_{\mathcal{C}|\bm{\mathcal{V}}=\bm{V}}(c)~p_{\bm{\mathcal{V}}}(\bm{V})\mathbb{E}(\mathcal{U}^2|\bm{\mathcal{V}}=\bm{V},\mathcal{C}=c)=\int_\mathbb{R}u^{d+2}p_{\bm{\mathcal{W}},\mathcal{B}}(u\bm{V},-uc)~\mathrm{d}u $. 
\end{proposition}
Using the above result we compute the explicit form of $\zeta$ for Gaussian initialization. %

\begin{theorem}[Explicit form of $\zeta$ for Gaussian initialization]
\label{proposition:explicit-rho-gaussian-2d}
\mbox{} 
Assume that $\bm{\mathcal{W}}$ and $\mathcal{B}$ are independent, $\bm{\mathcal{W}}\sim \mathcal{N}(0, \sigma_w^2\bm{I}_d)$ and $\mathcal{B}\sim \mathcal{N}(0, \sigma_b^2)$. Then 
$\zeta$ is given by 
\begin{equation*}
\zeta(\bm{V},c)=\frac{\sigma_w^3\sigma_b^{d+2}}{\pi^{(d+1)/2}\left(\sigma_b^2+c^2\sigma_w^2\right)^{(d+3)/2}}\Gamma(\frac{d+3}{2}). 
\end{equation*}
\end{theorem}

\section{Conclusion} 
We obtained explicit descriptions in function space for the implicit bias of gradient descent in mean squared error regression with wide shallow ReLU networks covering the univariate and multivariate cases. We also presented a generalization to networks with different activation functions and discussed a relaxation related to early stopping and training trajectories in function space. 

In the case of univariate regression, our main result shows that the trained network function interpolates the training data while minimizing a weighted 2-norm of the second derivative with respect to the input. Such functions correspond to spatially adaptive interpolating splines. 
In the case of multivariate regression, our results also characterize the trained network functions.  Under specific parameter initialization schemes, these functions correspond to polyharmonic interpolating splines. 
The spaces of interpolating splines are linear of dimension in the order of the number of data points. Hence, our results imply that, even if the network has many parameters, the complexity of the trained functions will be adjusted to the number of training data points. This can be used to explain why overparametrized networks do not overfit in practice, as the generalization error can be regarded as the precision of the spline interpolation \citep[see, e.g.,][]{wendland2004scattered}. 

\citet{zhang2019type} described the implicit bias of gradient descent as minimizing a RKHS norm from initialization. %
Our result can be regarded as making the RKHS norm explicit, providing an interpretable description of the bias in function space. 
Compared with \citet{zhang2019type}, our results describe the role of the parameter initialization scheme, which determines the curvature penalty function $1/\zeta$. This gives us a clearer picture of how the initialization affects the implicit bias of gradient descent. 
This could be used in order to select a good initialization scheme. 
For instance, one could conduct a pre-assessment of the data to estimate the locations of the input space where the solution should have a high curvature, and choose the parameter initialization accordingly. This is an interesting possibility to experiment with based on our theoretical results. 

Our results can also be interpreted in combination with early stopping. The training trajectory is approximated by a smoothing spline, meaning that the network will filter out high frequencies which are usually associated to noise in the training data. This behaviour is sometimes referred to as a spectral bias \citep{pmlr-v97-rahaman19a}. \hj{\cite{ijcai2021p304} studied spectral bias theoretically and showed that spherical harmonics of low frequency are easier to be learned by over-parameterized neural networks if the input data is uniformly distributed over the unit hypersphere.}

\acks{This project has been supported by ERC Starting Grant 757983, NSF CAREER Grant 2145630, DFG SPP 2298 Grant 464109215.}

\appendix
\newpage 
\section*{Appendix}
The appendix is organized as follows. 
\begin{itemize}[leftmargin=*]
\item 
In Appendix \ref{app:additional-experiments} we illustrate our theoretical results numerically, 
and %
provide details on the numerical implementation. 
\item 
In Appendix \ref{app:additional-comments} we briefly discuss definitions and settings around the parametrization and initialization of neural networks, as well as on the limiting NTK and the linearization of a neural network. 
\item 
In Appendices 
\ref{app:proof-thm1}, 
\hj{
\ref{proof1}, 
}
\ref{Proof2}, 
\ref{appendix:Lee-generalization}, 
\ref{Proof4},  
we provide proofs and supporting results for the results presented in Sections~\ref{sec:main}, 
\ref{parameter}, \ref{Training_only_the_output_layer}, 
and \ref{2.4}. 
\item 
In Appendices~\ref{app:univariate} and~\ref{app:multi-dimensional inputs}, we provide the proofs of the results in Sections~\ref{sec:implicit_bias_univariate} and \ref{sec:implicit_bias_multivariate} for univariate and multivariate regression respectively. 
\item 
In Appendix~\ref{Other_activation}, we prove a corresponding result for activation functions other than ReLU. 
\item 
In Appendix~\ref{Difference_between_solutions_of_variational_problems} we discuss the linear adjustment of the training data and why our result still gives a good description of training with the original data for non-linear target functions. 

\item \hj{In Appendix~\ref{app:skip_connection}, we introduce the network with skip connections and obtain the same implicit bias result without adjusting the training data. 
}
\item 
In Appendix~\ref{appendix:relation_NTK} we show the equivalence between our variational characterization of the implicit bias of gradient descent in function space and the description in terms of a kernel norm minimization problem. 
\item 
In Appendix~\ref{appendix:smoothingspline} we discuss the relation between the gradient descent optimization trajectory and a trajectory of spatially adaptive smoothing splines with decreasing smoothness regularization coefficient which converges to the spatially adaptive interpolating spline. 
\item 
In Appendix~\ref{appendix:splines} we give the explicit form of the solution to our variational problem, i.e., the spatially adaptive interpolating spline, which corresponds to the output function after gradient descent training in the infinite width limit. 
\item 
In Appendix~\ref{app:generalizations} we comment on possible extensions and generalizations of the analysis. 
\end{itemize}

\section{Numerical Illustration of the Theoretical Results}
\label{app:additional-experiments}

\paragraph{Implementation of gradient descent}
Training is implemented as full-batch gradient descent.
In practice we choose the learning rate as follows. We start with a large learning rate and keep decreasing it by half until we observe that the loss function decreases. After that, we start training with the fixed learning rate we found. We observe that the learning rate we found is inversely proportional to the width $n$ of the neural network. This observation is in accord with Theorem~\ref{minimum_weight} with respect to the upper bound of the learning rate in order to converge. 

We note that the implicit bias in parameter space shown in Theorem~\ref{minimum_weight} is independent of the specific step size that is used in the optimization, so long as it is small enough (see Appendix~\ref{proof1}). 
The stopping criterion for training of the neural network is that the change in the training loss in consecutive iterations is less than a pre-specified threshold:  $|L(\theta_{t})-L(\theta_{t-1})|\leq 10^{-8}$. 

We use ASI (see Appendix \ref{app:ASI}) at initialization. 
Then the initial output function of the network is $f(\cdot,\theta_0)\equiv 0$. %
Hence in the figures the network output function is actually equal to the difference from initialization. 

For the comparison of the functions $f(\cdot,\theta^\ast)$ and $g^\ast$, 
the infinity norm $\|f(\cdot,\theta^\ast)- g^\ast\|_\infty$ is computed over a discretization of $[-\max_i\|\mathbf{x}_i\|_2,\max_i\|\mathbf{x}_i\|_2]^d$.

\paragraph{Implementation of numerical solutions to the variational problem}
For univariate regression, the variational problem for cubic splines can be solved explicitly as described in Appendix~\ref{appendix:splines}. 
For a general non-constant curvature penalty function $1/\zeta$, we can obtain a numerical solution to problem \eqref{function_space} as follows. 
First we discretize the interval $[-I,I]$ evenly with points $x_j=-I+2jI/N$, $j=0,\ldots,N$. 
For simplicity we suppose that the $M$ input training data points are among these grid points, and we denote them by $x_{j_1},\ldots, x_{j_M}$. 
Then we initialize $f(x_j)=0$ for $x_j$ not in the training data (to be optimized) and $f(x_{j_i})=y_i$ (fixed values during optimization). 
We use the central difference to approximate the second derivative, $f''(x_j) = \frac{f(x_{j+1})-2f(x_{j})+f(x_{j-1})}{h^2}$, where $h = |x_{j+1}-x_j|$. 
Then the objective function in \eqref{function_space} is approximated by $\sum_{j=1}^{N-1} \frac{1}{\zeta(x_j)%
}\left(\frac{f(x_{j+1})-2f(x_{j})+f(x_{j-1})}{h^2}\right)^2$. 
This is a quadratic problem in $f(x_j)$, $j\in\{1,\ldots, N\}\setminus\{j_1,\ldots, j_M\}$. If we equate the gradient to zero, we obtain a linear system. 
The solution can be written in closed form in terms of the inverse of a design matrix. 
As with any linear regression problem, in practice we may still prefer to use an iterative approach to obtain a numerical solution. 
In our experiment, we discretize the interval $[-2,2]$ into $200$ pieces and use conjugate gradient descent for solving the linear system.  

For multivariate regression, it is not straightforward to numerically solve \eqref{gen_multi_dim}. Hence we numerically solve \eqref{continuous_version_multi_relax} instead. We discretize the interval $[-I_w,I_w]$ evenly with points $w_j=-I_w+2jI_w/n_w$, $j=0,\ldots,n_w$ and the interval $[-I_b,I_b]$ evenly with points $b_j=-I_b+2jI_b/n_b$, $j=0,\ldots,n_b$. Let $\alpha_{(i_1,\ldots,i_d,j)}=\alpha((w_{i_1},\ldots,w_{i_d}),b_j)$, $i_k=0,\ldots,n_w$, $j=0,\ldots,n_b$. We use numerical integration to approximate the objective and constraints of \eqref{continuous_version_multi_relax} and then get an optimization problem with search variables $\alpha_{(i_1,\ldots,i_d,j)}$. This is a quadratic programming problem which can be solved using an internal point method.

\paragraph{Gradient descent training and variational problem} 
To illustrate Theorem~\ref{thm:theorem1} across different initialization procedures, 
in Figures~\ref{error_against_n_gaussian}
and~\ref{error_against_n_gaussian0.1} we show analogous experiments to those in the left panel of Figure~\ref{fig:nr_neurons}, but using two types of Gaussian initialization instead of the uniform initialization. 
As we already observed in the right panel of Figure~\ref{fig:different_rho}, 
here the effect of the curvature penalty function is also visible. In portions of the input space where $\zeta$ is peaked, the solution function can have a high curvature, and, conversely, in portions of the input space where $\zeta$ takes small values, the solution function has a small second derivative and is more linear. 

To verify that the results are stable over different data sets, 
in Figure~\ref{fig:nr_neurons_mre_samples} we show an experiment similar to that of Figure~\ref{fig:nr_neurons}, but for a larger data set.

\begin{figure}
    \centering
\begin{tikzpicture}[x=\textwidth,y=\textwidth, inner sep = 1pt]
\node[above right] at (0,0) {\includegraphics[width=0.49\textwidth]{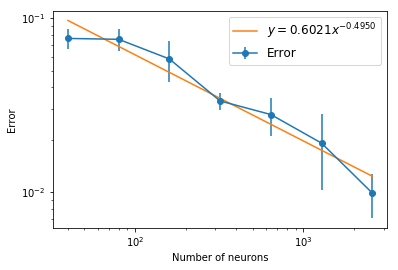}};
\node[above right] at (.5, .175) {\includegraphics[width=0.25\textwidth]{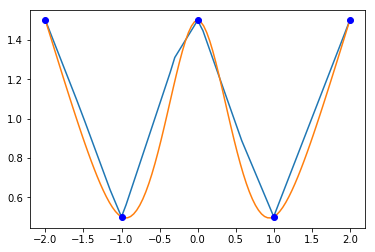}};
\node[above right] at (.75, .175) {\includegraphics[width=0.25\textwidth]{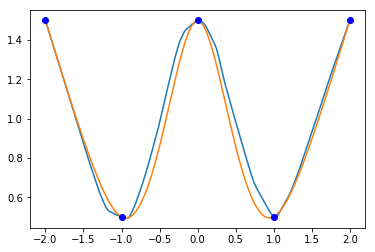}};
\node[above right] at (.5, 0) {\includegraphics[width=0.25\textwidth]{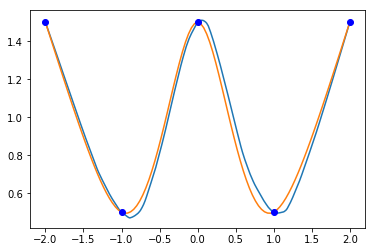}};
\node[above right] at (.75, 0) {\includegraphics[width=0.25\textwidth]{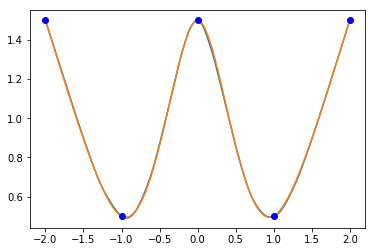}};
\node[above,fill=white] at (0.25, .33) {\small \textsf{Gaussian initialization $\sigma_B^2=1$}};
\node[above right,fill=white] at (.5+.03, .175+.08) {\small $n=20$};
\node[above right,fill=white] at (.75+.03, .175+.08) {\small $n=80$};
\node[above right,fill=white] at (.5+.03, 0+.08) {\small $n=320$};
\node[above right,fill=white] at (.75+.03, 0+.08) {\small $n=1280$};
\end{tikzpicture}
\caption{Illustration of Theorem~\ref{thm:theorem1}. Shown is the error between the output function $f(\cdot,\theta^\ast)$ of the trained neural network and the solution $g^\ast$ to the variational problem \eqref{function_space} against the number of neurons, $n$. 
    Shown is the average over $5$ repetitions, with error bars indicating the standard deviation. 
    Here the training data is fixed, and the parameters were initialized with $W\sim \mathcal{N}(0,1)$ and $B\sim \mathcal{N}(0,1)$. 
    The right panel shows the data (dots), trained network functions (blue) with $20$, $80$, $320$, $1280$ neurons, and the solution (orange) to the variational problem. 
}
\label{error_against_n_gaussian}
\end{figure}
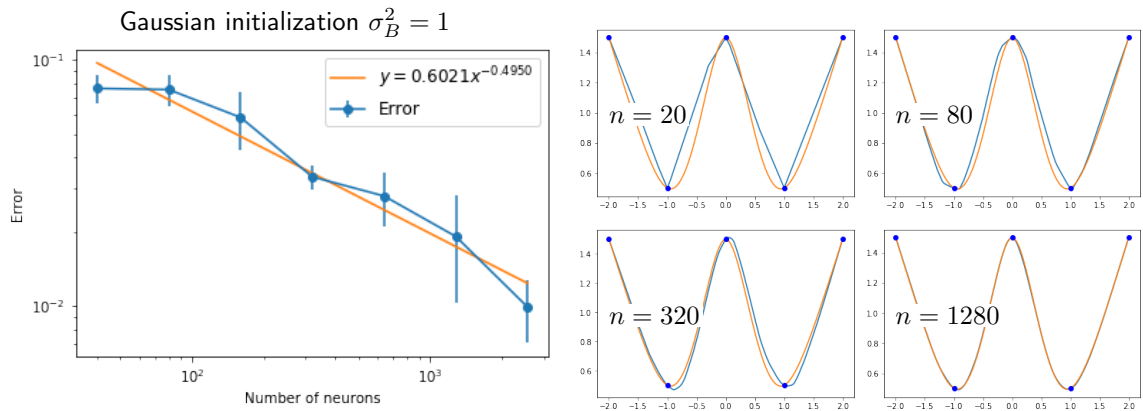

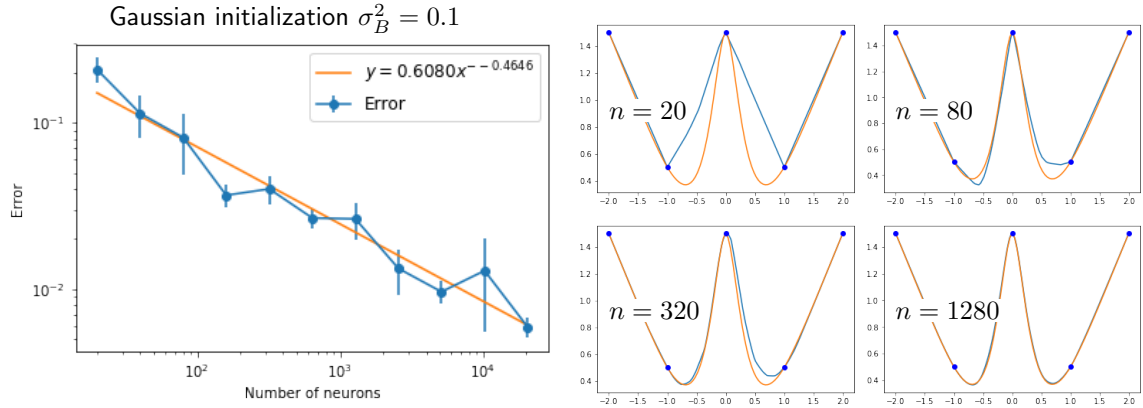
\begin{figure}
    \centering
\begin{tikzpicture}[x=\textwidth,y=\textwidth, inner sep = 1pt]
\node[above right] at (0,0) {\includegraphics[width=0.49\textwidth]{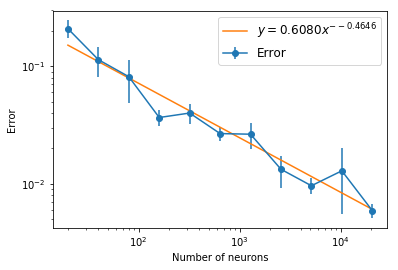}};
\node[above right] at (.5, .175) {\includegraphics[width=0.25\textwidth]{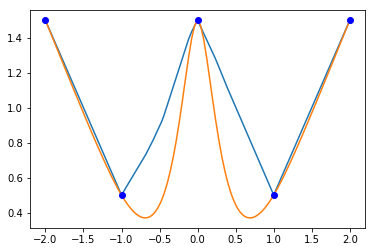}};
\node[above right] at (.75, .175) {\includegraphics[width=0.25\textwidth]{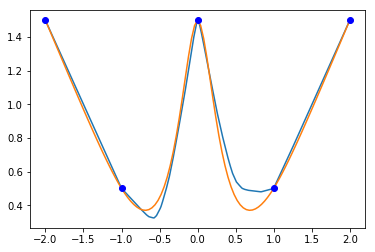}};
\node[above right] at (.5, 0) {\includegraphics[width=0.25\textwidth]{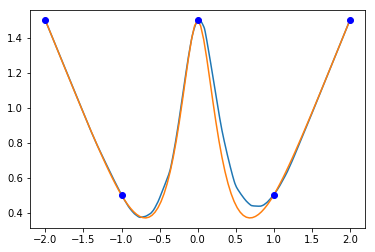}};
\node[above right] at (.75, 0) {\includegraphics[width=0.25\textwidth]{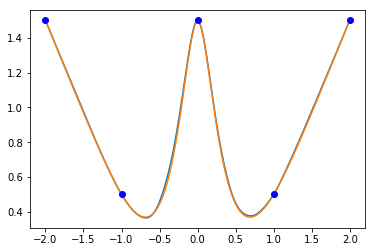}};
\node[above,fill=white] at (0.25, .33) {\small \textsf{Gaussian initialization $\sigma_B^2=0.1$}};
\node[above right,fill=white] at (.5+.03, .175+.08) {\small $n=20$};
\node[above right,fill=white] at (.75+.03, .175+.08) {\small $n=80$};
\node[above right,fill=white] at (.5+.03, 0+.08) {\small $n=320$};
\node[above right,fill=white] at (.75+.03, 0+.08) {\small $n=1280$};
\end{tikzpicture}    
\caption{Illustration of Theorem~\ref{thm:theorem1}. Similar to Figure~\ref{error_against_n_gaussian}, but with a different initialization $\mathcal{W}\sim \mathcal{N}(0,1)$ and $\mathcal{B}\sim \mathcal{N}(0,0.1)$, which gives rise to a curvature penalty function $\zeta$ that is more strongly peaked around $x=0$ (see Figure~\ref{fig:different_rho}). We observe in particular that the solutions are more curvy around $x=0$. 
}
    \label{error_against_n_gaussian0.1}
\end{figure}

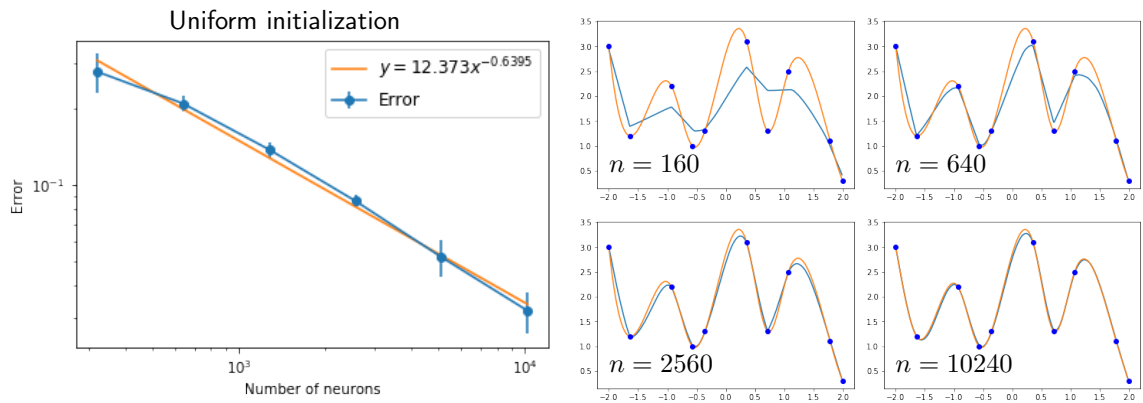
\begin{figure}
\centering
\begin{tikzpicture}[x=\textwidth,y=\textwidth, inner sep = 1pt]
\node[above right] at (0,0) {\includegraphics[width=0.49\textwidth]{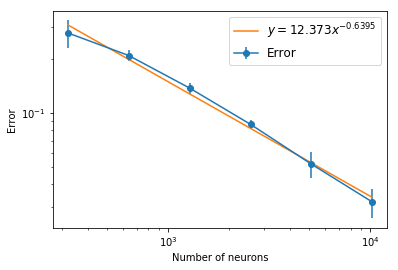}};
\node[above right] at (.5, .175) {\includegraphics[width=0.25\textwidth]{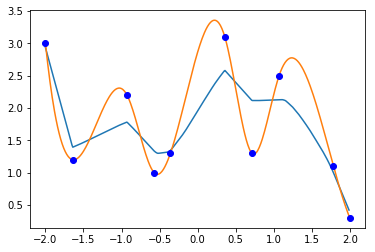}};
\node[above right] at (.75, .175) {\includegraphics[width=0.25\textwidth]{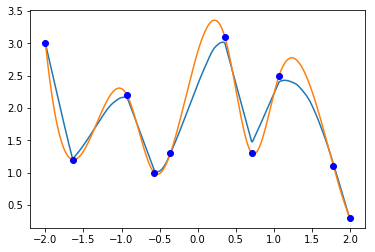}};
\node[above right] at (.5, 0) {\includegraphics[width=0.25\textwidth]{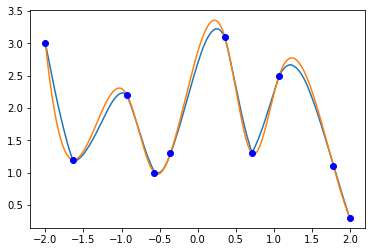}};
\node[above right] at (.75, 0) {\includegraphics[width=0.25\textwidth]{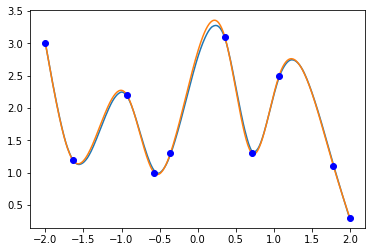}};
\node[above,fill=white] at (0.25, .33) {\small \textsf{Uniform initialization}};
\node[above right,fill=white] at (.5+.03, .175+.03) {\small $n=160$};
\node[above right,fill=white] at (.75+.03, .175+.03) {\small $n=640$};
\node[above right,fill=white] at (.5+.03, 0+.03) {\small $n=2560$};
\node[above right,fill=white] at (.75+.03, 0+.03) {\small $n=10240$};
\end{tikzpicture} 
\caption{Illustration of Theorem~\ref{thm:theorem1}. Similar to Figure~\ref{fig:nr_neurons}, with uniform initialization, but with a larger data set and larger networks. 
    }
\label{fig:nr_neurons_mre_samples}
\end{figure}

\paragraph{Training all layers versus training only the output layer} 
To illustrate Theorem~\ref{th-lin-onlyout}, we conduct the following experiment. We use the same training set as in Figure~\ref{fig:nr_neurons} and use uniform initialization. Starting from the same initial weights, we train the network in two ways. One way is only training the output layer and another way is training all layers of the network. The result is shown in Figure~\ref{fig:two_net}. The left panel plots the error between two trained network functions against the number of neurons $n$. In this experiment the error is of order $n^{-3/2}$, which is even smaller than the upper bound $n^{-1}$ given in Theorem~\ref{th-lin-onlyout}. 
Potentially the bound can be improved. 
The right panel plots two trained network functions with 20, 80, 320, 1280 neurons.

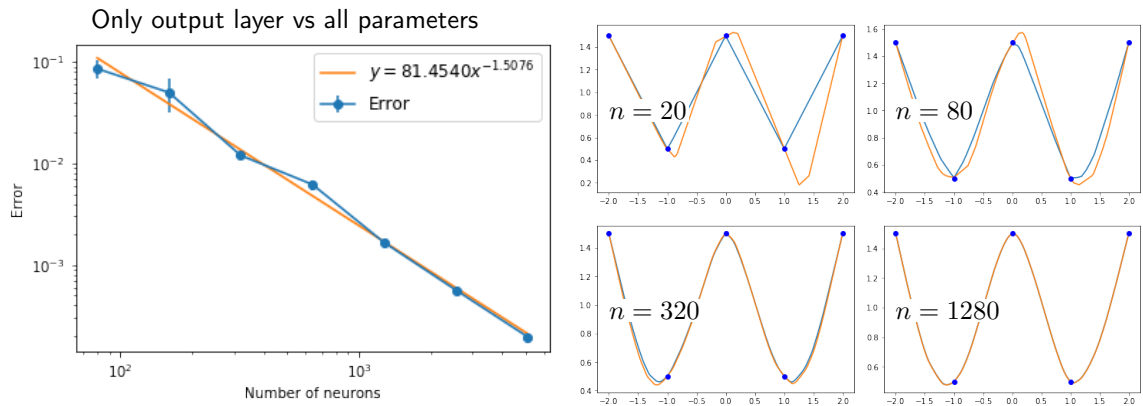
\begin{figure}
\centering
\begin{tikzpicture}[x=\textwidth,y=\textwidth, inner sep = 1pt]
\node[above right] at (0,0) {\includegraphics[width=0.49\textwidth]{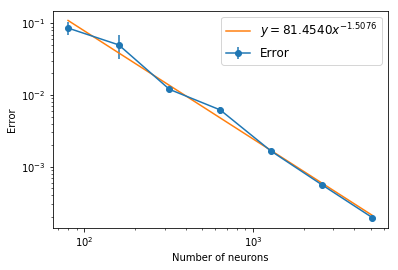}};
\node[above right] at (.5, .175) {\includegraphics[width=0.25\textwidth]{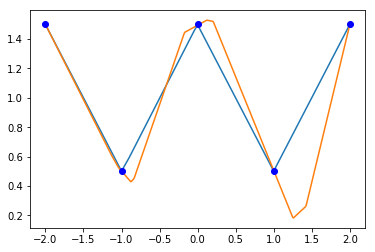}};
\node[above right] at (.75, .175) {\includegraphics[width=0.25\textwidth]{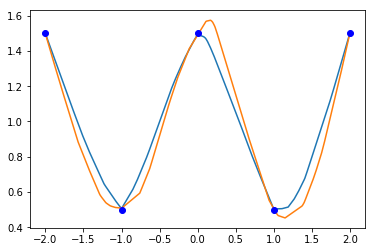}};
\node[above right] at (.5, 0) {\includegraphics[width=0.25\textwidth]{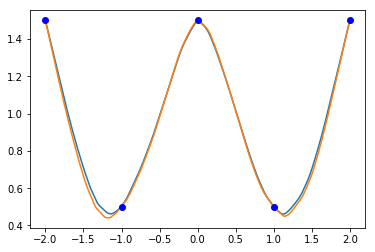}};
\node[above right] at (.75, 0) {\includegraphics[width=0.25\textwidth]{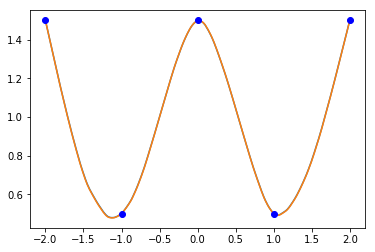}};
\node[above,fill=white] at (0.25, .33) {\small %
\textsf{Only output layer vs all parameters}};
\node[above right,fill=white] at (.5+.03, .175+.08) {\small $n=20$};
\node[above right,fill=white] at (.75+.03, .175+.08) {\small $n=80$};
\node[above right,fill=white] at (.5+.03, 0+.08) {\small $n=320$};
\node[above right,fill=white] at (.75+.03, 0+.08) {\small $n=1280$};
\end{tikzpicture}    
\caption{Illustration of Theorem~\ref{th-lin-onlyout}. Training only output layer vs training all parameters of the network. 
We use uniform initialization and the same training set as in Figure~\ref{fig:nr_neurons}. The left panel plots the error between two trained network functions against the number of neurons $n$. For one network, we only train the output layer while for the another one, we train all layers. The right panel shows the data (dots) and two trained network functions with 20, 80, 320, 1280 neurons. 
}
\label{fig:two_net}
\end{figure}

\paragraph{Effect of linear function on implicit bias} 
In Theorem~\ref{thm:theorem1}, since the variational problem defines functions only up to addition of linear functions, we need to adjust training data by subtracting a specific linear function $ux+v$. %
However, in our previous experiments, we observed that even if we do not adjust the training data, the statement of Theorem~\ref{thm:theorem1} still approximately holds. 
We attribute this to the fact that the linear function can be easily fit by the neural network. 
We provide details about this in Appendix~\ref{Difference_between_solutions_of_variational_problems}. 
In order evaluate the effect of this linear function on the implicit bias, we conduct the following experiment. 
Similar to Figure~\ref{fig:nr_neurons}, we use uniform initialization. We add a linear function $10x+10$ to the training data in Figure~\ref{fig:nr_neurons}. So the training data we use are $\{(-2,-8.5), (-1,0.5),  (0,11.5),  (1,20.5) , (2,31.5)\}$. 
In  Figure~\ref{fig:effect_linear} we show analogous experiments to those in the left panel of Figure~\ref{fig:nr_neurons}. In order to clearly show the difference between the trained network function and the solution to the variational problem, we subtract $10x+10$ from these two functions in the right panel of Figure~\ref{fig:effect_linear}. From the right panel of Figure~\ref{fig:effect_linear}, we see that the difference between plotted two functions is relatively larger than that in Figure~\ref{fig:nr_neurons}. From the left panel of Figure~\ref{fig:effect_linear}, we see that the error between these two functions stops to decrease when number of neurons $n$ is larger than $1280$. 
It means that the limit of trained network function as $n\to\infty$ is slightly different from the solution to the variational problem. 
If we choose bigger $u$ and $v$, we expect that the difference will become larger.

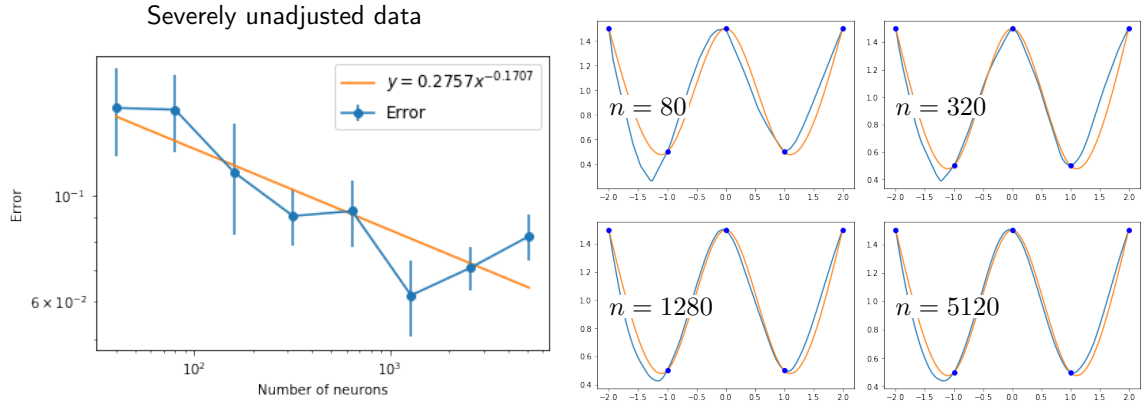
\begin{figure}
\centering
\begin{tikzpicture}[x=\textwidth,y=\textwidth, inner sep = 1pt]
\node[above right] at (0,0) {\includegraphics[width=0.49\textwidth]{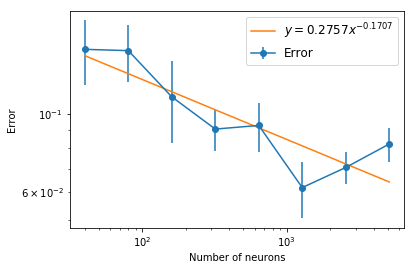}};
\node[above right] at (.5, .175) {\includegraphics[width=0.25\textwidth]{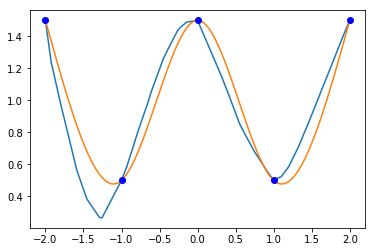}};
\node[above right] at (.75, .175) {\includegraphics[width=0.25\textwidth]{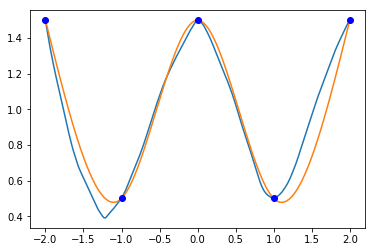}};
\node[above right] at (.5, 0) {\includegraphics[width=0.25\textwidth]{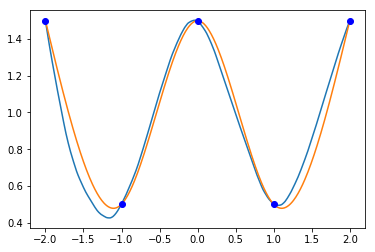}};
\node[above right] at (.75, 0) {\includegraphics[width=0.25\textwidth]{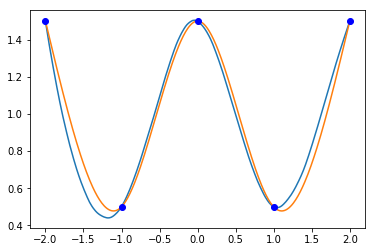}};
\node[above,fill=white] at (0.25, .33) {\small %
\textsf{Severely unadjusted data}};
\node[above right,fill=white] at (.5+.03, .175+.08) {\small $n=80$};
\node[above right,fill=white] at (.75+.03, .175+.08) {\small $n=320$};
\node[above right,fill=white] at (.5+.03, 0+.08) {\small $n=1280$};
\node[above right,fill=white] at (.75+.03, 0+.08) {\small $n=5120$};
\end{tikzpicture}    
\caption{
Effect of not adjusting the data. 
We use uniform initialization and add a linear function $10x+10$ to the training data of Figure~\ref{fig:nr_neurons}. 
In order to clearly show the difference between trained network function and the solution to the variational problem, we subtract $10x+10$ from these two functions in the right panel. 
In the right panel we see that if we ignore $u$ and $v$ in the variational problem \eqref{continuous_add_linear}, the solution is slightly different from \eqref{function_space}. 
    }
\label{fig:effect_linear}
\end{figure}

\paragraph{Experiments for two-dimensional regression problems}
We illustrate Theorem~\ref{thm:theorem_multi} numerically in Figure~\ref{error_against_n_2D}. 
We conduct experiments similar to Figure~\ref{fig:nr_neurons} and Figure~\ref{error_against_n_gaussian} for the bivariate case. 
The initialization used in Figure~\ref{error_against_n_2D} is $\bm{\mathcal{W}}\sim U(\mathbb{S}^1)$ and $\mathcal{B}\sim U(-2,2)$, thus we can use Theorem~\ref{closed_form_sol} to exactly compute the solution to the variational problem \eqref{gen_multi_dim}. 
In close agreement with the theory, the solution to the variational problem captures the solution of gradient descent training uniformly with error of order $n^{-1/2}$. %

To verify that the results are stable over different data sets, 
in Figure~\ref{error_against_n_2D_more_sample} we show an experiment similar to that of Figure~\ref{error_against_n_2D}, but for a larger data set.

\begin{figure}
    \centering
\begin{tikzpicture}[x=\textwidth,y=\textwidth, inner sep = 1pt]
\node[above right] at (0,0) {\includegraphics[width=0.49\textwidth]{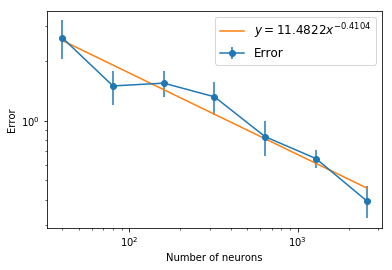}};
\node[above right] at (.5, .175) {\includegraphics[width=0.25\textwidth]{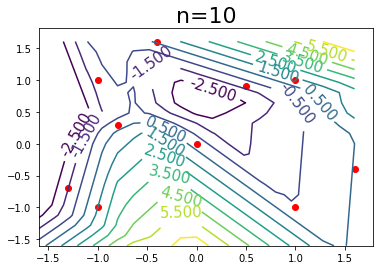}};
\node[above right] at (.75, .175) {\includegraphics[width=0.25\textwidth]{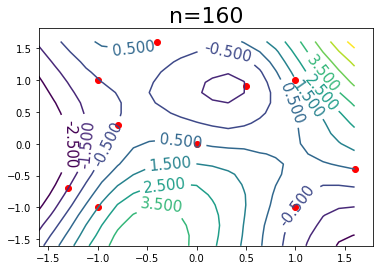}};
\node[above right] at (.5, 0) {\includegraphics[width=0.25\textwidth]{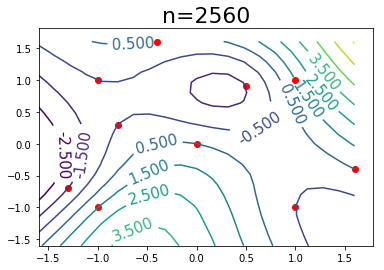}};
\node[above right] at (.75, 0) {\includegraphics[width=0.25\textwidth]{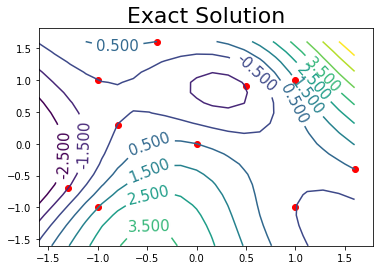}};
\end{tikzpicture}
\caption{Illustration of Theorem~\ref{thm:theorem_multi}. Similar to Figure~\ref{error_against_n_2D}, with the same initialization, but with a larger data set. 
}
\label{error_against_n_2D_more_sample}
\end{figure}
To illustrate Theorem~\ref{thm:theorem_multi} across different initialization procedures, %
in Figures~\ref{error_against_n_gaussian_2D} and~\ref{error_against_n_gaussian0.1_2D} we show analogous experiments to Figure~\ref{error_against_n_2D}, but using Gaussian initialization instead. 
The initialization used in Figure~\ref{error_against_n_gaussian_2D} is $\bm{\mathcal{W}}\sim \mathcal{N}(0,I_d)$ and $\mathcal{B}\sim \mathcal{N}(0,1)$, and the initialization used in Figure~\ref{error_against_n_gaussian0.1_2D} is $\bm{\mathcal{W}}\sim \mathcal{N}(0,I_d)$ and $\mathcal{B}\sim \mathcal{N}(0,0.1)$. 
So we can use Theorem~\ref{proposition:explicit-rho-gaussian-2d} to exactly compute the curvature penalty function and solve the variational problem \eqref{gen_multi_dim} numerically. 
\begin{figure}
    \centering
\begin{tikzpicture}[x=\textwidth,y=\textwidth, inner sep = 1pt]
\node[above right] at (0,0) {\includegraphics[width=0.49\textwidth]{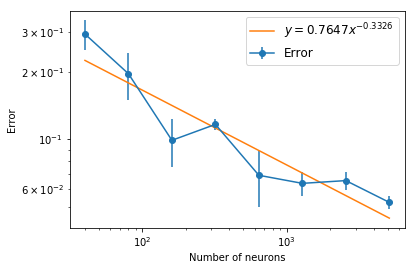}};
\node[above right] at (.5, .175) {\includegraphics[width=0.25\textwidth]{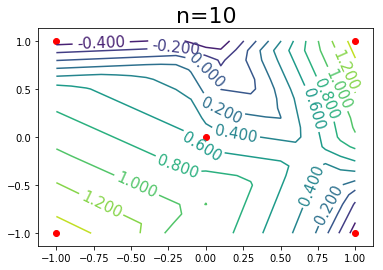}};
\node[above right] at (.75, .175) {\includegraphics[width=0.25\textwidth]{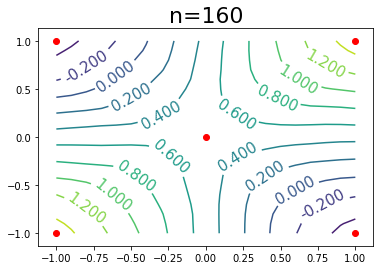}};
\node[above right] at (.5, 0) {\includegraphics[width=0.25\textwidth]{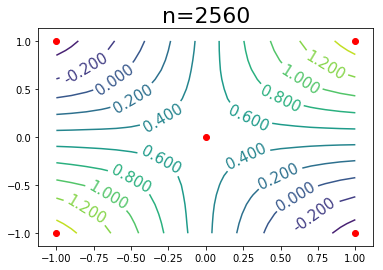}};
\node[above right] at (.75, 0) {\includegraphics[width=0.25\textwidth]{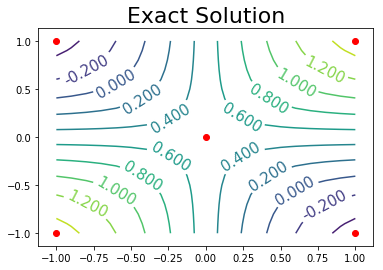}};
\end{tikzpicture}
\caption{Illustration of Theorem~\ref{thm:theorem_multi}. Similar to Figure~\ref{error_against_n_2D}, but with the Gaussian initialization $\bm{\mathcal{W}}\sim \mathcal{N}(0,I_d)$ and $\mathcal{B}\sim \mathcal{N}(0,1)$.
}
\label{error_against_n_gaussian_2D}
\end{figure}

\begin{figure}
    \centering
\begin{tikzpicture}[x=\textwidth,y=\textwidth, inner sep = 1pt]
\node[above right] at (0,0) {\includegraphics[width=0.49\textwidth]{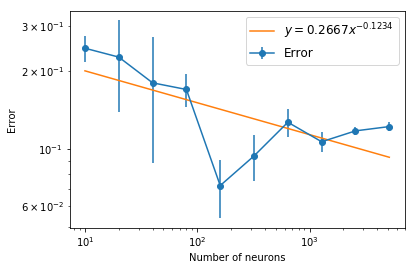}};
\node[above right] at (.5, .175) {\includegraphics[width=0.25\textwidth]{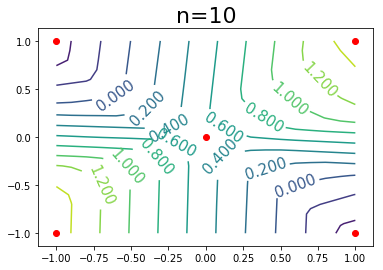}};
\node[above right] at (.75, .175) {\includegraphics[width=0.25\textwidth]{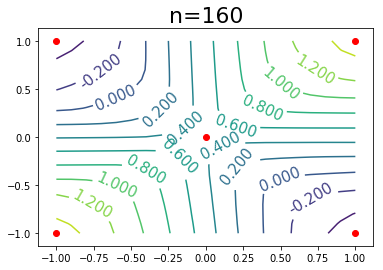}};
\node[above right] at (.5, 0) {\includegraphics[width=0.25\textwidth]{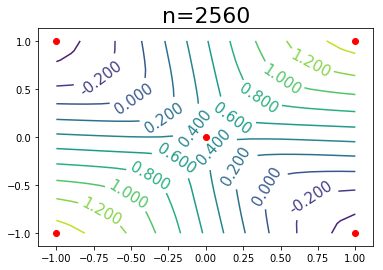}};
\node[above right] at (.75, 0) {\includegraphics[width=0.25\textwidth]{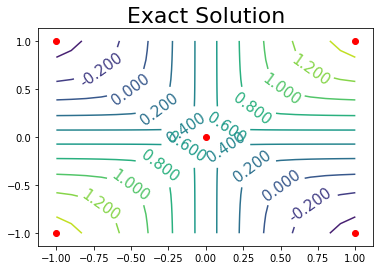}};
\end{tikzpicture}
\caption{Illustration of Theorem~\ref{thm:theorem_multi}. Similar to Figure~\ref{error_against_n_2D}, but with the Gaussian initialization $\bm{\mathcal{W}}\sim N(0,I_d)$ and $\mathcal{B}\sim N(0,0.1)$. Because of the linear adjustment, the exact solution of the variational problem \eqref{gen_multi_dim} is slightly different from the network output with a large number of hidden neurons.
}
\label{error_against_n_gaussian0.1_2D}
\end{figure}

\section{Additional Background on the NTK, Initialization, and Parametrization}
\label{app:additional-comments}

In this appendix we provide a few additional details on the NTK, ASI initialization, standard vs NTK parametrization, and discuss the difference between our results and weight norm minimization. 

\subsection{NTK Convergence and Positive-definiteness}
\label{app:NTK-conv-pos}

The convergence of the empirical NTK to a deterministic limiting NTK as the width of the network tends to infinity and the positive-definiteness of this limiting kernel can be ensured whenever the neural network converges to a Gaussian process. 
The arguments from \cite{jacot2018neural} to prove convergence and positive definiteness hold in this case. 
As they mention, the limiting NTK only depends on the choice of the network activation function, the depth of the network, and the variance of the parameters at initialization. 
They prove positive definiteness when the input data is supported on a sphere. 
More generally, positive definiteness can be proved based on the structure of the NTK as a covariance matrix. 
Let $\|f\|_p^2 = \mathbb{E}_{\mathbf{x}\sim p}[f(\mathbf{x})^T f(\mathbf{x})]$, where $p$ denotes the distribution of inputs. 
The NTK is positive definite when the span of the partial derivatives $\partial_{\theta_i} f(\cdot, \theta)$, $i=1,\ldots, d$, becomes dense in function space with respect to $\|\cdot\|_p$ as the width of the network tends to infinity \citep{jacot2018neural}. 
For a finite data set $\mathbf{x}_1,\ldots, \mathbf{x}_M$, positive definiteness of the corresponding Gram matrix is equivalent to $\partial_{\theta_i} f(\mathbf{x}_j, \cdot)$ being linearly independent \citep[Theorem~3.1]{du2018gradient}. 
This condition for positive definiteness does not depend on the specific distribution of the parameters, but if anything it only depends on the support of the distribution of parameters and on the input data. The precise value of the least eigenvalue may be affected by changes in the distribution however. 
The convergence of the network function to a Gaussian process in the limit of infinite width and independent parameter initialization is a classic result \citep{neal1996priors}. 
To verify this Gaussian process assumption it is sufficient that 
$\sum_{i}W_i^{(2)}\sigma(\langle \mathbf{W}^{(1)}_i,\mathbf{x}\rangle +b_i)$ is a sum of independent random variables with finite variance. 

\subsection{Anti-Symmetrical Initialization (ASI)} 
\label{app:ASI}

The AntiSymmetrical Initialization (ASI) trick as proposed by \citet{zhang2019type} creates duplicate hidden units with opposite output weights, ensuring that $f(\cdot, \theta_0)\equiv 0$. 
More precisely, ASI defines $f_{\mathrm{ASI}}(\mathbf{x},\vartheta) = \frac{\sqrt{2}}{2}f(\mathbf{x},\vartheta') - \frac{\sqrt{2}}{2}f(\mathbf{x},\vartheta'')$. Here $\vartheta = (\vartheta',\vartheta'')$ is initialized with $\vartheta_0'=\vartheta_0''$, so that 
\begin{equation*}
  f_{\mathrm{ASI}}(\mathbf{x},\vartheta_0)=\sum_{i=1}^{n} \frac{\sqrt{2}}{2}\overline{V}_i^{(2)}[\langle\bm{\overline{V}}_i^{(1)},\mathbf{x}\rangle +\overline{a}_i^{(1)}]_+ + \sum_{i=1}^n -\frac{\sqrt{2}}{2}\overline{V}_i^{(2)}[\langle\bm{\overline{V}}_i^{(1)},\mathbf{x}\rangle +\overline{a}_i^{(1)}]_+ \equiv 0. 
\end{equation*} 
The parameter vector at initialization is thus 
$\vartheta_0 = \mathrm{vec}(\bm{\overline{V}}^{(1)},\bm{\overline{V}}^{(1)},\bm{\overline{a}}^{(1)},\bm{\overline{a}}^{(1)},\frac{\sqrt{2}}{2} \bm{\overline{V}}^{(2)},-\frac{\sqrt{2}}{2} \bm{\overline{V}}^{(2)},\allowbreak \frac{\sqrt{2}}{2}\overline{a}^{(2)},-\frac{\sqrt{2}}{2}\overline{a}^{(2)})$.   

The basic statistics on the size of the parameters remains like \eqref{initialization2}, even if now there are perfectly correlated pairs of parameters. Hence the analysis and results on limits when the number of hidden units tends to infinity remain valid under ASI. 
The ASI is not needed for our analysis, which can be used to compare different types of initialization procedures, but it simplifies some of the presentation. 
One motivation for using ASI in practical applications is that it provides a simple way to implement a simple output function at initialization. 
Since the output function at initialization directly influences the bias of the gradient descent solution, this is a particular way to control the bias. 
Manipulating the bias from initialization is also the motivation presented by \cite{zhang2019type}. 
A related discussion also appears in \cite{sahs2020a}.
\subsection{Standard vs NTK Parametrization}
  \label{twoParametrization}

We have focused on the standard parametrization of the neural network. 
\cite{jacot2018neural} use a non-standard parametrization which is now known as the NTK parametrization. 
We briefly discuss the difference. 
A network with NTK parametrization is described as
  \begin{equation*}
    \displaystyle
    \begin{cases}
    \mathbf{h}^{(l+1)}=\sqrt{\frac{1}{n_l}}\mathbf{W}^{(l+1)}\mathbf{x}^l+\mathbf{b}^{(l+1)} \\
    \mathbf{x}^{(l+1)}=\phi(\mathbf{h}^{(l+1)})
    \end{cases}
    \quad
    \text{ and}
    \quad
    \begin{cases}
       W_{i,j}^{(l)} \sim  \mathcal{N}(0,1)\\
     b_{j}^{(l)} \sim  \mathcal{N}(0,1)
    \end{cases}.
  \end{equation*}
In contrast to the standard parametrization, in the NTK parametrization the factor $\sqrt{1/n_l}$ is carried outside of the trainable parameter. 
In this case, the scaling of the derivatives is $\nabla_{W_{i,j}^{(1)}} f(x,\theta_0)=O(n^{-\frac{1}{2}})$ and $\nabla_{W_i^{(2)}} f(x,\theta_0)=O(n^{-\frac{1}{2}})$. 
In turn, during training the changes of $W_{i,j}^{(1)}$ and $W_i^{(2)}$ are comparable in magnitude. 
This implies that we can not ignore the changes of $W_{i,j}^{(1)}$ and approximate the dynamics by that of the linearized model that trains only the output weights as we did in the case of the standard parametrization. 
In particular, we can not use problem %
\eqref{continuous_version} to describe the result of gradient descent as $n\to\infty$.

\subsection{Weight Norm Minimization} 
\label{Parameter_norm_min}
\citet{savarese2019infinite} studied networks of the form $f(x,\theta)=\sum_{i=1}^n W_i^{(2)}[W_i^{(1)}x +b_i^{(1)}]_+ +b^{(2)}$ allowing the width to tend to infinity. They showed that the minimum weight norm for approximating a given function $g$ is related to a measure of the smoothness of $g$ by $\lim_{\epsilon\to 0}(\inf_{\theta}C(\theta) \ \text{s.t.} \ \|f(\cdot,\theta)-g\|_\infty\leq \epsilon) =\max\{ \int_{-\infty}^{\infty}|g''(x)|~\mathrm{d}x, \; |g'(-\infty)+g'(\infty)| \}$, where $C(\theta)=\frac{1}{2}\sum_{i=1}^n((W_i^{(2)})^2+(W_i^{(1)})^2)$. 
Here the derivatives are understood in the weak sense. 
This implies that infinite width shallow networks trained with weight norm regularization (sparing biases) represent functions with smallest $1$-norm of the second derivative, an example of which are linear splines. 
(Note that $C(\theta)$ is not strictly convex in the space of all parameters and also the $1$-norm of the second derivative is not strictly convex, hence the solution is not unique).

The result of \cite{savarese2019infinite} is illuminating in that it connects properties of the parameters and properties of the represented functions. 
However, the result does not necessarily inform us about the functions represented by the network upon gradient descent training without explicit weight norm regularization. 
Indeed, if we initialize the parameters by \eqref{initialization2} with sub-Gaussian distribution, the neural network can be approximated by the linearized model. 
Then by Theorem~\ref{minimum_weight}, $\|\omega-\theta_0\|_2$ is minimized rather than $\|\omega\|_2$. But in this case $\|\theta_0\|_2$ is bounded away from zero with high probability and the $2$-norm of all parameters (or also of the weights only) is not minimized. 
On the other hand, if we initialize the parameters with $\|\theta_0\|_2$ close to $0$, then the neural network might not be well approximated by the linearized model. 
This has been observed experimentally by \cite{chizat2019lazy} and we further illustrate it in Appendix~\ref{linear-validity}. 

Even if we assume that the linearization of a network at the origin is valid, 
in order for the network to approximate certain complex functions, the weights necessarily have to be bounded away from zero. 
This means that reaching zero training error requires to move far from the basis point, where the difference between linearized and non-linearized model could become significant. In turn, the implicit bias description derived from a linearization at the origin may not accurately reflect the implicit bias of gradient descent in the original non-linearized model. 

The above paragraphs discuss why the result of \citet{savarese2019infinite} does not apply to gradient descent training without weight norm regularization. 
It is also interesting to discuss the difference between our result and the result of \cite{savarese2019infinite}. In our result, the implicit bias of gradient descent without weight norm regularization is characterized by 2-norm of the second derivative weighted by $1/\zeta$, which is a RKHS-norm. In the result of \cite{savarese2019infinite}, they showed that training with weight norm regularization (sparing biases) leads to functions with smallest 1-norm of the second derivative, which is not a RKHS norm. The reason why training without weight decay gives RKHS norm is because the training trajectory can be approximated by that of a linear model, which corresponds to a certain RKHS. And for training with weight norm regularization, the weight in the first layer is regularized, so it changes the feature space and we can no longer regard that as a linear model. Some works give empirical evidence that minimizing a non-RKHS norm can have better generalization than minimizing an RKHS norm because of the limitation of linear models and the kernel regime. However, as far as we know, there is no theory which shows that a non-RKHS-norm could result in better generalization than a RKHS norm. 

The paper by 
\cite{Parhi2019MinimumN} follows the approach of \citet{savarese2019infinite} and generalizes the result of \citet{savarese2019infinite} to different types of activation functions $\sigma$. Then they show that minimizing the weight “norm” of two-layer neural networks with activation function $\sigma$ is actually minimizing 1-norm of $\mathrm{L}f$ in place of the second derivative, where $f$ is the output function of the neural network. Here $\mathrm{L}$ and $\sigma$ satisfy $L\sigma=\delta$, i.e., $\sigma$ is a Green’s function of $\mathrm{L}$. 
Such activation functions can be used in combination with our analysis. We comment further on such generalizations in Appendix~\ref{Other_activation}.

\subsection{Basis Parameter for Linearization of the Model} 
\label{linear-validity} 

We discuss how the quality of the approximation of a neural network by a linearized model depends on the basis point. 
For a feedforward ReLU network and a list 
$\mathcal{X}=(x_i)_{i=1}^m$ of input data points, 
the mapping $\theta\mapsto f(\mathcal{X},\theta) = [f(x_1,\theta), \ldots, f(x_m,\theta)]$ is piecewise multilinear. Each of the pieces is smooth and we can assume that it is approximated reasonably well by its Taylor expansion. 
However, the quality of the approximation can drop when we cross the boundary between smooth pieces. 
Consider a single-input network with a layer of $n$ ReLUs and a single output unit. 
At an input $x$ the prediction is $f(x;\theta) =\sum_{j=1}^n W^{(2)}_j[W^{(1)}_j x + b^{(1)}_j]_+ + b^{(2)}$, where $\theta=\mathrm{vec}(\mathbf{W}^{(1)},\mathbf{b}^{(1)},\mathbf{W}^{(2)},b^{(2)})$. 
The Jacobian is non-smooth whenever $\theta \in H_{xj} =\{ W^{(1)}_{j}x + b^{(1)}_j=0\}$ for some $j=1,\ldots, n$. 
Hence for $m$ input data points $x_i$, $i=1,\ldots, m$, the locus of non-smoothness is given by $m$ central hyperplanes $H_{ij}$, $i=1,\ldots, m$ in the parameter space of each hidden unit $j=1,\ldots, n$. 
For an individual ReLU, if the parameter $\theta_0$ is drawn from a centrally symmetric probability distribution, the probability $p$ that an $\epsilon$ ball around $c \theta_0$ intersects one of the non-linearity hyperplanes $H_i$, $i=1,\ldots,m$, behaves roughly as $p = O(m c^{-1})$ as $c$ goes to infinity. 
Hence we can expect that the prediction function will be better approximated by its linearization $f^{\text{lin}}(x,\theta) = f(x,c\theta_0) + \nabla_\theta f(x,c\theta_0)(\theta-c\theta_0)$ at a point $c\theta_0$ if $c$ is larger. 
This is well reflected numerically in Figure~\ref{fig:linear-initialization}. 
As we see, for larger initialization the model looks more linear. 
We observed that this qualitative behavior remains same if we try to adjust the size of the window around the initial value.

\begin{figure}
    \centering
\begin{tabular}{cc}
\begin{minipage}{.42\textwidth}
\centering
\begin{tikzpicture}[x=1cm,y=1cm]
\draw[-,black, thick] (-.5,-1) -- (1,2); 
\node[right] at (1,2) {$H_1$};
\draw[-,black, thick] (1,-1) -- (-2,2); 
\node[right] at (1,-1) {$H_2$};
\node at (-1.8,-0) {$\mathbb{R}^d$}; 
\draw[fill,opacity=.5,blue] (.5,0.25) circle [radius=0.5]; 
\draw[fill] (.5,0.25) circle [radius=0.025]; 
\node[below] at (.5,0.25) {$c_1$};
\draw[fill,opacity=.5,red] (1.5,0.75) circle [radius=0.5]; 
\draw[fill] (1.5,0.75) circle [radius=0.025]; 
\node[below] at (1.5,0.75) {$c_2$};
\draw[dotted,->,black, thick] (0,0) -- (2.5,1.25); 
\node[right] at (2.5,1.25) {$\theta_0$};
\node at (0,-1.2) {}; 
\node at (0.2,2.6) {\footnotesize\textsf{Parameter space of a ReLU}}; 
\end{tikzpicture}\\
\begin{tikzpicture}
\node at (0,0) {\includegraphics[clip=true,trim=2cm 9.5cm 11cm 10.5cm,width=.9\textwidth]{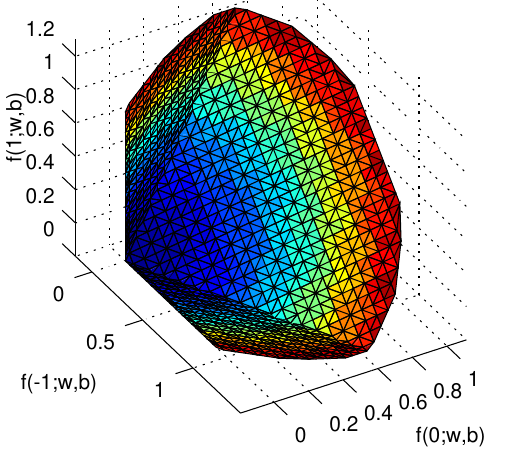}};
\node at (0,2.8) {\footnotesize\textsf{Predictions of a ReLU on 3 data points}};
\end{tikzpicture}
\end{minipage}
& 
\begin{minipage}{.59\textwidth}
\begin{tikzpicture}
\node at (0,0) {\includegraphics[width=7cm]{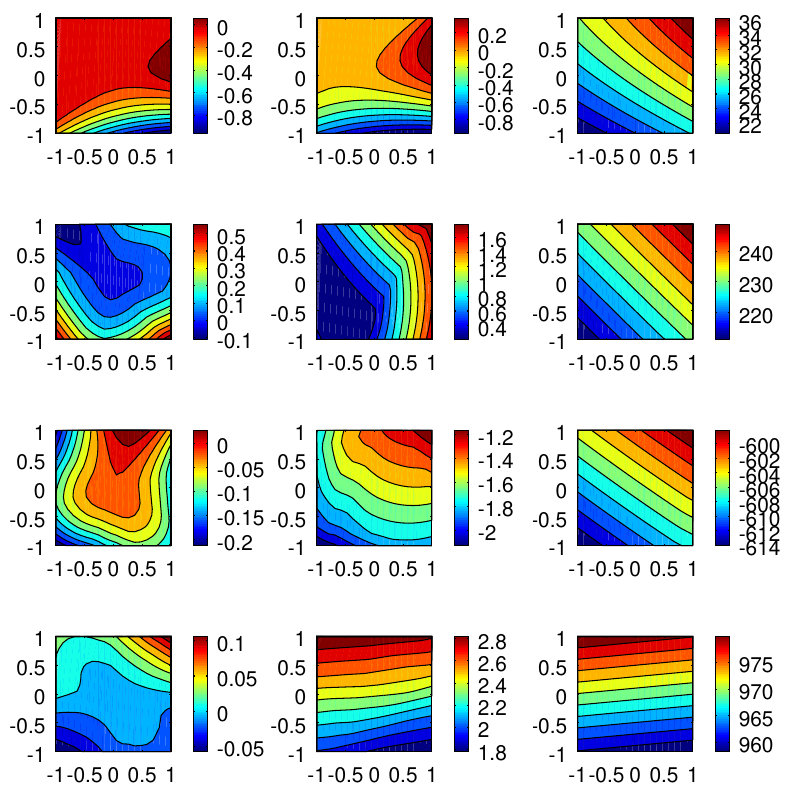}};
\node at (0,4) {\footnotesize\textsf{Network predictions over 2D parameter slices}};
\node at (-3.7,0) {\rotatebox[origin=c]{90}{\footnotesize\textsf{$\leftarrow$ more neurons}}};
\node  at (0,-3.8) {\footnotesize\textsf{larger initialization $\rightarrow$}};
\node at (0,-3.8) {}; 
\end{tikzpicture}
    \end{minipage}
\end{tabular}
    \caption{
Left: For a single ReLU, the map $\theta\mapsto f(\mathcal{X},\theta)$ from parameters to prediction vectors over a set $\mathcal{X} = \{x_1,\ldots, x_m\}$ of $m$ input data points is piecewise linear, with pieces separated by $m$ central hyperplanes. 
Right: Shown is the prediction $f(x,\theta)$ of a shallow ReLU network on a fixed input point $x$, over a 2D slice of parameters $\theta = c \theta_0 + v_1\xi_1 + v_2\xi_2$ 
spanned by two random orthogonal unit norm vectors $v_1$, $v_2$ and parametrized by $(\xi_1,\xi_2)\in[-1,1]^2$. 
    From top to bottom, the number of hidden units is $n = 1, 5, 25, 125$ and in each row the initial parameter $\theta_0$ is drawn i.i.d.\ from a standard Gaussian. 
    In each column we use a different scaling constant $c = 0, 0.5, 10$. As we see, for larger scaling $c$ of the initialization the model looks more linear. }
    \label{fig:linear-initialization}
\end{figure}
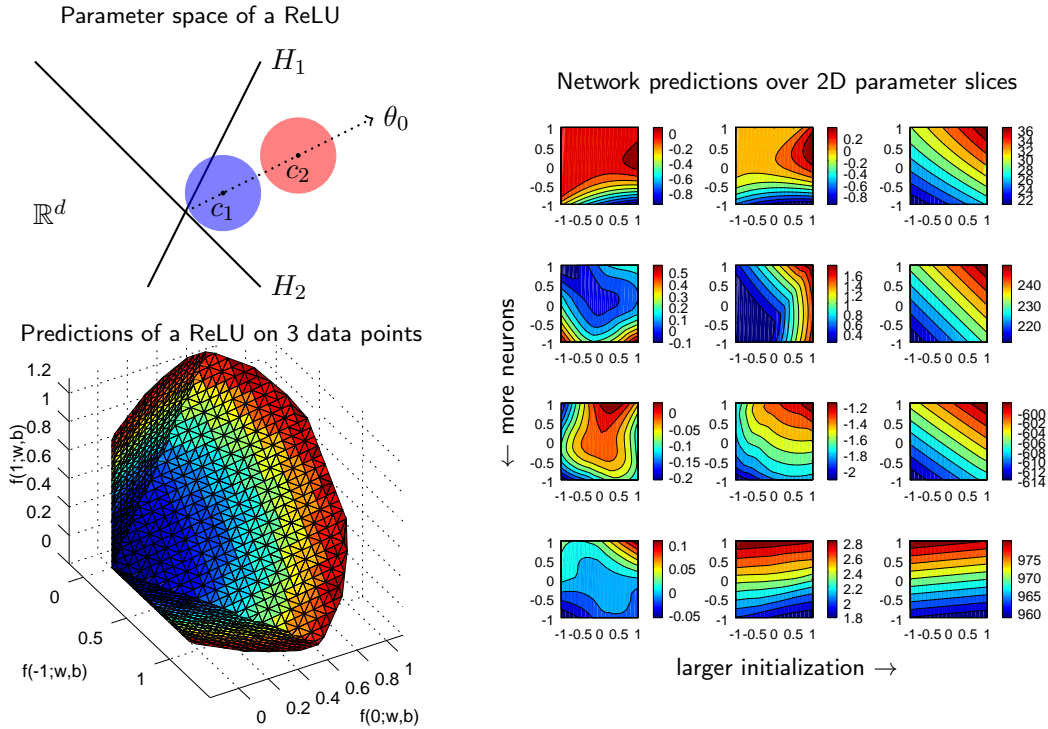

\section{Proof of Theorem~\ref{thm:theorem1} and Theorem~\ref{thm:theorem_multi}}
\label{app:proof-thm1}

The proof of Theorem \ref{thm:theorem1} and Theorem~\ref{thm:theorem_multi} is the compilation of results from Sections \ref{sec:3}, \ref{2.4}, \ref{sec:implicit_bias_univariate} and \ref{sec:implicit_bias_multivariate}. 
Next we give the proof of Theorem~\ref{thm:theorem_multi}. 
Theorem \ref{thm:theorem1} can be similarly proved. 
\begin{proof}[Proof of Theorem \ref{thm:theorem_multi}]
The convergence to zero training error for ReLU networks is by now a well known result \citep{du2018gradient,pmlr-v97-allen-zhu19a}. We proceed with the implicit bias result. 

For simplicity, we give the proof under ASI (see Appendix \ref{app:ASI}). 
In Section~\ref{sec:implicit_bias_multivariate}, we relax the optimization problem \eqref{continuous_version} to \eqref{continuous_version_multi_relax}. Suppose $(\overline{\alpha},\bm{\overline{u}},\overline{v})$ is the solution of \eqref{continuous_version_multi_relax}. The we can adjust the training samples $\{(\mathbf{x}_i,y_i)\}_{i=1}^M$ to  $\{(\mathbf{x}_i,y_i-\langle\bm{\overline{u}},\mathbf{x}_i\rangle-\overline{v})\}_{i=1}^M$. It's easy to see that on the adjusted training samples, $(\overline{\alpha},\mathbf{0},0)$ is the solution of \eqref{continuous_version_multi_relax}. Then $\overline{\alpha}$ is the solution of \eqref{continuous_version} on the adjusted data. Furthermore, the solution of \eqref{continuous_version} in function space, $g(\mathbf{x},\overline{\alpha})$, 
equals to the solution of \eqref{continuous_version_multi_relax} in function space, $g(\mathbf{x},(\overline{\alpha},0,0))$, i.e.,
\begin{equation}
\label{step1}
  g(\mathbf{x},\overline{\alpha})=g(\mathbf{x},(\overline{\alpha},\mathbf{0},0)) .  
\end{equation}

It we change the variable $\alpha$ to $\gamma$ as in Section \ref{sec:implicit_bias_multivariate}, we get
\begin{equation}
\label{step2}
    g(\mathbf{x},(\overline{\alpha},\mathbf{0},0))=g(\mathbf{x},(\overline{\gamma},\mathbf{0},0)),
\end{equation} 
On any compact set $D\subset \mathbb{R}^d$, %
according to Theorem \ref{theorem4}, 
\begin{equation}
\label{step3}
    \sup_{\mathbf{x}\in D}|g_n(\mathbf{x},\overline{\alpha}_n)- g(\mathbf{x},\overline{\alpha})|=O_p(n^{-1/2}),
\end{equation}
where $g_n(\mathbf{x},\overline{\alpha}_n)$ is the solution of problem \eqref{probablity_version} in function space. 
Since problem \eqref{probablity_version} is equivalent to problem \eqref{direct_version_non_ASI}, $g_n(\mathbf{x},\overline{\alpha}_n)$ is also the solution of \eqref{direct_version_non_ASI} in function space. According to discussion in Section~\ref{2.4}, $f^{\mathrm{lin}}(\mathbf{x}, \widetilde{\omega}_\infty)$ is the solution of \eqref{direct_version_non_ASI}. Then we have
\begin{equation}
\label{step4}
    g_n(\mathbf{x},\overline{\alpha}_n)=f^{\mathrm{lin}}(\mathbf{x}, \widetilde{\omega}_\infty).
\end{equation}
According to Corollary~\ref{cor:outlayer-only}, we get
\begin{equation}
\label{step5}
\sup_{\mathbf{x}\in D}|f^{\mathrm{lin}}(\mathbf{x},\tilde\omega_\infty)-f(\mathbf{x},\theta^*)|=O_p(n^{-\frac{1}{2}}).
\end{equation}
Finally, according to Theorem~\ref{theorem_func_multi_dim} (to prove Theorem \ref{thm:theorem1}, apply Theorem~\ref{theorem_func} and Proposition \ref{proposition:form-rho}), $g(\mathbf{x},(\overline{\gamma},\mathbf{0},0))$  %
is the solution of \eqref{gen_multi_dim}, which is $g^*(\mathbf{x})$. 
It means that %
\begin{equation}
\label{step6}
    g(\mathbf{x},(\overline{\gamma},0,0))=g^*(\mathbf{x}).
\end{equation}
Combining \eqref{step1}, \eqref{step2}, \eqref{step3}, \eqref{step4}, \eqref{step5}, \eqref{step6}, %
we prove the theorem.
\end{proof}

\section{Implicit Bias in Parameter Space for a Linearized Model}
  \label{proof1}

\citet{zhang2019type} show that gradient flow converges to the solution with zero empirical loss which is closest to the initial weights. 
We show a similar result for the case of gradient descent with small enough learning rate. 

\begin{theorem}[Bias of the linearized model in parameter space]
  \label{minimum_weight}
Consider a convex loss function $\ell$ with a unique finite minimum and its derivative is $K$-Lipschitz continuous, i.e., $|\frac{\mathrm{d}}{\mathrm{d}y}\ell(y_1,\hat{y})-\frac{\mathrm{d}}{\mathrm{d}y}\ell(y_2,\hat{y})|\leq K|y_1-y_2|$. 
If $\mathrm{rank}(\nabla_\theta f(\mathcal{X},\theta_0))=M$, then the gradient descent iteration~\eqref{linearized} with learning rate $\eta\leq\frac{M}{Kn\lambda_{\max}(\hat{\Theta}_n)}$ converges to the unique solution of following constrained optimization problem: 
\begin{equation}
  \min_{\omega} \|\omega- \theta_0\|_2 \quad \text{s.t. } f^{\mathrm{lin}}(\mathcal{X},\omega)=\mathcal{Y}. 
  \label{problem_appendix}
\end{equation}
\end{theorem}
\gmm{The derivative $\frac{\mathrm{d}}{\mathrm{d}y}$ is with respect to the first argument of $\ell$ and the gradient $\nabla_\theta$ is with respect to the second argument of $f$ (see notation in Section~\ref{sec:notation}).} 

\begin{remark}[Remark on Theorem~\ref{minimum_weight}, step size]
Note that this statement is valid for the linearization of any set of functions, not only neural networks. 
The proof remains valid for a changing step size as long as this satisfies the required inequality. 
\end{remark}
\begin{remark}[Remark on Theorem~\ref{minimum_weight}, rank assumption]
The assumption $\nabla_\theta f(\mathcal{X}, \theta_0)=M$ is satisfied in most cases when $n\geq M$ (here $n$ refers to the number of parameters in $\theta$ since we use the linearized model). 
This is because $\nabla_\theta f(\mathcal{X}, \theta_0)$ is a $M\times n$ matrix. The $M$ rows corresponds to $M$ training samples and they are almost always linearly independent.
\end{remark}
Here we give the proof of Theorem~\ref{minimum_weight}. 
We note that \cite{zhang2019type} prove a similar result for gradient flow. Our proof is for finite step size and different from theirs.  
  \begin{proof}[Proof of Theorem \ref{minimum_weight}]
    We use gradient descent to minimize $L^{\mathrm{lin}}(\omega)=\frac{1}{M}\sum_{i=1}^M \ell(f^{\mathrm{lin}}(\mathbf{x}_i,\omega),y_i)$. First we prove that $\nabla_\omega L^{\mathrm{lin}}(\omega)$
 is Lipschitz continuous as follows:

 \begin{equation*}
   \begin{aligned}
    &\|\nabla_\omega L^{\mathrm{lin}}(\omega_1)-\nabla_\omega L^{\mathrm{lin}}(\omega_2)\|_2\\
    =& \frac{1}{M}\|\nabla_\theta f(\mathcal{X},\theta_{0})^\top \nabla_{f^{\mathrm{lin}}(\mathcal{X},\omega_{1})}L-\nabla_\theta f(\mathcal{X},\theta_{0})^\top \nabla_{f^{\mathrm{lin}}(\mathcal{X},\omega_{2})}L\|_2\\
    \leq& \frac{1}{M}\| \nabla_\theta f(\mathcal{X},\theta_{0})^\top\|_2\| \nabla_{f^{\mathrm{lin}}(\mathcal{X},\omega_{1})}L-\nabla_{f^{\mathrm{lin}}(\mathcal{X},\omega_{2})}L\|_2\\
    =& \frac{1}{M} \|\nabla_\theta f(\mathcal{X},\theta_{0})^\top\|_2\sqrt{\sum_{i=1}^M\left(\frac{\mathrm{d}}{\mathrm{d}y}l(f^{\mathrm{lin}}(\mathbf{x}_i,\omega_{1}),y_i)-\frac{\mathrm{d}}{\mathrm{d}y}l(f^{\mathrm{lin}}(\mathbf{x}_i,\omega_{2}),y_i)\right)^2}\\
    \leq& \frac{K}{M} \| \nabla_\theta f(\mathcal{X},\theta_{0})^\top\|_2\| f^{\mathrm{lin}}(\mathcal{X},\omega_{1})-f^{\mathrm{lin}}(\mathcal{X},\omega_{2})\|_2\quad\text{(K-Lipschitz continuity of $\ell$)}\\
    =&\frac{K}{M}  \|\nabla_\theta f(\mathcal{X},\theta_{0})^\top\|_2\| \nabla_\theta f(\mathcal{X},\theta_{0})(\omega_1-\omega_2)\|_2\\
    \leq&\frac{K}{M}  \|\nabla_\theta f(\mathcal{X},\theta_{0})^\top\|_2\| \nabla_\theta f(\mathcal{X},\theta_{0})\|_2\|(\omega_1-\omega_2)\|_2\\
    \leq&\frac{Kn}{M}  \lambda_{\max}(\hat{\Theta}_n)\|\omega_1-\omega_2\|_2. 
   \end{aligned}
 \end{equation*}
 So $L^{\mathrm{lin}}(\omega)$
is Lipschitz continuous with Lipschitz constant $\frac{Kn}{M}  \lambda_{\max}(\hat{\Theta}_n)$. Since $L^{\mathrm{lin}}$ is convex over $\omega$, gradient descent with learning rate $\eta=\frac{M}{Kn \lambda_{\max}(\hat{\Theta}_n)}$ converges to a global minimium of $L^{\mathrm{lin}}(\omega)$. %
     By assumption that $\mathrm{rank}(\nabla_\theta f(\mathcal{X},\theta_0))=M$, the model can perfectly fit all data. Then the minimium of $L^{\mathrm{lin}}(\omega)$ is zero and
    gradient descent converges to zero loss.

    Let $\omega_\infty = \lim_{t\to\infty}\omega_t$. Then $f^{\mathrm{lin}}(\mathcal{X},\omega_\infty)=\mathcal{Y}$. According to gradient descent iteration,
    \begin{equation*}
      \begin{aligned}
        \omega_\infty&=\theta_0-\sum_{t=0}^\infty \eta\nabla_\theta f(\mathcal{X},\theta_{0})^T \nabla_{f^{\mathrm{lin}}(\mathcal{X},\omega_{t})}L^{\mathrm{lin}}\\
        &=\theta_0- \eta\nabla_\theta f(\mathcal{X},\theta_{0})^T \sum_{t=0}^\infty \nabla_{f^{\mathrm{lin}}(\mathcal{X},\omega_{t})}L^{\mathrm{lin}}.
      \end{aligned}
    \end{equation*}

    Since $f^{\mathrm{lin}}$ is linear over weights $\omega$ and $\|\omega- \theta_0\|_2$ is strongly convex, the constrained optimization problem \eqref{problem_appendix} is a strongly convex optimization problem. The first order optimality condition of the problem is
    \begin{equation}
      \begin{cases}
      \omega- \theta_0 +  \nabla_\theta f^{\mathrm{lin}}(\mathcal{X},\theta_{0})^T \lambda = 0,\\
      f^{\mathrm{lin}}(\mathcal{X},\omega)=\mathcal{Y}.
      \end{cases}
      \label{condition}
    \end{equation}
    Let $ \lambda = \sum_{t=0}^\infty \nabla_{f^{\mathrm{lin}}(\mathcal{X},\theta_{t})}L$, we can easily check that $\omega_\infty$ satisfies condition \eqref{condition}. So $\omega_\infty$ is the solution of problem \eqref{problem_appendix}. 
  \end{proof}

\begin{remark}[Remark on Theorem~\ref{minimum_weight}]
Making an analogous statement to Theorem~\ref{minimum_weight} to describe the bias in parameter space when training wide networks rather than the linearized model is interesting, but harder, because the gradient direction is no longer constant. 
\cite{oymak2018overparameterized} obtain bounds on the trajectory length in parameter space, putting the final solution within a factor $4\beta/\alpha$ of $\min_\theta\|\theta_0-\theta\|$, where $\beta$ and $\alpha$ are upper and lower bounds on the singular values of the Jacobian over the relevant region. However, currently it is unclear whether the solution upon gradient optimization is indeed the distance  minimizer from initialization. 
\end{remark}

\hj{Next we discuss the implicit bias of SGD (stochastic gradient descent) in parameter space. Consider the following stochastic gradient descent iteration for the linearized model:
\begin{equation}
 \omega_0=\theta_0,\quad 
  \omega_{t+1}
  = \omega_t- \eta_t \frac{\mathrm{d}}{\mathrm{d}y}\ell(f^{\mathrm{lin}}(\mathbf{x}_{r(t)},\omega_{t}),y_{r(t)})\nabla_\theta f(\mathbf{x}_{r(t)},\theta_{0}), 
  \label{eq:linearized_sgd}
\end{equation}
where $r(t)$ is evenly chosen from the set $\{1,2,...,M\}$ and $\eta_t$ is the learning rate at the step $t$. 
Typically, $\eta_t$ needs to decay in order for SGD
to converge. However, for overparametrized linearized model, we can show that SGD converges for constant learning rate and the implicit bias of SGD is the same as gradient descent under certain conditions. This is shown in the following theorem.}

  \hj{\begin{theorem}[Bias of the linearized model in parameter space, SGD]
  \label{minimum_weight_sgd}
Consider a convex loss function $\ell$ with a unique finite minimum and its derivative is $K$-Lipschitz continuous, i.e., $|\frac{\mathrm{d}}{\mathrm{d}y}\ell(y_1,\hat{y})-\frac{\mathrm{d}}{\mathrm{d}y}\ell(y_2,\hat{y})|\leq K|y_1-y_2|$. 
If $\mathrm{rank}(\nabla_\theta f(\mathcal{X},\theta_0))=M$, the stochastic gradient descent iteration~\eqref{eq:linearized_sgd} with constant learning rate $\eta_t=\eta\leq\frac{1}{K\max_j\|\nabla_\theta f(\mathbf{x}_{j},\theta_{0})\|_2^2}$ converges to the unique solution of following constrained optimization problem with probability $1$: 
\begin{equation}
  \min_{\omega} \|\omega- \theta_0\|_2 \quad \text{s.t. } f^{\mathrm{lin}}(\mathcal{X},\omega)=\mathcal{Y}. 
  \label{problem_appendix_sgd}
\end{equation}
\end{theorem}}
\begin{highlight}
\begin{proof}[Proof of Theorem \ref{minimum_weight_sgd}]
Let $\omega^*$ be the solution to the optimization problem \eqref{problem_appendix_sgd}. Let $\mathbf{z}_j=\nabla_\theta f(\mathbf{x}_{j},\theta_{0})$. It is easy to see that $\omega_t-\langle \omega_t-\omega^*,\frac{\mathbf{z}_j}{\|\mathbf{z}_j\|_2} \rangle \frac{\mathbf{z}_j}{\|\mathbf{z}_j\|_2}$ is the projection of $\omega_t$ onto the hyperplane $\{\langle \omega, \mathbf{z}_j \rangle\}=\{\langle \omega^*, \mathbf{z}_j \rangle\}$. So for any $\hat{\eta}\leq 1$, we have
\begin{align}
    \left\|\omega_t-\hat{\eta}\langle \omega_t-\omega^*,\frac{\mathbf{z}_j}{\|\mathbf{z}_j\|_2} \rangle \frac{\mathbf{z}_j}{\|\mathbf{z}_j\|_2}-\omega^*\right\|_2^2
    &=\|\omega_t-\omega^*\|^2_2-(1-(1-\hat{\eta})^2)\left|\langle \omega_t-\omega^*,\frac{\mathbf{z}_j}{\|\mathbf{z}_j\|_2^2} \rangle\right|^2\label{eq:proj_equal}
    \\
    &\leq \|\omega_t-\omega^*\|^2_2.
    \label{eq:proj_less}
\end{align}
The length of the stochastic gradient in \eqref{eq:linearized_sgd} can be bounded as follows:
\begin{align*}
    &\eta_t \frac{\mathrm{d}}{\mathrm{d}y}\ell(f^{\mathrm{lin}}(\mathbf{x}_{r(t)},\omega_{t}),y_{r(t)})\|\mathbf{z}_{r(t)}\|_2 \\
    &\leq \eta_t K|f^{\mathrm{lin}}(\mathbf{x}_{r(t)},\omega_{t})-y_{r(t)}|\|\mathbf{z}_{r(t)}\|_2 \\
    &\leq  K\frac{1}{K\max_j\|\nabla_\theta f(\mathbf{x}_{j},\theta_{0})\|_2^2}|f^{\mathrm{lin}}(\mathbf{x}_{r(t)},\omega_{t})-y_{r(t)}|\|\mathbf{z}_{r(t)}\|_2 \\
    &=  \frac{1}{\max_j\|\mathbf{z}_{j}\|_2^2}\langle \omega_t-\omega^*, \mathbf{z}_{r(t)}\rangle\|\mathbf{z}_{r(t)}\|_2 \\
    &\leq \frac{1}{\max_j\|\mathbf{z}_{j}\|_2^2}\|\mathbf{z}_{r(t)}\|^2_2\langle \omega_t-\omega^*, \frac{\mathbf{z}_{r(t)}}{\|\mathbf{z}_{r(t)}\|_2}\rangle \\ 
    &\leq \langle \omega_t-\omega^*, \frac{\mathbf{z}_{r(t)}}{\|\mathbf{z}_{r(t)}\|_2}\rangle . 
\end{align*}
Then according to \eqref{eq:proj_less}, we have \begin{equation*}
    \left\|\omega_t-\eta_t \frac{\mathrm{d}}{\mathrm{d}y}\ell(f^{\mathrm{lin}}(\mathbf{x}_{r(t)},\omega_{t}),y_{r(t)})\|\mathbf{z}_{r(t)}\|_2 \frac{\mathbf{z}_{r(t)}}{\|\mathbf{z}_{r(t)}\|_2}-\omega^*\right\|_2\leq \|\omega_t-\omega^*\|_2.
\end{equation*}
The above equation means that
\begin{equation}
    \|\omega_{t+1}-\omega^*\|_2\leq \|\omega_t-\omega^*\|_2.
    \label{eq:decreasing}
\end{equation}
Then $\|\omega_t\|_2$ is bounded and $\lim_{t\to\infty}\|\omega_t-\omega^*\|_2- \|\omega_{t+1}-\omega^*\|_2=0$. 
Next we show that for any convergent subsequence $\{\omega_{t_k}\}_{k\geq 1}$ of $\{\omega_t\}_{t\geq 1}$, we have $\lim_{k\to\infty}\omega_{t_k}=\omega^*$.

Let $\lim_{k\to\infty}\omega_{t_k}=\bar{\omega}$. Asuume that $\bar{\omega}\not=\omega^*$. According to the first order optimality \eqref{condition}, we have that $\omega^*=\theta_0+\sum_{j=1}^{M} \lambda_j \mathbf{z}_j$. From the stochastic gradient descent iterations, we have $\omega_t=\theta_0-\eta\sum_{s=1}^{t-1} \frac{\mathrm{d}}{\mathrm{d}y}\ell(f^{\mathrm{lin}}(\mathbf{x}_{r(s)},\omega_{s}),y_{r(s)})\mathbf{z}_{r(s)}$. Then $\omega_t-\omega^*$ is a linear combination of $\{\omega_j\}_{j=1}^M$. It means that $\bar{\omega}-\omega^*$ is a linear combination of $\{\mathbf{z}_j\}_{j=1}^M$. Since $\bar{\omega}-\omega^*$ is not zero, the set  $A=\left\{j:\left|\langle \bar{\omega}-\omega^*,\frac{\mathbf{z}_{j}}{\|\mathbf{z}_{j}\|_2} \rangle\right|>0\right\}$ is not empty. With probability 1, we have that $r(t)\in A$ infinitely many times. So for any given $k$, we can find $t_k'\geq t_k$ such that $r(t_k')\in A$ and $r(t)\not\in A$ for $t_k\leq t<t_k'$.

When we prove \eqref{eq:decreasing}, we only use the property that $f^{\mathrm{lin}}(\mathbf{x}_{r(t)},\omega^*)=y_{r(t)}$. When $t_k\leq t<t_k'$, we have $r(t)\not\in A$, so $\langle \bar{\omega},\frac{\mathbf{z}_{j}}{\|\mathbf{z}_{j}\|_2} \rangle=\langle \omega^*,\frac{\mathbf{z}_{r(t)}}{\|\mathbf{z}_{r(t)}\|_2} \rangle$. It means that $f^{\mathrm{lin}}(\mathbf{x}_{r(t)},\bar{\omega})=f^{\mathrm{lin}}(\mathbf{x}_{r(t)},\omega^*)=y_{r(t)}$. Using the same argument as  \eqref{eq:decreasing}, we have $\|\omega_{t+1}-\bar{\omega}\|_2\leq \|\omega_t-\bar{\omega}\|_2$ when $t_k\leq t<t_k'$. Then $\|\omega_{t_k'}-\bar{\omega}\|_2\leq \|\omega_{t_k}-\bar{\omega}\|_2$. Then $\lim_{k\to\infty}\omega_{t_k'}=\lim_{k\to\infty}\omega_{t_k}=\bar{\omega}$.
According to \eqref{eq:proj_equal}, we have
\begin{align*}
    \|\omega_{t+1}-\omega^*\|^2_2 &=\|\omega_t-\omega^*\|^2_2-(1-(1-\tilde{\eta}_t)^2)\left|\langle \omega_t-\omega^*,\frac{\mathbf{z}_{r(t)}}{\|\mathbf{z}_{r(t)}\|_2} \rangle\right|^2\\
    \text{and }\tilde{\eta}_t&=\frac{\eta \frac{\mathrm{d}}{\mathrm{d}y}\ell(f^{\mathrm{lin}}(\mathbf{x}_{r(t)},\omega_{t}),y_{r(t)})\|\mathbf{z}_{r(t)}\|_2}{\left|\langle \omega_t-\omega^*,\frac{\mathbf{z}_{r(t)}}{\|\mathbf{z}_{r(t)}\|_2} \rangle\right|}.
\end{align*}
Since  $\lim_{k\to\infty}\omega_{t_k'}=\bar{\omega}$, for sufficiently large $k$ we have
\begin{align}
    \left|\langle \omega_{t_k'}-\omega^*,\frac{\mathbf{z}_{r({t_k'})}}{\|\mathbf{z}_{r({t_k'})}\|_2} \rangle\right|^2 &\geq \frac{1}{2}\min_{j\in A}\left|\langle \bar{\omega}-\omega^*,\frac{\mathbf{z}_{r(j)}}{\|\mathbf{z}_{r(j)}\|_2} \rangle\right|^2 \nonumber\\
    &=\Omega(1),\label{eq:angle_estimate}
\end{align}
and
\begin{align}
    \tilde{\eta}_t&\geq \frac{1}{2}\frac{\eta \min_{j\in A}\frac{\mathrm{d}}{\mathrm{d}y}\ell(f^{\mathrm{lin}}(\mathbf{x}_{j},\bar{\omega}),y_{j})\min_{j\in A}\|\mathbf{z}_{j}\|_2}{\|\bar{\omega}-\omega^*\|_2}\nonumber\\
    &=\Omega(1)\min_{j\in A}\frac{\mathrm{d}}{\mathrm{d}y}\ell(f^{\mathrm{lin}}(\mathbf{x}_{j},\bar{\omega}),y_{j})\nonumber\\ 
    &=\Omega(1),\label{eq:mid_step}
\end{align}
where \eqref{eq:mid_step} holds because $f^{\mathrm{lin}}(\mathbf{x}_{j},\bar{\omega})-y_{j}=\langle \bar{\omega}-\omega^*,\mathbf{z}_{j} \rangle\not=0$ for all $j\in A$ and $\frac{\mathrm{d}}{\mathrm{d}y}\ell(y,\hat{y})=0$ if and only if $y=\hat{y}$ according to the fact that loss function $\ell$ has a unique finite minimum. From \eqref{eq:angle_estimate} and \eqref{eq:mid_step} we have $\|\omega_{t_k'}-\omega^*\|^2_2-\|\omega_{t_k'+1}-\omega^*\|^2_2=\Omega(1)$. This contradicts the fact that $\lim_{t\to\infty}\|\omega_t-\omega^*\|_2- \|\omega_{t+1}-\omega^*\|_2=0$. Then the assumption $\bar{\omega}\not=\omega^*$ is not true. So for any convergent subsequence $\{\omega_{t_k}\}_{k\geq 1}$ of $\{\omega_t\}_{t\geq 1}$, we have $\lim_{k\to\infty}\omega_{t_k}=\omega^*$. Combining the above statement with the fact that $\|\omega_t\|_2$ is bounded, we have $\lim_{t\to\infty}\omega_{t}=\omega^*$
\end{proof}
\end{highlight}

\hj{\begin{remark}[Remark on Theorem~\ref{minimum_weight_sgd}]
\label{rk:sgd_gd}
Theorem~\ref{minimum_weight_sgd} shows that SGD and gradient descent has the same implicit bias in parameter space. Then our main theorem also holds for SGD training.
\end{remark}}

\section{Proof of Theorem~\ref{th-lin-onlyout}} 
\label{Proof2}

We note that assumption  $\operatorname{liminf}_{n\to\infty} \lambda_{\operatorname{min}}(\hat \Theta_n)>0$ is satisfied if the empirical NTK converges and the limit NTK is positive definite. 
For details see Appendix~\ref{app:NTK-conv-pos}.

\begin{proof}[Proof of Theorem~\ref{th-lin-onlyout}] Since set $D$ is compact and $\mathbf{x}\in D$, we have $\|\mathbf{x}\|_2\leq C$ for a fixed constant $C$.
According to \eqref{linearized}, 
\begin{equation*}
  \omega_{t+1}
  = \omega_t- \eta\nabla_\theta f(\mathcal{X},\theta_{0})^T \nabla_{f^{\mathrm{lin}}(\mathcal{X},\omega_{t})}L^{\mathrm{lin}}.
\end{equation*}
Since we use the MSE loss, we have
\begin{equation*}
  \omega_{t+1}
  = \omega_t- \eta\nabla_\theta f(\mathcal{X},\theta_{0})^T (f^{\mathrm{lin}}(\mathcal{X},\omega_{t})-\mathcal{Y}).
\end{equation*}
Using \eqref{linearized_model}, we get
\begin{equation*}
\begin{aligned}
  f^{\mathrm{lin}}(\mathcal{X},\omega_{t+1})
  &= f^{\mathrm{lin}}(\mathcal{X},\omega_{t})- \eta \nabla_\theta f(\mathcal{X},\theta_{0})\nabla_\theta f(\mathcal{X},\theta_{0})^T (f^{\mathrm{lin}}(\mathcal{X},\omega_{t})-\mathcal{Y})\\
  &= f^{\mathrm{lin}}(\mathcal{X},\omega_{t})- n \eta\hat{\Theta}_n (f^{\mathrm{lin}}(\mathcal{X},\omega_{t})-\mathcal{Y}). 
\end{aligned}
\end{equation*}
Then we have
\begin{equation*}
    f^{\mathrm{lin}}(\mathcal{X},\omega_{t+1})-\mathcal{Y} = (I-n\eta\hat{\Theta}_n) (f^{\mathrm{lin}}(\mathcal{X},\omega_{t})-\mathcal{Y}),
\end{equation*}
and
\begin{equation*}
\begin{aligned}
    f^{\mathrm{lin}}(\mathcal{X},\omega_{t})-\mathcal{Y} &= (I-n\eta\hat{\Theta}_n)^t (f^{\mathrm{lin}}(\mathcal{X},\theta_{0})-\mathcal{Y})\\
    &= (I-n\eta\hat{\Theta}_n)^t (f(\mathcal{X},\theta_{0})-\mathcal{Y}).\\
\end{aligned}
\end{equation*}
According to the update rule of $\omega_t$, we know that $\omega_t=\nabla_\theta f(\mathcal{X},\theta_{0})^T \xi+\theta_0$, where $\xi$ is a column vector. Then we have
\begin{equation*}
\begin{aligned}
    f^{\mathrm{lin}}(\mathcal{X},\omega_{t})-\mathcal{Y} &= 
    f^{\mathrm{lin}}(\mathcal{X},\omega_{t})-f(\mathcal{X},\theta_{0})+f(\mathcal{X},\theta_{0})-\mathcal{Y}\\ 
    &=\nabla_\theta f(\mathcal{X},\theta_{0})(\omega_t-\theta_0)+f(\mathcal{X},\theta_{0})-\mathcal{Y}\\
    &=\nabla_\theta f(\mathcal{X},\theta_{0})\nabla_\theta f(\mathcal{X},\theta_{0})^T \xi+f(\mathcal{X},\theta_{0})-\mathcal{Y}\\
    &=n\hat{\Theta}_n\xi+f(\mathcal{X},\theta_{0})-\mathcal{Y}\\
    &=(I-n\eta\hat{\Theta}_n)^t (f(\mathcal{X},\theta_{0})-\mathcal{Y}).
\end{aligned}
\end{equation*}
From above equation we can solve for $\xi$:
\begin{equation*}
    \xi= -n^{-1}\hat{\Theta}_n^{-1}[I-(I-n\eta\hat{\Theta}_n )^t](f(\mathcal{X},\theta_{0})-\mathcal{Y}). 
\end{equation*}
Therefore
\begin{equation}
\label{hat_omega}
    \omega_t = -n^{-1}\nabla_\theta f(\mathcal{X},\theta_{0})^T\hat{\Theta}_n^{-1}[I-(I-n\eta\hat{\Theta}_n )^t](f(\mathcal{X},\theta_{0})-\mathcal{Y})+\theta_0. 
\end{equation}
For any $\mathbf{x}\in \mathbb{R}^d$,
\begin{equation}
\begin{aligned}
    f^{\mathrm{lin}}(\mathbf{x},\omega_t)&=f(\mathbf{x},\theta_0)+\nabla_\theta f(\mathbf{x},\theta_0)(\omega_t-\theta_0)\\
    &=f(\mathbf{x},\theta_0)-n^{-1}\nabla_\theta f(\mathbf{x},\theta_0)\nabla_\theta f(\mathcal{X},\theta_{0})^T\hat{\Theta}_n^{-1}[I-(I-n\eta\hat{\Theta}_n )^t](f(\mathcal{X},\theta_{0})-\mathcal{Y}). 
\end{aligned}
\label{training_all_parameters}
\end{equation}
For the training process \eqref{fix_first}, we can define the corresponding empirical neural tangent kernel in the following way:
\begin{equation*}
    \tilde{\Theta}_n = \frac{1}{n}\nabla_{\mathbf{W}^{(2)}} f(\mathcal{X},\theta_0) \nabla_{\mathbf{W}^{(2)}} f(\mathcal{X},\theta_0)^T . 
\end{equation*}
Using the same argument, we have
\begin{equation}
\label{tilde_omega}
    \widetilde{\mathbf{W}}^{(2)}_t = -n^{-1}\nabla_{\mathbf{W}^{(2)}} f(\mathcal{X},\theta_{0})^T\tilde{\Theta}_n^{-1}[I-(I-n\eta\tilde{\Theta}_n )^t](f(\mathcal{X},\theta_{0})-\mathcal{Y})+\overline{\mathbf{W}}^{(2)}_0 
\end{equation}
and
\begin{equation}
    f^{\mathrm{lin}}(\mathbf{x},\widetilde{\omega}_t)
    =f(\mathbf{x},\theta_0)-n^{-1}\nabla_{\mathbf{W}^{(2)}} f(\mathbf{x},\theta_0)\nabla_{\mathbf{W}^{(2)}} f(\mathcal{X},\theta_{0})^T\tilde{\Theta}_n^{-1}[I-(I-n\eta\tilde{\Theta}_n )^t](f(\mathcal{X},\theta_{0})-\mathcal{Y}). 
    \label{training_the_second_layer}
\end{equation}
According to \eqref{training_all_parameters} and \eqref{training_the_second_layer}, we have
\begin{equation}
\label{diff_hat_tilde}
\begin{aligned}
    &|f^{\mathrm{lin}}(\mathbf{x},\widetilde{\omega}_t)- f^{\mathrm{lin}}(\mathbf{x},\omega_t)|\\
    =&n^{-1}\left|\nabla_\theta f(\mathbf{x},\theta_0)\nabla_\theta f(\mathcal{X},\theta_{0})^T\hat{\Theta}_n^{-1}[I-(I-n\eta\hat{\Theta}_n )^t](f(\mathcal{X},\theta_{0})-\mathcal{Y})\right. \\
    &\left.-\nabla_{\mathbf{W}^{(2)}} f(\mathbf{x},\theta_0)\nabla_{\mathbf{W}^{(2)}} f(\mathcal{X},\theta_{0})^T\tilde{\Theta}_n^{-1}[I-(I-n\eta\tilde{\Theta}_n )^t](f(\mathcal{X},\theta_{0})-\mathcal{Y})\right|. 
\end{aligned}
\end{equation}
The next step is to compute the difference between $\tilde{\Theta}_n$ and $\hat{\Theta}_n$. Let $\Delta \Theta=\hat{\Theta}_n-\tilde{\Theta}_n$, then the $ij$-th entry of the matrix $\Delta \Theta$ is
\begin{equation}
\begin{aligned}
(\Delta \Theta)_{ij} =& \frac{1}{n}\left[\sum_{k=1}^n \left(\left\langle\nabla_{\mathbf{W}_k^{(1)}} f(\mathbf{x}_i,\theta_0), \nabla_{\mathbf{W}_k^{(1)}} f(\mathbf{x}_j,\theta_0)\right\rangle+\nabla_{b_k^{(1)}} f(\mathbf{x}_i,\theta_0)\nabla_{b_k^{(1)}} f(\mathbf{x}_j,\theta_0)\right) \right.\\
&+ \left.\nabla_{b^{(2)}} f(\mathbf{x}_i,\theta_0)\nabla_{b^{(2)}} f(\mathbf{x}_j,\theta_0)\vphantom{\sum_{k=1}^n}\right]. 
\end{aligned}
\end{equation}
Given $\mathbf{x}\in\mathbb{R}^d$, we have
\begin{align}
\label{gradient1}
  \|\nabla_{\mathbf{W}_k^{(1)}} f(\mathbf{x},\theta_0)\|=&\|W_k^{(2)} H(\langle\mathbf{W}_k^{(1)},\mathbf{x}\rangle+b_k^{(1)})\cdot \mathbf{x}\|\leq C|W_k^{(2)}|\\
  \label{gradient2}
   |\nabla_{b^{(1)}_k} f(\mathbf{x},\theta_0)|=&|W_k^{(2)} H(\langle\mathbf{W}_k^{(1)},\mathbf{x}\rangle+b_k^{(1)})|\leq |W_k^{(2)}|\\
   \label{gradient3}
  |\nabla_{W_k^{(2)}} f(\mathbf{x},\theta_0)|=&[\langle\mathbf{W}_k^{(1)},\mathbf{x}\rangle+b_k^{(1)}]_+\leq C\|\mathbf{W}_k^{(1)}\|\|+b_k^{(1)}\\
  \label{gradient4}
   |\nabla_{b^{(2)}} f(\mathbf{x},\theta_0)|=&1.
\end{align}
Therefore, 
\begin{equation}
\label{eq:delta_Theta}
\begin{aligned}
|(\Delta \Theta)_{ij}| &\leq \frac{1}{n}\left[\sum_{k=1}^n \left(|W_k^{(2)}|^2\|\mathbf{x}_i\|\|\mathbf{x}_j\|+|W_k^{(2)}|^2\right) \right.
+\left.1 \vphantom{\sum_{k=1}^n}\right]\\
&=\frac{
C^2+1}{n}\sum_{k=1}^n |W_k^{(2)}|^2+\frac{1}{n}. 
\end{aligned}
\end{equation}
According to initialization \eqref{initialization2}, $W_{k}^{(2)} \buildrel d \over = \sqrt{1/n} ~\mathcal{W}^{(2)}$. Then according to the law of large numbers, $\lim_{n\to\infty}\sum_{k=1}^n |W_k^{(2)}|^2=\mathbb{E}|\mathcal{W}^{(2)}|^2$ almost surely as $n\to\infty$. Then $\sum_{k=1}^n |W_k^{(2)}|^2=O_p(1)$ and $|(\Delta \Theta)_{ij}|=O_p(n^{-1})$.

Since the size of $\Delta \Theta$ is $M\times M$, which does not change as n goes up. So $\|\Delta \Theta\|_2=O_p(n^{-1})$, which means $\|\hat{\Theta}_n-\tilde{\Theta}_n\|_2=O_p(n^{-1})$. 

Now we measure the difference of each part in \eqref{diff_hat_tilde}. According to assumption  $\inf_n \lambda_{\min}(\hat{\Theta}_n)>0$, we have
\begin{align}
\label{NTK_hat_inv}
 \lambda_{\min}(\hat{\Theta}_n^{-1})&\geq \frac{1}{\inf_n \lambda_{\min}(\hat{\Theta}_n)}=O_p(1)\\
\label{NTK_tilde_inv}
 \lambda_{\min}(\tilde{\Theta}_n^{-1})&\geq \frac{1}{\inf_n \lambda_{\min}(\hat{\Theta}_n)-O_p(n^{-1})} =O_p(1).
\end{align} 
Therefore
\begin{equation}
\label{diff_inv}
\begin{aligned}
    \|\hat{\Theta}_n^{-1}-\tilde{\Theta}_n^{-1}\|_2&=\|\hat{\Theta}_n^{-1}(\tilde{\Theta}_n-\hat{\Theta}_n)\tilde{\Theta}_n^{-1}\|_2\\
    &\leq\|\hat{\Theta}_n^{-1}\|_2\|\Delta \Theta\|_2\|\tilde{\Theta}_n^{-1}\|_2\\
    &=O_p(n^{-1}).
\end{aligned}
\end{equation}
The assumption $\eta < \frac{2}{n\lambda_{\max}(\hat{\Theta}_n)}$ implies
\begin{equation}
\label{diff_I_NTK}
\|I-n\eta\hat{\Theta}_n\|_2<1,
\end{equation}
and 
\begin{equation*}
\begin{aligned}
    \|I-n\eta\tilde{\Theta}_n \|_2&\leq  \|I-n\eta\hat{\Theta}_n \|_2+n\eta\|\hat{\Theta}_n- \Theta\|_2 \\
    &\leq\max\{n\eta\frac{\lambda_{\max}(\Theta)}{2},1-n\eta\lambda_{\min}(\hat{\Theta}_n)\}+O_p(n^{-1}).
\end{aligned}
\end{equation*}
For any $\delta>0$, as $n$ is large enough, we also have $\|I-n\eta\tilde{\Theta}_n \|_2<1$ with probability at least $1-\delta$. Then as $n$ is large enough,
\begin{equation*}
\begin{aligned}
    &\|[I-(I-n\eta\hat{\Theta}_n )^t]-[I-(I-n\eta\tilde{\Theta}_n )^t]\|_2\\
    &=\|(I-n\eta\hat{\Theta}_n )^t-(I-n\eta\tilde{\Theta}_n )^t\|_2\\
    &\leq\|[(I-n\eta\hat{\Theta}_n )-(I-n\eta\tilde{\Theta}_n )](I-n\eta\hat{\Theta}_n )^{t-1} \|_2\\
    &\quad+\|(I-n\eta\tilde{\Theta}_n )[(I-n\eta\hat{\Theta}_n )-(I-n\eta\tilde{\Theta}_n )](I-n\eta\hat{\Theta}_n )^{t-2} \|_2\\
    &\quad+\cdots\\
    &\quad+\|(I-n\eta\tilde{\Theta}_n )^{t-1}[(I-n\eta\hat{\Theta}_n )-(I-n\eta\tilde{\Theta}_n )] \|_2\\
    &\leq\eta\|\hat{\Theta}_n-\tilde{\Theta}_n\|_2\|I-n\eta\hat{\Theta}_n \|_2^{t-1}\\
    &\quad+\eta\|I-n\eta\tilde{\Theta}_n \|_2\|\hat{\Theta}_n-\tilde{\Theta}_n\|_2\|I-n\eta\hat{\Theta}_n \|_2^{t-2} \\
    &\quad+\cdots\\
    &\quad+\eta\|I-n\eta\tilde{\Theta}_n \|_2^{t-1}\|\hat{\Theta}_n-\tilde{\Theta}_n\|_2\\
    &\leq \eta\|\hat{\Theta}_n-\tilde{\Theta}_n\|_2\cdot t\cdot(\max\{\|I-n\eta\hat{\Theta}_n \|_2,\|I-n\eta\tilde{\Theta}_n \|_2\})^{t-1}.
\end{aligned}
\end{equation*}
Since $\max\{\|I-n\eta\hat{\Theta}_n \|_2,\|I-n\eta\tilde{\Theta}_n \|_2\}<1$,  $\sup_{t>0} t\cdot(\max\{\|I-n\eta\hat{\Theta}_n \|_2,\|I-n\eta\tilde{\Theta}_n \|_2\})^{t-1}$ is a finite number. Then we have
\begin{equation}
\label{diff_power}
\begin{aligned}
    \|[I-(I-n\eta\hat{\Theta}_n )^t]-[I-(I-n\eta\tilde{\Theta}_n )^t]\|_2 &\leq O(\eta\|\hat{\Theta}_n-\tilde{\Theta}_n\|_2)\\
    &= O_p(n^{-1}). 
\end{aligned}
\end{equation}
Let $\Delta \Theta(\mathbf{x},\mathcal{X})=n^{-1}(\nabla_\theta f(\mathbf{x},\theta_0)\nabla_\theta f(\mathcal{X},\theta_{0})^T-\nabla_{\mathbf{W}^{(2)}} f(\mathbf{x},\theta_0)\nabla_{\mathbf{W}^{(2)}} f(\mathcal{X},\theta_{0})^T)$, then the $i$-th entry of the vector $\Delta \Theta(\mathbf{x},\mathcal{X})$ is
\begin{equation*}
\begin{aligned}
(\Delta \Theta(\mathbf{x},\mathcal{X}))_{i} =& \frac{1}{n}\left[\sum_{k=1}^n \left(\nabla_{\mathbf{W}_k^{(1)}} f(\mathbf{x},\theta_0)\nabla_{\mathbf{W}_k^{(1)}} f(\mathbf{x}_i,\theta_0)+\nabla_{b_k^{(1)}} f(\mathbf{x},\theta_0)\nabla_{b_k^{(1)}} f(\mathbf{x}_i,\theta_0)\right) \right.\\
&+ \left.\nabla_{b^{(2)}} f(\mathbf{x},\theta_0)\nabla_{b^{(2)}} f(\mathbf{x}_i,\theta_0)\vphantom{\sum_{k=1}^n}\right].
\end{aligned}
\end{equation*}
Similar to \eqref{eq:delta_Theta}, we have
\begin{equation}
    (\Delta \Theta(\mathbf{x},\mathcal{X}))_{i}|=O_p(n^{-1}).
\end{equation}
Since the size of $\Delta \Theta(\mathbf{x},\mathcal{X})$ is $M$, which does not change as n goes up. So 
\begin{equation}
\label{diff_d}
    \|\Delta \Theta(\mathbf{x},\mathcal{X})\|_2=O_p(n^{-1}).
\end{equation}
Let $\tilde{\Theta}_n(\mathbf{x},\mathcal{X})=n^{-1}(\nabla_{\mathbf{W}^{(2)}} f(\mathbf{x},\theta_0)\nabla_{\mathbf{W}^{(2)}} f(\mathcal{X},\theta_{0})^T))$, then the $i$-th entry of the vector $\tilde{\Theta}_n(\mathbf{x},\mathcal{X})$ is
\begin{equation}
\label{test_NTK}
\begin{aligned}
|(\tilde{\Theta}_n(\mathbf{x},\mathcal{X}))_{i}| &\leq \frac{1}{n}\sum_{k=1}^n |\nabla_{W_k^{(2)}} f(\mathbf{x},\theta_0)\nabla_{W_k^{(2)}} f(\mathbf{x}_i,\theta_0)|\\
&\leq \frac{1}{n}\sum_{k=1}^n |(\|\mathbf{W}_k^{(1)}\|\|\mathbf{x}\|_2+b_k^{(1)})(\|\mathbf{W}_k^{(1)}\|\|\mathbf{x}_i\|_2+b_k^{(1)})|.\\
&\leq \frac{1}{n}\sum_{k=1}^n |(C\|\mathbf{W}_k^{(1)}\|+b_k^{(1)})(C\|\mathbf{W}_k^{(1)}\|+b_k^{(1)})|.\\
\end{aligned}
\end{equation}
According to initialization \eqref{initialization2}, $(\mathbf{W}_{k}^{(1)},b_{k}^{(1)}) \buildrel d \over = ~(\bm{\mathcal{W}},\mathcal{B})$. Then according to the law of large numbers,
\begin{equation}
\label{eq:theta_xX}
    |(\tilde{\Theta}_n(\mathbf{x},\mathcal{X}))_{i}|= O_p(1).
\end{equation}
Since the size of $\tilde{\Theta}_n(\mathbf{x},\mathcal{X})$ is $M$, which does not change as $n$ goes up. So 
\begin{equation*}
    \|\tilde{\Theta}_n(\mathbf{x},\mathcal{X})\|_2=O_p(1).
\end{equation*}
\cite{neal1996priors}, \cite{lee2018deep} show that as $n$ goes to infinity, the output function at initialization $f(\cdot,\theta_{0})$ converges to a Gaussian process in distribution, which means that $f(\mathcal{X},\theta_{0})\sim \mathcal{N}(0,\mathcal{K}(\mathcal{X},\mathcal{X}))$. Here $\mathcal{K}(\mathcal{X},\mathcal{X})$ can be computed recursively. 
Then $f(\mathcal{X},\theta_{0})$ is bounded in probability and we get
\begin{equation}
\label{bound_init}
    \|f(\mathcal{X},\theta_{0})-\mathcal{Y}\|_2=O_p(1).
\end{equation}
Then following \eqref{diff_hat_tilde} and \eqref{bound_init}, we get
\begin{equation*}
\begin{aligned}
    &|f^{\mathrm{lin}}(\mathbf{x},\widetilde{\omega}_t)- f(\mathbf{x},\theta_t)|\\
    =&n^{-1}|\nabla_\theta f(\mathbf{x},\theta_0)\nabla_\theta f(\mathcal{X},\theta_{0})^T\hat{\Theta}_n^{-1}[I-(I-n\eta\hat{\Theta}_n )^t](f(\mathcal{X},\theta_{0})-\mathcal{Y})\\
    &-\nabla_{\mathbf{W}^{(2)}} f(\mathbf{x},\theta_0)\nabla_{\mathbf{W}^{(2)}} f(\mathcal{X},\theta_{0})^T\tilde{\Theta}_n^{-1}[I-(I-n\eta\tilde{\Theta}_n )^t](f(\mathcal{X},\theta_{0})-\mathcal{Y})|\\
    =&n^{-1}\|\nabla_\theta f(\mathbf{x},\theta_0)\nabla_\theta f(\mathcal{X},\theta_{0})^T\hat{\Theta}_n^{-1}[I-(I-n\eta\hat{\Theta}_n )^t]\\
    &-\nabla_{\mathbf{W}^{(2)}} f(\mathbf{x},\theta_0)\nabla_{\mathbf{W}^{(2)}} f(\mathcal{X},\theta_{0})^T\tilde{\Theta}_n^{-1}[I-(I-n\eta\tilde{\Theta}_n )^t]\|_2\|f(\mathcal{X},\theta_{0})-\mathcal{Y}\|_2\\
    =&n^{-1}\|\nabla_\theta f(\mathbf{x},\theta_0)\nabla_\theta f(\mathcal{X},\theta_{0})^T\hat{\Theta}_n^{-1}[I-(I-n\eta\hat{\Theta}_n )^t]\\
    &-\nabla_{\mathbf{W}^{(2)}} f(\mathbf{x},\theta_0)\nabla_{\mathbf{W}^{(2)}} f(\mathcal{X},\theta_{0})^T\tilde{\Theta}_n^{-1}[I-(I-n\eta\tilde{\Theta}_n )^t]\|_2\cdot O_p(1). 
\end{aligned}
\end{equation*}
According to 
\eqref{diff_inv}, \eqref{diff_I_NTK}, \eqref{diff_power},  \eqref{diff_d} and \eqref{eq:theta_xX}, we have that 
\begin{equation*}
\begin{aligned}
    &n^{-1}\|\nabla_\theta f(\mathbf{x},\theta_0)\nabla_\theta f(\mathcal{X},\theta_{0})^T\hat{\Theta}_n^{-1}[I-(I-n\eta\hat{\Theta}_n )^t]\\
    &-\nabla_{\mathbf{W}^{(2)}} f(\mathbf{x},\theta_0)\nabla_{\mathbf{W}^{(2)}} f(\mathcal{X},\theta_{0})^T\tilde{\Theta}_n^{-1}[I-(I-n\eta\tilde{\Theta}_n )^t]\|_2\\
    \leq&n^{-1}\|\nabla_\theta f(\mathbf{x},\theta_0)\nabla_\theta f(\mathcal{X},\theta_{0})^T-\nabla_{\mathbf{W}^{(2)}} f(\mathbf{x},\theta_0)\nabla_{\mathbf{W}^{(2)}} f(\mathcal{X},\theta_{0})^T\|\|\hat{\Theta}_n^{-1}\|_2\|I-(I-n\eta\hat{\Theta}_n )^t\|_2\\
    &+n^{-1}\|\nabla_{\mathbf{W}^{(2)}} f(\mathbf{x},\theta_0)\nabla_{\mathbf{W}^{(2)}} f(\mathcal{X},\theta_{0})^T\|_2\|\hat{\Theta}_n^{-1}-\tilde{\Theta}_n^{-1}\|_2\|I-(I-n\eta\hat{\Theta}_n )^t\|_2\\
    &+n^{-1}\|\nabla_{\mathbf{W}^{(2)}} f(\mathbf{x},\theta_0)\nabla_{\mathbf{W}^{(2)}} f(\mathcal{X},\theta_{0})^T\|_2\|\tilde{\Theta}_n^{-1}\|_2\|[I-(I-n\eta\hat{\Theta}_n )^t]-[I-(I-n\eta\tilde{\Theta}_n )^t]\|_2\\
    \leq &O_p(n^{-1})O_p(1)O_p(1)+O_p(1)O_p(n^{-1})O_p(1)+O_p(1)O_p(1)O_p(n^{-1})\\
    = &O_p(n^{-1}). 
\end{aligned}
\end{equation*}
So we have $|f^{\mathrm{lin}}(\mathbf{x},\widetilde{\omega}_t)- f(\mathbf{x},\theta_t)|=O_p(n^{-1})$, and the constants in $O_p(n^{-1})$ do not depend on $t$ and $\mathbf{x}$. 
Then we get
\begin{equation*}
    \sup_{\mathbf{x}\in D}\sup_t |f^{\mathrm{lin}}(\mathbf{x},\widetilde{\omega}_t)- f^{\mathrm{lin}}(\mathbf{x},\omega_t)|=O_p(n^{-1}), \text{ as } n\to\infty. 
\end{equation*}
For the difference of parameters, we have
\begin{equation*}
   \widetilde{\omega}_t-\omega_t=\mathrm{vec}(\overline{\mathbf{W}}^{(1)}-\widehat{\mathbf{W}}_t^{(1)},\overline{\mathbf{b}}^{(1)}-\widehat{\mathbf{b}}_t^{(1)},\widetilde{\mathbf{W}}_t^{(2)}-\widehat{\mathbf{W}}_t^{(2)},\overline{b}^{(2)}-\widehat{b}_t^{(2)}).
\end{equation*}
According to \eqref{hat_omega} and \eqref{tilde_omega},
\begin{equation*}
\begin{aligned}
\|\overline{\mathbf{W}}^{(1)}-\widehat{\mathbf{W}}_t^{(1)}\|_2&=\|n^{-1}\nabla_{\mathbf{W}^{(1)}} f(\mathcal{X},\theta_{0})^T\hat{\Theta}_n^{-1}[I-(I-n\eta\hat{\Theta}_n )^t](f(\mathcal{X},\theta_{0})-\mathcal{Y})\|_2\\
&\leq \|n^{-1}\nabla_{\mathbf{W}^{(1)}} f(\mathcal{X},\theta_{0})^T\|_2\|\hat{\Theta}_n^{-1}\|_2\|I-(I-n\eta\hat{\Theta}_n )^t\|_2\|f(\mathcal{X},\theta_{0})-\mathcal{Y}\|_2\\
&\leq n^{-1}\|\nabla_{\mathbf{W}^{(1)}}f(\mathcal{X},\theta_{0})^T\|_2\cdot O_p(1). 
\end{aligned}
\end{equation*}
Here $\nabla_{\mathbf{W}^{(1)}}f(\mathcal{X},\theta_{0})^T$ is a $n\times M$ matrix, the $ij$-th entry of the matrix is $\nabla_{\mathbf{W}^{(1)}_i}f(\mathbf{x}_j,\theta_{0})$.  According to \eqref{gradient1}, we have $\nabla_{\mathbf{W}^{(1)}_i}f(\mathbf{x}_j,\theta_{0})=O_p(n^{-1/2})$. Then we get $\|\nabla_{\mathbf{W}^{(1)}}f(\mathcal{X},\theta_{0})^T\|_2=O_p(1)$ by the law of large numbers. So we have $\|\overline{\mathbf{W}}^{(1)}-\widehat{\mathbf{W}}_t^{(1)}\|_2=O_p(n^{-1})$, and $O_p(n^{-1})$ does not contain any constant factor which is related to $t$.  Then we get
\begin{equation*}
\sup_t \|\overline{\mathbf{W}}^{1}- \widehat{\mathbf{W}}^{1}_t\|_2=O_p(n^{-1}), \text{ as } n\to\infty.
\end{equation*}
Similarly we can prove
\begin{align}
    \sup_t \|\overline{\mathbf{b}}^{1}- \widehat{\mathbf{b}}^{1}_t\|_2&=O_p(n^{-1}), \text{ as } n\to\infty, \\
    \sup_t \|\overline{b}^{2}- \widehat{b}^{2}_t\|&=O_p(n^{-1}), \text{ as } n\to\infty.
\end{align}
For $\widetilde{\mathbf{W}}_t^{(2)}-\widehat{\mathbf{W}}_t^{(2)}$, we have 
\begin{equation*}
\begin{aligned}
\|\overline{\mathbf{W}}^{(2)}-\widehat{\mathbf{W}}_t^{(2)}\|_2&=\|n^{-1}\nabla_{\mathbf{W}^{(2)}}f(\mathcal{X},\theta_{0})^T \left(\hat{\Theta}_n^{-1}[I-(I-n\eta\hat{\Theta}_n )^t]-\right.\\
&\quad\quad\left.\tilde{\Theta}_n^{-1}[I-(I-n\eta\tilde{\Theta}_n )^t]\right)(f(\mathcal{X},\theta_{0})-\mathcal{Y})\|_2\\
&\leq \|n^{-1}\nabla_{\mathbf{W}^{(2)}} f(\mathcal{X},\theta_{0})^T\|_2\left(\|\hat{\Theta}_n^{-1}-\tilde{\Theta}_n^{-1}\|_2\|I-(I-n\eta\hat{\Theta}_n )^t\|_2+\right.\\
&\quad\quad\left.\|\tilde{\Theta}_n^{-1}\|_2\|[I-(I-n\eta\hat{\Theta} )^t]-[I-(I-n\eta\tilde{\Theta} )^t]\|_2\right)\|f(\mathcal{X},\theta_{0})-\mathcal{Y}\|_2\\
&\leq n^{-1}\|\nabla_{\mathbf{W}^{(2)}}f(\mathcal{X},\theta_{0})^T\|_2(O_p(n^{-1})O_p(1)+O_p(1)O_p(n^{-1}))\cdot O_p(1)\\
&=O_p(n^{-2})\|\nabla_{\mathbf{W}^{(2)}}f(\mathcal{X},\theta_{0})^T\|_2 . 
\end{aligned}
\end{equation*}
Here $\nabla_{\mathbf{W}^{(2)}}f(\mathcal{X},\theta_{0})^T$ is a $n\times M$ matrix, the $ij$-th entry of the matrix is $\nabla_{W^{(2)}_i}f(\mathbf{x}_j,\theta_{0})$. According to \eqref{gradient3}, we have $\nabla_{W^{(2)}_i}f(\mathbf{x}_j,\theta_{0})=O_p(1)$. Then $\|\nabla_{\mathbf{W}^{(2)}}f(\mathcal{X},\theta_{0})^T\|_2=O_p(n^{1/2})$ by the law of large numbers. So we have $\|\widetilde{\mathbf{W}}_t^{(2)}-\widehat{\mathbf{W}}_t^{(2)}\|_2=O_p(n^{-3/2})$, and $O_p(n^{-3/2})$ does not contain any constant factor which is related to $t$. Then we get
\begin{equation*}
\sup_t \|\widetilde{\mathbf{W}}^{2}_t- \widehat{\mathbf{W}}^{2}_t\|_2=O_p(n^{-3/2}), \text{ as } n\to\infty.
\end{equation*}
\end{proof}

\section{Training Only the Output Layer Approximates Training a Wide Shallow Network}
\label{appendix:Lee-generalization}

Corollary~\ref{cor:outlayer-only} is obtained by combining Theorem~\ref{th-lin-onlyout} and the fact that training a linearized model approximates training a wide network \citep[Proposition 3.2]{lai2023generalization}. 
Although \citet[Proposition 3.2]{lai2023generalization} consider Gaussian initialization, the arguments extend to sub-Gaussian initialization if the initialization distribution has a continuous probability density. 

\begin{proof}[Proof of Corollary~\ref{cor:outlayer-only}]
Using Theorem~\ref{th-lin-onlyout}, we have that
\begin{equation}
    \sup_t |f^{\mathrm{lin}}(\mathbf{x},\widetilde{\omega}_t)- f^{\mathrm{lin}}(\mathbf{x},\omega_t)|=O_p(n^{-1}), \text{ as } n\to\infty. 
    \label{result_theorem3}
\end{equation}
According to \citet[Proposition 3.2]{lai2023generalization}, in the case of Gaussian initialization, we have
\begin{equation*}
    \sup_t |f^{\mathrm{lin}}(\mathbf{x},\omega_t)- f(\mathbf{x},\theta)|=O_p(n^{-\frac{1}{2}}), \text{ as } n\to\infty. 
\end{equation*}
Under our neural network setting, which is a one-input network with a single hidden layer of $n$ ReLUs and a linear output, we can generalize the above result to sub-Gaussian initialization. 
Combining the above equation with \eqref{result_theorem3} concludes the proof. 
\end{proof}

\section{Proof of Theorem~\ref{theorem4}}
\label{Proof4}

\begin{proof}[Proof of Theorem~\ref{theorem4}]
The Lagrangian of problem \eqref{probablity_version} is
\begin{equation*}
  L(\alpha_n,\lambda^{(n)})=\int_{\mathbb{R}^2} \alpha_n^2(\mathbf{W}^{(1)},b)~\mathrm{d}\mu_n(\mathbf{W}^{(1)},b)+\sum_{j=1}^M \lambda^{(n)}_j(g_n(\mathbf{x}_j,\alpha_n)-y_j).
\end{equation*}
The optimal condition is $\nabla_{\alpha_n} L=0$, which means
\begin{equation*}
  \nabla_{\alpha_n} L = 2\alpha_n(\mathbf{W}^{(1)},b)+\sum_{j=1}^M \lambda^{(n)}_j [\langle\mathbf{W}^{(1)},\mathbf{x}_j \rangle+b]_+ = 0 \textup{ when } (\mathbf{W}^{(1)},b)=(\mathbf{W}^{(1)}_i,b_i), \ i=1,\ldots,k.
\end{equation*}
Then we get
\begin{equation*}
  \overline{\alpha}_n(\mathbf{W}^{(1)},b) = -\frac{1}{2}\sum_{j=1}^M \lambda^{(n)}_j [\langle\mathbf{W}^{(1)},\mathbf{x}_j \rangle+b]_+ \textup{ when } (\mathbf{W}^{(1)},b)=(\mathbf{W}^{(1)}_i,b_i), \ i=1,\ldots,k.
\end{equation*}
Since only function values on $(\mathbf{W}_i^{(1)},b_i)_{i=1}^M$ are taken into account in problem \eqref{probablity_version}, we can let
\begin{equation}
  \overline{\alpha}_n(\mathbf{W}^{(1)},b) = -\frac{1}{2}\sum_{j=1}^M \lambda^{(n)}_j [\langle\mathbf{W}^{(1)},\mathbf{x}_j \rangle+b]_+ \quad \forall (\mathbf{W}^{(1)},b)\in\mathbb{R}^{d+1}
  \label{a1}
\end{equation}
without changing $\int_{\mathbb{R}^2} \overline{\alpha}_n^2(\mathbf{W}^{(1)},b)~\mathrm{d}\mu_n(\mathbf{W}^{(1)},b)$ and $g_n(\mathbf{x},\overline{\alpha}_n)$.

Here $\lambda^{(n)}_j$, $j=1,\ldots,M$ are chosen to make $g_n(\mathbf{x}_i,\overline{\alpha}_n)=y_i$, $i=1,\ldots,M$. This means that
\begin{equation}
  -\frac{1}{2}\sum_{j=1}^M \lambda^{(n)}_j \int_{\mathbb{R}^2} [\langle\mathbf{W}^{(1)},\mathbf{x}_j \rangle+b]_+[\langle\mathbf{W}^{(1)},\mathbf{x}_i \rangle+b]_+ ~\mathrm{d}\mu_n(\mathbf{W}^{(1)},b)= y_i,\; i=1,\ldots,M. 
  \label{e1}
\end{equation}

Similarly, the Lagrangian of problem \eqref{continuous_version} is 
\begin{equation*}
  \widetilde{L}(\alpha,\lambda)=\int_{\mathbb{R}^2} \alpha^2(\mathbf{W}^{(1)},b)~\mathrm{d}\mu(\mathbf{W}^{(1)},b)+\sum_{j=1}^M \lambda_j(g(\mathbf{x}_j,\alpha)-y_j) . 
\end{equation*}
The optimality condition is $\nabla_\alpha \widetilde{L}=0$, which means 
\begin{equation*}
  \nabla_\alpha \widetilde{L} = 2\alpha(\mathbf{W}^{(1)},b)+\sum_{j=1}^M \lambda_j [\langle\mathbf{W}^{(1)},\mathbf{x}_j \rangle+b]_+ = 0 \quad \forall (\mathbf{W}^{(1)},b)\in\mathbb{R}^{d+1}.
\end{equation*}
Then we get
\begin{equation}
  \overline{\alpha}(\mathbf{W}^{(1)},b) = -\frac{1}{2}\sum_{j=1}^M \lambda_j [\langle\mathbf{W}^{(1)},\mathbf{x}_j \rangle+b]_+  \quad \forall (\mathbf{W}^{(1)},b)\in\mathbb{R}^{d+1}.
  \label{a2}
\end{equation}
Here $\lambda_j$, $j=1,\ldots,M$ are chosen to make $g(\mathbf{x},\alpha)=y_i$, $i=1,\ldots,M$. This means that 
\begin{equation}
  -\frac{1}{2}\sum_{j=1}^M \lambda_j \int_{\mathbb{R}^2} [\langle\mathbf{W}^{(1)},\mathbf{x}_j \rangle+b]_+[\langle\mathbf{W}^{(1)},\mathbf{x}_i \rangle+b]_+ ~\mathrm{d}\mu(\mathbf{W}^{(1)},b)= y_i, \quad i=1,\ldots,M. 
  \label{e2}
\end{equation}
Compare \eqref{e1} and \eqref{e2}. Since the number of samples is finite, $\mathbf{x}_i$ is also bounded. Then by the assumption that $\mathcal{\bm{W}}$ and $\mathcal{B}$ have finite fourth moments, we have that $[\langle\mathbf{W}^{(1)},\mathbf{x}_j \rangle+b]_+[\langle\mathbf{W}^{(1)},\mathbf{x}_i \rangle+b]_+$ has finite variance. According to central limit theorem, as $n\to\infty$, $\int_{\mathbb{R}^2} [\langle\mathbf{W}^{(1)},\mathbf{x}_j \rangle+b]_+[\langle\mathbf{W}^{(1)},\mathbf{x}_i \rangle+b]_+ ~\mathrm{d}\mu_n(\mathbf{W}^{(1)},b)$ tends to a Gaussian distribution with variance $O(n^{-1})$. 
This implies that $\forall i=1,\ldots,M,~\forall j=1,\ldots,M$,
\begin{equation*}
\begin{aligned}
&|\int_{\mathbb{R}^2} [\langle\mathbf{W}^{(1)},\mathbf{x}_j \rangle+b]_+[\langle\mathbf{W}^{(1)},\mathbf{x}_i \rangle+b]_+ ~\mathrm{d}\mu_n(\mathbf{W}^{(1)},b)\\
&-\int_{\mathbb{R}^2} [\langle\mathbf{W}^{(1)},\mathbf{x}_j \rangle+b]_+[\langle\mathbf{W}^{(1)},\mathbf{x}_i \rangle+b]_+ ~\mathrm{d}\mu(\mathbf{W}^{(1)},b)|\\
&=O_p(n^{-1/2})
\end{aligned}
\end{equation*}
Since \eqref{e1} and \eqref{e2} are systems of linear equations and coefficients of \eqref{e1} converge to coefficients of \eqref{e2} at the rate of $O_p(n^{-1/2})$, then we get
\begin{equation}
  |\lambda_j^n-\lambda_j|=O_p(n^{-1/2}),\quad j=1,\ldots,M. \label{converge_lambda} 
\end{equation}
Compare \eqref{a1} and \eqref{a2}. Given $(\mathbf{W}^{(1)},b)$, we have
\begin{equation}
  |\overline{\alpha}_n(\mathbf{W}^{(1)},b)-\overline{\alpha}(\mathbf{W}^{(1)},b)|=O_p(n^{-1/2}).\label{alpha_n_converge}
\end{equation}
Next we want to prove that $\sup_{\mathbf{x}\in D}|g_n(\mathbf{x},\overline{\alpha}_n)-g(\mathbf{x},\overline{\alpha})|=O_p(n^{-1/2})$. 
Firstly, we prove that $\sup_{\mathbf{x}\in D}|g_n(\mathbf{x},\overline{\alpha})-g(\mathbf{x},\overline{\alpha})|=O_p(n^{-1/2})$. 
Note that 
\begin{equation*}
  \begin{aligned}
    g_n(\mathbf{x},\overline{\alpha})&=\int_{\mathbb{R}^2} \overline{\alpha}(\mathbf{W}^{(1)},b)[\langle\mathbf{W}^{(1)},\mathbf{x} \rangle+b]_+~\mathrm{d}\mu_n(\mathbf{W}^{(1)},b)\\
    g(\mathbf{x},\overline{\alpha})&=\int_{\mathbb{R}^2} \overline{\alpha}(\mathbf{W}^{(1)},b)[\langle\mathbf{W}^{(1)},\mathbf{x} \rangle+b]_+~\mathrm{d}\mu(\mathbf{W}^{(1)},b).
  \end{aligned}
\end{equation*}
Therefore, 
\begin{equation}
\begin{aligned}
   \mathbb{E}(g_n(\mathbf{x},\overline{\alpha}))&=g(\mathbf{x},\overline{\alpha}) \\
   \operatorname{Var}(g_n(\mathbf{x},\overline{\alpha}))&=\frac{1}{n}\int_{\mathbb{R}^2} [\overline{\alpha}(\mathbf{W}^{(1)},b)[\langle\mathbf{W}^{(1)},\mathbf{x} \rangle+b]_+-g(\mathbf{x},\overline{\alpha})]^2~\mathrm{d}\mu(\mathbf{W}^{(1)},b). 
\end{aligned}
\label{mean_variance}
\end{equation}
Here the expectation and the variance are with respect to $(\mathbf{W}_i^{(1)},b_i)_{i=1}^n$. According to \eqref{a2} and the assumption that $\mathcal{\bm{W}}$ and $\mathcal{B}$ have finite fourth moments, the integral in \eqref{mean_variance} is bounded on $ D$. So $\sup_{\mathbf{x}\in D}\operatorname{Var}g_n(\mathbf{x},\overline{\alpha})=O(n^{-1})$. According to central limit theorem, as $n\to\infty$, $g_n(\mathbf{x},\overline{\alpha})$ tends to Gaussian distribution of variance $O(n^{-1})$ for any $\mathbf{x}\in D$. Then $|g_n(\mathbf{x},\overline{\alpha})-g(\mathbf{x},\overline{\alpha})|=O_p(n^{-1/2})$ pointwise on $ D$.
Then we only need to prove that the sequence of functions $\{g_n(\mathbf{x},\overline{\alpha})\}_{n=1}^\infty$ is uniformly equicontinuous. Actually, $\forall \mathbf{x}_1,\mathbf{x}_2\in D$
\begin{equation*}
  \begin{aligned}
    &|g_n(\mathbf{x}_1,\overline{\alpha})-g_n(\mathbf{x}_2,\overline{\alpha})|\\
    \leq&\int_{\mathbb{R}^2} \left|\overline{\alpha}(\mathbf{W}^{(1)},b)[\langle\mathbf{W}^{(1)},\mathbf{x}_1 \rangle+b]_+ -\overline{\alpha}(\mathbf{W}^{(1)},b)[\langle\mathbf{W}^{(1)},\mathbf{x}_2 \rangle+b]_+\right|~\mathrm{d}\mu_n(\mathbf{W}^{(1)},b)\\
    \leq&\int_{\mathbb{R}^2} \left|\overline{\alpha}(\mathbf{W}^{(1)},b)\right|\left|\mathbf{W}_i^{(1)}\right|\left|\mathbf{x}_1-\mathbf{x}_2\right|~\mathrm{d}\mu_n(\mathbf{W}^{(1)},b)\\
    \leq&\int_{\mathbb{R}^2} \left|\overline{\alpha}(\mathbf{W}^{(1)},b)\right|\left|\mathbf{W}_i^{(1)}\right|~\mathrm{d}\mu_n(\mathbf{W}^{(1)},b)\left|\mathbf{x}_1-\mathbf{x}_2\right|. 
  \end{aligned}
\end{equation*}
Notice that $\int_{\mathbb{R}^2} \left|\overline{\alpha}(\mathbf{W}^{(1)},b)\right|\left|\mathbf{W}_i^{(1)}\right|~\mathrm{d}\mu_n(\mathbf{W}^{(1)},b)\to \int_{\mathbb{R}^2} \left|\overline{\alpha}(\mathbf{W}^{(1)},b)\right|\left|\mathbf{W}_i^{(1)}\right|~\mathrm{d}\mu(\mathbf{W}^{(1)},b)$ with probability 1 according to the law of large numbers. Hence $\int_{\mathbb{R}^2} \left|\overline{\alpha}(\mathbf{W}^{(1)},b)\right|\left|\mathbf{W}_i^{(1)}\right|~\mathrm{d}\mu_n(\mathbf{W}^{(1)},b)$ is bounded and the bound is independent of $n$.  So $\{g_n(\mathbf{x},\overline{\alpha})\}_{n=1}^\infty$ is uniformly equicontinuous. Then by similar arguments to the Arzela-Ascoli theorem, 
\begin{equation}
\sup_{\mathbf{x}\in D}|g_n(\mathbf{x},\overline{\alpha})-g(\mathbf{x},\overline{\alpha})|=O_p(n^{-1/2}). \label{part_converge}
\end{equation}
Finally, we prove that $\sup_{\mathbf{x}\in D}|g_n(\mathbf{x},\overline{\alpha}_n)-g_n(\mathbf{x},\overline{\alpha})|=O_p(n^{-1/2})$. Since $\forall \mathbf{x}\in  D$
\begin{equation*}
  \begin{aligned}
    &|g_n(\mathbf{x},\overline{\alpha}_n)-g_n(\mathbf{x},\overline{\alpha})|\\
    \leq&\int_{\mathbb{R}^2} \left|\overline{\alpha}_n(\mathbf{W}^{(1)},b)[\langle\mathbf{W}^{(1)},\mathbf{x} \rangle+b]_+ -\overline{\alpha}(\mathbf{W}^{(1)},b)[\langle\mathbf{W}^{(1)},\mathbf{x} \rangle+b]_+\right|~\mathrm{d}\mu_n(\mathbf{W}^{(1)},b)\\
    \leq&\int_{\mathbb{R}^2} \left|\overline{\alpha}_n(\mathbf{W}^{(1)},b)-\overline{\alpha}(\mathbf{W}^{(1)},b)\right|[\langle\mathbf{W}^{(1)},\mathbf{x} \rangle+b]_+~\mathrm{d}\mu_n(\mathbf{W}^{(1)},b)\\
    \leq&  \int_{\mathbb{R}^2}\left|-\frac{1}{2}\sum_{j=1}^M (\lambda_j^n-\lambda_j) [\langle\mathbf{W}^{(1)},\mathbf{x}_j \rangle+b]_+\right| [\langle\mathbf{W}^{(1)},\mathbf{x} \rangle+b]_+~\mathrm{d}\mu_n(\mathbf{W}^{(1)},b)\\
    \leq&  \frac{1}{2}\sum_{j=1}^M |\lambda_j^n-\lambda_j| \int_{\mathbb{R}^2}[\langle\mathbf{W}^{(1)},\mathbf{x}_j \rangle+b]_+ [\langle\mathbf{W}^{(1)},\mathbf{x} \rangle+b]_+~\mathrm{d}\mu_n(\mathbf{W}^{(1)},b)\\
    \leq&  \frac{1}{2} \left(\max_{\mathbf{x}\in D}\int_{\mathbb{R}^2}[\langle\mathbf{W}^{(1)},\mathbf{x}_j \rangle+b]_+ [\langle\mathbf{W}^{(1)},\mathbf{x} \rangle+b]_+~\mathrm{d}\mu_n(\mathbf{W}^{(1)},b)\right)\sum_{j=1}^M |\lambda_j^n-\lambda_j|. 
  \end{aligned}
\end{equation*}
Because $ D$ is compact and $\int_{\mathbb{R}^2}[\langle\mathbf{W}^{(1)},\mathbf{x}_j \rangle+b]_+ [\langle\mathbf{W}^{(1)},\mathbf{x} \rangle+b]_+~\mathrm{d}\mu_n(\mathbf{W}^{(1)},b)$ converges according to the law of large numbers, we have that $\max_{\mathbf{x}\in D}\int_{\mathbb{R}^2}[\langle\mathbf{W}^{(1)},\mathbf{x}_j \rangle+b]_+ [\langle\mathbf{W}^{(1)},\mathbf{x} \rangle+b]_+~\mathrm{d}\mu_n(\mathbf{W}^{(1)},b)$ is bounded by a finite number independent of $n$. Then according to \eqref{converge_lambda},%
\begin{equation*}
\sup_{\mathbf{x}\in D}|g_n(\mathbf{x},\overline{\alpha}_n)-g_n(\mathbf{x},\overline{\alpha})|=O_p(n^{-1/2}) . 
\end{equation*}
Combined with \eqref{part_converge}, we have
\begin{equation*}
\sup_{\mathbf{x}\in D}|g_n(\mathbf{x},\overline{\alpha}_n)-g(\mathbf{x},\overline{\alpha})|=O_p(n^{-1/2}) . 
\end{equation*}
This concludes the proof. 
\end{proof}

\section{Proofs of Results for Univariate Regression}
\label{app:univariate}
\subsection{Proof of Theorem~\ref{theorem_func}}
\label{Proof5}
The second derivative $g''$ is given by
\begin{equation}
  \begin{aligned}
    g''(x,\gamma)
&=p_{\mathcal{C}}(x)\int_{\mathbb{R}}\gamma(W^{(1)},x)\big|W^{(1)}\big|~\mathrm{d}\nu_{\mathcal{W}|\mathcal{C}=x}(W^{(1)}) . 
  \end{aligned}
  \label{2nd_derivative}
\end{equation}
The detailed calculation of \eqref{2nd_derivative} is as follows:  
\begin{equation}
  \begin{aligned}
    g''(x,\gamma)
    &=\int_{\mathbb{R}^2}\gamma(W^{(1)},c) \left|W^{(1)}\right| \delta(x-c)~\mathrm{d}\nu(W^{(1)},c)\\
    &=\int_{\mathrm{supp}(\nu_\mathcal{C})}\left(\int_{\mathbb{R}}\gamma(W^{(1)},c) \left|W^{(1)}\right|  %
    ~\mathrm{d}\nu_{\mathcal{W}|\mathcal{C}=c}(W^{(1)})\right)
    \delta(x-c)~\mathrm{d}\nu_\mathcal{C}(c)\\
    &=\int_{\mathrm{supp}(\nu_\mathcal{C})}\left(\int_{\mathbb{R}}\gamma(W^{(1)},c)\left|W^{(1)}\right|~\mathrm{d}\nu_{\mathcal{W}|\mathcal{C}=c}(W^{(1)})\right)\delta(x-c) p_\mathcal{C}(c)\mathrm{d}c\\
    &=p_{\mathcal{C}}(x)\int_{\mathbb{R}}\gamma(W^{(1)},x)\left|W^{(1)}\right|~\mathrm{d}\nu_{\mathcal{W}|\mathcal{C}=x}(W^{(1)}) . 
  \end{aligned}
  \label{2nd_derivative-long}
\end{equation}

\begin{proof}[Proof of Theorem~\ref{theorem_func}]
First, if $x\not\in\mathrm{supp}(\zeta)$, similar to \eqref{2nd_derivative}, we have
\begin{equation*}
    \begin{aligned}
g(x,(\overline{\gamma},\overline{u},\overline{v}))&=p_{\mathcal{C}}(x)\int_{\mathbb{R}}\gamma(W^{(1)},x)\left|W^{(1)}\right|~\mathrm{d}\nu_{\mathcal{W}|\mathcal{C}=x}(W^{(1)})\\    
  &=0 . 
    \end{aligned}
\end{equation*}
Next, we prove that $g(x,(\overline{\gamma},\overline{u},\overline{v}))$ restricted on $\mathrm{supp}(\zeta)$ is the solution of the following problem:
\begin{equation}
  \begin{aligned}
  \min_{h\in C^2(\mathrm{supp}(\zeta))}
  \quad &
\int_{\mathrm{supp}(\zeta)}
\frac{(h''(x))^2}{\zeta(x)}~\mathrm{d}x\\
    \textup{subject to}\quad & h(x_j)=y_j,\quad j=1,\ldots,m. 
  \end{aligned}
  \label{function_space_only_zeta}
\end{equation}
Let $L(f)=\int_{\mathrm{supp}(\zeta)} \frac{(f''(x))^2}{p(x)\mathbb{E}(\mathcal{W}^2|\mathcal{C}=x)}\mathrm{d}x$. Then the functional $L(f)$ is strictly convex on space $\{f\in C^2(\mathbb{R}^2)|f(x_i)=y_i,~i=1,\ldots,m\}$ when $m\geq 2$. 
This means that the minimizer of problem \eqref{function_space_only_zeta} is unique.

Suppose $h(x)$ is the minimizer of problem \eqref{function_space_only_zeta} and $h(x)$ is different from $g(x,(\overline{\gamma},\overline{u},\overline{v}))$ restricted on $\mathrm{supp}(\zeta)$. Then by uniqueness of the solution,
\begin{equation}
  L(h)<L(g(\cdot,(\overline{\gamma},\overline{u},\overline{v}))).
  \label{contradiction}
\end{equation}
Now our goal is to find a different $(\gamma,u,v)$ with smaller cost in problem \eqref{continuous_add_linear}. 
Then $(\overline{\gamma},\overline{u},\overline{v})$ is not the solution of \eqref{continuous_add_linear}, which is a contradiction. We set
\begin{equation*}
  \gamma(W^{(1)},c)=\frac{h''(c)|W^{(1)}|}{p_\mathcal{C}(c)\mathbb{E}(\mathcal{W}^2|\mathcal{C}=c)}, \quad c\in \mathrm{supp}(\zeta) . 
\end{equation*}
Then according to \eqref{2nd_derivative},
\begin{equation*}
  \begin{aligned}
    g''(x,\gamma)&=p(x)\int_{\mathbb{R}}\gamma(W^{(1)},x)\left|W^{(1)}\right|~\mathrm{d}\nu_{\mathcal{W}|\mathcal{C}=x}(W^{(1)})\\
    &=p(x)\int_{\mathbb{R}}\frac{h''(x)|W^{(1)}|}{p(x)\mathbb{E}(\mathcal{W}^2|\mathcal{C}=x)}\left|W^{(1)}\right|~\mathrm{d}\nu_{\mathcal{W}|\mathcal{C}=x}(W^{(1)})\\
    &=\frac{h''(x)}{\mathbb{E}(\mathcal{W}^2|\mathcal{C}=x)}\int_{\mathbb{R}}\left|W^{(1)}\right|^2~\mathrm{d}\nu_{\mathcal{W}|\mathcal{C}=x}(W^{(1)})\\
    &=\frac{h''(x)}{\mathbb{E}(\mathcal{W}^2|\mathcal{C}=x)}\mathbb{E}(\mathcal{W}^2|\mathcal{C}=x)\\
    &=h''(x), \quad x\in \mathrm{supp}(\zeta) . 
  \end{aligned}
\end{equation*}
This means that we can find $u,v\in\mathbb{R}$ such that $ux+v+g(x,\gamma)\equiv h(x)$. 
Then we find $(\gamma,u,v)$ such that $g(x,(\gamma,u,v))=ux+v+g(x,\gamma)=h(x)$ on $\mathrm{supp}(\zeta)$. So $g(x_j,(\gamma,u,v))=h(x_j)=y_j$. It means that $(\gamma,u,v)$ satisfies the condition in problem \eqref{continuous_add_linear}. Next we compute the cost of $(\gamma,u,v)$: 
\begin{equation}
  \begin{aligned}
    &\int_{\mathbb{R}^2} \gamma^2(W^{(1)},c)~\mathrm{d}\nu(W^{(1)},c)\\
   =&\int_{\mathbb{R}^2} \left(\frac{h''(c)|W^{(1)}|}{p_\mathcal{C}(c)\mathbb{E}(\mathcal{W}^2|\mathcal{C}=c)}\right)^2~\mathrm{d}\nu(W^{(1)},c)\\
   =&\int_{\mathrm{supp}(\zeta)}\left(\int_{\mathbb{R}}\left(\frac{h''(c)|W^{(1)}|}{p_\mathcal{C}(c)\mathbb{E}(\mathcal{W}^2|\mathcal{C}=c)}\right)^2~\mathrm{d}\nu_{\mathcal{W}|\mathcal{C}=c}(W^{(1)})\right)\mathrm{d}\nu_\mathcal{C}(c)\\
   =&\int_{\mathrm{supp}(\zeta)}\left(\frac{h''(c)}{p_\mathcal{C}(c)\mathbb{E}(\mathcal{W}^2|\mathcal{C}=c)}\right)^2\left(\int_{\mathbb{R}}|W^{(1)}|^2~\mathrm{d}\nu_{\mathcal{W}|\mathcal{C}=c}(W^{(1)})\right)p_\mathcal{C}(c)\mathrm{d}c\\
   =&\int_{\mathrm{supp}(\zeta)}\left(\frac{h''(c)}{p_\mathcal{C}(c)\mathbb{E}(\mathcal{W}^2|\mathcal{C}=c)}\right)^2\left(\int_{\mathbb{R}}|W^{(1)}|^2~\mathrm{d}\nu_{\mathcal{W}|\mathcal{C}=c}(W^{(1)})\right)p_\mathcal{C}(c)\mathrm{d}c\\
   =&\int_{\mathrm{supp}(\zeta)} \frac{(h''(c))^2}{p_\mathcal{C}(c)\mathbb{E}(\mathcal{W}^2|\mathcal{C}=c)}\mathrm{d}x\\
   =&L(h) . 
  \end{aligned}
  \label{long1}
\end{equation}
On the other hand, the cost of $(\overline{\gamma},\overline{u},\overline{v})$ is
\begin{equation}
  \begin{aligned}
    &\int_{\mathbb{R}^2} \overline{\gamma}^2(W^{(1)},c)~\mathrm{d}\nu(W^{(1)},c)\\
   =&\int_{\mathrm{supp}(\zeta)}\left(\int_{\mathbb{R}}\overline{\gamma}^2(W^{(1)},c)~\mathrm{d}\nu_{\mathcal{W}|\mathcal{C}=c}(W^{(1)})\right)p_\mathcal{C}(c)\mathrm{d}c\\
   \geq&\int_{\mathrm{supp}(\zeta)}\frac{\left(\int_{\mathbb{R}}\overline{\gamma}(W^{(1)},c)|W^{(1)}|~\mathrm{d}\nu_{\mathcal{W}|\mathcal{C}=c}\right)^2}{\int_{\mathbb{R}}|W^{(1)}|^2~\mathrm{d}\nu_{\mathcal{W}|\mathcal{C}=c}}p_\mathcal{C}(c)\mathrm{d}c\quad\text{(Cauchy-Schwarz inequality)}\\
   =&\int_{\mathrm{supp}(\zeta)}\frac{\left(g''(c,\overline{\gamma})/p_\mathcal{C}(c))\right)^2}{\int_{\mathbb{R}}|W^{(1)}|^2~\mathrm{d}\nu_{\mathcal{W}|\mathcal{C}=c}}p_\mathcal{C}(c)\mathrm{d}c\quad\text{(according to \eqref{2nd_derivative})}\\
   =&\int_{\mathrm{supp}(\zeta)}\frac{\left(g''(c,\overline{\gamma})\right)^2}{p_\mathcal{C}(c)\mathbb{E}(\mathcal{W}^2|\mathcal{C}=c)}\mathrm{d}c\\
   =&L(g(\cdot,\overline{\gamma}))\\
   =&L(g(\cdot,(\overline{\gamma},\overline{u},\overline{v}))) \quad (g(\cdot,(\overline{\gamma},\overline{u},\overline{v}))\text{ has the same second derivative as }g(\cdot,\overline{\gamma})) .
  \end{aligned}
  \label{long2}
\end{equation}
From this we have 
\begin{equation*}
  \begin{aligned}
    \int_{\mathbb{R}^2} \gamma^2(W^{(1)},c)~\mathrm{d}\nu(W^{(1)},c)&=L(h)\quad\text{(according to \eqref{long1})}\\
    &<L(g(\cdot,(\overline{\gamma},\overline{u},\overline{v})))\quad\text{(according to \eqref{contradiction})}\\
    &\leq \int_{\mathbb{R}^2} \overline{\gamma}^2(W^{(1)},c)~\mathrm{d}\nu(W^{(1)},c)\quad\text{(according to \eqref{long2})} . 
  \end{aligned}
\end{equation*}
It means that the cost of $(\gamma,u,v)$ is smaller than the cost of $(\overline{\gamma},\overline{u},\overline{v})$. So $(\overline{\gamma},\overline{u},\overline{v})$ is not the solution of \eqref{function_space_only_zeta}, which is a contradiction. So our assumption is wrong. So $h(x)\equiv g(x,(\overline{\gamma},\overline{u},\overline{v}))$ on $\mathrm{supp}(\zeta)$, and $g(x,(\overline{\gamma},\overline{u},\overline{v}))$ is the solution of problem \eqref{function_space_only_zeta}. In the last step we prove that $g''(x,(\overline{\gamma},\overline{u},\overline{v}))=0$ when $x\not\in[\min_i x_i, \max_i x_i]$ and $g(x,(\overline{\gamma},\overline{u},\overline{v}))$ restricted on $\operatorname{supp}(\zeta) \cap [\min_i x_i, \max_i x_i]$ is the solution of \eqref{function_space_only_zeta}. We only need to prove these statements for $h(x)$, which is the solution of \eqref{function_space_only_zeta}.

Since $|x_i|\in [\min_i x_i, \max_i x_i]$, the function values on $(-\infty,\min_i x_i)$ and $(\max_i x_i,\infty)$ are not related to constraints of problem \eqref{function_space_only_zeta}, so $h(x)$ can be replaced by following $\tilde{h}(x)$ which also satisfies the constraints of problem \eqref{function_space_only_zeta}:
\begin{equation*}
  \tilde{h}(x)=\begin{cases}
    h(x)& x\in [\min_i x_i, \max_i x_i]\\
    h'(\min_i x_i)(x-\min_i x_i)+h(\min_i x_i)&x\in (-\infty,\min_i x_i)\\
    h'(\max_i x_i)(x-\max_i x_i)+h(\max_i x_i)&x\in (\max_i x_i,\infty) .
\end{cases}
\end{equation*}
Then we get
 \begin{equation*}
  \tilde{h}''(x)=\begin{cases}
    h''(x)& x\in [\min_i x_i, \max_i x_i]\\
    0&x\in (-\infty,\min_i x_i)\\
    0&x\in (\max_i x_i,\infty) . 
\end{cases}
\end{equation*}
So the cost of $\tilde{h}(x)$ is less than that of $h(x)$. Then the fact $h(x)$ is the minimizer of \eqref{function_space_only_zeta} tell us that $h(x)\equiv\tilde{h}(x)$. So $h(x)$ should be linear on $(-\infty,\min_i x_i)$ and $(\max_i x_i,\infty)$. Then $h''(x)=0$ when $x\not\in[\min_i x_i, \max_i x_i]$. Let $h(x)|_S$ denote the function $h(x)$ restricted on $S=\operatorname{supp}(\zeta) \cap [\min_i x_i, \max_i x_i]$. Since $h(x)$ is the solution to problem \eqref{function_space_only_zeta}, we get $h(x)|_S$ is the solution to problem \eqref{function_space_only_zeta}.
This concludes the proof. 
\end{proof}

In the case of not using ASI, problem~\eqref{continuous_add_linear} becomes 
\begin{equation}
\begin{aligned}
 \min_{\gamma\in C(\mathbb{R}^2), u\in\mathbb{R}, v\in\mathbb{R}}\quad & \int_{\mathbb{R}^2} \gamma^2(W^{(1)},c)~\mathrm{d}\nu(W^{(1)},c)\\
 \textup{subject to}\quad & ux_j+v+\int_{\mathbb{R}^2} \gamma(W^{(1)},c)[W^{(1)}(x_j-c)]_+~\mathrm{d}\nu(W^{(1)},c)=y_j-f(x_{j}, \theta_0), 
 j=1,\ldots,M. 
\end{aligned}
\label{continuous_add_linear_non_ASI}
\end{equation}
Then Theorem~\ref{theorem_func} without ASI is stated as follows. 

\begin{theorem}[Theorem~\ref{theorem_func} without ASI]
Suppose $(\overline{\gamma},\overline{u},\overline{v})$ is the solution of \eqref{continuous_add_linear_non_ASI}, 
and consider the corresponding output function 
\begin{equation}
  g(x,(\overline{\gamma},\overline{u},\overline{v}))=\overline{u} %
  x+\overline{v}+\int_{\mathbb{R}^2} \overline{\gamma}(W^{(1)},c)[W^{(1)}(%
  x-c)]_+~\mathrm{d}\nu(W^{(1)},c)+f(x, \theta_0). 
  \label{func_g_non_ASI}
\end{equation}
Then $g(x,(\overline{\gamma},\overline{u},\overline{v}))$ 
satisfies $g''(x,(\overline{\gamma},\overline{u},\overline{v}))=f''(x, \theta_0)$ for 
$x\not\in S$ 
and for $x\in S$ it is the solution of the following problem: 
\begin{equation}
  \begin{aligned}
  \min_{h\in C^2(S)}
  \quad &
\int_{S}
\frac{(h''(x)-f''(x, \theta_0))^2}{\zeta(x)}~\mathrm{d}x\\
    \textup{subject to}\quad & h(x_j)=y_j,\quad j=1,\ldots,M. 
  \end{aligned}
  \label{function_space_non_ASI}
\end{equation}
\end{theorem}

\subsection{Proof of Proposition~\ref{proposition:form-rho} and Remarks to Proposition~\ref{set_rho}}
\label{app:set-rho}

\begin{proof}[Proof of Proposition~\ref{proposition:form-rho}]
Let $p_{\mathcal{W},\mathcal{C}}$ and $p_{\mathcal{W},\mathcal{B}}$ denote the joint density functions of $(\mathcal{W},\mathcal{C})$ and $(\mathcal{W},\mathcal{B})$, respectively. We have
\begin{equation*}
    p_{\mathcal{W},\mathcal{C}}(W,c)
    =\left|\frac{\partial(W,-Wc)}{\partial(W,c)}\right|p_{\mathcal{W},\mathcal{B}}(W,-Wc) 
    =|W|p_{\mathcal{W},\mathcal{B}}(W,-Wc), 
\end{equation*}
and 
\begin{equation}
\label{rho_x}
\begin{aligned}
\mathbb{E}(W^2|C=x)p_\mathcal{C}(x)
&=\int_\mathbb{R}W^2p_{\mathcal{W}|\mathcal{C}=x}(W)~\mathrm{d}W\cdot p_\mathcal{C}(x)\\
&=\int_\mathbb{R}W^2p_{\mathcal{W},\mathcal{C}}(W,x)~\mathrm{d}W\\
&=\int_\mathbb{R}|W|^3p_{\mathcal{W},\mathcal{B}}(W,-Wx)~\mathrm{d}W . 
\end{aligned}
\end{equation}
\end{proof}

\begin{proof}[Proof of Proposition~\ref{set_rho}]
The construction is given in the statement of the proposition. 
\end{proof}

\begin{remark}[Remark to Proposition~\ref{set_rho}, sampling the initial parameters]
The variables $(\mathcal{W},\mathcal{B})$ can be sampled by first sampling $C$ from $p_\mathcal{C}(x)=\frac{1}{Z}\frac{1}{\varrho(x)}$, then independently sampling $W$ from a standard Gaussian distribution and setting $B=-WC$. 
In this construction, in general $\mathcal{W}$ and $\mathcal{B}$ are not independent. 
\end{remark}

Intuitively, if we want the output function to be smooth at a certain point $x_0$, we can let the conditional distribution of $\mathcal{W}$ given $\mathcal{C}$ be concentrated around zero for $\mathcal{C}=x_0$, or we can let the probability density function of $\mathcal{C}$ to be small at $\mathcal{C}=x_0$. 
Note that $p_\mathcal{C}$ is the breakpoint density at initialization. The form of this has been studied for uniform initialization by \citet{sahs2020a}. 
We provide the explicit form of the smoothness penalty function for several types of initialization in Appendix~\ref{appendix:proof-explicit-rho}.

\begin{remark}[Remark to Proposition~\ref{set_rho}, independent initialization]
Note that constructing an arbitrary curvature penalty function will necessitate in general a non-independent joint distribution of $\mathcal{W}$ and $\mathcal{B}$. 
If $\mathcal{W}$ and $\mathcal{B}$ are required to be independent random variables, \eqref{rho_x} gives 
\[
\zeta(x) = \mathbb{E}(W^2|C=x)p_\mathcal{C}(x)
=\int_\mathbb{R}|W|^3p_{\mathcal{W}}(W)p_\mathcal{B}(-Wx)~\mathrm{d}W.
\]
Given a desired function for the left hand side, we can still try to solve for the parameter densities. 
This type of integral equation problem has been studied \citep{10.2307/2038642} and one can write a formal solution, although it is not always clear whether it will be a density. 
\end{remark}

\subsection{Proof of Theorem~\ref{proposition:explicit-rho}}
\label{appendix:proof-explicit-rho}

We prove the statement for the three considered types of initialization distributions in turn.

\begin{proof}[Proof of Theorem~\ref{proposition:explicit-rho} for Gaussian initialization] 
Using \eqref{rho_x}, we have
\begin{equation*}
\begin{aligned}
\mathbb{E}(W^2|C=x)p_\mathcal{C}(x)&=\int_\mathbb{R}|W|^3p_{\mathcal{W}}(W)p_\mathcal{B}(-Wx)\mathrm{d}W\\
&=\int_\mathbb{R}|W|^3\frac{1}{\sqrt{2\pi}\sigma_w}e^{-\frac{W^2}{2\sigma_w^2}}\frac{1}{\sqrt{2\pi}\sigma_b}e^{-\frac{W^2x^2}{2\sigma_b^2}}\mathrm{d}W\\
&=\frac{1}{2\pi\sigma_w\sigma_b}\int_\mathbb{R}|W|^3e^{-(\frac{1}{2\sigma_w^2}+\frac{x^2}{2\sigma_b^2})W^2}\mathrm{d}W.\\
\end{aligned}
\end{equation*}
Let $\sigma^2=1/\left(\frac{1}{\sigma_w^2}+\frac{x^2}{\sigma_b^2}\right)$, then we get
\begin{equation*}
\begin{aligned}
\mathbb{E}(W^2|C=x)p_\mathcal{C}(x)
&=\frac{\sigma}{\sqrt{2\pi}\sigma_w\sigma_b}\int_\mathbb{R}|W|^3\frac{1}{\sqrt{2\pi}\sigma}e^{-\frac{W^2}{2\sigma^2}}\mathrm{d}W\\
&=\frac{\sigma}{\sqrt{2\pi}\sigma_w\sigma_b}\sigma^3\cdot 2 \cdot \sqrt{\frac{2}{\pi}}\\
&=\frac{2\sigma^4}{\pi\sigma_w\sigma_b}\\
&=\frac{2\sigma_w^3\sigma_b^3}{\pi(\sigma_b^2+x^2\sigma_w^2)^2}.\\
\end{aligned}
\end{equation*}
Then we have
\begin{equation*}
\begin{aligned}
\zeta(x)&=\mathbb{E}(W^2|C=x)p_\mathcal{C}(x)\\
&=\frac{2\sigma_w^3\sigma_b^3}{\pi(\sigma_b^2+x^2\sigma_w^2)^2}.
\end{aligned}
\end{equation*}
\end{proof}

\begin{proof}[Proof of Theorem~\ref{proposition:explicit-rho} for binary-uniform initialization]
Since $\mathcal{W}$ is either $-1$ or $1$, $\mathbb{E}(\mathcal{W}^2|\mathcal{C}=x)=1$ for any $x\in \mathrm{supp}(\nu_{\mathcal{C}})$. Since $\mathcal{B}\sim \mathrm{Unif}(-a_b, a_b)$, it is easy to check $-\mathcal{B}/\mathcal{W}\sim \mathrm{Unif}(-a_b, a_b)$. So $\zeta(x)=1/2a_b,~x\in[-a_b,a_b]$. 
\end{proof}

\begin{proof}[Proof of Theorem~\ref{proposition:explicit-rho} for uniform initialization]
According to Theorem 1 in \cite{sahs2020a}, the density function $p_\mathcal{C}(c)$ of $\nu_\mathcal{C}$ is
\begin{equation*}
  p_\mathcal{C}(c)=\frac{1}{4a_wa_b}\left(\min\left\{\frac{a_b}{|c|},a_w\right\}\right)^2,\quad c\in\mathrm{supp}(\nu_\mathcal{C}).
\end{equation*}
When $|c|\leq \frac{a_b}{a_w}$, %
then $p_\mathcal{C}(c)=\frac{1}{4a_wa_b}\left(a_w\right)^2$. It means that $p_\mathcal{C}(c)$ is constant when $|c|\leq \frac{a_b}{a_w}$.

Let $p_{\mathcal{W},\mathcal{B}}(W^{(1)},b)$ denote the density function of $\mu$, $p_{\mathcal{W},\mathcal{C}}(W^{(1)},c)$ denote the density function of $\nu$, so
\begin{equation*}
  \begin{aligned}
    p_{\mathcal{W},\mathcal{C}}(W^{(1)},c)&=p_{\mathcal{W},\mathcal{B}}(W^{(1)},-cW^{(1)})\frac{\partial b}{\partial c}\\
    &= \frac{1}{4a_wa_b}\mathbbm{1}_{W^{(1)}\in[-a_w,a_w]}\cdot\mathbbm{1}_{-cW^{(1)}\in[-a_b,a_b]}\cdot(-W^{(1)}).
  \end{aligned}
\end{equation*}
Here $\mathbbm{1}_{a}$ is the indicator function which equals to $1$ when condition $a$ is true, and $0$ otherwise. Then density function $p_{\mathcal{W}|\mathcal{C}}(W^{(1)}|c)$ of the conditional distribution %
$\nu_{\mathcal{W}|\mathcal{C}=c}$ is
\begin{equation*}
  \begin{aligned}
    p_{\mathcal{W}|\mathcal{C}}(W^{(1)}|c)&=\frac{p_{\mathcal{W},\mathcal{C}}(W^{(1)},c)}{p_\mathcal{C}(c)}\\
    &=\frac{\frac{1}{4a_wa_b}\mathbbm{1}_{W^{(1)}\in[-a_w,a_w]}\cdot\mathbbm{1}_{-cW^{(1)}\in[-a_b,a_b]}\cdot(-W^{(1)})}{p_\mathcal{C}(c)}.
  \end{aligned}
\end{equation*}
When $|c|\leq \frac{a_b}{a_w}$, $|-cW^{(1)}|\leq\frac{a_b}{a_w}a_w=a_b$. So $-cW^{(1)}\in[-a_b,a_b]$ is true and $\mathbbm{1}_{-cW^{(1)}\in[-a_b,a_b]}=1$. Combined with the fact that $p_\mathcal{C}(c)$ is constant when $|c|\leq \frac{a_b}{a_w}$, we have $p_{\mathcal{W}|\mathcal{C}}(W^{(1)}|c)$ is independent of $c$ when $|c|\leq \frac{a_b}{a_w}$. So $\mathbb{E}(\mathcal{W}^2|\mathcal{C}=c)$ is constant when $|c|\leq \frac{a_b}{a_w}$. Since $\frac{a_b}{a_w} \geq I$, $\mathbb{E}(\mathcal{W}^2|\mathcal{C}=c)$ and $p_\mathcal{C}(c)$ are constant when $c\in [-I,I]$. Then $\zeta(x) =\mathbb{E}(W^2|C=x)p_\mathcal{C}(x)$ is constant when $c\in [-I,I]$. 
\end{proof}

\section{Proofs of Results for Multivariate Regression}

\label{app:multi-dimensional inputs}

\subsection{Proof of Theorem~\ref{theorem_func_multi_dim}}
\label{proof_theorem_func_multi_dim}
In this section, we prove Theorem \ref{theorem_func_multi_dim}. 
We will need the following lemmas: 
\begin{lemma}
\label{Delta_s_lipschitz}
Let $f\in\operatorname{Lip}(\mathbb{R}^d)$ be considered as a tempered distribution 
and $(-\Delta)^s f\equiv 0$, $s>0$. Then $f$ is linear, i.e., $f(\mathbf{x})=\langle \mathbf{u},\mathbf{x}\rangle+v$. 
\end{lemma}
\begin{proof}[Proof of Lemma \ref{Delta_s_lipschitz}] 
\hj{In the following proof we regard $f$ as a tempered distribution, thus the fractional Laplacian and Fourier transform of $f$ can be defined. We first give a brief introduction of tempered distribution. }

\hj{The space of tempered distributions $S'(\mathbb{R}^d)$ is the space of continuous linear functionals on the space of Schwartz test functions $S(\mathbb{R}^d)$. 
The space of Schwartz test functions on
$\mathbb{R}^{d}$ is the rapidly decreasing function space
\begin{equation}
    S(\mathbb{R}^d) := \left \{ \psi \in C^\infty(\mathbb{R}^d) \mid \forall \alpha, \beta \in\mathbb{N}^d, \sup_{\mathbf{x}\in \mathbb{R}^d}|\mathbf{x}^\beta D^\alpha \psi(\mathbf{x})| < \infty\right \}.
\end{equation}
The details of defining norms and the topology on $S(\mathbb{R}^d)$ is shown in \cite[Chapter 1]{melrose2008introduction}.} 

\hj{For any $f\in\operatorname{Lip}(\mathbb{R}^d)$, we can define a corresponding tempered distribution $T_f$ by
\begin{equation}
    T_f:S(\mathbb{R}^d)\mapsto \mathbb{R},\ T_f(\psi)=\int_{\mathbb{R}^d} f\psi\mathrm{d}\mathbf{x}.
\end{equation}
So any $f\in\operatorname{Lip}(\mathbb{R}^d)$ can be naturally regarded as a tempered distribution $T_f$.}

\hj{Let $\mathcal{F}$ be the Fourier transform. Since $\mathcal{F}$ and its adjoint maps a Schwartz function to a Schwartz function, we can define the Fourier transform of a tempered distribution by
\begin{equation}
 \mathcal{F}:S'(\mathbb{R}^d)\mapsto S'(\mathbb{R}^d),\ (\mathcal{F}T_f)(\psi)=\int_{\mathbb{R}^d} f\cdot \mathcal{G}\psi\mathrm{d}\mathbf{x}, 
\end{equation}
where $\mathcal{G}$ is the adjoint of $\mathcal{F}$. Details of Fourier transform on tempered distributions can be found in \cite[Chapter 1.7]{melrose2008introduction}.}

\hj{Similarly the fractional Laplacian of a tempered distribution is defined by
\begin{equation}
 (-\Delta)^s:S'(\mathbb{R}^d)\mapsto S'(\mathbb{R}^d),\ ((-\Delta)^sT_f)(\psi)=\int_{\mathbb{R}^d} f\cdot (-\Delta)^s\psi\mathrm{d}\mathbf{x}, 
\end{equation}}

Since $(-\Delta)^s f\equiv 0$, in Fourier domain we have $\|\bm\xi\|^{2s} \mathcal{F}f\equiv 0$.
It means that the support of $\mathcal{F}f$ is $\{0\}$. According to \citet[Chapter~9]{folland1999real}, $\mathcal{F}f$ is a linear combination of $\delta$ (Dirac's Delta) and derivatives of $\delta$.\footnote{The $k$-th derivative of $\delta$ can be defined as a tempered distribution on the Schwartz test function $\phi$ by: $\delta^{(k)}(\phi)=(-1)^k\phi^{(k)}(0)$.} Then $f$ is a polynomial. Since $f$ is Lipschitz continuous, we conclude that $f$ is linear. 
\end{proof}

\begin{lemma}
\label{even_odd_2_norm}
Let $\alpha\in L^2(\mathbb{S}^{d-1}\times \mathbb{R})$. Suppose that $\alpha=\alpha^++\alpha^-$ where $\alpha^+$ is even and $\alpha^-$ is odd. Then $\|\alpha\|_2\geq \|\alpha^+\|_2$ and $\|\alpha\|_2\geq \|\alpha^-\|_2$.
\end{lemma}
\begin{proof}[Proof of Lemma \ref{even_odd_2_norm}]
Since
\begin{equation*}
    \begin{aligned}
      \|\alpha\|_2^2&=\|\alpha^++\alpha^-\|_2^2\\
      &=\|\alpha^+\|_2^2+\|\alpha^-\|_2^2+2\langle \alpha^+,\alpha^+\rangle\\
      &=\|\alpha^+\|_2^2+\|\alpha^-\|_2^2+2\int_{\mathbb{S}^{d-1}\times\mathbb{R}}\alpha^+\cdot \alpha^-~\mathrm{d}\sigma^{d-1}(\bm{V})\mathrm{d}c\\
      &=\|\alpha^+\|_2^2+\|\alpha^-\|_2^2,
    \end{aligned}
\end{equation*}
where the last equality holds true since $\alpha^+\cdot \alpha^-$ is odd. Then we have $\|\alpha\|_2\geq \|\alpha^+\|_2$ and $\|\alpha\|_2\geq \|\alpha^-\|_2$.
\end{proof}
The next lemma shows that the output of the infinite-width network is Lipschitz continuous. This is also observed in \cite[Proposition~8]{ongie2019function}. 
\begin{lemma}
\label{network_Lipschitz}
Assume that (1) the norm of the random vector $\|\bm{\mathcal{W}}\|$ has the finite second moment; (2) $\int_{\mathbb{R}^+\times\mathbb{S}^{d-1}\times\mathbb{R}} \gamma^2(u,\bm{V},c)~\mathrm{d}\nu(u,\bm{V},c)<+\infty$; (3) $\mathbf{u}\in\mathbb{R}^d$ and $v\in\mathbb{R}$.  Then $g(\mathbf{x},(\gamma, \mathbf{u},v))$ is Lipschitz continuous.
\end{lemma}
\begin{proof}[Proof of Lemma \ref{network_Lipschitz}]
Let $\alpha(u\bm{V},-cu)=\gamma(u,\bm{V},c)$. For all $\mathbf{x}_1,\mathbf{x}_2\in\mathbb{R}^d$, we have
\begin{equation*}
    \begin{aligned}
      &|g(\mathbf{x}_1,(\gamma, \mathbf{u},v))-g(\mathbf{x}_2,(\gamma, \mathbf{u},v))|\\
      \leq&\left|\int_{\mathbb{R}^d\times \mathbb{R}} |\alpha(\mathbf{W}^{(1)},b)|\left|[\langle \mathbf{W}^{(1)},\mathbf{x}_1\rangle+b]_+-[\langle \mathbf{W}^{(1)},\mathbf{x}_2\rangle+b]_+\right|~\mathrm{d}\mu(\mathbf{W}^{(1)},b)\right|+|\langle  \mathbf{u},\mathbf{x}_1-\mathbf{x}_2 \rangle| \\
      \leq&\left|\int_{\mathbb{R}^d\times \mathbb{R}} |\alpha(\mathbf{W}^{(1)},b)|\left|\langle \mathbf{W}^{(1)},\mathbf{x}_1-\mathbf{x}_2\rangle\right|~\mathrm{d}\mu(\mathbf{W}^{(1)},b)\right|+|\langle  \mathbf{u},\mathbf{x}_1-\mathbf{x}_2 \rangle| \\
      \leq&\left(\int_{\mathbb{R}^d\times \mathbb{R}} |\alpha(\mathbf{W}^{(1)},b)| \|\mathbf{W}^{(1)}\|~\mathrm{d}\mu(\mathbf{W}^{(1)},b)+\| \mathbf{u}\|\right) \|\mathbf{x}_1-\mathbf{x}_2\|\\
      \leq&\left(\int_{\mathbb{R}^d\times \mathbb{R}} \alpha^2(\mathbf{W}^{(1)},b) ~\mathrm{d}\mu(\mathbf{W}^{(1)},b)\cdot\int_{\mathbb{R}^d\times \mathbb{R}}\|\mathbf{W}^{(1)}\|^2~\mathrm{d}\mu(\mathbf{W}^{(1)},b)+\| \mathbf{u}\|\right) \|\mathbf{x}_1-\mathbf{x}_2\|\\
      \leq&\left(\int_{\mathbb{R}^d\times \mathbb{R}} \alpha^2(\mathbf{W}^{(1)},b) ~\mathrm{d}\mu(\mathbf{W}^{(1)},b)\cdot\mathbb{E}(\|\bm{\mathcal{W}}\|^2) +\| \mathbf{u}\|\right) \|\mathbf{x}_1-\mathbf{x}_2\|.\\
    \end{aligned}
\end{equation*}
According to the assumptions, $\int_{\mathbb{R}^d\times \mathbb{R}} \alpha^2(\mathbf{W}^{(1)},b) ~\mathrm{d}\mu(\mathbf{W}^{(1)},b)$, $\mathbb{E}(\|\bm{\mathcal{W}}\|^2)$ and $\| \mathbf{u}\|$ are all finite. Then $g(\mathbf{x},(\gamma, \mathbf{u},v))$ is Lipschitz continuous.
\end{proof}

\begin{lemma}
Given a function $h\in\operatorname{Lip}(\mathbb{R}^d)\cap C(\mathbb{R}^d)$. Define $\psi:\mathbb{S}^{d-1}\times \mathbb{R}\to \mathbb{R}$ by $\psi\coloneqq-\frac{1}{2(2\pi)^{d-1}}\mathcal{R}\{(-\Delta)^{(d+1)/2}h\}$. Assume that (1) $\int_{\operatorname{supp}(\zeta)} \left(\psi(\bm{V},c)\right)^2/\zeta(\bm{V},c)~\mathrm{d}\sigma^{d-1}(\bm{V})\mathrm{d}c<+\infty$, where $\zeta(\bm{V},c)$ is define in \eqref{definition_of_zeta}, and $\psi(\bm{V},c)=0, \ \forall(\bm{V},c)\not\in \operatorname{supp}(\zeta)$; (2) $\|\bm{\mathcal{W}}\|_2$ and $\mathcal{B}$ both have finite second moments; (3) $(-\Delta)^{(d+1)/2}h \in L^p(\mathbb{R}^d),\ 1\leq p<d/(d-1)$. Then there exist $\mathbf{u}\in \mathbb{R}^d$ and $v\in \mathbb{R}$ such that $h(\mathbf{x})=\int_{\mathbb{S}^{d-1}\times\mathbb{R}}\psi(\bm{V},c)[\langle\bm{V},\mathbf{x}\rangle-c]_+ ~\mathrm{d}\sigma^{d-1}(\bm{V})\mathrm{d}c+\langle \mathbf{u},\mathbf{x} \rangle +v$.
\label{recover_lemma}
\end{lemma}
\begin{proof}[Proof of Lemma~\ref{recover_lemma}]
Since $\|\bm{\mathcal{W}}\|_2$ and $\mathcal{B}$ both have finite second moments, we have
\begin{equation*}
\begin{aligned}
    \int_{\operatorname{supp}(\zeta)} \zeta(\bm{V},c)~\mathrm{d}\sigma^{d-1}(\bm{V})\mathrm{d}c&=\int_{\operatorname{supp}(\zeta)} p_{\mathcal{C}|\bm{\mathcal{V}}=\bm{V}}(c)~p_{\bm{\mathcal{V}}}(\bm{V})\mathbb{E}(\mathcal{U}^2|\bm{\mathcal{V}}=\bm{V},\mathcal{C}=c)~\mathrm{d}\sigma^{d-1}(\bm{V})\mathrm{d}c\\
    &=\mathbb{E}\left(\mathbb{E}(\mathcal{U}^2|\bm{\mathcal{V}},\mathcal{C})\right)\\
    &=\mathbb{E}(\mathcal{U}^2)\\
    &=\mathbb{E}(\|\bm{\mathcal{W}}\|_2^2)\\
    &<+\infty, 
\end{aligned}
\end{equation*}
and 
\begin{equation*}
\begin{aligned}
    \int_{\operatorname{supp}(\zeta)} \zeta(\bm{V},c)\cdot c^2~\mathrm{d}\sigma^{d-1}(\bm{V})\mathrm{d}c&=\int_{\operatorname{supp}(\zeta)} p_{\mathcal{C}|\bm{\mathcal{V}}=\bm{V}}(c)~p_{\bm{\mathcal{V}}}(\bm{V})\mathbb{E}(\mathcal{U}^2|\bm{\mathcal{V}}=\bm{V},\mathcal{C}=c)\cdot c^2~\mathrm{d}\sigma^{d-1}(\bm{V})\mathrm{d}c\\
    &=\mathbb{E}\left(\mathbb{E}(\mathcal{U}^2\mathcal{C}^2|\bm{\mathcal{V}},\mathcal{C})\right)\\
    &=\mathbb{E}(\mathcal{U}^2\mathcal{C}^2)\\
    &=\mathbb{E}(\mathcal{B}^2)\\
    &<+\infty.
\end{aligned}
\end{equation*}
Let $\tilde{h}(\mathbf{x})=\int_{\mathbb{S}^{d-1}\times\mathbb{R}}\psi(\bm{V},c)[\langle\bm{V},\mathbf{x}\rangle-c]_+ ~\mathrm{d}\sigma^{d-1}(\bm{V})\mathrm{d}c$. For any $\mathbf{x}\in \mathbb{R}^d$, we have
\begin{equation*}
    \begin{aligned}
      &\int_{\mathbb{S}^{d-1}\times\mathbb{R}}\left|\psi(\bm{V},c)\right|[\langle\bm{V},\mathbf{x}\rangle-c]_+ ~\mathrm{d}\sigma^{d-1}(\bm{V})\mathrm{d}c\\
      \leq&\int_{\operatorname{supp}(\zeta)}\left|\psi(\bm{V},c)\right|(\|\bm{V}\|_2\|\mathbf{x}\|_2+|c|) ~\mathrm{d}\sigma^{d-1}(\bm{V})\mathrm{d}c \\
      \leq&\int_{\operatorname{supp}(\zeta)}\left|\psi(\bm{V},c)\right|(\|\mathbf{x}\|_2+|c|) ~\mathrm{d}\sigma^{d-1}(\bm{V})\mathrm{d}c \\
      \leq&\|\mathbf{x}\|_2\sqrt{\int_{\operatorname{supp}(\zeta)} \frac{\left(\psi(\bm{V},c)\right)^2}{\zeta(\bm{V},c)}~\mathrm{d}\sigma^{d-1}(\bm{V})\mathrm{d}c\cdot \int_{\operatorname{supp}(\zeta)} \zeta(\bm{V},c)~\mathrm{d}\sigma^{d-1}(\bm{V})\mathrm{d}c}\\
      +&\sqrt{\int_{\operatorname{supp}(\zeta)} \frac{\left(\psi(\bm{V},c)\right)^2}{\zeta(\bm{V},c)}~\mathrm{d}\sigma^{d-1}(\bm{V})\mathrm{d}c\cdot \int_{\operatorname{supp}(\zeta)} \zeta(\bm{V},c)\cdot c^2~\mathrm{d}\sigma^{d-1}(\bm{V})\mathrm{d}c}\\
      < &+\infty.
    \end{aligned}
\end{equation*}
So $\tilde{h}(\mathbf{x})$ is well-defined. The above inequality also implies that the Lipschitz constant of $\tilde{h}(\mathbf{x})$ is bounded by $\int_{\operatorname{supp}(\zeta)}\left|\psi(\bm{V},c)\right|\|\bm{V}\|_2 ~\mathrm{d}\sigma^{d-1}(\bm{V})\mathrm{d}c$, which is finite. So $\tilde{h}(\mathbf{x})$ is Lipschitz continuous. Then we have
\begin{equation}
\begin{aligned}
    (-\Delta)^{(d+1)/2}  \tilde{h}
    &= -(-\Delta)^{(d-1)/2}\int_{\mathbb{S}^{d-1}\times\mathbb{R}}\psi(\bm{V},c)\delta(\langle\bm{V},\mathbf{x}\rangle-c) ~\mathrm{d}\sigma^{d-1}(\bm{V})\mathrm{d}c\\
    &= -(-\Delta)^{(d-1)/2}\int_{\mathbb{S}^{d-1}}\psi(\bm{V},\langle\bm{V},\mathbf{x}\rangle) ~\mathrm{d}\sigma^{d-1}(\bm{V})\\
    &= -(-\Delta)^{(d-1)/2}\mathcal{R}^*\{\psi\}.
\end{aligned}
\label{Laplace_output}
\end{equation}
Since $(-\Delta)^{(d+1)/2}h \in L^p(\mathbb{R}^d),\ 1\leq p<d/(d-1)$, we can apply the inversion formula of the Radon transform \citep{solmon1987asymptotic}:
\begin{equation*}
\begin{aligned}
    (-\Delta)^{(d+1)/2}h
    &=\frac{1}{2(2\pi)^{d-1}}(-\Delta)^{(d-1)/2}\mathcal{R}^*\{\mathcal{R}\{(-\Delta)^{(d+1)/2}h\}\}\\
    &= -(-\Delta)^{(d-1)/2}\mathcal{R}^*\{\psi\}\\
    &= (-\Delta)^{(d+1)/2}  \tilde{h}.
\end{aligned}
\end{equation*}
According to Lemma \ref{Delta_s_lipschitz}, we have that $h-\tilde{h}$ is linear, which gives the claim.
\end{proof}
Lemma \ref{recover_lemma} immediately gives the following corollary:
\begin{corollary}
\label{equal_radon}
If $\mathcal{R}\{(-\Delta)^{(d+1)/2}g\}\equiv\mathcal{R}\{(-\Delta)^{(d+1)/2}h\}$, and $(-\Delta)^{(d+1)/2}g, (-\Delta)^{(d+1)/2}h\in L^p(\mathbb{R}^d),\ 1\leq p<d/(d-1)$, then $g-h$ is linear.
\end{corollary}

The next lemma shows that the minimizer $h(\mathbf{x})$ of problem \eqref{function_space_multi} satisfies that $\mathcal{R}\{(-\Delta)^{(d+1)/2}h\}$ is compactly supported.
\begin{lemma}
\label{tight_support}
  Consider the training data $\{(\mathbf{x}_i,y_i)\}_{i=1}^M$. Let $R$ be the maximum 2-norm of training inputs, i.e., $R=\max_i \|\mathbf{x}_i\|_2$. Suppose $h(\mathbf{x})$ is the solution of the optimization problem \eqref{function_space_multi}. Then ${\mathcal{R}\{(-\Delta)^{(d+1)/2}h\}(\bm{V},c)}=0,\ \forall (\bm{V},c)\not\in \mathbb{S}^{d-1}\times [-R,R]$. 
\end{lemma}
\begin{proof}[Proof of Lemma \ref{tight_support}]
Define $\psi:\mathbb{S}^{d-1}\times \mathbb{R}\to \mathbb{R}$ by $\psi\coloneqq-\frac{1}{2(2\pi)^{d-1}}\mathcal{R}\{(-\Delta)^{(d+1)/2}h\}$. Then we construct the function $\overline{\psi}:\mathbb{S}^{d-1}\times \mathbb{R}\to \mathbb{R}$ as follows:
\begin{equation*}
    \overline{\psi}(\bm{V},c) =
    \begin{cases}
    \psi(\bm{V},c),&\text{for }|c|\leq R\\
    0.&\text{for }|c|> R.
    \end{cases}
\end{equation*}
Since the Radon transform is even, we have that $\psi$ and $\overline{\psi}$ are both even. Since $h$ is the solution of \eqref{function_space_multi}, $\psi$ satisfies all assumptions of Lemma \ref{recover_lemma}. Then according to Lemma \ref{recover_lemma},  $h(\mathbf{x})=\int_{\mathbb{S}^{d-1}\times\mathbb{R}}\psi(\bm{V},c)[\langle\bm{V},\mathbf{x}\rangle-c]_+ ~\mathrm{d}\sigma^{d-1}(\bm{V})\mathrm{d}c+\langle \mathbf{u},\mathbf{x} \rangle +v$. Let $\overline{h}(\mathbf{x})=\int_{\mathbb{S}^{d-1}\times\mathbb{R}}\overline{\psi}(\bm{V},c)[\langle\bm{V},\mathbf{x}\rangle-c]_+ ~\mathrm{d}\sigma^{d-1}(\bm{V})\mathrm{d}c$. Then $\overline{h}(\mathbf{x})-h(\mathbf{x})=\int_{\mathbb{S}^{d-1}\times\mathbb{R}}(\psi-\overline{\psi})(\bm{V},c)[\langle\bm{V},\mathbf{x}\rangle-c]_+ ~\mathrm{d}\sigma^{d-1}(\bm{V})\mathrm{d}c+\langle \mathbf{u},\mathbf{x} \rangle +v$. When $|c|\leq R$,  $\psi-\overline{\psi}=0$. When $|c|> R$, $[\langle\bm{V},\mathbf{x}\rangle-c]_+$ is linear with respect to $x$ on $\{\mathbf{x}:\|\mathbf{x}\|_2\leq R\}$. It means that $\overline{h}(\mathbf{x})-h(\mathbf{x})$ is linear on $\{\mathbf{x}:\|\mathbf{x}\|_2\leq R\}$. Then we can find out $\overline{u}$ and $\overline{v}$ such that $h(\mathbf{x})=\overline{h}(\mathbf{x})+\langle \overline{\mathbf{u}},\mathbf{x} \rangle +\overline{v}$ on $\{\mathbf{x}:\|\mathbf{x}\|_2\leq R\}$. Let $\widetilde{h}(\mathbf{x})=\overline{h}(\mathbf{x})+\langle \overline{\mathbf{u}},\mathbf{x} \rangle +\overline{v}$. Since all training inputs satisfy $\|\mathbf{x}_i\|\leq R$, we have that $\widetilde{h}(\mathbf{x})$ fits all training data. Similar to \eqref{Laplace_output}, we have that $\Delta \widetilde{h}=\mathcal{R}^*\{\overline{\psi}\}$. Since $\overline{\psi}$ has compact support, the inversion formula of the Radon transform \citep{solmon1987asymptotic} gives that $\overline{\psi}=-\frac{1}{2(2\pi)^{d-1}}\mathcal{R}\{(-\Delta)^{(d+1)/2}\widetilde{h}\}$. Since the support of $\overline{\psi}$ is contained in the support of $\psi$, we have $\mathcal{R}\{(-\Delta)^{(d+1)/2}\widetilde{h}\}(\bm{V},c)=0, \ \forall(\bm{V},c)\not\in \operatorname{supp}(\zeta)$. Since $(-\Delta)^{(d+1)/2}\widetilde{h}=-(-\Delta)^{(d-1)/2}\mathcal{R}^*\{\overline{\psi}\}$ and $\overline{\psi}$ is compactly supported, we have $(-\Delta)^{(d-1)/2}\mathcal{R}^*\{\overline{\psi}\}\in L^p(\mathbb{R}^d),\ 1\leq p<d/(d-1)$ according to  \citep[Lemma~4.1]{solmon1987asymptotic}. The above argument shows that $\widetilde{h}$ satisfies all constrains of the problem \eqref{function_space_multi}. Since $h$ is the solution of \eqref{function_space_multi}, we have $\int_{\operatorname{supp}(\zeta)} \frac{\left(\psi(\bm{V},c)\right)^2}{\zeta(\bm{V},c)}~\mathrm{d}\sigma^{d-1}(\bm{V})\mathrm{d}c\leq \int_{\operatorname{supp}(\zeta)} \frac{\left(\overline{\psi}(\bm{V},c)\right)^2}{\zeta(\bm{V},c)}~\mathrm{d}\sigma^{d-1}(\bm{V})\mathrm{d}c$. It means that $\psi(\bm{V},c)=0$ when $|c|>R$, which gives the claim.
\end{proof}
The proof of Lemma \ref{tight_support} also applies to the
optimization problem without the constraint ${\mathcal{R}\{(-\Delta)^{(d+1)/2}h\}(\bm{V},c)}=0, \ \forall(\bm{V},c)\not\in \operatorname{supp}(\zeta)$. Then we have the following corollary.
\begin{corollary}
\label{tight_support_corollary}
  Consider the training data $\{(\mathbf{x}_i,y_i)\}_{i=1}^M$. Let $R$ be the maximum 2-norm of training inputs, i.e., $R=\max_i \|\mathbf{x}_i\|_2$. Suppose $h(\mathbf{x})$ is the solution of the following optimization problem:
\begin{equation}
\begin{aligned}
 \min_{h\in \operatorname{Lip}(\mathbb{R}^d)\cap C(\mathbb{R}^d)}\quad & \int_{\operatorname{supp}(\zeta)} \frac{\left({\mathcal{R}\{(-\Delta)^{(d+1)/2}h\}(\bm{V},c)}\right)^2}{\zeta(\bm{V},c)}~\mathrm{d}\sigma^{d-1}(\bm{V})\mathrm{d}c\\
 \textup{subject to}\quad & h(\mathbf{x}_j)=y_j,\quad j=1,\ldots,M,\\
 &(-\Delta)^{(d+1)/2}h \in L^p(\mathbb{R}^d),\ 1\leq p<d/(d-1).
\end{aligned}
\label{function_space_multi_less_constraint}
\end{equation}
Then ${\mathcal{R}\{(-\Delta)^{(d+1)/2}h\}(\bm{V},c)}=0,\ \forall (\bm{V},c)\not\in \mathbb{S}^{d-1}\times [0,R]$. It means that 
  if  $\mathbb{S}^{d-1}\times [0,R] \subset \operatorname{supp}(\zeta)$, $h(\mathbf{x})$ is also the solution of \eqref{function_space_multi}.
\end{corollary}

Now we are ready to prove Theorem~\ref{theorem_func_multi_dim}. 
We use the proof technique of Theorem~\ref{theorem_func} and \eqref{alpha_multi}. 

\begin{proof}[Proof of Theorem~\ref{theorem_func_multi_dim}]
First, according to \eqref{definition_of_kappa} and \eqref{alpha_multi}, if $(\bm{V},c)\not\in \operatorname{supp}(\zeta)$, we have
\begin{equation}
    \begin{aligned}
&|\mathcal{R}\{(-\Delta)^{(d+1)/2}g(\cdot,(\overline{\gamma},\mathbf{u},v))\}(\bm{V},c)|\\
=&|2(2\pi)^{d-1}\int_{\mathbb{R}}\gamma(u,\bm{V},c)\cdot u~\mathrm{d}\nu_{\mathcal{U}|\bm{\mathcal{V}}=\bm{V},\mathcal{C}=c}(u)\cdot p_{\mathcal{C}|\bm{\mathcal{V}}=\bm{V}}(c)~p_{\bm{\mathcal{V}}}(\bm{V})|\\
\leq& |2(2\pi)^{d-1}\int_{\mathbb{R}}\gamma^2(u,\bm{V},c)~\mathrm{d}\nu_{\mathcal{U}|\bm{\mathcal{V}}=\bm{V},\mathcal{C}=c}(u)\cdot\mathbb{E}(\mathcal{U}^2|\bm{\mathcal{V}}=\bm{V},\mathcal{C}=c) p_{\mathcal{C}|\bm{\mathcal{V}}=\bm{V}}(c)~p_{\bm{\mathcal{V}}}(\bm{V})|\\
=&0.
    \end{aligned}
\label{zero_outside}
\end{equation}
By Lemma \ref{network_Lipschitz}, we have that $g(\mathbf{x},(\overline{\gamma},\overline{ \mathbf{u}},\overline{v}))$ is Lipschitz continuous, thus $g(\mathbf{x},(\overline{\gamma},\overline{ \mathbf{u}},\overline{v}))$ satisfies all constraints of \eqref{function_space_multi}. Next, we prove that $g(\mathbf{x},(\overline{\gamma},\overline{ \mathbf{u}},\overline{v}))$ is the solution of \eqref{function_space_multi}.

Let $L(f)=\int_{\operatorname{supp}(\zeta)} \frac{\left({\mathcal{R}\{(-\Delta)^{(d+1)/2}g\}(\bm{V},c)}\right)^2}{\zeta(\bm{V},c)}~\mathrm{d}\sigma^{d-1}(\bm{V})\mathrm{d}c$. We first show that when $m\geq d+1$, the functional $L(f)$ is strictly convex on the feasible set, %
which means that the minimizer of problem \eqref{function_space_multi} is unique.

Suppose $f_1,f_2$ are two different functions in the feasible set of \eqref{function_space_multi}. Then $\mathcal{R}\{(-\Delta)^{(d+1)/2}f_1\}$ and  $\mathcal{R}\{(-\Delta)^{(d+1)/2}f_2\}$ should be different. Otherwise, according to Corollary \ref{equal_radon}, $f_1-f_2$ is a linear function. We know that $(f_1-f_2)(\mathbf{x}_i)=0$, $i=1,\ldots,m$. So $f_1=f_2$ on at least $d+1$ points. Then $f_1-f_2\equiv0$ and this is a contradiction. Since $\mathcal{R}\{(-\Delta)^{(d+1)/2}(f_1)\}$ and  $\mathcal{R}\{(-\Delta)^{(d+1)/2}(f_2)\}$ are different, by strict convexity of the square function, we have that $L(f)$ is strictly convex on the feasible set. 

Suppose $h(\mathbf{x})$ is the minimizer of problem \eqref{function_space_multi} and $h(\mathbf{x})$ is different from $g(\mathbf{x},(\overline{\gamma},\overline{ \mathbf{u}},\overline{v}))$. Then by uniqueness of the solution,
\begin{equation}
  L(h)<L(g(\mathbf{x},(\overline{\gamma},\overline{\mathbf{u}},\overline{v}))).
  \label{contradiction_multi}
\end{equation}

Now our goal is to find a different $(\gamma,\mathbf{u},v)$ with smaller cost in problem \eqref{continuous_version_multi_relax_gamma}. Then $(\overline{\gamma},\overline{ \mathbf{u}},\overline{v})$ is not the solution of \eqref{continuous_version_multi_relax_gamma}, which is a contradiction. We set
\begin{equation*}
  \gamma(u,\bm{V},c)=\begin{cases}
  \displaystyle\frac{\mathcal{R}\{(-\Delta)^{(d+1)/2}h\}(\bm{V},c)\cdot u}{-2(2\pi)^{d-1}\zeta(\bm{V},c)}, & (\bm{V},c)\in \mathrm{supp}(\zeta),\\
  0,& (\bm{V},c)\not\in \mathrm{supp}(\zeta). 
  \end{cases}
\end{equation*}
According to \eqref{2nd_derivative-long_multi_sphere_beta}, we have $\Delta g(\cdot, (\gamma,\mathbf{0},0))=\mathcal{R}^*\{\beta\}$ where $\beta$ is defined in \eqref{definition_of_kappa} and  \eqref{definition_of_beta}. Using Lemma \ref{tight_support}, we know that $\mathcal{R}\{(-\Delta)^{(d+1)/2}h\}$ is compactly supported. Then we can easily verify that $\beta$ is also compactly supported. According to  \citep[Lemma~4.1]{solmon1987asymptotic}, $(-\Delta)^{(d-1)/2}\mathcal{R}^*\{\beta\}\in L^p(\mathbb{R}^d),\ 1\leq p<d/(d-1)$, which means that $g(\cdot, (\gamma,\mathbf{0},0))$ satisfies the third constraint of the optimization problem \eqref{function_space_multi}.

Since the Radon transform is an even function, we have $\gamma(u,\bm{V},c)=\gamma(u,-\bm{V},-c)$. Since the distribution of $(\bm{\mathcal{W}},\mathcal{B})$ is symmetric, we have that $\nu_{\mathcal{U}|\bm{\mathcal{V}}=\bm{V},\mathcal{C}=c}$ is the same probability measure as $\nu_{\mathcal{U}|\bm{\mathcal{V}}=-\bm{V},\mathcal{C}=-c}$ and $p_{\mathcal{C}|\bm{\mathcal{V}}=\bm{V}}(c)p_{-\bm{\mathcal{V}}}(\bm{V})=p_{\mathcal{C}|\bm{\mathcal{V}}=\bm{V}}(-c)p_{\bm{\mathcal{V}}}(-\bm{V})$. From the definition of $\kappa$ \eqref{definition_of_kappa} and $\beta$ \eqref{definition_of_beta}, we have that $\kappa$ and $\beta$ are even. Then the odd part $\beta^-$ of $\beta$ is $0$. According to \eqref{alpha_multi}, 
\begin{equation}
  \begin{aligned}
    &\mathcal{R}\{(-\Delta)^{(d+1)/2}g(\cdot, (\gamma,\mathbf{0},0)))\}(\bm{V},c)\\
    =&-2(2\pi)^{d-1}p_{\mathcal{C}|\bm{\mathcal{V}}=\bm{V}}(c)p_{\bm{\mathcal{V}}}(\bm{V})\int_{\mathbb{R}^+}\gamma(u,\bm{V},c)\cdot u~\mathrm{d}\nu_{\mathcal{U}|\bm{\mathcal{V}}=\bm{V},\mathcal{C}=c}(u)\\
    =&-2(2\pi)^{d-1}p_{\mathcal{C}|\bm{\mathcal{V}}=\bm{V}}(c)p_{\bm{\mathcal{V}}}(\bm{V})\int_{\mathbb{R}^+}\frac{\mathcal{R}\{(-\Delta)^{(d+1)/2}h\}(\bm{V},c)\cdot u^2}{-2(2\pi)^{d-1}\zeta(\bm{V},c)}~\mathrm{d}\nu_{\mathcal{U}|\bm{\mathcal{V}}=\bm{V},\mathcal{C}=c}(u)\\
    =&-2(2\pi)^{d-1}p_{\mathcal{C}|\bm{\mathcal{V}}=\bm{V}}(c)p_{\bm{\mathcal{V}}}(\bm{V})\frac{\mathcal{R}\{(-\Delta)^{(d+1)/2}h\}(\bm{V},c)\cdot \mathbb{E}(\mathcal{U}^2|\bm{\mathcal{V}}=\bm{V},\mathcal{C}=c)}{-2(2\pi)^{d-1}\zeta(\bm{V},c)}\\
    =&\mathcal{R}\{(-\Delta)^{(d+1)/2}h\}(\bm{V},c), \quad (\bm{V},c)\in \mathrm{supp}(\zeta) . 
  \end{aligned}
  \label{Radon_of_new_function}
\end{equation}
It is not difficult to show that if $(\bm{V},c)\not\in \operatorname{supp}(\zeta)$, then $\mathcal{R}\{(-\Delta)^{(d+1)/2}g(\cdot, (\gamma,\mathbf{0},0)))\}(\bm{V},c)=0$ as in \eqref{zero_outside}. Then, according to \eqref{Radon_of_new_function}, $\mathcal{R}\{(-\Delta)^{(d+1)/2}g(\cdot, (\gamma,\mathbf{0},0)))\}\equiv\mathcal{R}\{(-\Delta)^{(d+1)/2}h\}$. According to Corollary~\ref{equal_radon}, we have
that $g(\cdot, (\gamma,\mathbf{0},0)))-h$ is a linear function. This means that we can find  $\mathbf{u}\in\mathbb{R}^d,v\in\mathbb{R}$ such that $\langle \mathbf{u}, \mathbf{x}\rangle+v+g(\mathbf{x}, (\gamma,\mathbf{0},0)))\equiv h(\mathbf{x})$. Then we find $(\gamma,\mathbf{u},v)$ such that $g(\mathbf{x},(\gamma,\mathbf{u},v))=\langle \mathbf{u}, \mathbf{x}\rangle+v+g(\mathbf{x}, (\gamma,\mathbf{0},0)))=h(\mathbf{x})$ on $\mathrm{supp}(\zeta)$. So $g(\mathbf{x}_j,(\gamma,\mathbf{u},v))=h(\mathbf{x}_j)=y_j$. This means that $(\gamma,\mathbf{u},v)$ satisfies the condition in problem \eqref{continuous_version_multi_relax_gamma}. Next we compute the cost of $(\gamma,\mathbf{u},v)$:
\begin{equation}
  \begin{aligned}
    &\int_{\mathbb{R}^+\times\mathbb{S}^{d-1}\times\mathbb{R}}\gamma^2(u,\bm{V},c)~\mathrm{d}\nu(u,\bm{V},c)\\
   =&\int_{\mathbb{R}^+\times\mathbb{S}^{d-1}\times\mathbb{R}} \left(\frac{\mathcal{R}\{(-\Delta)^{(d+1)/2}h\}(\bm{V},c)\cdot u}{-2(2\pi)^{d-1}\zeta(\bm{V},c)}\right)^2~\mathrm{d}\nu(u,\bm{V},c)\\
   =&\int_{\mathbb{S}^{d-1}\times\mathbb{R}}\left(\frac{\mathcal{R}\{(-\Delta)^{(d+1)/2}h\}(\bm{V},c)}{\zeta(\bm{V},c)}\right)^2\left(\frac{\int_{\mathbb{R}^+}u^2~\mathrm{d}\nu_{\mathcal{U}|\bm{\mathcal{V}}=\bm{V},\mathcal{C}=c}(u)}{4(2\pi)^{2(d-1)}}\right)~\mathrm{d}\nu_{\bm{\mathcal{V}},\mathcal{C}}(\bm{V},c)\\
   =&\int_{\mathbb{S}^{d-1}\times\mathbb{R}}\left(\frac{\mathcal{R}\{(-\Delta)^{(d+1)/2}h\}(\bm{V},c)}{\zeta(\bm{V},c)}\right)^2\frac{\mathbb{E}(\mathcal{U}^2|\bm{\mathcal{V}}=\bm{V},\mathcal{C}=c)p_{\mathcal{C}|\bm{\mathcal{V}}=\bm{V}}(c)~p_{\bm{\mathcal{V}}}(\bm{V})}{4(2\pi)^{2(d-1)}}~\mathrm{d}\sigma^{d-1}(\bm{V})\mathrm{d}c\\
   =&\frac{1}{4(2\pi)^{2(d-1)}}\int_{\mathbb{S}^{d-1}\times\mathbb{R}}\frac{\left(\mathcal{R}\{(-\Delta)^{(d+1)/2}h\}(\bm{V},c)\right)^2}{\zeta(\bm{V},c)}~\mathrm{d}\sigma^{d-1}(\bm{V})\mathrm{d}c\\
   =&\frac{1}{4(2\pi)^{2(d-1)}}L(h).
  \end{aligned}
  \label{long_multi_dim1}
\end{equation}
According to \eqref{alpha_multi}, the cost of $(\overline{\gamma},\mathbf{\overline{u}},\overline{v})$ is
\begin{equation}
  \begin{aligned}
&\int_{\mathbb{R}^+\times\mathbb{S}^{d-1}\times\mathbb{R}}\overline{\gamma}^2(u,\bm{V},c)~\mathrm{d}\nu(u,\bm{V},c)\\
=&\int_{\mathbb{S}^{d-1}\times\mathbb{R}}\left(\int_{\mathbb{R}^+}\overline{\gamma}^2(u,\bm{V},c)~\mathrm{d}\nu_{\mathcal{U}|\bm{\mathcal{V}}=\bm{V},\mathcal{C}=c}(u)\right)~\mathrm{d}\nu_{\bm{\mathcal{V}},\mathcal{C}}(\bm{V},c)\\
\geq&\int_{\mathbb{S}^{d-1}\times\mathbb{R}}\frac{\left(\int_{\mathbb{R}^+}\overline{\gamma}(u,\bm{V},c)\cdot u~\mathrm{d}\nu_{\mathcal{U}|\bm{\mathcal{V}}=\bm{V},\mathcal{C}=c}(u)\right)^2}{\mathbb{E}(\mathcal{U}^2|\bm{\mathcal{V}}=\bm{V},\mathcal{C}=c)}~\mathrm{d}\nu_{\bm{\mathcal{V}},\mathcal{C}}(\bm{V},c)\\
=&\int_{\mathbb{S}^{d-1}\times \mathbb{R}} \left(\frac{\mathcal{R}\{(-\Delta)^{(d+1)/2}g(\cdot, (\overline{\gamma},\mathbf{0},0))\}(\bm{V},c)-2(2\pi)^{d-1}\beta^-}{2(2\pi)^{d-1}p_{\mathcal{C}|\bm{\mathcal{V}}=\bm{V}}(c)~p_{\bm{\mathcal{V}}}(\bm{V})}\right)^2\frac{1}{\mathbb{E}(\mathcal{U}^2|\bm{\mathcal{V}}=\bm{V},\mathcal{C}=c)}~\mathrm{d}\nu_{\bm{\mathcal{V}},\mathcal{C}}(\bm{V},c)\\
\geq&\int_{\mathbb{S}^{d-1}\times \mathbb{R}} \left(\frac{\mathcal{R}\{(-\Delta)^{(d+1)/2}g(\cdot, (\overline{\gamma},\mathbf{0},0))\}(\bm{V},c)}{2(2\pi)^{d-1}p_{\mathcal{C}|\bm{\mathcal{V}}=\bm{V}}(c)~p_{\bm{\mathcal{V}}}(\bm{V})}\right)^2\frac{p_{\mathcal{C}|\bm{\mathcal{V}}=\bm{V}}(c)~p_{\bm{\mathcal{V}}}(\bm{V})}{\mathbb{E}(\mathcal{U}^2|\bm{\mathcal{V}}=\bm{V},\mathcal{C}=c)}~\mathrm{d}\sigma^{d-1}(\bm{V})\mathrm{d}c\\
=&\int_{\mathbb{S}^{d-1}\times \mathbb{R}} \frac{\left({\mathcal{R}\{(-\Delta)^{(d+1)/2}g(\cdot, (\overline{\gamma},\mathbf{0},0))\}(\bm{V},c)}\right)^2}{4(2\pi)^{2(d-1)}p_{\mathcal{C}|\bm{\mathcal{V}}=\bm{V}}(c)~p_{\bm{\mathcal{V}}}(\bm{V})\mathbb{E}(\mathcal{U}^2|\bm{\mathcal{V}}=\bm{V},\mathcal{C}=c)}~\mathrm{d}\sigma^{d-1}(\bm{V})\mathrm{d}c\\
=&\frac{1}{4(2\pi)^{2(d-1)}}L(g(\cdot, (\overline{\gamma},\mathbf{0},0)))\\
=&\frac{1}{4(2\pi)^{2(d-1)}}L(g(\cdot, (\overline{\gamma},\mathbf{u},v))) \quad \text{(since $(-\Delta)^{(d+1)/2}$ is invariant up to a linear function)},
  \end{aligned}
  \label{long_multi_dim}
\end{equation}
where the first inequality is by the Cauchy-Schwarz inequality and the second inequality is by the Lemma \ref{even_odd_2_norm} and the fact that $\frac{\mathcal{R}\{(-\Delta)^{(d+1)/2}g(\cdot, (\overline{\gamma},\mathbf{0},0))\}(\bm{V},c)}{2(2\pi)^{d-1}p_{\mathcal{C}|\bm{\mathcal{V}}=\bm{V}}(c)~p_{\bm{\mathcal{V}}}(\bm{V})}$ is an even function and $\frac{-2(2\pi)^{d-1}\beta^-}{2(2\pi)^{d-1}p_{\mathcal{C}|\bm{\mathcal{V}}=\bm{V}}(c)~p_{\bm{\mathcal{V}}}(\bm{V})}$ is an odd function. Then we have
\begin{equation*}
  \begin{aligned}
    &\int_{\mathbb{R}^+\times\mathbb{S}^{d-1}\times\mathbb{R}}\gamma^2(u,\bm{V},c)~\mathrm{d}\nu(u,\bm{V},c)\\
    =&\frac{1}{4(2\pi)^{2(d-1)}}L(h)\quad\text{(according to \eqref{long_multi_dim1})}\\
    <&\frac{1}{4(2\pi)^{2(d-1)}}L(g(\cdot, (\overline{\gamma},\mathbf{u},v)))\quad\text{(according to \eqref{contradiction_multi})}\\
    \leq& \int_{\mathbb{R}^+\times\mathbb{S}^{d-1}\times\mathbb{R}}\frac{1}{4(2\pi)^{2(d-1)}}\overline{\gamma}^2(u,\bm{V},c)~\mathrm{d}\nu(u,\bm{V},c)\quad\text{(according to \eqref{long_multi_dim})} . 
  \end{aligned}
\end{equation*}
This means that the cost of $(\gamma,\mathbf{u},v)$ is smaller than the cost of $(\overline{\gamma},\mathbf{\overline{u}},\overline{v})$. 
This implies that $(\overline{\gamma},\mathbf{\overline{u}},\overline{v})$ is not the solution of \eqref{continuous_version_multi_relax_gamma}, which is a contradiction and hence the assumption cannot be true. 
In turn, $h(\mathbf{x})\equiv g(\mathbf{x},(\overline{\gamma},\mathbf{\overline{u}},\overline{v}))$, and $g(x,(\overline{\gamma},\mathbf{\overline{u}},\overline{v})$ is the solution of problem \eqref{continuous_version_multi_relax_gamma}. 
This concludes the proof. 
\end{proof}

\subsection{Proof of Theorem~\ref{uniform_curvature}}
\label{proof_uniform_curvature}
\begin{proof}[Proof of Theorem \ref{uniform_curvature}]
To simplify the analysis, we let $f(\mathbf{x}, \theta_0)\equiv 0$. 
The analysis still holds without this simplification. 
It is easy to verify that $\operatorname{supp}(\zeta)=\mathbb{S}^{d-1}\times [-a_b,a_b]$ and  $\zeta(\bm{V},c)$ is constant over $\operatorname{supp}(\zeta)$ according to Proposition~\ref{constant_zeta}. 
According to Corollary \ref{tight_support_corollary}, we have that the variational problem \eqref{gen_multi_dim} is equivalent to the following variational problem: 
\begin{equation}
\begin{aligned}
 \min_{h\in \operatorname{Lip}(\mathbb{R}^d)\cap C(\mathbb{R}^d)}\quad & \int_{\operatorname{supp}(\zeta)} \left({\mathcal{R}\{(-\Delta)^{(d+1)/2}h\}(\bm{V},c)}\right)^2~\mathrm{d}\sigma^{d-1}(\bm{V})\mathrm{d}c\\
 \textup{subject to}\quad & h(\mathbf{x}_j)=y_j,\quad j=1,\ldots,M,\\
 &(-\Delta)^{(d+1)/2}h \in L^p(\mathbb{R}^d),\ 1\leq p<d/(d-1). 
\end{aligned}
\label{function_space_multi_less_constraint_copy}
\end{equation}
The solution $h(\mathbf{x})$ of \eqref{function_space_multi_less_constraint_copy} satisfies that  ${\mathcal{R}\{(-\Delta)^{(d+1)/2}h\}(\bm{V},c)}=0,\ \forall (\bm{V},c)\not\in \mathbb{S}^{d-1}\times [0,\max_i \|\mathbf{x}_i\|_2]$. 
The assumption $a_b\geq\max_i \|\mathbf{x}_i\|_2$ means that $\mathbb{S}^{d-1}\times [0,\max_i\|\mathbf{x}_i\|_2] \subset \operatorname{supp}(\zeta)$. 
So ${\mathcal{R}\{(-\Delta)^{(d+1)/2}h\}(\bm{V},c)}=0,\ \forall (\bm{V},c)\not\in \operatorname{supp}(\zeta)$, which means that $h(\mathbf{x})$ is also the solution of the following variational problem: 
\begin{equation}
\begin{aligned}
 \min_{h\in \operatorname{Lip}(\mathbb{R}^d)\cap C(\mathbb{R}^d)}\quad & \int_{\mathbb{S}^{d-1}\times\mathbb{R}} \left({\mathcal{R}\{(-\Delta)^{(d+1)/2}h\}(\bm{V},c)}\right)^2~\mathrm{d}\sigma^{d-1}(\bm{V})\mathrm{d}c\\
 \textup{subject to}\quad & h(\mathbf{x}_j)=y_j,\quad j=1,\ldots,M.\\
 &(-\Delta)^{(d+1)/2}h \in L^p(\mathbb{R}^d),\ 1\leq p<d/(d-1).
\end{aligned}
\label{function_space_multi_less_constraint_copy2}
\end{equation}
So it is sufficient to prove that if $h\in \operatorname{Lip}(\mathbb{R}^d)$ and $(-\Delta)^{(d+1)/2}h \in L^p(\mathbb{R}^d),\ 1\leq p<d/(d-1)$, we have
\begin{equation*}
    \int_{\mathbb{S}^{d-1}\times\mathbb{R}} \left({\mathcal{R}\{(-\Delta)^{(d+1)/2}h\}(\bm{V},c)}\right)^2~\mathrm{d}\sigma^{d-1}(\bm{V})\mathrm{d}c=\int_{\mathbb{R}^{d}} \left((-\Delta)^{(d+3)/4}h(\mathbf{x})\right)^2~\mathrm{d}\mathbf{x}.
\end{equation*} 
Given $f:\mathbb{S}^{d-1}\times \mathbb{R}\to \mathbb{R}$, let $\widetilde{f}$ be the Fourier transform over affine parameter:
\[
\widetilde{f}(\bm{V},\tau)=\int_\infty^\infty f(\bm{V},c)e^{-ic\tau}\mathrm{d}c.
\]
According to \citet[Lemma~4.5]{solmon1987asymptotic}, we have
\begin{equation*}
\begin{aligned}
    \widetilde{\mathcal{R}}\{(-\Delta)^{(d+1)/2}h\}(\bm{V},\tau)&=\widehat{(-\Delta)^{(d+1)/2}h}(\tau\bm{V})\\
    \\
    &=\|\tau\|^{d+1}\widehat{h}(\tau\bm{V}) \quad \mathrm{a.e.},
\end{aligned}
\end{equation*}
where $\widehat{h}$ is the Fourier transform of $h$. Then we have
\begin{equation*}
\begin{aligned}
    &\int_{\mathbb{S}^{d-1}\times\mathbb{R}} \left({\mathcal{R}\{(-\Delta)^{(d+1)/2}h\}(\bm{V},c)}\right)^2~\mathrm{d}\sigma^{d-1}(\bm{V})\mathrm{d}c\\
    =&\int_{\mathbb{S}^{d-1}\times\mathbb{R}} \left(\widetilde{\mathcal{R}}\{(-\Delta)^{(d+1)/2}h\}(\bm{V},\tau)\right)^2~\mathrm{d}\sigma^{d-1}(\bm{V})\mathrm{d}\tau\\
    =&\int_{\mathbb{S}^{d-1}\times\mathbb{R}} \left(\|\tau\|^{d+1}\widehat{h}(\tau\bm{V})\right)^2~\mathrm{d}\sigma^{d-1}(\bm{V})\mathrm{d}\tau\\
    =&\int_{\mathbb{R}^{d}} \left(\|\tau\|^{(d+3)/2}\widehat{h}(\mathbf{x})\right)^2~\mathrm{d}\mathbf{x}\\
    =&\int_{\mathbb{R}^{d}} \left(\widehat{(-\Delta)^{(d+3)/4}h}(\mathbf{x})\right)^2~\mathrm{d}\mathbf{x}\\
    =&\int_{\mathbb{R}^{d}} \left((-\Delta)^{(d+3)/4}h(\mathbf{x})\right)^2~\mathrm{d}\mathbf{x}.\\
\end{aligned}
\end{equation*} 
\end{proof}
\subsection{Proof of Theorem~\ref{closed_form_sol}}
\label{proof_closed_form_sol}
In order to prove Theorem \ref{closed_form_sol}, we need following lemmas:
\begin{lemma}
\label{fractional_Lap_radial_power_function}
For any $d\geq 2$ and $\mathbf{x}_1,\mathbf{x}_2\in\mathbb{R}^d$, we have $(-\Delta)^{(d+1)/2}(\|\mathbf{x}-\mathbf{x}_1\|^3-\|\mathbf{x}-\mathbf{x}_2\|^3)=C_d(\Gamma(\mathbf{x}-\mathbf{x}_1)-\Gamma(\mathbf{x}-\mathbf{x}_2))$, where $C_d$ is a constant.
\end{lemma}
\begin{proof}[Proof of Lemma \ref{fractional_Lap_radial_power_function}]
In order to prove the lemma, we need the following simple fact that 
\begin{equation}
    (-\Delta)\|\mathbf{x}\|^p=\tilde{C}_p\|\mathbf{x}\|^{p-2},
    \label{simple_fact}
\end{equation}
where $\tilde{C}_p$ is a constant depends on $p$.

For $d\geq 3$, we can actually prove that $(-\Delta)^{(d+1)/2}\|\mathbf{x}\|^3=C_d(\Gamma(\mathbf{x}))$. We discuss the cases of odd $d$ and even $d$ separately. If $d$ is odd, we apply \eqref{simple_fact} for $(d+1)/2$ times and get
\begin{equation*}
\begin{aligned}
    (-\Delta)^{(d+1)/2}\|\mathbf{x}\|^3&=\bar{C}\|\mathbf{x}\|^{3-(d+1)}\\
    &=C_d(\Gamma(\mathbf{x})),
\end{aligned}
\end{equation*}
where $C_d$ and $\bar{C}$ are some constants. 

If $d$ is even, we apply \eqref{simple_fact} for $d/2$ times and get
\begin{equation*}
\begin{aligned}
    (-\Delta)^{(d+1)/2}\|\mathbf{x}\|^3&=\bar{C}(-\Delta)^{1/2}\|\mathbf{x}\|^{3-d}.
\end{aligned}
\end{equation*}
Then we only need to prove that $(-\Delta)^{1/2}\|\mathbf{x}\|^{3-d}=C\|\mathbf{x}\|^{2-d}$ for some constant $C$. Let $g(\mathbf{x})=(-\Delta)^{1/2}\|\mathbf{x}\|^{3-d}$. Since the fractional Laplacian can be written \gmm{as} a singular integral \citep{kwasnicki2017ten}, we have
\begin{equation*}
    \begin{aligned}
      g(\mathbf{x})&=C_1\int_{\mathbb{R}^d}\frac{\|\mathbf{x}\|^{3-d}-\|\mathbf{y}\|^{3-d}}{\|\mathbf{x}-\mathbf{y}\|^{d+1}}\mathrm{d}\mathbf{y},
    \end{aligned}
\end{equation*}
where $C_1$ is some constant. Since the fractional Laplacian of a radially symmetric function is also radially symmetric, we have that $g(\mathbf{x})$ is radially symmetric, which means $g(\mathbf{x})$ only depends on $\|\mathbf{x}\|$. For any positive number $k>0$, we have
\begin{equation*}
    \begin{aligned}
      g(k\mathbf{x})&=C_1\int_{\mathbb{R}^d}\frac{\|k\mathbf{x}\|^{3-d}-\|\mathbf{y}\|^{3-d}}{\|k\mathbf{x}-\mathbf{y}\|^{d+1}}\mathrm{d}\mathbf{y}\\
      &=C_1\int_{\mathbb{R}^d}k^d\cdot\frac{\|k\mathbf{x}\|^{3-d}-k\|\mathbf{y}\|^{3-d}}{\|k\mathbf{x}-k\mathbf{y}\|^{d+1}}\mathrm{d}\mathbf{y}\\
      &=C_1\int_{\mathbb{R}^d}k^{2-d}\cdot\frac{\|\mathbf{x}\|^{3-d}-\|\mathbf{y}\|^{3-d}}{\|\mathbf{x}-\mathbf{y}\|^{d+1}}\mathrm{d}\mathbf{y}\\
      &=k^{2-d}g(\mathbf{x}).\\
    \end{aligned}
\end{equation*}
Combining the above equation with the fact that  $g(\mathbf{x})$ is radially symmetric, we show that $g(\mathbf{x})=\|\mathbf{x}\|^{2-d}g(\frac{\mathbf{x}}{\|\mathbf{x}\|})=C\|\mathbf{x}\|^{2-d}$ for some constant $C$.

Now we have proved the lemma for $d\geq 3$. Next we consider the case when $d=2$. Since $(-\Delta)^{3/2}(\|\mathbf{x}-\mathbf{x}_1\|^3-\|\mathbf{x}-\mathbf{x}_2\|^3)=(-\Delta)^{1/2}(\|\mathbf{x}-\mathbf{x}_1\|-\|\mathbf{x}-\mathbf{x}_2\|)$, we only need to prove that $(-\Delta)^{1/2}(\|\mathbf{x}-\mathbf{x}_1\|-\|\mathbf{x}-\mathbf{x}_2\|)=C(\log\|\mathbf{x}-\mathbf{x}_1\|-\log\|\mathbf{x}-\mathbf{x}_2\|)$, where $C$ is a constant. Using the singular integral definition of fractional Laplacian, we get
\begin{equation*}
\begin{aligned}
    &(-\Delta)^{1/2}(\|\mathbf{x}-\mathbf{x}_1\|-\|\mathbf{x}-\mathbf{x}_2\|)\\
    =&C_1\int_{\mathbb{R}^d}\frac{\|\mathbf{x}-\mathbf{x}_1\|-\|\mathbf{x}-\mathbf{x}_2\|-\|\mathbf{y}-\mathbf{x}_1\|+\|\mathbf{y}-\mathbf{x}_2\|}{\|\mathbf{x}-\mathbf{y}\|^{3}}\mathrm{d}\mathbf{y}\\
    =&C_1\lim_{R\to\infty}\int_{B(\mathbf{x}_1,R)\cup B(\mathbf{x}_2,R)}\frac{\|\mathbf{x}-\mathbf{x}_1\|-\|\mathbf{x}-\mathbf{x}_2\|-\|\mathbf{y}-\mathbf{x}_1\|+\|\mathbf{y}-\mathbf{x}_2\|}{\|\mathbf{x}-\mathbf{y}\|^{3}}\mathrm{d}\mathbf{y}\\
    =&C_1\lim_{R\to\infty}\int_{B(\mathbf{x}_1,R)}\frac{\|\mathbf{x}-\mathbf{x}_1\|-\|\mathbf{y}-\mathbf{x}_1\|}{\|\mathbf{x}-\mathbf{y}\|^{3}}\mathrm{d}\mathbf{y}-C_1\lim_{R\to\infty}\int_{B(\mathbf{x}_2,R)}\frac{\|\mathbf{x}-\mathbf{x}_2\|-\|\mathbf{y}-\mathbf{x}_2\|}{\|\mathbf{x}-\mathbf{y}\|^{3}}\mathrm{d}\mathbf{y}\\
    +&C_1\lim_{R\to\infty}\int_{B(\mathbf{x}_2,R)\backslash B(\mathbf{x}_1,R)}\frac{\|\mathbf{x}-\mathbf{x}_1\|-\|\mathbf{y}-\mathbf{x}_1\|}{\|\mathbf{x}-\mathbf{y}\|^{3}}\mathrm{d}\mathbf{y}\\
    -&C_1\lim_{R\to\infty}\int_{B(\mathbf{x}_1,R)\backslash B(\mathbf{x}_2,R)}\frac{\|\mathbf{x}-\mathbf{x}_2\|-\|\mathbf{y}-\mathbf{x}_2\|}{\|\mathbf{x}-\mathbf{y}\|^{3}}\mathrm{d}\mathbf{y}.\\
\end{aligned}
\end{equation*}
Since for $\mathbf{y}\in B(\mathbf{x}_2,R)\backslash B(\mathbf{x}_1,R)$, we have $\|\mathbf{y}\|\geq R-\|\mathbf{x}_1\|$. And the area of $B(\mathbf{x}_2,R)\backslash B(\mathbf{x}_1,R)$ is at most $2R\|\mathbf{x}_1-\mathbf{x}_2\|$. So
\begin{equation*}
\begin{aligned}
    &\lim_{R\to\infty}\int_{B(\mathbf{x}_2,R)\backslash B(\mathbf{x}_1,R)}\left|\frac{\|\mathbf{x}-\mathbf{x}_1\|-\|\mathbf{y}-\mathbf{x}_1\|}{\|\mathbf{x}-\mathbf{y}\|^{3}}\right|\mathrm{d}\mathbf{y}\\
    \leq &\lim_{R\to\infty}2R\|\mathbf{x}_1-\mathbf{x}_2\|\cdot\frac{\|\mathbf{x}-\mathbf{x}_1\|+R+\|\mathbf{x}_1\|+\|\mathbf{x}_2\|}{(R-\|\mathbf{x}\|-\|\mathbf{x}_1\|)^{3}}\\
    =&0.
\end{aligned}
\end{equation*}
Similarly we have $\lim_{R\to\infty}\int_{B(\mathbf{x}_1,R)\backslash B(\mathbf{x}_2,R)}\frac{\|\mathbf{x}-\mathbf{x}_2\|-\|\mathbf{y}-\mathbf{x}_2\|}{\|\mathbf{x}-\mathbf{y}\|^{3}}\mathrm{d}\mathbf{y}=0$. Then we get
\begin{equation*}
\begin{aligned}
    &(-\Delta)^{1/2}(\|\mathbf{x}-\mathbf{x}_1\|-\|\mathbf{x}-\mathbf{x}_2\|)\\
    =&C_1\lim_{R\to\infty}\int_{B(\mathbf{x}_1,R)}\frac{\|\mathbf{x}-\mathbf{x}_1\|-\|\mathbf{y}-\mathbf{x}_1\|}{\|\mathbf{x}-\mathbf{y}\|^{3}}\mathrm{d}\mathbf{y}-C_1\lim_{R\to\infty}\int_{B(\mathbf{x}_2,R)}\frac{\|\mathbf{x}-\mathbf{x}_2\|-\|\mathbf{y}-\mathbf{x}_2\|}{\|\mathbf{x}-\mathbf{y}\|^{3}}\mathrm{d}\mathbf{y}\\
    =&C_1\lim_{R\to\infty}\int_{B(0,R)}\frac{\|\mathbf{x}-\mathbf{x}_1\|-\|\mathbf{y}\|}{\|\mathbf{x}-\mathbf{x}_1-\mathbf{y}\|^{3}}\mathrm{d}\mathbf{y}-C_1\lim_{R\to\infty}\int_{B(0,R)}\frac{\|\mathbf{x}-\mathbf{x}_2\|-\|\mathbf{y}\|}{\|\mathbf{x}-\mathbf{x}_2-\mathbf{y}\|^{3}}\mathrm{d}\mathbf{y}.\\
\end{aligned}
\end{equation*}
Let $f(\mathbf{x},R)=\int_{B(0,R)}\frac{\|\mathbf{x}\|-\|\mathbf{y}\|}{\|\mathbf{x}-\mathbf{y}\|^{3}}\mathrm{d}\mathbf{y}$. Then $(-\Delta)^{1/2}(\|\mathbf{x}-\mathbf{x}_1\|-\|\mathbf{x}-\mathbf{x}_2\|)=\lim_{R\to\infty} f(\mathbf{x}-\mathbf{x}_1,R)-f(\mathbf{x}-\mathbf{x}_2,R)$. Next we show that $f(\lambda\mathbf{x},\lambda R)=f(\mathbf{x},R)$ for any $\lambda>0$. Actually
\begin{equation*}
\begin{aligned}
    f(\lambda\mathbf{x},\lambda R)
    &=\int_{B(0,\lambda R)}\frac{\|\lambda\mathbf{x}\|-\|\mathbf{y}\|}{\|\lambda\mathbf{x}-\mathbf{y}\|^{3}}\mathrm{d}\mathbf{y}\\
    &=\int_{B(0,R)}\frac{\lambda\|\mathbf{x}\|-\lambda\|\mathbf{y}\|}{\|\lambda\mathbf{x}-\lambda\mathbf{y}\|^{3}}\lambda^d\mathrm{d}\mathbf{y}\\
    &=\int_{B(0,R)}\frac{\|\mathbf{x}\|-\|\mathbf{y}\|}{\|\mathbf{x}-\mathbf{y}\|^{3}}\mathrm{d}\mathbf{y}=f(\mathbf{x},R).
\end{aligned}
\end{equation*}
Also it is easy to see that $f(\mathbf{x},R)$ is radially symmetric over $\mathbf{x}$. So $f(\mathbf{x},R)=f(\|\mathbf{x}\|\mathbf{u},R)$ for any unit vector $\mathbf{u}\in\mathbb{R}^2$. Then we get
\begin{equation*}
    \begin{aligned}
      \lim_{R\to\infty} f(\mathbf{x}-\mathbf{x}_1,R)-f(\mathbf{x}-\mathbf{x}_2,R)&=\lim_{R\to\infty}f(\mathbf{u},\frac{R}{\|\mathbf{x}-\mathbf{x}_1\|})-f(\mathbf{u},\frac{R}{\|\mathbf{x}-\mathbf{x}_2\|})\\
      &=\lim_{R\to\infty}\int_{B(0,\frac{R}{\|\mathbf{x}-\mathbf{x}_1\|})\backslash B(0,\frac{R}{\|\mathbf{x}-\mathbf{x}_2\|}) }\frac{\|\mathbf{u}\|-\|\mathbf{y}\|}{\|\mathbf{u}-\mathbf{y}\|^{3}}\mathrm{d}\mathbf{y}\\
       &=\lim_{R\to\infty}\int_{B(0,\frac{R}{\|\mathbf{x}-\mathbf{x}_1\|})\backslash B(0,\frac{R}{\|\mathbf{x}-\mathbf{x}_2\|}) }\frac{-\|\mathbf{y}\|}{\|-\mathbf{y}\|^{3}}\mathrm{d}\mathbf{y}\\
       &=-\lim_{R\to\infty}\int_{[\frac{R}{\|\mathbf{x}-\mathbf{x}_2\|},\frac{R}{\|\mathbf{x}-\mathbf{x}_1\|}] }\frac{2\pi}{r}\mathrm{d}r\\
       &=-2\pi\lim_{R\to\infty}\log\frac{R}{\|\mathbf{x}-\mathbf{x}_1\|}-\log\frac{R}{\|\mathbf{x}-\mathbf{x}_2\|}\\
       &=2\pi(\log\|\mathbf{x}-\mathbf{x}_1\|-\log\|\mathbf{x}-\mathbf{x}_2\|).
    \end{aligned}
\end{equation*}
So we proved the lemma for case $d=2$.
\end{proof}
The problem \eqref{function_space_multi} is over the Lipschitz continuous function space, which is hard to analyse. The following Lemma shows that we can consider the optimization problem over $ -\Delta h$.
\begin{lemma}
\label{on_Delta_h}
Suppose $h(\mathbf{x})$ is the solution of the variational problem \eqref{function_space_multi}. Then there exist $\mathbf{u}\in\mathbb{R}^{d}, v\in\mathbb{R}$ such that $(-\Delta h(\mathbf{x}),\mathbf{u},v)$ is the solution of the following variational problem:
\begin{equation}
\begin{aligned}
 \min_{\substack{f\in  C(\mathbb{R}^d),\\ \mathbf{u}\in\mathbb{R}^{d}, v\in\mathbb{R}}}\quad & \int_{\operatorname{supp}(\zeta)} \frac{\left({\mathcal{R}\{(-\Delta)^{(d-1)/2}f\}(\bm{V},c)}\right)^2}{\zeta(\bm{V},c)}~\mathrm{d}\sigma^{d-1}(\bm{V})\mathrm{d}c\\
 \textup{subject to}\quad & \int_{\mathbb{R}^d}f(\mathbf{s})\left[\Gamma(\mathbf{x}_j-\mathbf{s})-\Gamma(-\mathbf{s})-\langle\mathbf{x}_j,\nabla\Gamma(-\mathbf{s})\rangle\right]\mathrm{d}\mathbf{s}+\langle \mathbf{u},\mathbf{x}_j \rangle +v=y_j,\quad j=1,\ldots,M,\\
 & {\mathcal{R}\{(-\Delta)^{(d-1)/2}f\}(\bm{V},c)}=0, \ \forall(\bm{V},c)\not\in \operatorname{supp}(\zeta),\\
 &(-\Delta)^{(d-1)/2}f \in L^p(\mathbb{R}^d),\ 1\leq p<d/(d-1),\\
 &\sup_{\mathbf{x}\in \mathbb{R}^d} \|\mathbf{x}\|\cdot|f(\mathbf{x})|<\infty,
\end{aligned}
\label{function_space_multi_Laplace}
\end{equation}
where $\Gamma(\mathbf{x})$ is the fundamental solution of the Laplace equation $-\Delta\Gamma(\mathbf{x})=\delta(\mathbf{x})$. The closed form of $\Gamma(\mathbf{x})$ is
\begin{equation*}
    \Gamma(\mathrm{x})=
    \begin{cases}
    -\frac{1}{2\pi}\log\|\mathbf{x}\|, & d=2,\\
    \frac{1}{d(d-2)V_d\|\mathbf{x}\|^{d-2}}, & d\geq 3,
    \end{cases}
\end{equation*}
where $V_d$ is the volume of the unit ball in $\mathbb{R}^d$.
\end{lemma}
\begin{proof}[Proof of Lemma \ref{on_Delta_h}]
First we prove that $\sup_{\mathbf{x}\in \mathbb{R}^d} \|\mathbf{x}\|\cdot|-\Delta h|<\infty$. According to Lemma \ref{tight_support} and Lemma \ref{recover_lemma}, we have $-\Delta h=\mathcal{R}^*\{\psi\}$, where $\psi$ is tightly supported. Then \citep[Corollary~3.6]{solmon1987asymptotic} shows that $\mathcal{R}^*\{\psi\}=O(\|\mathbf{x}\|^{-1})$, which gives that $\sup_{\mathbf{x}\in \mathbb{R}^d} \|\mathbf{x}\|\cdot|-\Delta h|<\infty$.

Now it is sufficient to prove that for any $\overline{h}\in\operatorname{Lip}(\mathbb{R}^d)$ satisfying that $-\Delta\overline{h}\in  C(\mathbb{R}^d)$ and $\sup_{\mathbf{x}\in \mathbb{R}^d} \|\mathbf{x}\|\cdot|-\Delta\overline{h}(\mathbf{x})|<\infty$, there exist $\mathbf{u}\in\mathbb{R}^{d}, v\in\mathbb{R}$ such that
\begin{equation*}
    \int_{\mathbb{R}^d}[-\Delta\overline{h}(\mathbf{s})]\left[\Gamma(\mathbf{x}-\mathbf{s})-\Gamma(-\mathbf{s})-\langle\mathbf{x},\nabla\Gamma(-\mathbf{s})\rangle\right]\mathrm{d}\mathbf{s}+\langle \mathbf{u},\mathbf{x} \rangle +v=\overline{h}(\mathbf{x}).
\end{equation*}
Let $\overline{g}(\mathbf{x})=\int_{\mathbb{R}^d}[-\Delta\overline{h}(\mathbf{s})]\left[\Gamma(\mathbf{x}-\mathbf{s})-\Gamma(-\mathbf{s})-\langle\mathbf{x},\nabla\Gamma(-\mathbf{s})\rangle\right]\mathrm{d}\mathbf{s}$. First we show that $\overline{g}(\mathbf{x})$ is well-defined. Since $\int_{\|\mathbf{s}\|\geq 1}\|\mathbf{s}\|^{-(d+1)}\mathrm{d}\mathbf{s}<\infty$, we only need to prove that $\Gamma(\mathbf{x}-\mathbf{s})-\Gamma(-\mathbf{s})-\langle\mathbf{x},\nabla\Gamma(-\mathbf{s})\rangle=O(\|\mathbf{s}\|^{-d})$ as $\|\mathbf{s}\|\to \infty$ for any given $\mathbf{x}$. Using Taylor's expansion, we have 
\begin{equation}
\begin{aligned}
  \Gamma(\mathbf{x}-\mathbf{s})-\Gamma(-\mathbf{s})-\langle\mathbf{x},\nabla\Gamma(-\mathbf{s})\rangle&=\mathbf{x}^{T}H_{\Gamma}(c\mathbf{x}-\mathbf{s})\mathbf{x} \text{ for some }c\in[0,1],\label{Taylor}
\end{aligned}
\end{equation}
where $H_{\Gamma}$ is the Hessian matrix of $\Gamma$. Since
\begin{equation}
\begin{aligned}
  \frac{\partial \Gamma}{\partial s_i\partial s_j}(\mathbf{s})=-\frac{\delta_{ij}\|\mathbf{s}\|^2-ds_is_j}{dV_d\|\mathbf{s}\|^{d+2}}=O(\|\mathbf{s}\|^{-d}),
  \label{Hessian_asymptotic}
\end{aligned}
\end{equation}
where $\delta_{ij}=1$ when $i=j$, and $\delta_{ij}=0$ otherwise. According to \eqref{Taylor} we have $\Gamma(\mathbf{x}-\mathbf{s})-\Gamma(-\mathbf{s})-\langle\mathbf{x},\nabla\Gamma(-\mathbf{s})\rangle=O(\|\mathbf{s}\|^{-d})$ as $\|\mathbf{s}\|\to \infty$. Then we proved that $\overline{g}(\mathbf{x})$ is well-defined.

Next we prove that $\|\nabla \overline{g}(\mathbf{x})\|=O(\log\|\mathbf{x}\|)$. We only need to consider the large enough $\mathbf{x}$. Suppose $\|\mathbf{x}\|\geq 2$. The partial derivative of $\overline{g}$ is given by
\begin{equation}
\label{partial_g}
\begin{aligned}
  \frac{\partial \overline{g}}{\partial x_i}(\mathbf{x})&=\int_{\mathbb{R}^d}-\frac{1}{dV_d}[-\Delta\overline{h}(\mathbf{s})]\left[\frac{x_i-s_i}{\|\mathbf{x}-\mathbf{s}\|^d}+\frac{s_i}{\|\mathbf{s}\|^d}\right]\mathrm{d}\mathbf{s}.\\
\end{aligned}
\end{equation}
Since $\sup_{\mathbf{x}\in \mathbb{R}^d} \|\mathbf{x}\|\cdot|-\Delta\overline{h}(\mathbf{x})|<\infty$, we have $\|-\Delta\overline{h}(\mathbf{x})\|\leq C\cdot\min\{1,\frac{1}{\mathbf{x}})\}$ for some constant $C$. It is easy to see that the integrand of \eqref{partial_g} is $O(\|\mathbf{s}\|^{d+1})$. So $|\frac{\partial \overline{g}}{\partial x_i}(\mathbf{x})|<\infty$. Next we estimate the integral \eqref{partial_g} on $\mathbb{R}^d\backslash B(0,\|\mathbf{x}\|/2)$:
\begin{equation}
    \begin{aligned}
      &\left|\int_{\|\mathbf{s}\|>\|\mathbf{x}\|/2}[-\Delta\overline{h}(\mathbf{s})]\left[\frac{x_i-s_i}{\|\mathbf{x}-\mathbf{s}\|^d}+\frac{s_i}{\|\mathbf{s}\|^d}\right]\mathrm{d}\mathbf{s}\right|\\
      \leq&\left|\int_{\|\mathbf{s}\|>\|\mathbf{x}\|/2}\frac{1}{\|\mathbf{s}\|}\left[\frac{x_i-s_i}{\|\mathbf{x}-\mathbf{s}\|^d}+\frac{s_i}{\|\mathbf{s}\|^d}\right]\mathrm{d}\mathbf{s}\right|\\
      =&\left|\int_{\|\mathbf{s}\|>1/2}\frac{1}{\|\mathbf{s}\|}\left[\frac{x_i/\|\mathbf{x}\|-s_i}{\|\mathbf{x}/\|\mathbf{x}\|-\mathbf{s}\|^d}+\frac{s_i}{\|\mathbf{s}\|^d}\right]\mathrm{d}\mathbf{s}\right|\\
      \leq&\max_{\|\widehat{\mathbf{x}}\|=1}\left|\int_{\|\mathbf{s}\|>1/2}\frac{1}{\|\mathbf{s}\|}\left[\frac{\widehat{x}_i-s_i}{\|\widehat{\mathbf{x}}-\mathbf{s}\|^d}+\frac{s_i}{\|\mathbf{s}\|^d}\right]\mathrm{d}\mathbf{s}\right|.\\
      \label{exterior}
    \end{aligned}
\end{equation}
Since $\left|\int_{\|\mathbf{s}\|>1/2}\frac{1}{\|\mathbf{s}\|}\left[\frac{\widehat{x}_i-s_i}{\|\widehat{\mathbf{x}}-\mathbf{s}\|^d}+\frac{s_i}{\|\mathbf{s}\|^d}\right]\mathrm{d}\mathbf{s}\right|$ is well-defined and continuous function over $\widehat{\mathbf{x}}$. Then $\max_{\|\widehat{\mathbf{x}}\|=1}\left|\int_{\|\mathbf{s}\|>1/2}\frac{1}{\|\mathbf{s}\|}\left[\frac{\widehat{x}_i-s_i}{\|\widehat{\mathbf{x}}-\mathbf{s}\|^d}+\frac{s_i}{\|\mathbf{s}\|^d}\right]\mathrm{d}\mathbf{s}\right|$ is a finite number.

Next we estimate the integral \eqref{partial_g} on $B(0,\|\mathbf{x}\|/2)$:
\begin{equation}
    \begin{aligned}
      &\left|\int_{\|\mathbf{s}\|\leq\|\mathbf{x}\|/2}[-\Delta\overline{h}(\mathbf{s})]\left[\frac{x_i-s_i}{\|\mathbf{x}-\mathbf{s}\|^{d-1}}+\frac{s_i}{\|\mathbf{s}\|^d}\right]\mathrm{d}\mathbf{s}\right|\\
      \leq&\left|\int_{\|\mathbf{s}\|\leq 1}C\left[\frac{1}{\|\mathbf{x}-\mathbf{s}\|^d}+\frac{1}{\|\mathbf{s}\|^{d-1}}\right]\mathrm{d}\mathbf{s}\right|+\left|\int_{1<\|\mathbf{s}\|\leq\|\mathbf{x}\|/2}\frac{C}{\|\mathbf{s}\|}\left[\frac{1}{\|\mathbf{x}-\mathbf{s}\|^{d-1}}+\frac{1}{\|\mathbf{s}\|^{d-1}}\right]\mathrm{d}\mathbf{s}\right|\\
      \leq&\left|\int_{\|\mathbf{s}\|\leq 1}C\left[\frac{2^d}{\|\mathbf{x}\|^d}+\frac{1}{\|\mathbf{s}\|^{d-1}}\right]\mathrm{d}\mathbf{s}\right|+\left|\int_{1<\|\mathbf{s}\|\leq\|\mathbf{x}\|/2}\frac{C}{\|\mathbf{s}\|}\left[\frac{2^d}{\|\mathbf{x}\|^{d-1}}+\frac{1}{\|\mathbf{s}\|^{d-1}}\right]\mathrm{d}\mathbf{s}\right|\\
      \leq&C_1+\frac{2^dC}{\|\mathbf{x}\|^{d-1}}\left|\int_{1<\|\mathbf{s}\|\leq\|\mathbf{x}\|/2}\frac{1}{\|\mathbf{s}\|}\mathrm{d}\mathbf{s}\right|+C\left|\int_{1<\|\mathbf{s}\|\leq\|\mathbf{x}\|/2}\frac{1}{\|\mathbf{s}\|^{d}}\mathrm{d}\mathbf{s}\right|\\
      \leq&C_1+\frac{2^dC_2}{\|\mathbf{x}\|^{d-1}}\|\mathbf{x}\|^{d-1}+C_3\log \|\mathbf{x}\|\\
      \leq&C_4+C_3\log \|\mathbf{x}\|,
    \end{aligned}
    \label{interior}
\end{equation}
where $C_1$, $C_2$, $C_3$ and $C_4$ are some constants. Combining \eqref{exterior} and \eqref{interior} we proved that $\|\nabla \overline{g}(\mathbf{x})\|=O(\log\|\mathbf{x}\|)$.

In our last step, we prove that $\overline{g}-\overline{h}$ is linear. Because of the property of the fundamental solution, we have $-\Delta(\overline{g}-\overline{h})\equiv 0$. Since $\overline{h}$ is Lipschitz continuous and $\|\nabla \overline{g}(\mathbf{x})\|=O(\log\|\mathbf{x}\|)$, we have $\nabla(\overline{g}-\overline{h})=O(\log\|\mathbf{x}\|)$. So we can regard $\overline{g}-\overline{h}$ as a tempered distribution. Using the proof technique of Lemma \ref{Delta_s_lipschitz}, we have that $\overline{g}-\overline{h}$ is a polynomial. Since $\nabla(\overline{g}-\overline{h})=O(\log\|\mathbf{x}\|)$, $\overline{g}-\overline{h}$ must be a linear function, which gives the claim.
\end{proof}

\begin{proof}[Proof of Theorem \ref{closed_form_sol}]
To simplify the proof, we let $f(\mathbf{x}, \theta_0)\equiv 0$.
The analysis still holds without this simplification.
Let $h(\mathbf{x})$ be the solution of \eqref{function_space_multi_simple}. Then %
Lemma \ref{on_Delta_h} tell us that there exist $\mathbf{u}\in\mathbb{R}^{d}, v\in\mathbb{R}$ such that $(-\Delta h(\mathbf{x}),\mathbf{u},v)$ is the solution of the following variational problem:
\begin{equation}
\begin{aligned}
 \min_{\substack{f\in  C(\mathbb{R}^d),\\ \mathbf{u}\in\mathbb{R}^{d}, v\in\mathbb{R}}}\quad & \int_{\mathbb{R}^{d}} \left((-\Delta)^{(d-1)/4}f(\mathbf{x})\right)^2~\mathrm{d}\mathbf{x}\\
 \textup{subject to}\quad & \int_{\mathbb{R}^d}f(\mathbf{s})\left[\Gamma(\mathbf{x}_j-\mathbf{s})-\Gamma(-\mathbf{s})-\langle\mathbf{x}_j,\nabla\Gamma(-\mathbf{s})\rangle\right]\mathrm{d}\mathbf{s}+\langle \mathbf{u},\mathbf{x}_j \rangle +v=y_j,\quad j=1,\ldots,M\\
 &(-\Delta)^{(d-1)/2}f \in L^p(\mathbb{R}^d),\ 1\leq p<d/(d-1)\\
 &\sup_{\mathbf{x}\in \mathbb{R}^d} \|\mathbf{x}\|\cdot|f(\mathbf{x})|<\infty,
\end{aligned}
\label{function_space_multi_Laplace_simplified}
\end{equation}
Suppose that $f(\mathbf{x})$ is the solution of \eqref{function_space_multi_Laplace_simplified}. Let $J(f,\mathbf{u},v)=\int_{\mathbb{R}^{d}} \left((-\Delta)^{(d-1)/4}f(\mathbf{x})\right)^2~\mathrm{d}\mathbf{x}$ and $G_j(f,\mathbf{u},v)=\int_{\mathbb{R}^d}f(\mathbf{s})\left[\Gamma(\mathbf{x}_j-\mathbf{s})-\Gamma(-\mathbf{s})-\langle\mathbf{x}_j,\nabla\Gamma(-\mathbf{s})\rangle\right]\mathrm{d}\mathbf{s}+\langle \mathbf{u},\mathbf{x}_j \rangle +v$. For any function $\varphi$ in Schwartz space $\mathcal{S}(\mathbb{R}^d)$,\footnote{The Schwartz functions on $\mathbb{R}^d$ is the function space $\mathcal{S}(\mathbb{R}^d)=\{f\in\mathcal{C}^\infty(\mathbb{R}^d):\forall \alpha,\beta\in\mathbb{N}^d, \sup_{\mathbf{x}\in\mathbb{R}^d}\mathbf{x}^\alpha(D^\beta f)(\mathbf{x})<\infty\}$, where $\alpha$ and $\beta$ are multi-indices.} $\tilde{\mathbf{u}}\in\mathbb{R}^d$ and $\tilde{v}\in\mathbb{R}$, we consider the perturbation $(\epsilon\varphi,\epsilon\tilde{\mathbf{u}},\epsilon\tilde{v})$ to the solution $(-\Delta h,\mathbf{u},v)$. It is easy to verify that $-\Delta h+\epsilon\varphi$ satisfies that $(-\Delta)^{(d-1)/2}(-\Delta h+\epsilon\varphi) \in L^p(\mathbb{R}^d),\ 1\leq p<d/(d-1)$ and $\sup_{\mathbf{x}\in \mathbb{R}^d} \|\mathbf{x}\|\cdot|(-\Delta h+\epsilon\varphi)(\mathbf{x})|<\infty$. Next we have
\begin{equation*}
\begin{aligned}
      \frac{\mathrm{d}}{\mathrm{d}\epsilon} J(-\Delta h+\epsilon\varphi,\mathbf{u}+\epsilon\tilde{\mathbf{u}},v+\epsilon\tilde{v})&=2\int_{\mathbb{R}^{d}} \left((-\Delta)^{(d-1)/4}(-\Delta h)\right)\left((-\Delta)^{(d-1)/4}\varphi)\right)~\mathrm{d}\mathbf{x}\\
      &=2\int_{\mathbb{R}^{d}} \varphi\cdot\left((-\Delta)^{(d-1)/2}(-\Delta h))\right)~\mathrm{d}\mathbf{x},
\end{aligned}
\end{equation*}
The last equality holds because $\varphi\in\mathcal{S}(\mathbb{R}^d)$. Also we have
\begin{equation*}
\begin{aligned}
      \frac{\mathrm{d}}{\mathrm{d}\epsilon} G_j(-\Delta h+\epsilon\varphi,\mathbf{u}+\epsilon\tilde{\mathbf{u}},v+\epsilon\tilde{v})&=\int_{\mathbb{R}^d}\varphi(\mathbf{s})\left[\Gamma(\mathbf{x}_j-\mathbf{s})-\Gamma(-\mathbf{s})-\langle\mathbf{x}_j,\nabla\Gamma(-\mathbf{s})\rangle\right]\mathrm{d}\mathbf{s}+\langle \tilde{\mathbf{u}},\mathbf{x}_j \rangle +\tilde{v}.
\end{aligned}
\end{equation*}
Then according to the first-order optimality condition, there are scalars $\bar{\lambda}_1,\ldots,\bar{\lambda}_M$ such that
\begin{equation*}
\begin{cases}
      (-\Delta)^{(d-1)/2}(-\Delta h(\mathbf{x}))=\sum_{j=1}^M \bar{\lambda}_j\left[\Gamma(\mathbf{x}_j-\mathbf{x})-\Gamma(-\mathbf{x})-\langle\mathbf{x}_j,\nabla\Gamma(-\mathbf{x})\rangle\right]\\
      \sum_{j=1}^M \bar{\lambda}_j=0\\
      \sum_{j=1}^M \bar{\lambda}_j\mathbf{x}_j=\mathbf{0}
    \end{cases}\,,
\end{equation*}
which can be simplified to 
\begin{equation}
\begin{cases}
      (-\Delta)^{(d+1)/2}h(\mathbf{x})=\sum_{j=1}^M \bar{\lambda}_j\left[\Gamma(\mathbf{x}-\mathbf{x}_j)-\Gamma(\mathbf{x})\right]\\
      \sum_{j=1}^M \bar{\lambda}_j=0\\
      \sum_{j=1}^M \bar{\lambda}_j\mathbf{x}_j=\mathbf{0}
    \end{cases}\,.
    \label{optimality_condition}
\end{equation}
According to Lemma~\ref{fractional_Lap_radial_power_function} and Lemma~\ref{Delta_s_lipschitz}, we can find out $\mathbf{u}\in\mathbb{R}^{d}, v\in\mathbb{R}$ such that 
\begin{equation*}
    \begin{aligned}
      h(\mathbf{x})&=\frac{1}{C_d}\sum_{j=1}^M \bar{\lambda}_j\left[\|\mathbf{x}-\mathbf{x}_j\|^3-\|\mathbf{x}\|^3\right]+\langle \mathbf{u},\mathbf{x} \rangle +v\\
      &=\frac{1}{C_d}\sum_{j=1}^M \bar{\lambda}_j\|\mathbf{x}-\mathbf{x}_j\|^3+\langle \mathbf{u},\mathbf{x} \rangle +v,
    \end{aligned}
\end{equation*}
which gives \eqref{exact_solution_h} after substituting $\frac{\bar{\lambda}_j}{C_d}$ by $\lambda_j$. Since $h(\mathbf{x})$ should fit all training data and $\lambda_j$ should satisfy \eqref{optimality_condition}, the coefficients $\lambda_j$, $\mathbf{u}$ and $v$ satisfy \eqref{coefficients}. Now $h(\mathbf{x})$ satisfies the first-order optimality condition and fits all training data. Since the variational problem \eqref{function_space_multi_less_constraint_copy2} is convex, we only need to check that  $h\in\operatorname{Lip}(\mathbb{R}^d)$ and $(-\Delta)^{(d+1)/2}h \in L^p(\mathbb{R}^d),\ 1\leq p<d/(d-1)$ then we can conclude that $h(\mathbf{x})$ is the solution of \eqref{function_space_multi_less_constraint_copy2}. Using \eqref{Taylor}, we have
\begin{equation*}
    \begin{aligned}
      (-\Delta)^{(d+1)/2}h(\mathbf{x})&=\sum_{j=1}^M \bar{\lambda}_j\left[\Gamma(\mathbf{x}_j-\mathbf{x})-\Gamma(-\mathbf{x})-\langle\mathbf{x}_j,\nabla\Gamma(-\mathbf{x})\rangle\right]\\
      &=\sum_{j=1}^M\bar{\lambda}_j\mathbf{x}_j^{T}H_{\Gamma}(c\mathbf{x}_j-\mathbf{x})\mathbf{x}_j \text{ for some }c\in[0,1].\\
    \end{aligned}
\end{equation*}
According to \eqref{Hessian_asymptotic}, we get that $(-\Delta)^{(d+1)/2}h(\mathbf{x})=O(\|\mathbf{x}\|^{-d})$. We set $p=(d+1)/d$ which satisfies $1 \leq p<d/(d-1)$. It is easy to verify that $\int_{B(\mathbf{x}_j,\epsilon)}\Gamma^p(\mathbf{x}_j-\mathbf{x}) \mathrm{d}\mathbf{x}$ is integrable for small enough $\epsilon$ and $\int_{\mathbb{R}^d\backslash B(0,1)}\|\mathbf{x}\|^{-pd} \mathrm{d}\mathbf{x}$ is integrable. Then $(-\Delta)^{(d+1)/2}h \in L^p(\mathbb{R}^d)$.

Similarly we have
\begin{equation*}
    \begin{aligned}
      h(\mathbf{x})&=\sum_{j=1}^M \bar{\lambda}_j\left[\|\mathbf{x}_j-\mathbf{x}\|^3-\|-\mathbf{x}\|^3-\langle\mathbf{x}_j,\nabla(\|\cdot\|^3)(-\mathbf{x})\rangle\right]\\
      &=\sum_{j=1}^M\bar{\lambda}_j\mathbf{x}_j^{T}H_{\|\cdot\|^3}(c\mathbf{x}_j-\mathbf{x})\mathbf{x}_j \text{ for some }c\in[0,1],
    \end{aligned}
\end{equation*}
where $H_{\|\cdot\|^3}$ is the Hessian matrix of $\|\mathbf{x}\|^3$. As $\|\mathbf{x}\|\to\infty$, we have
\begin{equation}
\begin{aligned}
  \frac{\partial \|\cdot\|^3}{\partial x_i\partial x_j}(\mathbf{x})=3\delta_{ij}\|\mathbf{x}\|-3\frac{x_ix_j}{\|\mathbf{x}\|}=O(\|\mathbf{x}\|),
  \label{Hessian_asymptotic_cube}
\end{aligned}
\end{equation}
where $\delta_{ij}=1$ when $i=j$, and $\delta_{ij}=0$ otherwise. Then we have $h\in\operatorname{Lip}(\mathbb{R}^d)$.
\end{proof}
\subsection{Explicit Form of the Curvature Penalty Function} 
\label{proof_multi_curvature_penalty_function}
\begin{proof}[Proof of Proposition \ref{constant_zeta}]
Since $\bm{\mathcal{W}}\sim U(\mathbb{S}^{d-1})$, we have that $p_{\bm{\mathcal{V}}}(\bm{V})$ is constant over $\mathbb{S}^{d-1}$ and $\mathbb{E}(\mathcal{U}^2|\bm{\mathcal{V}}=\bm{V},\mathcal{C}=c)=1$ because $U=\|\bm{\mathcal{W}}\|=1$. Since $\mathcal{B}\sim U(-a,a)$ and $\bm{\mathcal{W}}$ and $\mathcal{B}$ are independent, we have $p_{\mathcal{C}|\bm{\mathcal{V}}=\bm{V}}(c)=\frac{1}{2a}\mathbbm{1}_{[-a,a]}(c)$. Then we get
\begin{equation*}
\begin{aligned}
\zeta(\bm{V},c) &= p_{\mathcal{C}|\bm{\mathcal{V}}
=\bm{V}}(c)~p_{\bm{\mathcal{V}}}(\bm{V})\mathbb{E}(\mathcal{U}^2|\bm{\mathcal{V}}=\bm{V},\mathcal{C}=c)\\
&=C_1\mathbbm{1}_{[-a,a]}(c),
\end{aligned}
\end{equation*}
where $C_1$ is a constant.
\end{proof}
\begin{proof}[Proof of Proposition~\ref{proposition:form-rho-2D}]
Let $p_{\bm{\mathcal{W}},\mathcal{B}}$ and $p_{\mathcal{U},\bm{\mathcal{V}},\mathcal{C}}$ denote the joint density functions of $(\bm{\mathcal{W}},\mathcal{B})$ and $(\mathcal{U},\bm{\mathcal{V}},\mathcal{C})$, respectively. We have
\begin{equation*}
    p_{\mathcal{U},\bm{\mathcal{V}},\mathcal{C}}(u,\bm{V},c)
    =\left|\frac{\partial(u\bm{V},-uc)}{\partial(u,\bm{V},c)}\right|p_{\bm{\mathcal{W}},\mathcal{B}}(u\bm{V},-uc) 
    =u^{d}p_{\bm{\mathcal{W}},\mathcal{B}}(u\bm{V},-uc) , 
\end{equation*}
and 
\begin{equation}
\label{rho_2d}
\begin{aligned}
&p_{\mathcal{C}|\bm{\mathcal{V}}=\bm{V}}(c)~p_{\bm{\mathcal{V}}}(\bm{V})\mathbb{E}(\mathcal{U}^2|\bm{\mathcal{V}}=\bm{V},\mathcal{C}=c)\\
=&p_{\mathcal{C}|\bm{\mathcal{V}}=\bm{V}}(c)~p_{\bm{\mathcal{V}}}(\bm{V})\cdot\int_\mathbb{R^+}u^2p_{\mathcal{U}|\bm{\mathcal{V}}=\bm{V},\mathcal{C}=c}(u)~\mathrm{d}u \\
=&\int_\mathbb{R^+}u^2p_{\mathcal{U},\bm{\mathcal{V}},\mathcal{C}}(u,\bm{V},c)~\mathrm{d}u \\
=&\int_\mathbb{R^+}u^{d+2}p_{\bm{\mathcal{W}},\mathcal{B}}(u\bm{V},-uc)~\mathrm{d}u . 
\end{aligned}
\end{equation}
\end{proof}

\begin{proof}[Proof of Theorem~\ref{proposition:explicit-rho-gaussian-2d}]
Using \eqref{rho_2d}, we have
\begin{equation*}
\begin{aligned}
\zeta(\bm{V},c)&=\int_\mathbb{R^+}u^{d+2}p_{\bm{\mathcal{W}},\mathcal{B}}(u\bm{V},-uc)~\mathrm{d}u \\
&=\int_\mathbb{R^+}u^{d+2}\frac{1}{\sqrt{(2\pi)^d}\sigma_w^d}e^{-\frac{\|u\bm{V}\|_2^2}{2\sigma_w^2}}\frac{1}{\sqrt{2\pi}\sigma_b}e^{-\frac{u^2c^2}{2\sigma_b^2}}~\mathrm{d}u\\ 
&=\frac{1}{(2\pi)^{(d+1)/2}\sigma_w^d\sigma_b}\int_\mathbb{R^+}u^{d+2}e^{-(\frac{1}{2\sigma_w^2}+\frac{c^2}{2\sigma_b^2})u^2}\mathrm{d}u.
\end{aligned}
\end{equation*}
Let $\sigma^2=1/\left(\frac{1}{\sigma_w^2}+\frac{c^2}{\sigma_b^2}\right)$, then we have
\begin{equation*}
\begin{aligned}
\zeta(\bm{V},c)
&=\frac{\sigma}{(2\pi)^{d/2}\sigma_w^d\sigma_b}\int_\mathbb{R^+}u^{d+2}\frac{1}{\sqrt{2\pi}\sigma}e^{-\frac{u^2}{2\sigma^2}}\mathrm{d}u\\
&=\frac{\sigma}{(2\pi)^{d/2}\sigma_w^d\sigma_b}\sigma^{d+2}\cdot 2^{d/2} \cdot \frac{\Gamma(\frac{d+3}{2})}{\sqrt{\pi}}\\
&=\frac{\sigma^{d+3}}{\pi^{(d+1)/2}\sigma_w^d\sigma_b}\Gamma(\frac{d+3}{2})\\
&=\frac{1}{\pi^{(d+1)/2}\sigma_w^d\sigma_b\left(\frac{1}{\sigma_w^2}+\frac{c^2}{\sigma_b^2}\right)^{(d+3)/2}}\Gamma(\frac{d+3}{2})\\
&=\frac{\sigma_w^3\sigma_b^{d+2}}{\pi^{(d+1)/2}\left(\sigma_b^2+c^2\sigma_w^2\right)^{(d+3)/2}}\Gamma(\frac{d+3}{2}).\\
\end{aligned}
\end{equation*}
\end{proof}

\section{Other Activation Functions for Univariate Regression}
\label{Other_activation}
We have focused on networks with ReLUs. 
The ReLU is special in that the second derivative of ReLU is a delta function. 
For other activation functions the variational problem on function space will look different. 

The paper by \cite{Parhi2019MinimumN} considers different types of activation functions $\sigma$. These are then related to different types of linear operators $\mathrm{L}$ in the definition of the smoothness regularizer. Here $\mathrm{L}$ and $\sigma$ satisfy $\mathrm{L}\sigma=\delta$, i.e., $\sigma$ is a Green’s function of $\mathrm{L}$. Suppose $\sigma$ is homogeneous. Then \cite{Parhi2019MinimumN} show that minimizing the weight ``norm''\footnote{Here the form of ``norm'' depends on the degree of homogeneity of the activation $\sigma$. We use quotation marks because it is a generalized notion of norm which may not satisfy the property of a norm.} of two-layer neural networks with activation function $\sigma$ is actually minimizing 1-norm of $\mathrm{L}f$ where $f$ is the output function of the neural network.

The approach in \cite{Parhi2019MinimumN} can be combined with our analysis. So if for example we replace the ReLU by another homogeneous activation, we can replace the operator accordingly and get an analogous result. 
\begin{proof}[Proof of Corollary \ref{cor:diff_activation}]
Use the same notation as in Section \ref{2.4}, and let $\sigma$ be the activation function, where we assume that $\sigma$ is a Green’s function of a linear operator $\mathrm{L}$. Then optimization problem \eqref{probablity_version} becomes:
\begin{equation}
\begin{aligned}
 \min_{\alpha_n\in C(\mathbb{R}^2)}\quad & \int_{\mathbb{R}^2} \alpha_n^2(W^{(1)},b)~\mathrm{d}\mu_n(W^{(1)},b)\\
 \textup{subject to}\quad & \int_{\mathbb{R}^2} \alpha_n(W^{(1)},b)\sigma( W^{(1)}x_j+b)~\mathrm{d}\mu_n(W^{(1)},b)=y_j,\quad j=1,\ldots,M. 
\end{aligned}
\label{probablity_version_activation}
\end{equation}
The limit of the problem \eqref{probablity_version_activation} as width $n\to \infty$ is 
\begin{equation}
\begin{aligned}
 \min_{\alpha\in C(\mathbb{R}^2)}\quad & \int_{\mathbb{R}^2} \alpha^2(W^{(1)},b)~\mathrm{d}\mu(W^{(1)},b)\\
 \textup{subject to}\quad & \int_{\mathbb{R}^2} \alpha(W^{(1)},b)\sigma( W^{(1)}x_j+b)~\mathrm{d}\mu(W^{(1)},b)=y_j,\quad j=1,\ldots,M. 
\end{aligned}
\label{continuous_version_activation}
\end{equation}
As in Section~\ref{sec:implicit_bias_univariate}, we can change the variables and relax the optimization problem \eqref{continuous_version_activation} to 
\begin{equation}
\begin{aligned}
 \min_{\substack{\gamma\in C(\mathbb{R}^2),\\ p\in C(\mathbb{R})} }\quad & \int_{\mathbb{R}^2} \gamma^2(W^{(1)},c)~\mathrm{d}\nu(W^{(1)},c)\\
 \textup{subject to}\quad & p(x_j)+\int_{\mathbb{R}^2} \gamma(W^{(1)},c)\sigma\left(W^{(1)}(x_j-c)\right)~\mathrm{d}\nu(W^{(1)},c)=y_j,\quad j=1,\ldots,M \\
 &\mathrm{L}~p\equiv 0 .
\end{aligned}
\label{continuous_version_activation_relax}
\end{equation}
If the activation function $\sigma$ is ReLU, $p$ is a linear function. Then \eqref{continuous_version_activation_relax} becomes the optimization problem \eqref{continuous_add_linear}. Define the output function $g$ of the neural network by
\[
g(x,(\gamma,p))=p(x)+\int_{\mathbb{R}^2} \gamma(W^{(1)},c)[W^{(1)}(%
  x-c)]_+~\mathrm{d}\nu(W^{(1)},c). 
 \]
Assume that the activation function $\sigma$ is homogeneous of degree $k$, i.e.,\ $\sigma(ax) = a^k
\sigma(x)$ for all $a>0$. Similar to \eqref{2nd_derivative-long}, we have
\begin{equation}
  \begin{aligned}
    (\mathrm{L}g)(x,(\gamma,p))
    &=\mathrm{L}\left(\int_{\mathbb{R}^2}\gamma(W^{(1)},c) \left|W^{(1)}\right|^k \sigma\left(\operatorname{sign}(W^{(1)})\cdot(x-c)\right)~\mathrm{d}\nu(W^{(1)},c)\right)\\
    &=\int_{\mathbb{R}^2}\gamma(W^{(1)},c) \left|W^{(1)}\right|^k \delta(x-c)~\mathrm{d}\nu(W^{(1)},c)\\
    &=\int_{\mathrm{supp}(\nu_\mathcal{C})}\left(\int_{\mathbb{R}}\gamma(W^{(1)},c) \left|W^{(1)}\right|^k  %
    ~\mathrm{d}\nu_{\mathcal{W}|\mathcal{C}=c}(W^{(1)})\right)
    \delta(x-c)~\mathrm{d}\nu_\mathcal{C}(c)\\
    &=\int_{\mathrm{supp}(\nu_\mathcal{C})}\left(\int_{\mathbb{R}}\gamma(W^{(1)},c)\left|W^{(1)}\right|^k~\mathrm{d}\nu_{\mathcal{W}|\mathcal{C}=c}(W^{(1)})\right)\delta(x-c) p_\mathcal{C}(c)\mathrm{d}c\\
    &=p_{\mathcal{C}}(x)\int_{\mathbb{R}}\gamma(W^{(1)},x)\left|W^{(1)}\right|^k~\mathrm{d}\nu_{\mathcal{W}|\mathcal{C}=x}(W^{(1)}) . 
  \end{aligned}
  \label{Lg}
\end{equation}
Then similar to Theorem \ref{theorem_func}, we show that the solution of \eqref{continuous_version_activation_relax} in function space actually solves the following optimization problem:
\begin{equation}
  \min_{h\in C^2(S)}
\int_{S}
\frac{\left((\mathrm{L}h)(x)\right)^2}{\zeta(x)}~\mathrm{d}x %
\quad 
\text{s.t.}\quad h(x_j)=y_j,\quad j=1,\ldots,m,
  \label{function_space_activation}
\end{equation}
where $\zeta(x) = p_\mathcal{C}(x)\mathbb{E}(\mathcal{W}^{2k}|\mathcal{C}=x) $ and $S = \operatorname{supp}(\zeta) \cap [\min_i x_i, \max_i x_i]$. Then Corollary \ref{cor:diff_activation} can be shown by using \eqref{function_space_activation} and the technique used in proof of Theorem  \ref{thm:theorem1}.
\end{proof}

\section{Effect of Linear Adjustment of the Training Data}
\label{Difference_between_solutions_of_variational_problems}

In this section, we show that the solution of the variational problem  with linearly adjusted training data \eqref{continuous_version_multi_relax} is close to the solution of training with the original training data \eqref{continuous_version}. 
This means that our characterization of the implicit bias in Theorem~\ref{thm:theorem1} gives a close description of the solution of gradient descent training with the original data set. 
The high level intuition is that fitting a linear function only requires a very small adjustment of the parameters of the network in comparison with the parameter adjustment needed to fit a non-linear function. 

For the reader's convenience, we restate the continuous version of the problem~\eqref{continuous_version}: 
\begin{equation}
\begin{aligned}
 \min_{\alpha\in C(\mathbb{R}^d\times \mathbb{R})}\quad & \int_{\mathbb{R}^d\times \mathbb{R}} \alpha^2(\mathbf{W}^{(1)},b)~\mathrm{d}\mu(\mathbf{W}^{(1)},b)\\
 \textup{subject to}\quad & \int_{\mathbb{R}^d\times \mathbb{R}} \alpha(\mathbf{W}^{(1)},b)[\langle \mathbf{W}^{(1)},\mathbf{x}_j\rangle+b]_+~\mathrm{d}\mu(\mathbf{W}^{(1)},b)=y_j,\quad j=1,\ldots,M,
\end{aligned}
\label{continuous_version_multi_appendix}
\end{equation}
and the linearly adjusted variational problem:%
\begin{equation}
\begin{aligned}
 \min_{\substack{\alpha\in C(\mathbb{R}^d\times \mathbb{R}),\\ \mathbf{u}\in\mathbb{R}^{d}, v\in\mathbb{R}}}\ & \int_{\mathbb{R}^d\times \mathbb{R}} \alpha^2(\mathbf{W}^{(1)},b)~\mathrm{d}\mu(\mathbf{W}^{(1)},b)\\
 \textup{subject to}\ & \int_{\mathbb{R}^d\times \mathbb{R}} \alpha(\mathbf{W}^{(1)},b)[\langle \mathbf{W}^{(1)},\mathbf{x}_j\rangle+b]_+~\mathrm{d}\mu(\mathbf{W}^{(1)},b)+\langle \mathbf{u},\mathbf{x}_j \rangle +v=y_j,\ j=1,\ldots,M.
\end{aligned}
\label{continuous_version_multi_relax_appendix}
\end{equation}
In this paper, our main focus is on the variational problem \eqref{continuous_version_multi_relax_appendix}, thus we derive our main result Theorem \ref{thm:theorem1} and Theorem \ref{thm:theorem_multi} which are statements on linearly adjusted training data. In this section, we try to analyze the difference between solutions of variational problems \eqref{continuous_version_multi_appendix} and \eqref{continuous_version_multi_relax_appendix}, and thus show that to what extent the variational problem \eqref{main_result} and \eqref{gen_multi_dim} in Theorem \ref{thm:theorem1} and Theorem \ref{thm:theorem_multi} describes the implicit bias of gradient descent on original training data. 

Suppose the solution of problem \eqref{continuous_version_multi_appendix} is $\bar\alpha_1$, and the corresponding output function is
\begin{equation*}
  g(\mathbf{x},\bar\alpha_1)=\int_{\mathbb{R}^2} \bar\alpha_1(\mathbf{W}^{(1)},b)[\langle \mathbf{W}^{(1)},\mathbf{x}\rangle+b]_+~\mathrm{d}\mu(\mathbf{W}^{(1)},b). 
\end{equation*}
The solution of problem \eqref{continuous_version_multi_relax_appendix} is $(\bar\alpha_2,\bar{\mathbf{u}},\bar{v})$ and the corresponding output function is:
\begin{equation*}
  g(\mathbf{x},(\bar\alpha_2,\bar{\mathbf{u}},\bar{v}))=\langle \bar{\mathbf{u}},\mathbf{x} \rangle+\bar{v}+\int_{\mathbb{R}^2} \bar\alpha_2(\mathbf{W}^{(1)},b)[\langle \mathbf{W}^{(1)},\mathbf{x}\rangle+b]_+~\mathrm{d}\mu(\mathbf{W}^{(1)},b). 
\end{equation*}
Our goal is to show that $g(\mathbf{x},\bar\alpha_1)$ and $g(\mathbf{x},(\bar\alpha_2,\bar{\mathbf{u}},\bar{v}))$ are close to each other.

Suppose that the linear function $\langle \bar{\mathbf{u}},\mathbf{x} \rangle +\bar v$ can be fitted by an infinite width network with parameters $\alpha_s$, i.e.,
\begin{equation}
    \int_{\mathbb{R}^2} \alpha_s(\mathbf{W}^{(1)},b)[\langle \mathbf{W}^{(1)},\mathbf{x}\rangle+b]_+~\mathrm{d}\mu(\mathbf{W}^{(1)},b)=\langle \bar{\mathbf{u}},\mathbf{x} \rangle +\bar v.
    \label{fit_line}
\end{equation}
Then $\bar\alpha_2+\alpha_s$ is a feasible solution of the problem
\eqref{continuous_version_multi_appendix}. It is easy to show that $g(\mathbf{x},\bar\alpha_2+\alpha_s)=g(\mathbf{x},(\bar\alpha_2,\bar{\mathbf{u}},\bar{v}))$. So we only need to measure the difference between $g(\mathbf{x},\bar\alpha_1)$ and $g(\mathbf{x},\bar\alpha_2+\alpha_s)$. The next theorem characterizes the relative difference between $\bar\alpha_1$ and $\bar\alpha_2+\alpha_s$. 
\begin{theorem}
\label{relative_error_parameters}
Suppose that the solution of the optimization problem \eqref{continuous_version_multi_appendix} is $\bar\alpha_1$ and the solution of the optimization problem \eqref{continuous_version_multi_relax_appendix} is $(\bar\alpha_2,\bar{\mathbf{u}},\bar{v})$. Suppose that $\alpha_s$ satisfies \eqref{fit_line}. Then we have 
\begin{equation*}
    \frac{\int_{\mathbb{R}^2} (\bar\alpha_1-\bar\alpha_2-\alpha_s)^2~\mathrm{d}\mu(\mathbf{W}^{(1)},b)}{\int_{\mathbb{R}^2} \bar\alpha_1^2~\mathrm{d}\mu(\mathbf{W}^{(1)},b)}
      \leq2\sqrt{\frac{\int_{\mathbb{R}^2}\alpha_s^2~\mathrm{d}\mu(\mathbf{W}^{(1)},b)}{\int_{\mathbb{R}^2} \bar\alpha_1^2~\mathrm{d}\mu(\mathbf{W}^{(1)},b)}}+\frac{\int_{\mathbb{R}^2}\alpha_s^2~\mathrm{d}\mu(\mathbf{W}^{(1)},b)}{\int_{\mathbb{R}^2} \bar\alpha_1^2~\mathrm{d}\mu(\mathbf{W}^{(1)},b)}.
\end{equation*}
\end{theorem}
\begin{proof}[Proof of Theorem \ref{relative_error_parameters}]
Since $(\bar\alpha_2,\bar u,\bar v)$ is the minimizer of \eqref{continuous_version_multi_relax_appendix}, we have that $(\bar\alpha_1, 0,0)$ is a feasible solution of \eqref{continuous_version_multi_appendix} but not optimal, which means
\begin{equation}
    \int_{\mathbb{R}^2} \bar\alpha_1^2(\mathbf{W}^{(1)},b)~\mathrm{d}\mu(\mathbf{W}^{(1)},b)\geq\int_{\mathbb{R}^2} \bar\alpha_2^2(\mathbf{W}^{(1)},b)~\mathrm{d}\mu(\mathbf{W}^{(1)},b).
    \label{left_side}
\end{equation}
From the optimality of $\alpha_1$, we have
\begin{equation*}
    \int_{\mathbb{R}^2} \bar\alpha_1^2(\mathbf{W}^{(1)},b)~\mathrm{d}\mu(\mathbf{W}^{(1)},b)\leq\int_{\mathbb{R}^2} (\bar\alpha_2+\alpha_s)^2(\mathbf{W}^{(1)},b)~\mathrm{d}\mu(\mathbf{W}^{(1)},b).
\end{equation*}
Using the first order optimality condition on the problem \eqref{continuous_version_multi_appendix}, we have that there exist $\lambda_j\in\mathbb{R}$ such that
\begin{equation}
    \alpha_1(\mathbf{W}^{(1)},b)=\sum_{j=1}^M \lambda_j[\langle \mathbf{W}^{(1)},\mathbf{x}\rangle+b]_+. 
    \label{alpha1_optimality}
\end{equation}
Since both $\bar\alpha_1$ and $\bar\alpha_2+\alpha_s$ are the feasible solutions of the problem \eqref{continuous_version_activation},
\begin{equation}
    \int_{\mathbb{R}^2} (\bar\alpha_1-\bar\alpha_2-\alpha_s)\cdot[\langle \mathbf{W}^{(1)},\mathbf{x}_j\rangle+b]_+~\mathrm{d}\mu(\mathbf{W}^{(1)},b)=0, \quad j=1,\ldots,M.
    \label{subspace}
\end{equation}
Using \eqref{alpha1_optimality} and \eqref{subspace}, we have 
\begin{equation}
\begin{aligned}
  &\int_{\mathbb{R}^2} (\bar\alpha_1-\bar\alpha_2-\alpha_s)\bar\alpha_1~\mathrm{d}\mu(\mathbf{W}^{(1)},b)\\
  =&\int_{\mathbb{R}^2} (\bar\alpha_1-\bar\alpha_2-\alpha_s)\sum_{j=1}^M \lambda_j[\langle \mathbf{W}^{(1)},\mathbf{x}\rangle+b]_+~\mathrm{d}\mu(\mathbf{W}^{(1)},b)\\
   =&\sum_{j=1}^M\lambda_j\int_{\mathbb{R}^2} (\bar\alpha_1-\bar\alpha_2-\alpha_s)\cdot [\langle \mathbf{W}^{(1)},\mathbf{x}\rangle+b]_+~\mathrm{d}\mu(\mathbf{W}^{(1)},b)\\
   =&0. 
\end{aligned}
\label{A113}
\end{equation}
Then we measure the difference between $\bar\alpha_1$ and $\bar\alpha_2+\alpha_s$:
\begin{equation*}
\begin{aligned}
    &\int_{\mathbb{R}^2} (\bar\alpha_1-\bar\alpha_2-\alpha_s)^2~\mathrm{d}\mu(\mathbf{W}^{(1)},b)\\
    =&\int_{\mathbb{R}^2} (\bar\alpha_2+\alpha_s)^2-(2\bar\alpha_2+2\alpha_s-\bar\alpha_1)\bar\alpha_1~\mathrm{d}\mu(\mathbf{W}^{(1)},b)\\
    =&\int_{\mathbb{R}^2} (\bar\alpha_2+\alpha_s)^2-\bar\alpha_1^2+(2\bar\alpha_2+2\alpha_s-2\bar\alpha_1)\bar\alpha_1~\mathrm{d}\mu(\mathbf{W}^{(1)},b)\\
    =&\int_{\mathbb{R}^2} (\bar\alpha_2+\alpha_s)^2-\bar\alpha_1^2~\mathrm{d}\mu(\mathbf{W}^{(1)},b)\quad \text{(use  \eqref{A113})}\\
    =&\int_{\mathbb{R}^2} (\bar\alpha_2^2+2\bar\alpha_2\alpha_s+\alpha_s^2)-\bar\alpha_1^2~\mathrm{d}\mu(\mathbf{W}^{(1)},b)\\
    \leq&\int_{\mathbb{R}^2} (\bar\alpha_1^2+2\bar\alpha_2\alpha_s+\alpha_s^2)-\bar\alpha_1^2~\mathrm{d}\mu(\mathbf{W}^{(1)},b)\quad \text{(use  \eqref{left_side})}\\
    \leq&\int_{\mathbb{R}^2} 2\bar\alpha_2\alpha_s+\alpha_s^2~\mathrm{d}\mu(\mathbf{W}^{(1)},b)\\
    \leq&2\sqrt{\int_{\mathbb{R}^2}\bar\alpha_2^2~\mathrm{d}\mu(\mathbf{W}^{(1)},b)\cdot\int_{\mathbb{R}^2}\alpha_s^2~\mathrm{d}\mu(\mathbf{W}^{(1)},b)}+\int_{\mathbb{R}^2}\alpha_s^2~\mathrm{d}\mu(\mathbf{W}^{(1)},b)\\
    \leq&2\sqrt{\int_{\mathbb{R}^2}\bar\alpha_1^2~\mathrm{d}\mu(\mathbf{W}^{(1)},b)\cdot\int_{\mathbb{R}^2}\alpha_s^2~\mathrm{d}\mu(\mathbf{W}^{(1)},b)}+\int_{\mathbb{R}^2}\alpha_s^2~\mathrm{d}\mu(\mathbf{W}^{(1)},b) \quad \text{(use  \eqref{left_side})}. %
\end{aligned}
\end{equation*}
Then we bound the relative difference between $\bar\alpha_1$ and $\bar\alpha_2+\alpha_s$:
\begin{equation*}
    \begin{aligned}
      &\frac{\int_{\mathbb{R}^2} (\bar\alpha_1-\bar\alpha_2-\alpha_s)^2~\mathrm{d}\mu(\mathbf{W}^{(1)},b)}{\int_{\mathbb{R}^2} \bar\alpha_1^2~\mathrm{d}\mu(\mathbf{W}^{(1)},b)}\\
      \leq&\frac{2\sqrt{\int_{\mathbb{R}^2}\bar\alpha_1^2~\mathrm{d}\mu(\mathbf{W}^{(1)},b)\cdot\int_{\mathbb{R}^2}\alpha_s^2~\mathrm{d}\mu(\mathbf{W}^{(1)},b)}+\int_{\mathbb{R}^2}\alpha_s^2~\mathrm{d}\mu(\mathbf{W}^{(1)},b)}{\int_{\mathbb{R}^2} \bar\alpha_1^2~\mathrm{d}\mu(\mathbf{W}^{(1)},b)}\\
      =&2\sqrt{\frac{\int_{\mathbb{R}^2}\alpha_s^2~\mathrm{d}\mu(\mathbf{W}^{(1)},b)}{\int_{\mathbb{R}^2} \bar\alpha_1^2~\mathrm{d}\mu(\mathbf{W}^{(1)},b)}}+\frac{\int_{\mathbb{R}^2}\alpha_s^2~\mathrm{d}\mu(\mathbf{W}^{(1)},b)}{\int_{\mathbb{R}^2} \bar\alpha_1^2~\mathrm{d}\mu(\mathbf{W}^{(1)},b)} . 
    \end{aligned}
\end{equation*}
\end{proof}
The above theorem means that if $\int_{\mathbb{R}^2} \alpha_s^2~\mathrm{d}\mu(\mathbf{W}^{(1)},b)$ is much smaller than $\int_{\mathbb{R}^2} \bar\alpha_1^2~\mathrm{d}\mu(\mathbf{W}^{(1)},b)$, the relative difference between $\bar\alpha_1$ and $\bar\alpha_2+\alpha_s$ is quite small. Here $\alpha_s$ fits a linear function and $\bar\alpha_1$ fits the original training data. Since it is much easier for a neural network to fit a linear function than a non-linear function, in practice we observe that $\int_{\mathbb{R}^2} \alpha_s^2~\mathrm{d}\mu(\mathbf{W}^{(1)},b)$ is indeed much smaller than $\int_{\mathbb{R}^2} \bar\alpha_1^2~\mathrm{d}\mu(\mathbf{W}^{(1)},b)$ when the training data is not highly linearly correlated. This is shown in the right panel of Figure~\ref{fit:linear_effect}.

Generally speaking, %
the relative difference between $g(\mathbf{x},\bar\alpha_1)$ and $g(\mathbf{x},(\bar\alpha_2,\bar{\mathbf{u}},\bar{v}))$ can be related to the relative difference between $\bar\alpha_1$ and $\bar\alpha_2+\alpha_s$, which can be bounded by using $D_1\coloneqq\frac{\int_{\mathbb{R}^2} \alpha_s^2~\mathrm{d}\mu(\mathbf{W}^{(1)},b)}{\int_{\mathbb{R}^2} \alpha_1^2~\mathrm{d}\mu(\mathbf{W}^{(1)},b)}$. In experiments, the relative difference between $g(\mathbf{x},\bar\alpha_1)$ and $g(\mathbf{x},(\bar\alpha_2,\bar{\mathbf{u}},\bar{v}))$ is measured by $D\coloneqq\frac{\int_{[-R,R]^{d}} \left(g(\mathbf{x},\bar\alpha_1)-g(\mathbf{x},(\bar\alpha_2,\bar{\mathbf{u}},\bar{v}))\right)^2~\mathrm{d}\mathbf{x}}{\int_{[-R,R]^{d}} \left(g(\mathbf{x},\bar\alpha_1)\right)^2~\mathrm{d}\mathbf{x}}$, where $R$ is the minimal positive number such that $[-R,R]^{d}$ includes all training samples. %
In order to compute $\int_{\mathbb{R}^2} \bar\alpha_1^2~\mathrm{d}\mu(\mathbf{W}^{(1)},b)$ we only need to solve the optimization  problem~\eqref{continuous_version_multi_appendix} and get $\alpha_1$. 
To compute $\int_{\mathbb{R}^2} \alpha_s^2~\mathrm{d}\mu(\mathbf{W}^{(1)},b)$, we first need to solve the optimization  problem~\eqref{continuous_version_multi_relax_appendix} and get $(\bar\alpha_2,\bar{\mathbf{u}},\bar{v})$. Then we need to find out $\alpha_s$ which satisfies \eqref{fit_line}. We can give an easy form of $\alpha_s$ if we assume that the distribution of $(\bm{\mathcal{W}},\mathcal{B})$ is symmetric over each component, i.e., $(\mathcal{W}_1,\ldots,\mathcal{W}_i,\ldots,\mathcal{W}_d,\mathcal{B})$ and $(\mathcal{W}_1,\ldots,-\mathcal{W}_i,\ldots,\mathcal{W}_d,\mathcal{B})$ have the same distribution for $i=1,\ldots,d$. In this case we can choose $\alpha_s(\mathbf{W}^{(1)},b) = C_1\langle\mathbf{W}^{(1)}, \bar{u}\rangle+C_2\bar{v}$ where $C_1$, $C_2$ are constants which is determined by \eqref{fit_line}.
\begin{table}
\begin{tabular}{l|p{15mm}|p{38mm}|p{32mm}|p{25mm}}
&dimension of inputs&training input set $\mathcal{X}$& training output $\mathcal{Y}$ &distribution of $(\bm{\mathcal{W}},\mathcal{B})$\\
\hline
Setting 1&1&$-2$, $-1.6$, 0.3, 0.6, 2& 1.5, 0.5, 1.5, 0.5, 1.5&\begin{tabular}[t]{@{}l@{}}
$W\sim U(-1,1)$\\$B\sim U(-2,2)$
\end{tabular}\\
\hline 
Setting 2&2&\begin{tabular}[t]{@{}l@{}}
     $(-1,-1)$, $(1,1)$, $(0,0)$,\\$(-1,1)$, $(1,-1)$
\end{tabular}&1.5, 1.5, 0.5, $-0.5$, $-0.5$&\begin{tabular}[t]{@{}l@{}}
$\bm{\mathcal{W}}\sim U(\mathbb{S}^1)$%
\\$B\sim U(-2,2)$
\end{tabular}\\
\hline
Setting 3&2&\begin{tabular}[t]{@{}l@{}}
     $(-1,1)$, $(1,1)$, $(0.5,0.9)$,\\
     $(-1,-1)$, $(1,-1)$, $(0,0)$,\\
     $(-1.3,-0.7)$, $(-0.8,0.3)$,\\
     $(-0.4,1.6)$, $(1.6,-0.4)$
\end{tabular}&1.5, 1.5, 0.5, $-0.5$, $-0.5$, $-1.5$, $-1.5$, $-0.5$, 0.5, 0.5&\begin{tabular}[t]{@{}l@{}}
$\bm{\mathcal{W}}\sim U(\mathbb{S}^1)$\\$B\sim U(-2,2)$
\end{tabular}
\end{tabular}
\caption{Experimental settings.}
\label{settings}
\end{table}

Next, we conduct some experiments to verify the above argument. We try three different settings and they are summarized in Table \ref{settings}. For each setting, we add different linear functions to training data and compute corresponding $D_1$ and $D$. In order to verify the idea that $D_1$ is small if training data is not highly correlated, we compute the coefficient of determination $R^2$ of the training data and then compare it with $D_1$. In Figure \ref{fit:linear_effect} we plot $D$ against $D_1$ and $D_1$ against $R^2$. We observe that $D_1$ is small when $R^2$ is small and $D_1$ is a loose upper bound of $D$. Actually, $D$ is very small even if $D_1$ is relatively large, which implies that the relative difference between solutions of \eqref{continuous_version_multi_appendix} and \eqref{continuous_version_multi_relax_appendix} is small in practice. 
\begin{figure}
    \centering
    \begin{tikzpicture}
    \node at (0,0) {\includegraphics[ width=.49\textwidth]{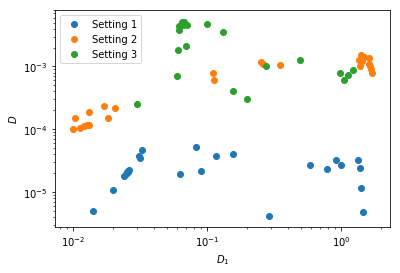}};
    \node at (8,0) {\includegraphics[ width=.49\textwidth]{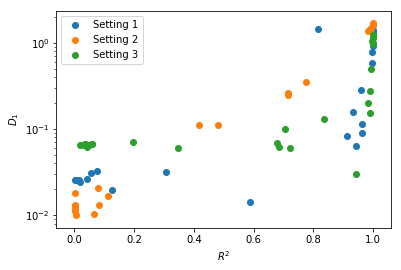}};
    \node at (0.5,2.8) {\footnotesize\textsf{$D$ against $D_1$}};
    \node at (8.5,2.8) {\footnotesize\textsf{$D_1$ against $R^2$}};
    \end{tikzpicture}
    \caption{Scatter plots of $D_1$, $D$ and $R^2$. The left panel is the scatter plot of $D$ against $D_1$, which shows that $D_1$ is a very loose upper bound of $D$. Even when $D_1$ is around $1$, $D$ is still around $10^{-3}$. The right panel is the scatter plot of $D_1$ against $R^2$, which shows that $D_1$ is small when training data are not highly linearly correlated and $D_1$ is large when training data are highly linearly correlated.}
    \label{fit:linear_effect}
\end{figure}

\section{\hj{Neural Networks with Skip Connections}}
\label{app:skip_connection}
\hj{ 
For any given input $\mathbf{x}\in\mathbb{R}^d$, the network with skip connections from the inputs to the outputs computes a function of the form 
\begin{equation}
f(\mathbf{x},\theta)=\sum_{i=1}^n W_i^{(2)}\phi(\langle \mathbf{W}_i^{(1)},\mathbf{x}\rangle +b_i^{(1)}) + \langle \mathbf{u},\mathbf{x} \rangle +v.
\label{eq:skip_connection}
\end{equation}
The skip connection corresponds to the term $\langle \mathbf{u}, \mathbf{x} \rangle$. 
The initializations of $\mathbf{W}_{i}^{(1)},b_{i}^{(1)},W_{i}^{(2)}$ are the same as \eqref{initialization2}. The parameters of skip connections are initialized by zero. We also train this network by gradient descent. The learning rate  of parameters $\mathbf{W}_{i}^{(1)},b_{i}^{(1)},W_{i}^{(2)}$ is $\eta_r$ and the learning rate  of parameters of skip connections $\mathbf{u}, v$ is $\eta_s$. Let $\theta_0=\mathrm{vec}(\overline{\mathbf{W}}^{(1)},\overline{\mathbf{b}}^{(1)},\overline{\mathbf{W}}^{(2)}, \mathbf{0}, 0)$ be the parameters at initialization and $\theta_t=\mathrm{vec}(\mathbf{W}^{(1)}_t,\mathbf{b}^{(1)}_t,\mathbf{W}^{(2)}_t, \mathbf{u}_{t}, v_t)$ be the parameters after $t$ steps of gradient descent. Then the gradient descent iterations are
\begin{equation}
\begin{aligned}
  &\mathbf{W}^{(1)}_0=\overline{\mathbf{W}}^{(1)},
  &\mathbf{W}^{(1)}_{t+1}={\mathbf{W}}^{(1)}_{t}- \eta_r\nabla_{\mathbf{W}^{(1)}} L^{\mathrm{lin}}(\theta_t)\\ 
  &\mathbf{b}^{(1)}_0=\overline{\mathbf{b}}^{(1)},
  &\mathbf{b}^{(1)}_{t+1}={\mathbf{b}}^{(1)}_{t}- \eta_r\nabla_{\mathbf{b}^{(1)}} L^{\mathrm{lin}}(\theta_t)\\ 
  &\mathbf{W}^{(2)}_0=\overline{\mathbf{W}}^{(2)},
  &\mathbf{W}^{(2)}_{t+1}={\mathbf{W}}^{(2)}_{t}- \eta_r\nabla_{\mathbf{W}^{(2)}} L^{\mathrm{lin}}(\theta_t)\\ 
  &\mathbf{u}_0=\mathbf{0}, 
  &\mathbf{u}_{t+1}={\mathbf{u}}_{t}- \eta_s\nabla_{{\mathbf{u}}} L^{\mathrm{lin}}(\theta_t)\\
  &v_0=0, 
  &v_{t+1}={v}_{t}- \eta_s\nabla_{{v}} L^{\mathrm{lin}}(\theta_t)
  \label{eq:gradient_skip_connection}
\end{aligned}
\end{equation}
Let $\widetilde{\omega}_t=\mathrm{vec}(\overline{\mathbf{W}}^{(1)},\overline{\mathbf{b}}^{(1)},\widetilde{\mathbf{W}}_t^{(2)}, \widetilde{\mathbf{u}}, \widetilde{v})$ be the parameters at time $t$ under the update rule where  $\overline{\mathbf{\mathbf{W}}}^{(1)},\overline{\mathbf{b}}^{(1)}$ are kept fixed at their initial values, and 
\begin{equation}
\begin{aligned}
  &\widetilde{\mathbf{W}}^{(2)}_0=\overline{\mathbf{W}}^{(2)},
  &\widetilde{\mathbf{W}}^{(2)}_{t+1}=\widetilde{\mathbf{W}}^{(2)}_{t}- \eta_r\nabla_{\mathbf{W}^{(2)}} L^{\mathrm{lin}}(\widetilde{\omega}_t)\\ 
  &\widetilde{\mathbf{u}}_0=\mathbf{0}, 
  &\widetilde{\mathbf{u}}_{t+1}=\widetilde{\mathbf{u}}_{t}- \eta_s\nabla_{{\mathbf{u}}} L^{\mathrm{lin}}(\widetilde{\omega}_t)\\
  &\widetilde{v}_0=0, 
  &\widetilde{v}_{t+1}=\widetilde{v}_{t}- \eta_s\nabla_{{v}} L^{\mathrm{lin}}(\widetilde{\omega}_t)
  \label{eq:fix_first_skip_connection}
\end{aligned}
\end{equation}
Let $\Psi=\sum_{j=1}^M (\mathbf{x}_j, 1)^T(\mathbf{x}_j, 1)$. Using the similar argument in Section~\ref{sec:3}, we can show that training all parameters can be approximated by training only output weights and skip connections parameters, which is actually a linearized model. Then we can apply Theorem~\ref{problem_appendix} with some modifications and show that gradient descent training of the output weights \eqref{eq:fix_first_skip_connection} on mean squared loss with $\eta_r\leq \frac{M}{4n\lambda_{\max}(\hat{\Theta}_n)}, \eta_s\leq\frac{M}{4\lambda_{\max}(\Psi)}$,  achieves zero loss and solves the following optimization problem:
\begin{equation}
\begin{aligned}
     \min_{\mathbf{W}^{(2)}}\quad  &\frac{1}{\eta_r}\|\mathbf{W}^{(2)}-\overline{\mathbf{W}}^{(2)}\|^2_2 +\frac{1}{\eta_s}\left(\|\mathbf{u}\|_2^2+v^2\right)\\
     \text{s.t.}\quad&\sum_{i=1}^n (W_i^{(2)}-\overline{W}_i^{(2)})[\langle \overline{\mathbf{W}}_i^{(1)}, \mathbf{x}_{j}\rangle+\overline{b}_i^{(1)}]_++\langle\mathbf{u},\mathbf{x}_j\rangle+v=y_j-f(\mathbf{x}_{j}, \theta_0), \;  j=1,\ldots, M.
\end{aligned}
\label{eq:direct_version_non_ASI_skip_connections}
\end{equation} 
Similar to Section~\ref{2.4}, we let $f^{\mathrm{lin}}(\mathbf{x}, \theta_0)\equiv 0$ by using the Anti-Symmetrical Initialization (ASI) trick.
Let $\mu_n$ denote the empirical distribution of the samples $(\overline{\mathbf{W}}_i^{(1)},\overline{b}_i^{(1)})_{i=1}^n$, 
i.e., $\mu_n(A)=\frac{1}{n}\sum_{i=1}^n \mathbbm{1}_A\left((\overline{\mathbf{W}}_i^{(1)},\overline{b}_i^{(1)})\right)$, where $\mathbbm{1}_A$ denotes the indicator function for measurable subsets $A$ in $\mathbb{R}^2$. 
We further consider a function $\alpha_n\colon \mathbb{R}^2\to\mathbb{R}$, $\alpha_n(\overline{\mathbf{W}}_i^{(1)},\overline{b}_i^{(1)})=n(W_i^{(2)}-\overline{W}_i^{(2)})$. 
Then \eqref{eq:direct_version_non_ASI_skip_connections} with ASI can be rewritten as 
\begin{equation}
\begin{aligned}
     \min_{\alpha_n\in C(\mathbb{R}^2)}\  &\int_{\mathbb{R}^2} \alpha_n^2(\mathbf{W}^{(1)},b)~\mathrm{d}\mu_n(\mathbf{W}^{(1)},b)+\frac{n\eta_r}{\eta_s}\left(\|\mathbf{u}\|_2^2+v^2\right)\\
     \textup{ s.t.}\ &\int_{\mathbb{R}^2} \alpha_n(\mathbf{W}^{(1)},b)[\langle \mathbf{W}^{(1)}, \mathbf{x}_{j}\rangle+b]_+~\mathrm{d}\mu_n(\mathbf{W}^{(1)},b)+\langle\mathbf{u},\mathbf{x}_j\rangle+v=y_j,\;  j=1,\ldots, M.
\end{aligned}
\label{eq:probablity_version_skip_connections}
\end{equation} 
}

\hj{ 
Now we can consider the infinite width limit. 
Let $\mu$ be the probability measure of $(\bm{\mathcal{W}},\mathcal{B})$. 
Assume that $\eta_r\leq n^{-1.5}\eta_s$. Then $\frac{n\eta_r}{\eta_s}=o(1)$ as $n\to\infty$, thus it can be ignored in the infinite width limit.
By substituting $\mu$ for $\mu_n$, we obtain a continuous version of problem~\eqref{eq:probablity_version_skip_connections} as follows: 
\begin{equation}
\begin{aligned}
     \min_{\alpha\in C(\mathbb{R}^2)}\  &\int_{\mathbb{R}^2} \alpha^2(\mathbf{W}^{(1)},b)~\mathrm{d}\mu(\mathbf{W}^{(1)},b)\\
     \textup{ s.t.}\ &\int_{\mathbb{R}^2} \alpha(\mathbf{W}^{(1)},b)[\langle \mathbf{W}^{(1)}, \mathbf{x}_{j}\rangle+b]_+~\mathrm{d}\mu(\mathbf{W}^{(1)},b)+\langle\mathbf{u},\mathbf{x}_j\rangle+v=y_j,\;  j=1,\ldots, M.
\end{aligned}
\label{eq:continuous_version_skip_connections}
\end{equation} 
Using that $\mu_n$ weakly converges to $\mu$, 
we show that in fact the solution of 
problem \eqref{eq:probablity_version_skip_connections} converges to the solution of 
\eqref{eq:continuous_version_skip_connections} in Theorem~\ref{eq:theorem_skip_connection}. 
}

\hj{
\begin{theorem}[Infinite width limit for network with skip connections]
\label{eq:theorem_skip_connection}
Let $(\overline{\mathbf{W}}_i^{(1)},\overline{b}_i^{(1)})_{i=1}^n$ be i.i.d.\ samples from a pair $(\bm{\mathcal{W}},\mathcal{B})$ with finite fourth moment. 
Suppose $\mu_n$ is the empirical distribution of $(\overline{\mathbf{W}}_i^{(1)},\overline{b}_i^{(1)})_{i=1}^n$ and $(\overline{\alpha}_n,\overline{ \mathbf{u}}_n,\overline{v}_n)$ is the solution of \eqref{eq:probablity_version_skip_connections}. 
Let $(\overline{\alpha},\overline{ \mathbf{u}},\overline{v})$ be the solution of \eqref{eq:continuous_version_skip_connections}. Assume that $\eta_r\leq n^{-1.5}\eta_s$.
Then, for any compact set $D\subset \mathbb{R}^d$, we have $\sup_{\mathbf{x}\in D}|g_n(\mathbf{x},(\overline{\alpha}_n,\overline{ \mathbf{u}}_n,\overline{v}_n))- g(\mathbf{x},(\overline{\alpha},\overline{ \mathbf{u}},\overline{v}))|=O_p(n^{-1/2})$ 
, where $g_n(\mathbf{x},(\overline{\alpha}_n,\overline{ \mathbf{u}}_n,\overline{v}_n)) = \int_{\mathbb{R}^2} \alpha_n(\mathbf{W}^{(1)},b)[\langle \mathbf{W}^{(1)}, \mathbf{x}\rangle+b]_+~\mathrm{d}\mu_n(\mathbf{W}^{(1)},b)+\langle\mathbf{u}_n,\mathbf{x}\rangle+v_n$
is the function represented by a network with $n$ hidden neurons and skip connections after training, 
and $g(\mathbf{x},(\overline{\alpha},\overline{ \mathbf{u}},\overline{v})) = \int_{\mathbb{R}^2} \alpha(\mathbf{W}^{(1)},b)[\langle \mathbf{W}^{(1)}, \mathbf{x}\rangle+b]_+~\mathrm{d}\mu(\mathbf{W}^{(1)},b)+\langle\mathbf{u},\mathbf{x}\rangle+v$ is
the function represented by the infinite-width network with skip connections. 
\end{theorem}
}
The proof of Theorem~\ref{eq:theorem_skip_connection} is provided at the end of the section. In Section~\ref{sec:implicit_bias_univariate} and 
Section~\ref{sec:implicit_bias_multivariate}, we show that the optimization problem \eqref{eq:continuous_version_skip_connections} is equivalent to \eqref{function_space} in the univariate case and equivalent to \eqref{function_space_multi} in the multivariate case. 
From this we immediately obtain 
our main theorems for networks with skip connections without adjusting the training data, namely the following Theorem~\ref{thm:theorem1_skip} and Theorem~\ref{thm:theorem_multi_skip}.  

\hj{ 
\begin{theorem}
[Implicit bias of networks with skip connections, univariate]
\label{thm:theorem1_skip}
Consider a two-layer feedforward network with skip connections \eqref{eq:skip_connection}. Assume  parameter initialization \eqref{initialization2}, which means for each hidden unit the input weight and bias are initialized from a sub-Gaussian $(\mathcal{W},\mathcal{B})$ with joint density $p_{\mathcal{W},\mathcal{B}}$. %
Then, for any finite data set $\{(x_j,y_j)\}_{j=1}^M$ 
and sufficiently large $n$, the optimization of the mean squared error on the training data 
 $\{(x_j,y_j)\}_{j=1}^M$ 
by gradient descent iterations \eqref{eq:gradient_skip_connection} with learning rate $\eta_s\leq\frac{M}{4\lambda_{\max}(\Psi)}, \eta_r\leq n^{-1.5}\eta_s$ converges to a parameter $\theta^\ast$ for which the output function $f(x,\theta^\ast)$ %
attains zero training error. 
Furthermore, letting $\zeta(x) = \int_\mathbb{R} |W|^3 p_{\mathcal{W},\mathcal{B}}(W,-Wx)~\mathrm{d}W$ and $S = \operatorname{supp}(\zeta) \cap [\min_j x_j, \max_j x_j]$, we have  $\sup_{x\in S}\|f(x,\theta^\ast) - g^\ast(x)\|_2 = O_p(n^{-\frac{1}{2}})$%
over the random initialization $\theta_0$, 
where $g^\ast$ solves following variational problem:  
\begin{equation}
  \begin{aligned}
\min_{g\in C^2(S)}\quad & \int_S \frac{1}{\zeta(x)} (g''(x) - f''(x,\theta_0))^2~\mathrm{d}x\\
\textup{subject to}\quad & g(x_j) = y_j,\quad j=1,\ldots, M .
\end{aligned}
\label{eq:main_result_skip}
\end{equation}
\end{theorem}
}

\hj{
\begin{theorem}
[Implicit bias of networks with skip connections, multivariate]
\label{thm:theorem_multi_skip}
Consider the same network settings as in Theorem~\ref{thm:theorem1_skip} except with $d$ input units instead of a single input unit. 
Assume that $\bm{\mathcal{W}}$ is a random vector with $\mathbb{P}(\|\bm{\mathcal{W}}\|=0)=0$ and $\mathcal{B}$ is a random variable; the distribution of $(\bm{\mathcal{W}},\mathcal{B})$ is symmetric, i.e., $(\bm{\mathcal{W}},\mathcal{B})$ and $(-\bm{\mathcal{W}},-\mathcal{B})$ have the same distribution; and $\|\bm{\mathcal{W}}\|_2$ and $\mathcal{B}$ are both sub-Gaussian.
Then, for any finite data set $\{(\mathbf{x}_j,y_j)\}_{i=1}^M$ 
and sufficiently large $n$, the optimization of the mean squared error on the training data 
 $\{(\mathbf{x}_j,y_j)\}_{j=1}^M$ 
by gradient descent iterations \eqref{eq:gradient_skip_connection} with learning rate $\eta_s\leq\frac{M}{4\lambda_{\max}(\Psi)}, \eta_r\leq n^{-1.5}\eta_s$ converges to a parameter $\theta^\ast$ for which $f(\mathbf{x},\theta^\ast)$ attains zero training error. 
Furthermore, let $\mathcal{U}=\|\bm{\mathcal{W}}\|_2$, $\bm{\mathcal{V}}=\bm{\mathcal{W}}/\|\bm{\mathcal{W}}\|_2$, $\mathcal{C}=-\mathcal{B}/\|\bm{\mathcal{W}}\|_2$ and $\zeta(\bm{V},c) = p_{\bm{\mathcal{V}},\mathcal{C}}(\bm{V},c)\mathbb{E}(\mathcal{U}^2|\bm{\mathcal{V}}=\bm{V},\mathcal{C}=c)$, where $p_{\bm{\mathcal{V}},\mathcal{C}}$ is the joint density of $(\bm{\mathcal{V}},\mathcal{C})$. 
Then, for any compact set $D\subset \mathbb{R}^d$, we have  $\sup_{\mathbf{x}\in D}\|f(\mathbf{x},\theta^\ast) - g^\ast(\mathbf{x})\|_2 = O_p(n^{-\frac{1}{2}})$
over the random initialization $\theta_0$, 
where $g^\ast$ solves following variational problem:  
\begin{equation}
\begin{aligned}
 \min_{g\in \operatorname{Lip}(\mathbb{R}^d)}\quad & \int_{\operatorname{supp}(\zeta)} \frac{\left({\mathcal{R}\{(-\Delta)^{(d+1)/2}(g-f(\cdot,\theta_0))\}(\bm{V},c)}\right)^2}{\zeta(\bm{V},c)}~\mathrm{d}\bm{V}\mathrm{d}c\\
 \textup{subject to}\quad & g(\mathbf{x}_j)=y_j,\quad j=1,\ldots,M \\
 & \mathcal{R}\{(-\Delta)^{(d+1)/2}(g-f(\cdot,\theta_0))\}(\bm{V},c)=0,\quad (\bm{V},c)\not\in\operatorname{supp}(\zeta) \\
 &(-\Delta)^{(d+1)/2}(g-f(\cdot,\theta_0)) \in L^p(\mathbb{R}^d),\ 1\leq p<d/(d-1). 
\end{aligned}
\label{eq:gen_multi_dim_skip}
\end{equation}
\end{theorem} 
}

\begin{proof}[Proof of Theorem~\ref{eq:theorem_skip_connection}]
\hj{ 
The Lagrangian of problem \eqref{eq:probablity_version_skip_connections} is
\begin{equation*}
  L((\alpha_n,\mathbf{u}_n,v_n),\lambda^{(n)})=\int_{\mathbb{R}^2} \alpha_n^2(\mathbf{W}^{(1)},b)~\mathrm{d}\mu_n(\mathbf{W}^{(1)},b)+\frac{n\eta_r}{\eta_s}\left(\|\mathbf{u}_n\|_2^2+v_n^2\right)+\sum_{j=1}^M \lambda^{(n)}_j(g_n(\mathbf{x}_j,\alpha_n)-y_j).
\end{equation*}
The optimal condition is $\nabla_{\alpha_n} L=0$, which means
\begin{align*}
   2\alpha_n(\mathbf{W}^{(1)},b)+\sum_{j=1}^M \lambda^{(n)}_j [\langle\mathbf{W}^{(1)},\mathbf{x}_j \rangle+b]_+ &= 0 \textup{ when } (\mathbf{W}^{(1)},b)=(\mathbf{W}^{(1)}_i,b_i), \ i=1,\ldots,k\\
  \frac{2n\eta_r}{\eta_s}\mathbf{u}_n+\sum_{j=1}^M \lambda^{(n)}_j \mathbf{x}_j&=0\\
  \frac{2n\eta_r}{\eta_s}v_n+\sum_{j=1}^M \lambda^{(n)}_j &=0.
\end{align*}
Since only function values on $(\mathbf{W}_i^{(1)},b_i)_{i=1}^M$ are taken into account in problem \eqref{eq:probablity_version_skip_connections}, we can let
\begin{equation}
  \overline{\alpha}_n(\mathbf{W}^{(1)},b) = -\frac{1}{2}\sum_{j=1}^M \lambda^{(n)}_j [\langle\mathbf{W}^{(1)},\mathbf{x}_j \rangle+b]_+ \quad \forall (\mathbf{W}^{(1)},b)\in\mathbb{R}^{d+1}
  \label{eq:a1}
\end{equation}
without changing $\int_{\mathbb{R}^2} \overline{\alpha}_n^2(\mathbf{W}^{(1)},b)~\mathrm{d}\mu_n(\mathbf{W}^{(1)},b)$ and $g_n(\mathbf{x},\overline{\alpha}_n)$.
}

\hj{
Here $\lambda^{(n)}_j$, $j=1,\ldots,M$ are chosen to make $g_n(\mathbf{x}_i,\overline{\alpha}_n)=y_i$, $i=1,\ldots,M$. So we get a system of linear equations in variables $\{\lambda^{(n)}_j\}_{j=1}^M, \mathbf{u}_n$ and $v_n$:
\begin{equation}
\begin{aligned}
  -\frac{1}{2}\sum_{j=1}^M \lambda^{(n)}_j \int_{\mathbb{R}^2} [\langle\mathbf{W}^{(1)},\mathbf{x}_j \rangle+b]_+[\langle\mathbf{W}^{(1)},\mathbf{x}_i \rangle+b]_+ ~\mathrm{d}\mu_n(\mathbf{W}^{(1)},b)+\langle\mathbf{u}_n,\mathbf{x}_i\rangle+v_n&= y_i,\; i=1,\ldots,M\\
  \sum_{j=1}^M \lambda^{(n)}_j \mathbf{x}_j+\frac{2n\eta_r}{\eta_s}\mathbf{u}_n&=0\\
  \sum_{j=1}^M \lambda^{(n)}_j+\frac{2n\eta_r}{\eta_s}v_n&=0 . 
  \label{eq:e1}
\end{aligned}
\end{equation}
}

\hj{
Similarly, the Lagrangian of problem \eqref{eq:continuous_version_skip_connections} is 
\begin{equation*}
  \widetilde{L}(\alpha,\lambda)=\int_{\mathbb{R}^2} \alpha^2(\mathbf{W}^{(1)},b)~\mathrm{d}\mu(\mathbf{W}^{(1)},b)+\sum_{j=1}^M \lambda_j(g(\mathbf{x}_j,\alpha)-y_j) . 
\end{equation*}
The optimality condition is $\nabla_\alpha \widetilde{L}=0$, which means 
\begin{align*}
   2\alpha(\mathbf{W}^{(1)},b)+\sum_{j=1}^M \lambda^{(n)}_j [\langle\mathbf{W}^{(1)},\mathbf{x}_j \rangle+b]_+ &= 0  \quad \forall (\mathbf{W}^{(1)},b)\in\mathbb{R}^{d+1}\\
  0\cdot\mathbf{u}+\sum_{j=1}^M \lambda^{(n)}_j \mathbf{x}_j&=0\\
  0\cdot v+\sum_{j=1}^M \lambda^{(n)}_j &=0.
\end{align*}
Then we get
\begin{equation}
  \overline{\alpha}(\mathbf{W}^{(1)},b) = -\frac{1}{2}\sum_{j=1}^M \lambda_j [\langle\mathbf{W}^{(1)},\mathbf{x}_j \rangle+b]_+  \quad \forall (\mathbf{W}^{(1)},b)\in\mathbb{R}^2.
  \label{eq:a2}
\end{equation}
Here $\lambda_j$, $j=1,\ldots,M$ are chosen to make $g(\mathbf{x},\alpha)=y_i$, $i=1,\ldots,M$. This means that 
\begin{equation}
\begin{aligned}
  -\frac{1}{2}\sum_{j=1}^M \lambda_j \int_{\mathbb{R}^2} [\langle\mathbf{W}^{(1)},\mathbf{x}_j \rangle+b]_+[\langle\mathbf{W}^{(1)},\mathbf{x}_i \rangle+b]_+ ~\mathrm{d}\mu(\mathbf{W}^{(1)},b)+\langle\mathbf{u},\mathbf{x}_i\rangle+v&= y_i,\; i=1,\ldots,M\\
  \sum_{j=1}^M \lambda_j \mathbf{x}_j+0\cdot\mathbf{u}&=0\\
  \sum_{j=1}^M \lambda_j+0\cdot v&=0\\.
  \label{eq:e2}
\end{aligned}
\end{equation}
}

\hj{ 
Compare \eqref{eq:e1} and \eqref{eq:e2}. Since the number of samples is finite, $\mathbf{x}_i$ is also bounded. Then by the assumption that $\mathcal{\bm{W}}$ and $\mathcal{B}$ have finite fourth moments, we have that $[\langle\mathbf{W}^{(1)},\mathbf{x}_j \rangle+b]_+[\langle\mathbf{W}^{(1)},\mathbf{x}_i \rangle+b]_+$ has finite variance. According to central limit theorem, as $n\to\infty$, $\int_{\mathbb{R}^2} [\langle\mathbf{W}^{(1)},\mathbf{x}_j \rangle+b]_+[\langle\mathbf{W}^{(1)},\mathbf{x}_i \rangle+b]_+ ~\mathrm{d}\mu_n(\mathbf{W}^{(1)},b)$ tends to a Gaussian distribution with variance $O(n^{-1})$. 
This implies that $\forall i=1,\ldots,M,~\forall j=1,\ldots,M$,
\begin{equation*}
\begin{aligned}
&|\int_{\mathbb{R}^2} [\langle\mathbf{W}^{(1)},\mathbf{x}_j \rangle+b]_+[\langle\mathbf{W}^{(1)},\mathbf{x}_i \rangle+b]_+ ~\mathrm{d}\mu_n(\mathbf{W}^{(1)},b)\\
&-\int_{\mathbb{R}^2} [\langle\mathbf{W}^{(1)},\mathbf{x}_j \rangle+b]_+[\langle\mathbf{W}^{(1)},\mathbf{x}_i \rangle+b]_+ ~\mathrm{d}\mu(\mathbf{W}^{(1)},b)|\\
&=O_p(n^{-1/2})
\end{aligned}
\end{equation*}
Also according to the assumption $\eta_r\leq n^{-1.5}\eta_s$, we have $\frac{2n\eta_r}{\eta_s}=O(n^{-1/2})$. So coefficients of \eqref{eq:e1} converge to coefficients of \eqref{eq:e2} at the rate of $O_p(n^{-1/2})$, then we get
\begin{equation}
  |\lambda_j^n-\lambda_j|=O_p(n^{-1/2}),\quad j=1,\ldots,M. \label{eq:converge_lambda} 
\end{equation}
Compare \eqref{eq:a1} and \eqref{eq:a2}. Given $(\mathbf{W}^{(1)},b)$, we have
\begin{equation}
  |\overline{\alpha}_n(\mathbf{W}^{(1)},b)-\overline{\alpha}(\mathbf{W}^{(1)},b)|=O_p(n^{-1/2}).\label{eq:alpha_n_converge}
\end{equation}
Next we want to prove that $\sup_{\mathbf{x}\in D}|g_n(\mathbf{x},(\overline{\alpha}_n,\overline{ \mathbf{u}}_n,\overline{v}_n))- g(\mathbf{x},(\overline{\alpha},\overline{ \mathbf{u}},\overline{v}))|=O_p(n^{-1/2})$. 
Firstly, we prove that $\sup_{\mathbf{x}\in D}|g_n(\mathbf{x},(\overline{\alpha},\overline{ \mathbf{u}},\overline{v}))- g(\mathbf{x},(\overline{\alpha},\overline{ \mathbf{u}},\overline{v}))|=O_p(n^{-1/2})$. Note that $|g_n(\mathbf{x},(\overline{\alpha},\overline{ \mathbf{u}},\overline{v}))- g(\mathbf{x},(\overline{\alpha},\overline{ \mathbf{u}},\overline{v}))|=|g_n(\mathbf{x},(\overline{\alpha},\mathbf{0},0))- g(\mathbf{x},(\overline{\alpha},\mathbf{0},0))|$. According to \eqref{part_converge} in the proof of Theorem~\ref{theorem4} in Appendix~\ref{Proof4}, we have $\sup_{\mathbf{x}\in D}|g_n(\mathbf{x},(\overline{\alpha},\mathbf{0},0))- g(\mathbf{x},(\overline{\alpha},\mathbf{0},0))|=O_p(n^{-1/2})$. Then we have
\begin{equation}
    \sup_{\mathbf{x}\in D}|g_n(\mathbf{x},(\overline{\alpha},\overline{ \mathbf{u}},\overline{v}))- g(\mathbf{x},(\overline{\alpha},\overline{ \mathbf{u}},\overline{v}))|=O_p(n^{-1/2}).
    \label{eq:part_converge}
\end{equation} 
}

\hj{ 
Finally, we prove that $\sup_{\mathbf{x}\in D}|g_n(\mathbf{x},(\overline{\alpha}_n,\overline{ \mathbf{u}}_n,\overline{v}_n))- g_n(\mathbf{x},(\overline{\alpha},\overline{ \mathbf{u}},\overline{v}))|=O_p(n^{-1/2})$. Since $\forall \mathbf{x}\in  D$
\begin{equation*}
  \begin{aligned}
    &|g_n(\mathbf{x},(\overline{\alpha}_n,\overline{ \mathbf{u}}_n,\overline{v}_n))- g_n(\mathbf{x},(\overline{\alpha},\overline{ \mathbf{u}},\overline{v}))|\\
    \leq&\int_{\mathbb{R}^2} \left|\overline{\alpha}_n(\mathbf{W}^{(1)},b)[\langle\mathbf{W}^{(1)},\mathbf{x} \rangle+b]_+ -\overline{\alpha}(\mathbf{W}^{(1)},b)[\langle\mathbf{W}^{(1)},\mathbf{x} \rangle+b]_+\right|~\mathrm{d}\mu_n(\mathbf{W}^{(1)},b)\\
    &+\|\mathbf{x}\|_2\|\overline{ \mathbf{u}}_n-\overline{ \mathbf{u}}\|_2+|\overline{v}_n-\overline{v}|\\
    \leq&\int_{\mathbb{R}^2} \left|\overline{\alpha}_n(\mathbf{W}^{(1)},b)-\overline{\alpha}(\mathbf{W}^{(1)},b)\right|[\langle\mathbf{W}^{(1)},\mathbf{x} \rangle+b]_+~\mathrm{d}\mu_n(\mathbf{W}^{(1)},b)+\|\mathbf{x}\|_2\|\overline{ \mathbf{u}}_n-\overline{ \mathbf{u}}\|_2+|\overline{v}_n-\overline{v}|\\
    \leq&  \int_{\mathbb{R}^2}\left|-\frac{1}{2}\sum_{j=1}^M (\lambda_j^n-\lambda_j) [\langle\mathbf{W}^{(1)},\mathbf{x}_j \rangle+b]_+\right| [\langle\mathbf{W}^{(1)},\mathbf{x} \rangle+b]_+~\mathrm{d}\mu_n(\mathbf{W}^{(1)},b)\\
    &+\|\mathbf{x}\|_2\|\overline{ \mathbf{u}}_n-\overline{ \mathbf{u}}\|_2+|\overline{v}_n-\overline{v}|\\
    \leq&  \frac{1}{2}\sum_{j=1}^M |\lambda_j^n-\lambda_j| \int_{\mathbb{R}^2}[\langle\mathbf{W}^{(1)},\mathbf{x}_j \rangle+b]_+ [\langle\mathbf{W}^{(1)},\mathbf{x} \rangle+b]_+~\mathrm{d}\mu_n(\mathbf{W}^{(1)},b)\\
    &+\|\mathbf{x}\|_2\|\overline{ \mathbf{u}}_n-\overline{ \mathbf{u}}\|_2+|\overline{v}_n-\overline{v}|\\
    \leq&  \frac{1}{2} \left(\max_{\mathbf{x}\in D}\int_{\mathbb{R}^2}[\langle\mathbf{W}^{(1)},\mathbf{x}_j \rangle+b]_+ [\langle\mathbf{W}^{(1)},\mathbf{x} \rangle+b]_+~\mathrm{d}\mu_n(\mathbf{W}^{(1)},b)\right)\sum_{j=1}^M |\lambda_j^n-\lambda_j|\\
    &+\max_{\mathbf{x}\in D}\|\mathbf{x}\|_2\|\overline{ \mathbf{u}}_n-\overline{ \mathbf{u}}\|_2+|\overline{v}_n-\overline{v}|. 
  \end{aligned}
\end{equation*}
Because $ D$ is compact and $\int_{\mathbb{R}^2}[\langle\mathbf{W}^{(1)},\mathbf{x}_j \rangle+b]_+ [\langle\mathbf{W}^{(1)},\mathbf{x} \rangle+b]_+~\mathrm{d}\mu_n(\mathbf{W}^{(1)},b)$ converges according to the law of large numbers, we have that $\max_{\mathbf{x}\in D}\int_{\mathbb{R}^2}[\langle\mathbf{W}^{(1)},\mathbf{x}_j \rangle+b]_+ [\langle\mathbf{W}^{(1)},\mathbf{x} \rangle+b]_+~\mathrm{d}\mu_n(\mathbf{W}^{(1)},b)$ and $\max_{\mathbf{x}\in D}\|\mathbf{x}\|_2$ is bounded by a finite number independent of $n$. Then according to \eqref{eq:converge_lambda},%
\begin{equation*}
\sup_{\mathbf{x}\in D}|g_n(\mathbf{x},(\overline{\alpha}_n,\overline{ \mathbf{u}}_n,\overline{v}_n))- g_n(\mathbf{x},(\overline{\alpha},\overline{ \mathbf{u}},\overline{v}))|=O_p(n^{-1/2}). 
\end{equation*}
Combined with \eqref{eq:part_converge}, we have
\begin{equation*}
\sup_{\mathbf{x}\in D}|g_n(\mathbf{x},(\overline{\alpha}_n,\overline{ \mathbf{u}}_n,\overline{v}_n))- g(\mathbf{x},(\overline{\alpha},\overline{ \mathbf{u}},\overline{v}))|=O_p(n^{-1/2}). 
\end{equation*}
This concludes the proof. } 
\end{proof}

\section{Equivalence of Our Characterization and NTK Norm Minimization for Univariate Regression}
\label{appendix:relation_NTK} 

In this section we demonstrate that NTK norm minimization \citep{zhang2019type}, which characterizes the implicit bias of training a linearized model by gradient descent, is equivalent to our characterization in Section~\ref{2.4} and Section~\ref{sec:implicit_bias_univariate}. For simplicity, we only consider univariate regression in this section.
Following \cite{jacot2018neural}, \cite{zhang2019type} show that gradient descent can be regarded as a kernel gradient descent in function space, whereby the kernel is given by the NTK. 
Then for a linearized model, gradient descent finds the global minimum that is closest to the initial output function in the corresponding reproducing kernel Hilbert space (RKHS). 
Let $\tilde{\Theta}_n$ be the  empirical neural tangent kernel of training only the output layer, i.e.,\ 
\begin{equation*}
\begin{aligned}
    \tilde{\Theta}_n(x_1,x_2) &= \frac{1}{n}\nabla_{W^{(2)}} f(\mathbf{x}_1,\theta_0) \nabla_{W^{(2)}} f(x_2,\theta_0)^T \\
    &=\frac{1}{n}\sum_{i=1}^n\nabla_{W_i^{(2)}} f(x_1,\theta_0) \nabla_{W_i^{(2)}} f(x_2,\theta_0)\\
    &=\frac{1}{n}\sum_{i=1}^n[W_i^{(1)}x_1+b^{(1)}_i]_+  [W_i^{(1)}x_2+b^{(1)}_i]_+. 
\end{aligned}
\end{equation*}
As $n\to\infty$, $\tilde{\Theta}_n\to \tilde{\Theta}$, where 
\begin{equation}
\begin{aligned}
    \tilde{\Theta}(x_1,x_2) 
    &=\int_{\mathbb{R}^2} [W^{(1)}x_1+b^{(1)}]_+  [W^{(1)}x_2+b^{(1)}]_+~\mathrm{d}\mu(W^{(1)},b).
    \label{Theta_form}
\end{aligned}
\end{equation}
Equivalently, using the notation in Section \ref{sec:implicit_bias_univariate}, we have
\begin{equation}
\begin{aligned}
    \tilde{\Theta}(x_1,x_2) 
    &=\int_{\mathbb{R}^2} [W^{(1)}(x_1-c)]_+[W^{(1)}(x_2-c)]_+~\mathrm{d}\nu(W^{(1)},c). 
    \label{Theta_form_nu}
\end{aligned}
\end{equation}

Next, \cite{zhang2019type} construct a RKHS $\mathcal{H}_{\tilde{\Theta}}(S)$ by kernel $\tilde{\Theta}$, and the inner product of the RKHS is denoted by $\langle\cdot,\cdot \rangle_{\tilde{\Theta}}$. Then $\mathcal{H}_{\tilde{\Theta}}(S)$ satisfies:
\begin{align}
\text{(i) \quad} &\forall x\in S, {\tilde{\Theta}}(\cdot,x)\in \mathcal{H}_{\tilde{\Theta}}(S);\\
\text{(ii) \quad} &\forall x\in S, \forall f\in \mathcal{H}_{\tilde{\Theta}}, \langle f(\cdot), {\tilde{\Theta}}(\cdot,x) \rangle_{\tilde{\Theta}} = f(x);\label{kernel_norm_property2}\\
\text{(iii) \quad} &\forall x,y\in S, \langle {\tilde{\Theta}}(\cdot,x), \tilde{\Theta}(\cdot,y) \rangle_{\tilde{\Theta}} = {\tilde{\Theta}}(x,y).    
\end{align}
Here the domain is $S= \operatorname{supp}(\zeta) \cap [\min_i x_i, \max_i x_i]$, which is the same as in Theorem~\ref{thm:theorem1} 
and Theorem~\ref{theorem_func}. 
Using the reproducing kernel Hilbert space, \cite{zhang2019type} prove that  $f^{\mathrm{lin}}(x,\widetilde{\omega}_\infty)$ (defined in Section~\ref{Training_only_the_output_layer}) is the solution of the following optimization problem:
\begin{equation*}
  \min_{g\in\mathcal{H}_{\tilde{\Theta}}(S)} \|g\|_{\tilde{\Theta}_n} \quad \text{s.t. } g(x_j)=y_j,\ j=1,\ldots,M. 
\end{equation*}
As the width $n$ tends to infinity, the above optimization problem becomes 
\begin{equation}
  \min_{g\in\mathcal{H}_{\tilde{\Theta}}(S)} \|g\|_{\tilde{\Theta}} \quad \text{s.t. } g(x_j)=y_j,\ j=1,\ldots,M. 
  \label{kernel_norm_min_continuous}
\end{equation}
In Section~\ref{2.4}, we show that $f^{\mathrm{lin}}(x,\widetilde{\omega}_\infty)$ is the solution of the optimization problem \eqref{probablity_version} in function space. As width $n$ tends to infinity, the optimization problem \eqref{probablity_version} becomes \eqref{continuous_version}, which we repeat below: 
\begin{equation}
\begin{aligned}
 \min_{\alpha\in C(\mathbb{R}^2)}\quad & \int_{\mathbb{R}^2} \alpha^2(W^{(1)},b)~\mathrm{d}\mu(W^{(1)},b)\\
 \textup{subject to}\quad & \int_{\mathbb{R}^2} \alpha(W^{(1)},b)[W^{(1)}x_j+b]_+~\mathrm{d}\mu(W^{(1)},b)=y_j, \quad j=1,\ldots,M .
\end{aligned}
\label{continuous_version_1}
\end{equation}
Since optimization problems \eqref{kernel_norm_min_continuous} and \eqref{continuous_version_1} both characterize the implicit bias of training a linearized model by gradient descent, they must have the same solution in function space. We express this formally in the following theorem: 

\begin{theorem}[Equivalence of our variational problem and NTK norm minimization]
\label{th:equivalence}
Assume that optimization problems \eqref{kernel_norm_min_continuous} and \eqref{continuous_version_1} are both feasible. Suppose $\overline{\alpha}$ is the solution of \eqref{continuous_version_1}, 
and consider the corresponding output function:
\begin{equation}
  \overline{g}(x)=\int_{\mathbb{R}^2} \overline{\alpha}(W^{(1)},b)[W^{(1)}x+b]_+~\mathrm{d}\mu(W^{(1)},b). 
   \label{func_g_1}
\end{equation}
Then $\overline{g}(x)$ restricted on $S$ is the solution of the optimization problem \eqref{kernel_norm_min_continuous}. 
\end{theorem}

Next, we give a standalone proof of this theorem using the property of kernel norm. The proof gives us an idea of what the kernel norm actually looks like. 
\begin{proof}[Proof of Theorem~\ref{th:equivalence}]
Since $\overline{\alpha}(W^{(1)},b)$ is the solution of 
\eqref{continuous_version_1}, according to \eqref{a2} in the proof of Theorem~\ref{theorem4}, 
\begin{equation*}
  \overline{\alpha}(W^{(1)},b) = -\frac{1}{2}\sum_{j=1}^M \lambda_j [W^{(1)}x_j+b]_+  \quad \forall (W^{(1)},b)\in\mathbb{R}^2
\end{equation*}
for some constants $\lambda_j, j=1,\ldots,M$. Then we write $\overline{\alpha}(W^{(1)},b)$ in the following form:
\begin{equation}
  \overline{\alpha}(W^{(1)},b) = \int_{S} h(x)[W^{(1)}x+b]_+  \mathrm{d}x,
  \label{form_of_alpha}
\end{equation}
where $h(x)$ can be a combination of Dirac delta functions. 
Then substitute \eqref{form_of_alpha} into the expression of $\overline{g}(x)$ \eqref{func_g_1} to obtain  
\begin{equation}
\begin{aligned}
    \overline{g}(x)&=\int_{\mathbb{R}^2\times S} h(\tilde{x})[W^{(1)}\tilde{x}+b]_+[W^{(1)}x+b]_+~\mathrm{d}\mu(W^{(1)},b) \mathrm{d}\tilde{x}\\
    &=\int_{S} h(\tilde{x})\tilde{\Theta}(x,\tilde{x}) \mathrm{d}\tilde{x}, %
    \label{form_of_g}
\end{aligned}
\end{equation}
where we use the expression of the NTK in equation  \eqref{Theta_form}. Then we get
\begin{equation}
\begin{aligned}
    \langle g(x),g(x)\rangle_{\tilde{\Theta}}&=\langle g(x),\int_{S} h(\tilde{x})\tilde{\Theta}(x,\tilde{x}) \mathrm{d}\tilde{x}\rangle_{\tilde{\Theta}}\mathrm{d}\tilde{x}\\
    &=\int_{S} h(\tilde{x})\langle g(x),\tilde{\Theta}(x,\tilde{x})\rangle_{\tilde{\Theta}}\mathrm{d}\tilde{x}\\
    &=\int_{S} h(\tilde{x})g(\tilde{x})\mathrm{d}\tilde{x}\quad \text{ (here we use the property of RKHS norm \eqref{kernel_norm_property2})}\\ 
    &=\int_{S\times S} h(\tilde{x})h(\bar{x})\tilde{\Theta}(\tilde{x},\bar{x})\mathrm{d}\tilde{x}\mathrm{d}\bar{x} \quad \text{ (use \eqref{form_of_g})} . 
    \label{ob1}
\end{aligned}
\end{equation}
On the other hand, using \eqref{form_of_alpha}, the objective of \eqref{continuous_version_1} becomes
\begin{equation}
\begin{aligned}
&\int_{S^2} \overline{\alpha}^2(W^{(1)},b)~\mathrm{d}\mu(W^{(1)},b)\\
=&\int_{S\times S\times \mathbb{R}^2} h(\tilde{x})[W^{(1)}\tilde{x}+b]_+ h(\bar{x})[W^{(1)}\bar{x}+b]_+ ~\mathrm{d}\tilde{x}\mathrm{d}\bar{x}\mathrm{d}\mu(W^{(1)},b)\\
=&\int_{S\times S} h(\tilde{x}) h(\bar{x})\int_{ \mathbb{R}^2}[W^{(1)}\tilde{x}+b]_+[W^{(1)}\bar{x}+b]_+ \mathrm{d}\mu(W^{(1)},b) ~\mathrm{d}\tilde{x}\mathrm{d}\bar{x}\\
=&\int_{S\times S} h(\tilde{x}) h(\bar{x})\tilde{\Theta}(\bar{x},\tilde{x}) ~\mathrm{d}\tilde{x}\mathrm{d}\bar{x}\quad \text{ (use \eqref{Theta_form})}. %
\label{ob2}
\end{aligned}
\end{equation}
Comparing \eqref{ob1} and \eqref{ob2}, we have that optimization problems \eqref{kernel_norm_min_continuous} and \eqref{continuous_version_1} are equivalent if $\alpha(W^{(1)},b)$ has the form \eqref{form_of_alpha} and $g(x)$ has the form \eqref{form_of_g}. Moreover, if every function $g\in\mathcal{H}_{\tilde{\Theta}}(S)$ can be approximated by the shallow network, we can find $\alpha(W^{(1)},b)$ in form of \eqref{form_of_alpha} such that $g(x)$ is expressed in the form of \eqref{form_of_g}.
In this sense we show that optimization problems \eqref{kernel_norm_min_continuous} and \eqref{continuous_version_1} are equivalent. 
\end{proof}

In Section \ref{sec:implicit_bias_univariate}, we relax the optimization problem \eqref{continuous_new} to \eqref{continuous_add_linear} in order to characterize the implicit bias in function space. This relaxation can also be done in 
the NTK norm minimization setting. 
It means that we can equivalently relax the problem \eqref{kernel_norm_min_continuous} to the following problem:
\begin{equation}
  \min_{g\in\mathcal{H}_{\tilde{\Theta}}(S),u\in\mathbb{R},v\in\mathbb{R}} \|g-ux-v\|_{\tilde{\Theta}} \quad \text{s.t. } g(x_j)=y_j,\ j=1,\ldots,M.
  \label{kernel_norm_min_continuous_add_linear}
\end{equation}
Then the optimization problems \eqref{continuous_add_linear} and \eqref{kernel_norm_min_continuous_add_linear} are equivalent. Theorem \ref{theorem_func} shows that  \eqref{continuous_add_linear} and \eqref{function_space} have the same solution on the set $S = \operatorname{supp}(\zeta) \cap [\min_i x_i, \max_i x_i]$. Then we have that optimization problems \eqref{kernel_norm_min_continuous_add_linear} and \eqref{function_space} are equivalent, which means that
\begin{equation}
    \min_{u\in\mathbb{R},v\in\mathbb{R}} \|g-ux-v\|_{\tilde{\Theta}}=\int_{S}
\frac{(g''(x))^2}{\zeta(x)}~\mathrm{d}x, \quad \forall g\in\mathcal{H}_{\tilde{\Theta}}(S) . 
\label{des_kernel_norm}
\end{equation}
Next, we directly prove the above equation \eqref{des_kernel_norm}. Given function $g\in\mathcal{H}_{\tilde{\Theta}}(S)$, let $h=\operatorname{argmin}_{h\in\mathcal{H}_{\tilde{\Theta}}(S)}\|h\|_{\tilde{\Theta}}$, s.t.\ $h=g-ux-v$ for some $u\in\mathbb{R},v\in\mathbb{R}$. Then according to optimality of $h$, we have $\langle h,x\rangle_{\tilde{\Theta}}=0$ and $\langle h,1\rangle_{\tilde{\Theta}}=0$. Consider the space $G=\{h\in\mathcal{H}_{\tilde{\Theta}}(S):\langle h,x\rangle_{\tilde{\Theta}}=0, \langle h,1\rangle_{\tilde{\Theta}}=0\}$, 
which is the orthogonal complement of $\operatorname{span}\{1,x\}$. Then $h$ is the projection of $g$ on $G$. Since $h=g-ux-v$, $h''=g''$. So we can reformulate the equation \eqref{des_kernel_norm} which we want to prove in the following theorem:

\begin{theorem}[Explicit form of the kernel norm]
\label{th:explicitkernelnorm}
The kernel norm on the space $G=\{h\in\mathcal{H}_{\tilde{\Theta}}(S):\langle h,x\rangle_{\tilde{\Theta}}=0, \langle h,1\rangle_{\tilde{\Theta}}=0\}$ is given as follows:
\begin{equation}
  \|h\|^2_{\tilde{\Theta}}=\int_{S}
\frac{(h''(x))^2}{\zeta(x)}~\mathrm{d}x, \quad \forall h\in G. 
\label{des_kernel_norm_orthogonal_complement}
\end{equation}
\end{theorem}
This theorem gives the explicit form of the kernel norm in a subspace of $\mathcal{H}_{\tilde{\Theta}}(S)$. Next we prove the above theorem using the property of kernel norm. 
\begin{proof}[Proof of Theorem~\ref{th:explicitkernelnorm}]
Let $\tilde{\Theta}_x(\cdot)=\tilde{\Theta}(\cdot,x)$. We can find the orthogonal projection of $\tilde{\Theta}_x$ on space $G$, which is denoted by $\tilde{\Theta}_{x,G}$. 
Then we only need to prove that $\langle h,\tilde{\Theta}_{x,G}\rangle_{\tilde{\Theta}}=\int_{S}
\frac{h''(y)\tilde{\Theta}_{x,G}''(y)}{\zeta(y)}~\mathrm{d}y$ for any $h\in G$ and $x\in S$.

First, $\tilde{\Theta}_{x,G}=\tilde{\Theta}_{x}-ux-v$ for some constant $u,v\in\mathbb{R}$. Since $h\in G$, $\langle h,1\rangle_{\tilde{\Theta}}=0$ and $\langle h,x\rangle_{\tilde{\Theta}}=0$. Then we have
\begin{equation}
\begin{aligned}
    \langle h,\tilde{\Theta}_{x,G}\rangle_{\tilde{\Theta}}&=\langle h,\tilde{\Theta}_{x}-ux-v\rangle_{\tilde{\Theta}}\\
    &=\langle h,\tilde{\Theta}_{x}\rangle_{\tilde{\Theta}}-u\langle h,x\rangle_{\tilde{\Theta}}-v\langle h,1\rangle_{\tilde{\Theta}}\\
    &=\langle h,\tilde{\Theta}_{x}\rangle_{\tilde{\Theta}}\\
    &=h(x) \quad \text{(use the reproducing property of the kernel \eqref{kernel_norm_property2})} . 
\end{aligned}
\label{hernel_norm_lhs}
\end{equation}
Next, using the notation from Section~\ref{sec:implicit_bias_univariate} we have 
\begin{equation*}
\begin{aligned}
    \tilde{\Theta}_{x,G}''(y)&=(\tilde{\Theta}_{x}(y)-uy-v)''=\tilde{\Theta}_{x}(y)''=\frac{\partial^2}{\partial y^2}\tilde{\Theta}(x,y)\\
    &=\frac{\partial^2}{\partial y^2}\int_{\mathbb{R}^2} [W^{(1)}(x-c)]_+[W^{(1)}(y-c)]_+~\mathrm{d}\nu(W^{(1)},c) \quad \text{(use \eqref{Theta_form_nu})}\\
    &=\frac{\partial^2}{\partial y^2}\int_{\mathbb{R}^2}(W^{(1)})^2 [\operatorname{sign}(W^{(1)})(x-c)]_+[\operatorname{sign}(W^{(1)})(y-c)]_+~\mathrm{d}\nu_{\mathcal{W}|\mathcal{C}=c}(W^{(1)})\mathrm{d}\nu_\mathcal{C}(c)\\
    &=\frac{\partial^2}{\partial y^2}\int_{\mathbb{R}}\left(\mathbb{E}(\mathcal{W}^2    \mathbbm{1}(\mathcal{W}\geq0)|\mathcal{C}=c) [x-c]_+[y-c]_+\right.\\
    &\quad\quad\quad\quad \left.+\mathbb{E}(\mathcal{W}^2    \mathbbm{1}(\mathcal{W}<0)|\mathcal{C}=c) [c-x]_+[c-y]_+\right)p_\mathcal{C}(c)~\mathrm{d}c\\
    &=\int_{\mathbb{R}}\left(\mathbb{E}(\mathcal{W}^2    \mathbbm{1}(\mathcal{W}\geq0)|\mathcal{C}=c) [x-c]_+\frac{\partial^2}{\partial y^2}[y-c]_+\right.\\
    &\quad\quad \left.+\mathbb{E}(\mathcal{W}^2    \mathbbm{1}(\mathcal{W}<0)|\mathcal{C}=c) [c-x]_+\frac{\partial^2}{\partial y^2}[c-y]_+\right)p_\mathcal{C}(c)~\mathrm{d}c\\
    &=\int_{\mathbb{R}}\left(\mathbb{E}(\mathcal{W}^2    \mathbbm{1}(\mathcal{W}\geq0)|\mathcal{C}=c) [x-c]_+\delta(y-c)\right.\\
    &\quad\quad \left.+\mathbb{E}(\mathcal{W}^2    \mathbbm{1}(\mathcal{W}<0)|\mathcal{C}=c) [c-x]_+\delta(y-c)\right)p_\mathcal{C}(c)~\mathrm{d}c\\
    &=\left(\mathbb{E}(\mathcal{W}^2    \mathbbm{1}(\mathcal{W}\geq0)|\mathcal{C}=y) [x-y]_++\mathbb{E}(\mathcal{W}^2    \mathbbm{1}(\mathcal{W}<0)|\mathcal{C}=y) [y-x]_+\right)p_\mathcal{C}(y) . 
\end{aligned}
\end{equation*}
Then we have
\begin{equation*}
    \begin{aligned}
&\int_{S}
\frac{h''(y)\tilde{\Theta}_{x,G}''(y)}{\zeta(y)}~\mathrm{d}y\\
=&\int_{S}
\frac{h''(y)\left(\mathbb{E}(\mathcal{W}^2    \mathbbm{1}(\mathcal{W}\geq0)|\mathcal{C}=y) [x-y]_++\mathbb{E}(\mathcal{W}^2 \mathbbm{1}(\mathcal{W}<0)|\mathcal{C}=y) [y-x]_+\right)p_\mathcal{C}(y)}{\zeta(y)}~\mathrm{d}y\\
=&\int_{S}
\frac{h''(y)\left(\mathbb{E}(\mathcal{W}^2    \mathbbm{1}(\mathcal{W}\geq0)|\mathcal{C}=y) [x-y]_++\mathbb{E}(\mathcal{W}^2    \mathbbm{1}(\mathcal{W}<0)|\mathcal{C}=y) [y-x]_+\right)}{\mathbb{E}(\mathcal{W}^2|\mathcal{C}=y)}~\mathrm{d}y\\
=&\int_{S}
\frac{\mathbb{E}(\mathcal{W}^2    \mathbbm{1}(\mathcal{W}\geq0)|\mathcal{C}=y)}{\mathbb{E}(\mathcal{W}^2|\mathcal{C}=y)}h''(y)[x-y]_++\frac{\mathbb{E}(\mathcal{W}^2    \mathbbm{1}(\mathcal{W}<0)|\mathcal{C}=y)}{\mathbb{E}(\mathcal{W}^2|\mathcal{C}=y)}h''(y)[y-x]_+~\mathrm{d}y . 
    \end{aligned}
\end{equation*}
Now, if we regard $\int_{S}
\frac{h''(y)\tilde{\Theta}_{x,G}''(y)}{\zeta(y)}~\mathrm{d}y$ as a function of $x$, then we get
\begin{equation*}
    \begin{aligned}
&\frac{\partial^2}{\partial x^2}\int_{S}
\frac{h''(y)\tilde{\Theta}_{x,G}''(y)}{\zeta(y)}~\mathrm{d}y\\
=&\frac{\partial^2}{\partial x^2}\int_{S}
\frac{\mathbb{E}(\mathcal{W}^2    \mathbbm{1}(\mathcal{W}\geq0)|\mathcal{C}=y)}{\mathbb{E}(\mathcal{W}^2|\mathcal{C}=y)}h''(y)[x-y]_++\frac{\mathbb{E}(\mathcal{W}^2    \mathbbm{1}(\mathcal{W}<0)|\mathcal{C}=y)}{\mathbb{E}(\mathcal{W}^2|\mathcal{C}=y)}h''(y)[y-x]_+~\mathrm{d}y\\
=&\int_{S}
\frac{\mathbb{E}(\mathcal{W}^2    \mathbbm{1}(\mathcal{W}\geq0)|\mathcal{C}=y)}{\mathbb{E}(\mathcal{W}^2|\mathcal{C}=y)}h''(y)\delta(x-y)+\frac{\mathbb{E}(\mathcal{W}^2    \mathbbm{1}(\mathcal{W}<0)|\mathcal{C}=y)}{\mathbb{E}(\mathcal{W}^2|\mathcal{C}=y)}h''(y)\delta(y-x)~\mathrm{d}y\\
=&\frac{\mathbb{E}(\mathcal{W}^2    \mathbbm{1}(\mathcal{W}\geq 0)|\mathcal{C}=x)}{\mathbb{E}(\mathcal{W}^2|\mathcal{C}=x)}h''(x)+\frac{\mathbb{E}(\mathcal{W}^2    \mathbbm{1}(\mathcal{W}<0)|\mathcal{C}=x)}{\mathbb{E}(\mathcal{W}^2|\mathcal{C}=x)}h''(x)\\
=&h''(x) . 
    \end{aligned}
\end{equation*}
From the definition of the space $G$, we see that the second derivative uniquely determines the element in $G$. Since $h\in G$, in order to show that $\int_{S}
\frac{h''(y)\tilde{\Theta}_{x,G}''(y)}{\zeta(y)}~\mathrm{d}y=h(x)$, we only need to show $\int_{S}
\frac{h''(y)\tilde{\Theta}_{x,G}''(y)}{\zeta(y)}~\mathrm{d}y\in G$, i.e., $\langle\int_{S}
\frac{h''(y)\tilde{\Theta}_{x,G}''(y)}{\zeta(y)}~\mathrm{d}y,1\rangle_{\tilde{\Theta}}=0$ and $\langle\int_{S}
\frac{h''(y)\tilde{\Theta}_{x,G}''(y)}{\zeta(y)}~\mathrm{d}y,x\rangle_{\tilde{\Theta}}=0$. Then we get
\begin{equation*}
    \begin{aligned}
    \langle\int_{S}
\frac{h''(y)\tilde{\Theta}_{x,G}''(y)}{\zeta(y)}~\mathrm{d}y,1\rangle_{\tilde{\Theta}} %
=&\langle\int_{S}
\frac{h''(y)\frac{\partial^2}{\partial y^2}\tilde{\Theta}(x,y)}{\zeta(y)}~\mathrm{d}y,1\rangle_{\tilde{\Theta}}\\
=&\langle\int_{S}
\frac{h''(y)\lim_{h\to 0}\frac{\tilde{\Theta}(x,y+h)-2\tilde{\Theta}(x,y)+\tilde{\Theta}(x,y-h)}{h^2}}{\zeta(y)}~\mathrm{d}y,1\rangle_{\tilde{\Theta}}\\
=&\lim_{h\to 0}\langle\int_{S}
\frac{h''(y)\frac{\tilde{\Theta}(x,y+h)-2\tilde{\Theta}(x,y)+\tilde{\Theta}(x,y-h)}{h^2}}{\zeta(y)}~\mathrm{d}y,1\rangle_{\tilde{\Theta}}\\
=&\lim_{h\to 0}\int_{S}
\frac{h''(y)\frac{\langle\tilde{\Theta}(x,y+h),1\rangle_{\tilde{\Theta}}-2\langle\tilde{\Theta}(x,y),1\rangle_{\tilde{\Theta}}+\langle\tilde{\Theta}(x,y-h),1\rangle_{\tilde{\Theta}}}{h^2}}{\zeta(y)}~\mathrm{d}y\\
=&\lim_{h\to 0}\int_{S}
\frac{h''(y)\frac{y+h-2y+y-h}{h^2}}{\zeta(y)}~\mathrm{d}y\\
=&0 . 
    \end{aligned}
\end{equation*}
Similarly we can show that $\langle\int_{S}
\frac{h''(y)\tilde{\Theta}_{x,G}''(y)}{\zeta(y)}~\mathrm{d}y,x\rangle_{\tilde{\Theta}}=0$. This concludes the proof. 
\end{proof}

\section{Gradient Descent Trajectory and Trajectory of Smoothing Splines for Univariate Regression}
\label{appendix:smoothingspline}

In the following we discuss the relation between the trajectory of functions obtained by gradient descent training of a neural network and a trajectory of solutions to 
the variational problem with the data fitting constraints replaced by a MSE for decreasing smoothness regularization strength. 
This Lagrange version of the variational problem is solved by so-called smoothing splines. 
Smoothing splines have been studied intensively in the literature and in particular they can be written explicitly. 
We give the explicit form of the solution for the trajectory in the context of our discussion.

\subsection{Regularized Regression and Early Stopping}

\cite{bishop1995regularization} shows that for linear regression with quadratic loss, early stopping and $L_2$ regularization lead to similar solutions. 
Let us recall some details of his analysis, before proceeding with our particular setting. 
He considers the loss function $E(\mathbf{w})=\|X\mathbf{w}-\mathbf{y}\|_2^2$, where $X=[\mathbf{x}_1,\ldots,\mathbf{x}_M]^T$ is the matrix of training inputs, $\mathbf{y}=[y_1,\ldots,y_M]^T$ is the vector %
of training outputs, and $\mathbf{w}$ is the weight vector of the linear model. 
Next the loss function can be written in the form of a quadratic function: 
\begin{equation*}
    \begin{aligned}
    E(W)&=\|X\mathbf{w}-\mathbf{y}\|_2^2\\
    &=\mathbf{w}^TX^TX\mathbf{w}-2\mathbf{y}^TX\mathbf{w}+\mathbf{y}^T\mathbf{y}\\
    &=\mathbf{w}^TX^TX\mathbf{w}-2\mathbf{y}^TX\mathbf{w}+\mathbf{y}^T\mathbf{y}\\
    &=\frac{1}{2}(\mathbf{w}-\mathbf{w}^*)^TH(\mathbf{w}-\mathbf{w}^*)+E_0,\\
    \end{aligned}
\end{equation*}
where $H=2X^TX$, $E_0$ is the minimum of the loss function, and $\mathbf{w}^*$ is the %
minimizer. %
The eigenvalues and eigenvectors of $H$ are as follows: 
\begin{equation*}
    H\mathbf{u}_j=\lambda_j\mathbf{u}_j. 
\end{equation*}
Then expand $\mathbf{w}$ and $\mathbf{w}^*$ in terms of the eigenvectors of $H$:
\begin{equation*}
    \mathbf{w}=\sum_j w_j\mathbf{u}_j,\quad \quad \mathbf{w}^*=\sum_j w_j^*\mathbf{u}_j. 
\end{equation*}
For the $L_2$ regularized regression problem, consider the regularized loss function
$\tilde{E}(\mathbf{w})=E(\mathbf{w})+c\|\mathbf{w}\|^2_2$.  Denote the minimizer by $\mathbf{w}=\tilde{\mathbf{w}}$ and consider its expansion as $\tilde{\mathbf{w}}=\sum_j \tilde{w}_j\mathbf{u}_j$. \cite{bishop1995regularization} shows that 
\begin{equation}
    \tilde{w}_j=\frac{\lambda_j}{\lambda_j+c}w_j^* . 
    \label{regularization}
\end{equation}
For early stopping, consider the gradient descent on $E(\mathbf{w})$ with zero initial weight vector: 
\begin{equation*}
\begin{aligned}
    \mathbf{w}^{(\tau)}&= \mathbf{w}^{(\tau-1)}-\eta \nabla E\\
    &= \mathbf{w}^{(\tau-1)}-\eta H(\mathbf{w}^{(\tau-1)}-\mathbf{w}^*),\\
    \mathbf{w}^{(0)}&= \mathbf{0}. 
\end{aligned}
\end{equation*}
Writing $\mathbf{w}^{(\tau)}=\sum_j w^{(\tau)}_j\mathbf{u}_j$, we have
\begin{equation*}
     w^{(\tau)}_j=(1-(1-\eta\lambda_j)^\tau) w^{*}_j. 
\end{equation*}
Note that $1-(1-\eta\lambda_j)^\tau\to1-e^{-\eta\tau\lambda_j}$ as $\eta\to 0$. Hence choosing a sufficiently small learning rate, approximately we have
\begin{equation}
     w^{(\tau)}_j=(1-e^{-\eta\tau\lambda_j}) w^{*}_j . 
     \label{early_stopping}
\end{equation}
From \eqref{regularization} and \eqref{early_stopping}, \cite{bishop1995regularization} observes that if $c$ is much larger than $\lambda_j$, then the regularized solution has coordinate $\tilde w_j$ close to $0$, 
and similarly if $1/(\eta \tau)$ is much larger than $\lambda_j$, then the early-stopping solution has coordinate $w_j^{(\tau)}$ close to the initial value $0$. 
We note that analogous observations apply when the regularization term has a reference point different from zero, $c\|\mathbf{w} - \overline{\mathbf{w}}\|_2^2$, 
and the gradient descent iteration is initialized at a point different from zero, $\mathbf{w}^{(0)} = \overline{\mathbf{w}}$. 

Now we want to take a closer look at the trajectories. 
Consider the following two functions:
\begin{equation*}
    h_1(x)=\frac{\lambda_j}{\lambda_j+x}, \quad\quad h_2(x)=1-e^{-\lambda_j/x}.
\end{equation*}
Actually we can verify that $h_1(0)=h_2(0)=1$ and $\lim_{x\to\infty} \frac{h_1(x)}{h_2(x)}=1$. It implies that these two functions are close to each other on $[0,\infty)$. Figure \ref{fig:two_h} shows the plot of functions $h_1(x)$ and $h_2(x)$. 
\begin{figure}
    \centering
    \includegraphics[width=.49\textwidth]{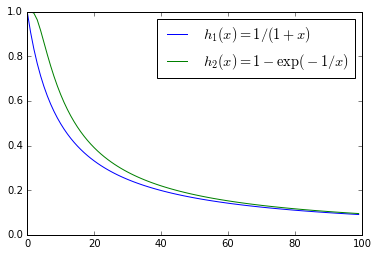}
    \includegraphics[width=.49\textwidth]{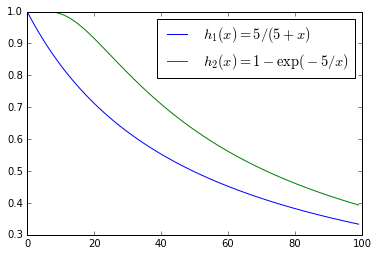}
    \caption{Plot of functions $h_1(x)$ and $h_2(x)$. The left panel plots the two function when $\lambda_j=1$. The right panel plots the two function when $\lambda_j=5$.}
    \label{fig:two_h}
\end{figure}

Now we choose the coefficient of regularization $c=\frac{1}{\eta\tau}$. Comparing \eqref{regularization} and \eqref{early_stopping}, and using the fact that $h_1(x)$ and $h_2(x)$ are close to each other on $[0,\infty)$, we show that early stopping and $L_2$ regularization lead to similar solutions across different values of $c=\frac{1}{\eta\tau}$. 

Back to our problem, we repeat the gradient descent procedures \eqref{fix_first} here:
\begin{equation*}
  \widetilde{W}^{(2)}_0=\overline{W}^{(2)},\quad 
  \widetilde{W}^{(2)}_{t+1}=\widetilde{W}^{(2)}_{t}- \eta\nabla_{W^{(2)}} L^{\mathrm{lin}}(\widetilde{\omega}_t). 
\end{equation*}
It is actually minimizing the following loss function of $W^{(2)}-\overline{W}$:
\begin{equation*}
    E(W^{(2)}-\overline{W})=\sum_{j=1}^M\left(\sum_{i=1}^n(W_i^{(2)}-\overline{W}_i^{(2)})[W_i^{(1)}x_j+b_i]_+-(y_j-f(x_{j}, \theta_0))\right)^2. 
\end{equation*}
Here we change the variable from $W^{(2)}$ to $W^{(2)}-\overline{W}$. 
Then $W_t^{(2)}-\overline{W}=0$ when $t=0$, so that gradient descent starts from the zero initial weight vector. 
Since the above model is linear with respect to $W^{(2)}-\overline{W}$,  we can apply the above argument about early stopping and $L_2$ regularization. Suppose that we use learning rate $\mu_n$ for the neural network of width $n$. We show that the solution $\widetilde{W}^{(2)}_t$ at iteration $t$ is close to the minimizer of the following regularized optimization problem: 
\begin{equation}
    \min_{W^{(2)}} \sum_{j=1}^M\left(\sum_{i=1}^n(W_i^{(2)}-\overline{W}_i^{(2)})[W_i^{(1)}x_j+b_i]_+-(y_j-f(x_{j}, \theta_0))\right)^2+c\|W^{(2)}-\overline{W}\|_2^2,
\label{regularized_problem}
\end{equation}
where $c=\frac{1}{\eta_n t}$. Using the same approach and notation as in Section~\ref{2.4}, the optimization problem \eqref{regularized_problem} is equivalent to 
\begin{equation}
\begin{aligned}
 \min_{\alpha_n\in C(\mathbb{R}^2)}\quad & \sum_{j=1}^M\left(\int_{\mathbb{R}^2} \alpha_n(W^{(1)},b)[W^{(1)}x_j+b]_+~\mathrm{d}\mu_n(W^{(1)},b)-y_j\right)^2\\
 &+\frac{1}{n\eta_n t}\int_{\mathbb{R}^2} \alpha_n^2(W^{(1)},b)~\mathrm{d}\mu_n(W^{(1)},b),
\end{aligned}
\label{regularized_probabilty}
\end{equation}
where we use the ASI trick (see Appendix~\ref{app:ASI}). 
Here \eqref{regularized_probabilty} has an extra factor $\frac{1}{n}$ compared to \eqref{regularized_problem}. 
This is because we define  $\alpha_n(W_i^{(1)},b_i)=n(W_i^{(2)}-\overline{W}_i^{(2)})$. According to Theorem ~\ref{minimum_weight}, $\eta_n\leq \frac{M}{Kn\lambda_{\max}(\hat{\Theta}_n)}$ is sufficient in order to ensure convergence. %
Then we suppose that $\eta_n=\bar{\eta}/n$, where $\bar{\eta}$ is a constant so that the %
requirement on the learning rate in Theorem~\ref{minimum_weight} is satisfied. 
The limit of the optimization problem \eqref{regularized_probabilty} as the width $n$ tends to infinity is:
\begin{equation}
\begin{aligned}
 \min_{\alpha\in C(\mathbb{R}^2)}\quad & \sum_{j=1}^M\left(\int_{\mathbb{R}^2} \alpha(W^{(1)},b)[W^{(1)}x_j+b]_+~\mathrm{d}\mu(W^{(1)},b)-y_j\right)^2\\
 &+\frac{1}{\bar{\eta} t}\int_{\mathbb{R}^2} \alpha^2(W^{(1)},b)~\mathrm{d}\mu(W^{(1)},b) . 
\end{aligned}
\label{regularized_continuous}
\end{equation}
Following the same reasoning of Section~\ref{sec:implicit_bias_univariate}, we relax the optimization problem \eqref{regularized_continuous} to the following one:
\begin{equation}
\begin{aligned}
 \min_{\alpha\in C(\mathbb{R}^2), u\in\mathbb{R}, v\in\mathbb{R}}\quad & \sum_{j=1}^M\left(ux_j+v+\int_{\mathbb{R}^2} \alpha(W^{(1)},b)[W^{(1)}x_j+b]_+~\mathrm{d}\mu(W^{(1)},b)-y_j\right)^2\\
 &+\frac{1}{\bar{\eta} t}\int_{\mathbb{R}^2} \alpha^2(W^{(1)},b)~\mathrm{d}\mu(W^{(1)},b) . 
\end{aligned}
\label{regularized_continuous_relax}
\end{equation}
Using the same technique and notation as in Theorem~\ref{theorem_func}, we can prove that the solution of \eqref{regularized_continuous_relax} actually solves the following optimization problem: 
\begin{equation}
  \label{regularized_function_space}
  \min_{h\in C^2(S)}
\sum_{j=1}^M\left[ h(x_j)-y_j\right]^2+\frac{1}{\bar{\eta} t}
\int_{S}
\frac{(h''(x))^2}{\zeta(x)}~\mathrm{d}x .
\end{equation}
Then in order to study the trajectory of gradient descent, we can study the optimization problem \eqref{regularized_function_space} with varying $t$. Figure~\ref{fig:trajectories} illustrates smoothing spline and gradient descent trajectories. The solution of \eqref{regularized_function_space} is called spatially adaptive smoothing spline. Here the curvature penalty function is $\frac{1}{\bar\eta t} \frac{1}{\zeta(x)}$, with time dependent smoothness regularization coefficient $\frac{1}{\bar \eta t}$. Next, we give the solution of \eqref{regularized_function_space} in the following two cases: (1) uniform case ($\zeta$ is constant over domain $S$); (2) spatially adaptive case ($\zeta$ is not constant over domain $S$). 
\begin{figure}
    \centering
    \begin{tikzpicture}
    \node at (0,0) {\includegraphics[clip=true, trim= 2.5cm 7cm 2.25cm 7cm, width=7cm]{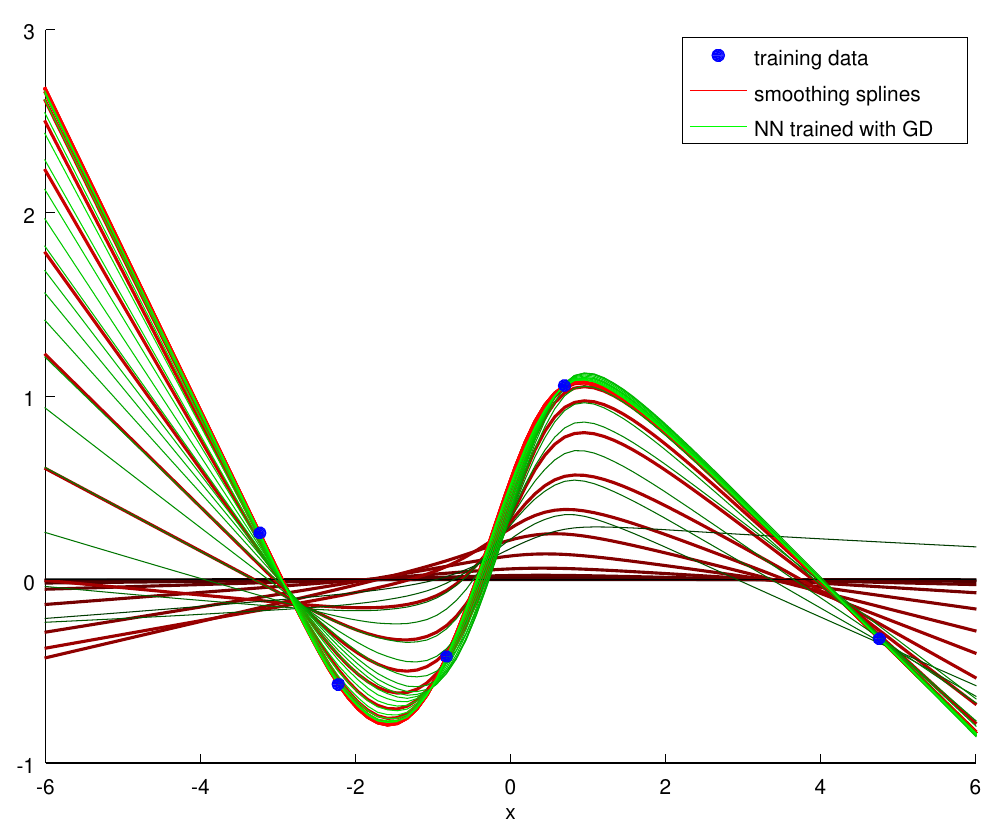}};
    \node at (7,0) {\includegraphics[clip=true, trim=  2.5cm 7cm 2.25cm 7cm, width=7cm]{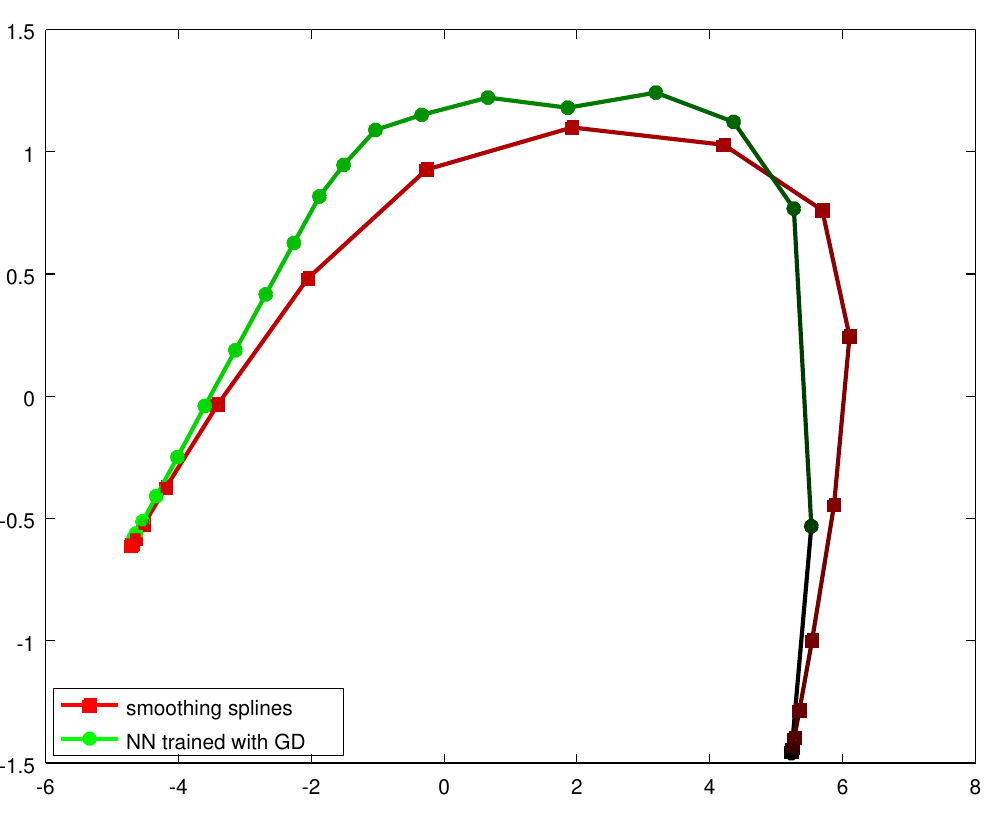}};
    \node at (0,3.1) {\footnotesize\textsf{Trajectories of functions}};
    \node at (7,3.1) {\footnotesize\textsf{2D PCA of the trajectories}};
    \end{tikzpicture}
    \caption{Trajectories of functions obtained by gradient descent training a neural network and by smoothing splines of the training data with decreasing regularization strength (from dark to bright). The left panel plots $20$ functions along each trajectory. The right panel shows the same functions in a two dimensional PCA representation. With asymmetric initialization of the network parameters and adjusting the training data by ordinary linear regression, both trajectories start at the zero function. The trajectories are not equivalent, but are close, and both converge to the same (spatially adaptive) cubic spline interpolation of the training data (in the limit of infinite wide networks). 
    Here we used a large network with $n=2000$ hidden units and Gaussian initialization $\mathcal{W}\sim \mathcal{N}(0,1)$, $\mathcal{B}\sim \mathcal{N}(0,1)$. The results are similar for smaller networks and different initializations.}
    \label{fig:trajectories}
\end{figure}

\begin{remark}[Spectral bias]
We have thus that the gradient descent optimization trajectory can be described approximately by a trajectory of smoothing splines which gradually relaxes the smoothness regularization (relative to initialization) until perfectly fitting the training data. If the function at initialization is at the zero function, e.g.,\ by ASI, then the regularization is on the function itself. 
Hence the result provides a theoretical explanation for the 
spectral bias phenomenon that has been observed by \cite{pmlr-v97-rahaman19a}. 
The spectral bias is that lower frequencies are learned first. 
\end{remark}

\subsection{Trajectory of Smoothing Splines with Uniform Curvature Penalty} 
Suppose the reciprocal curvature penalty is constant $\zeta(x)\equiv z$ on the domain $S$. Let $\lambda=\frac{1}{\bar{\eta}tz}$. Then \eqref{regularized_function_space} becomes the following optimization problem:
\begin{equation}
  \label{regularized_function_space_uniform_appendix}
  \min_{h\in C^2(S)}
\sum_{j=1}^M\left[ h(x_j)-y_j\right]^2+\lambda
\int_{S}
(h''(x))^2~\mathrm{d}x . 
\end{equation}
\cite{german2001smoothing} gives the explicit form of the minimizer $\hat h$ of \eqref{regularized_function_space_uniform_appendix}, which is called a smoothing spline. 
The minimizer $\hat h$ is a natural cubic spline with knots at the sample points $x_1,\ldots,x_M$. 
The %
smoothing spline does not fit the training data exactly, but rather it balances fitting and smoothness. 
The smoothing parameter $\lambda\geq0$ controls the trade off between fitting and roughness. 
The values of the smoothing spline at the knots can be obtained as 
\begin{equation}
(\hat h(x_1),\ldots, \hat h(x_M))^\top = (I + \lambda A)^{-1} Y. 
\label{eq:values-smoothing-spline}
\end{equation}
The matrix $A$ has entries $A_{ij} = \int_S h''_i(x) h''_j(x)~\mathrm{d}x$, where $h_i$ are spline basis functions which satisfy $h_i(x_j)=0$ for $j\not=i$ and $h_i(x_j)=1$ for $j=i$. 
\cite{german2001smoothing} gives a rather explicit form of matrix $A$, which is an $M\times M$ matrix given by $A=\Delta^TW^{-1}\Delta$. 
Here $\Delta$ is an $(M-2)\times M$ matrix of second differences with elements: 
\[
\Delta_{ii}=\frac{1}{h_i},\quad\Delta_{i,i+1}=-\frac{1}{h_i}-\frac{1}{h_{i+1}},\quad\Delta_{i,i+2}=\frac{1}{h_{i+1}} . 
\]
And $W$ is an $(M-2)\times (M-2)$ symmetric tri-diagonal matrix with elements:
\[
W_{i-1,i}=W_{i,i-1}=\frac{h_i}{6},\quad W_{i,i}=\frac{h_i+h_{i+1}}{3}, \text{ here } h_i=x_{i+1}-x_i . 
\]
As $\lambda\to 0$, the smoothing spline converges to the interpolating spline, and as $\lambda \to \infty$, it converges to the linear least squares estimate. 

\subsection{Trajectory of Spatially Adaptive Smoothing Splines} 

Let the curvature penalty $\rho(x)=\frac{1}{\bar\eta t} \frac{1}{\zeta(x)}\frac{1}{M}$. 
Then \eqref{regularized_function_space} can be written as 
\begin{equation}
  \label{regularized_function_space_adaptive}
\min_{h\in W_2(S)} \frac{1}{M}\sum_{i=1}^M\left[ h(x_j)-y_j\right]^2 + \int_S \rho(x) (h''(x))^2~\mathrm{d}x,
\end{equation}
where $W_2(S) = \{f\colon f, f' \text{ absolutely continuous and } f''\in L^2(S)\}$, with $L^2(S)$ the square integrable functions over the domain $S$. 
\citet{ABRAMOVICH1996327, 10.1093/biomet/93.1.113} give the solution of \eqref{regularized_function_space_adaptive} explicitly, which is called a spatially adaptive smoothing spline. 

According to \citet{10.1093/biomet/93.1.113}, the solution can be derived in terms of an appropriate RKHS representation of $W_2^0$ with inner product $\langle f,g \rangle_\rho = \int f''(x) g''(x) \rho(x)~\mathrm{d}x$. 
Here $W_0^2(S)=W_2(S)\cap B_2(S)$, 
where $W_2(S)$ is defined above,
and $B_2(S) = \{f: f(0)=f'(0)=0 \}$. Notice that when defining $B_2(S)$ we need $0\in S$. Actually we can choose any point in $S$. \citet{10.1093/biomet/93.1.113} define $B_2(S)$ in this way just for simplicity.
Then the kernel of the space $W_0^2(S)$ is given by
\begin{equation}
\label{kernel_rho_0}
K_{\rho}(x_1,x_2)=\int_S \rho(u)^{-1}[x_1-u]_+ [x_2-u]_+ \mathrm{d}u.
\end{equation}
Then the minimizer $\hat{h}$ of \eqref{regularized_function_space_adaptive} is given by
\begin{equation}
\label{solution_adaptive}
    \hat{h}(x)=\sum_{j=1}^M c_j K_{\rho}(x_j,x)+a+bx.
\end{equation}
Now define the $M\times M$ matrix
\begin{equation}
\label{def_Sigma}
    \Sigma_\rho=\{K_{\rho}(x_i,x_j)\}_{i,j=1,\ldots,M}, 
\end{equation}
and the $M\times 2$ matrix
\begin{equation}
\label{def_T}
    T=\begin{bmatrix}
1 &x_1\\
1 &x_2\\
\vdots& \vdots\\ 
1 &x_M\\
\end{bmatrix}. 
\end{equation}
Denote the vector of coefficients $\mathbf{c}=(c_1, \ldots , c_M)^T$ and the vector of output values $\mathbf{y}=(y_1, \ldots , y_M)^T$. 
Then the coefficients in \eqref{solution_adaptive} satisfy the following conditions: 
\begin{equation}
\label{conditions_coefficients}
    \Sigma_\rho\left[(\Sigma_\rho+MI)\mathbf{c}+T\begin{pmatrix}
    a\\b
    \end{pmatrix}\right]=\Sigma_\rho \mathbf{y} \quad\text{and}\quad T^\top\left[\Sigma_\rho\mathbf{c}+T\begin{pmatrix}
    a\\b
    \end{pmatrix}\right]=T^\top\mathbf{y}.
\end{equation}
After solving for \eqref{conditions_coefficients}, we get the values of $\mathbf{c}$, $a$ and $b$. Plug them into \eqref{solution_adaptive}, then we get the exact form of the minimizer of \eqref{regularized_function_space_adaptive}.

\section{Solution to the Variational Problems for Univariate Regression after Training}
\label{appendix:splines}

\subsection{Interpolating Splines with Uniform Curvature Penalty}
Theorem~\ref{proposition:explicit-rho}\,\ref{pro9:bin} and \ref{pro9:unif} show that for certain distributions of $(\mathcal{W},\mathcal{B})$, $\zeta$ is constant. In this case problem~\eqref{main_result} with ASI is solved by the cubic spline interpolation of the data with natural boundary conditions \citep{1967theory}. 
\begin{theorem}
[\citealt{1967theory}]
\label{thm:cubic}
  For training samples $\{(x_i,y_i)\}_{i=1}^M$, suppose $x_j\in S,~j=1,\ldots,M$. Then cubic spline interpolation of data $\{(x_i,y_i)\}_{i=1}^M$ with natural boundary condition is the solution of
  \begin{equation*}
    \begin{aligned}
    \min_{h\in C^2(S)} \quad &
      \int_{S} (h''(x))^2\mathrm{d}x\\
      \textup{subject to}\quad & h(x_j)=y_j, \quad j=1,\ldots,m.
    \end{aligned}
  \end{equation*}
\end{theorem}

As already mentioned in Appendix~\ref{appendix:smoothingspline}, cubic spline interpolation is a finite dimensional linear problem and can be solved exactly. 
A cubic spline is a piecewise polynomial of order $3$ with $(M-1)$ pieces. The $j$-th piece has the form $S_j(x) = a_j +b_jx + c_jx^2 +d_jx^3$, $j=1,\ldots,M-1$. 
These $(M-1)$ pieces satisfy equations $S_i(x_i)=y_i$, $S_{i}(x_{i+1})=y_{i+1}$, $i=1,\ldots,M-1$ and $S'_i(x_{i+1})=S'_{i+1}(x_{i+1})$, $S''_i(x_{i+1})=S''_{i+1}(x_{i+1})$, $i=1,\ldots,M-2$, and $S''_1(x_1)=S''_{M-1}(x_{M})=0$. 
Hence computing the spline amounts to solving a linear system in $4(M-1)$ indeterminates. 

\subsection{Spatially Adaptive Interpolating Splines}
In the case that $\zeta$ is not constant, we can still give the form of the solution to the variational problem~\eqref{main_result} with ASI by using the result in Appendix \ref{appendix:smoothingspline}. We multiply by a coefficient $\lambda$ the regularization term in the optimization problem \eqref{regularized_function_space_adaptive} and choose $\rho(x)=\frac{1}{\zeta(x)}$. Then we get 
\begin{equation}
  \label{regularized_function_space_adaptiv_lambda}
\min_{h\in W_2(S)} \frac{1}{M}\sum_{i=1}^M\left[ h(x_j)-y_j\right]^2 + \lambda \int_S \frac{1}{\zeta(x)} (h''(x))^2~\mathrm{d}x.
\end{equation}
As $\lambda\to 0$, the minimizer of \eqref{regularized_function_space_adaptiv_lambda} converges to the solution of the following optimization problem:
\begin{equation*}
  \min_{h\in W^2(S)}
\int_{S}
\frac{(h''(x))^2}{\zeta(x)}~\mathrm{d}x %
\quad 
\text{s.t.}\quad h(x_j)=y_j,\quad j=1,\ldots,m,
\end{equation*}
which is the variational problem~\eqref{main_result} with ASI. According to Appendix \ref{appendix:smoothingspline}, the solution of \eqref{regularized_function_space_adaptiv_lambda} is given by:
\begin{equation}
\label{sol_lambda}
\hat{h}^{(\lambda)}(x)=\sum_{j=1}^M c^{(\lambda)}_j K_{\frac{\lambda}{\zeta}}(x_j,x)+a^{(\lambda)}+b^{(\lambda)}x.
\end{equation}
And the vector $\mathbf{c}^{(\lambda)}=(c^{(\lambda)}_1, \ldots , c^{(\lambda)}_M)^T$, $a^{(\lambda)}$ and $b^{(\lambda)}$ satisfy the following conditions:
\begin{equation}
\label{conditions_coefficients_lambda}
    \Sigma_{\frac{\lambda}{\zeta}}\left[(\Sigma_{\frac{\lambda}{\zeta}}+MI)\mathbf{c}^{(\lambda)}+T\begin{pmatrix}
    a^{(\lambda)}\\b^{(\lambda)}
    \end{pmatrix}\right]=\Sigma_{\frac{\lambda}{\zeta}} \mathbf{y} \quad\text{and}\quad T^\top\left[\Sigma_{\frac{\lambda}{\zeta}}\mathbf{c}^{(\lambda)}+T\begin{pmatrix}
    a^{(\lambda)}\\b^{(\lambda)}
    \end{pmatrix}\right]=T^\top\mathbf{y},
\end{equation}
where $K_{\frac{\lambda}{\zeta}}$, $\Sigma_{\frac{\lambda}{\zeta}}$ and $T$ are defined in \eqref{kernel_rho_0}, \eqref{def_Sigma} and \eqref{def_T}. Next we show that $K_{\frac{\lambda}{\zeta}}$ is inversely proportional to $\lambda$:
\begin{equation}
\label{kernel_rho}
\begin{aligned}
K_{\frac{\lambda}{\zeta}}(x_1,x_2)&=\int_S \left(\frac{\lambda}{\zeta}\right)^{-1}[x_1-u]_+ [x_2-u]_+ \mathrm{d}u\\
&=\lambda^{-1}\int_S \left(\frac{1}{\zeta}\right)^{-1}[x_1-u]_+ [x_2-u]_+ \mathrm{d}u\\
&=\lambda^{-1}K_{\frac{1}{\zeta}}(x_1,x_2).
\end{aligned}
\end{equation}
Also $\Sigma_{\frac{\lambda}{\zeta}}=\lambda^{-1}\Sigma_{\frac{1}{\zeta}}$. Then we let $\bar{c}^{(\lambda)}_j=\lambda^{-1}c^{(\lambda)}_j$ and $\bar{\mathbf{c}}^{(\lambda)}=\lambda^{-1}\mathbf{c}^{(\lambda)}$. So we can rewrite \eqref{sol_lambda} and \eqref{conditions_coefficients_lambda} as
\begin{equation}
\label{sol_new}
    \hat{h}^{(\lambda)}(x)=\sum_{j=1}^M \bar{c}^{(\lambda)}_j K_{\frac{1}{\zeta}}(x_j,x)+a^{(\lambda)}+b^{(\lambda)}x,
\end{equation}
where $\bar{\mathbf{c}}^{(\lambda)}$, $a^{(\lambda)}$ and $b^{(\lambda)}$ satisfy the following conditions:
\begin{equation}
\label{conditions_coefficients_new}
    \Sigma_{\frac{1}{\zeta}}\left[(\Sigma_{\frac{1}{\zeta}}+\lambda MI)\bar{\mathbf{c}}^{(\lambda)}+T\begin{pmatrix}
    a^{(\lambda)}\\b^{(\lambda)}
    \end{pmatrix}\right]=\Sigma_{\frac{1}{\zeta}} \mathbf{y} \quad\text{and}\quad T^\top\left[\Sigma_{\frac{1}{\zeta}}\bar{\mathbf{c}}^{(\lambda)}+T\begin{pmatrix}
    a^{(\lambda)}\\b^{(\lambda)}
    \end{pmatrix}\right]=T^\top\mathbf{y},
\end{equation}
Now, as $\lambda\to 0$, \eqref{sol_new} and \eqref{conditions_coefficients_new} become: 
\begin{equation}
\label{sol_new_limit}
    \hat{h}^{(0^+)}(x)=\sum_{j=1}^M \bar{c}^{(0^+)}_j K_{\frac{1}{\zeta}}(x_j,x)+a^{(0^+)}+b^{(0^+)}x,
\end{equation}
where $\bar{\mathbf{c}}^{(0^+)}$, $a^{(0^+)}$, and $b^{(0^+)}$ satisfy the following conditions: 
\begin{equation}
\label{conditions_coefficients_new_limit}
    \Sigma_{\frac{1}{\zeta}}\left[\Sigma_{\frac{1}{\zeta}}\bar{\mathbf{c}}^{(0^+)}+T\begin{pmatrix}
    a^{(0^+)}\\b^{(0^+)}
    \end{pmatrix}\right]=\Sigma_{\frac{1}{\zeta}} \mathbf{y} \quad\text{and}\quad T^\top\left[\Sigma_{\frac{1}{\zeta}}\bar{\mathbf{c}}^{(\lambda)}+T\begin{pmatrix}
    a^{(0^+)}\\b^{(0^+)}
    \end{pmatrix}\right]=T^\top\mathbf{y}. 
\end{equation}
The expressions \eqref{sol_new_limit} and \eqref{conditions_coefficients_new_limit} give the solution of \eqref{regularized_function_space_adaptiv_lambda} as $\lambda\to 0$, which is also the solution to the variational problem~\eqref{function_space}.

\section{Possible Generalizations}
\label{app:generalizations}
\subsection{Deep Networks and Other Architectures}
For deep networks with $L$ layers, if we only train the output layer, then we actually train a linear model. We can actually write down the exact form of the NTK. However it is unclear whether we can write the explicit form of implicit bias in this case.

In the case of shallow networks, we show that training only the output layer is similar to training all parameters. 
Our analysis of shallow networks is based on this. 
However, in the case of a deep network, training only the output layer is no longer similar to training all parameters. 
If we train all model parameters, the results from \cite{lee2019wide,lai2023generalization} show that the model still is approximated by a linearized model. 
The result on kernel norm minimization \citep{zhang2019type} holds in this case. 
It will be interesting to study the explicit form of the kernel norm, and extensions of our analysis to the case of training all parameters of deep networks.

\subsection{Other Loss Functions} 
We have focused on the implicit bias of gradient descent for regression. 
For this type of problems, one often considers a loss function (per example) which has a single finite minimum. 
Roughly speaking, our description of the bias is in terms of smoothness properties of the solution functions. 
There are various works on the implicit bias of gradient descent for classification problems, e.g., \cite{soudry2018implicit}. In this case, the implicit bias is often formulated in terms of maximum margins. 

In our analysis, some theorems require that the loss function is mean square error~(MSE). In Theorem \ref{th-lin-onlyout}, the gradient flow is a linear differential equation if we use MSE. 
If we use a different loss, this will be more complicated. However, we think that the results can be generalized. 
We are also using the result from \cite{lee2018deep}, which is based on MSE. 
According to them it is not clear whether their result will still apply for other loss functions. 
Theorems \ref{theorem4} and~\ref{theorem_func} are about a variational problem that is derived from Theorem \ref{minimum_weight}, in relation to the minimization of $\|\theta-\theta_2\|_2$. 
Theorem \ref{minimum_weight} remains valid for other loss functions beside MSE. 
To sum up, if we can generalize the Theorem \ref{th-lin-onlyout} and the result of \cite{lee2018deep} to other loss functions, then we can generalize our main result in Theorem \ref{thm:theorem1} to other loss functions as well.

\subsection{Other Optimization Procedures} 
\label{app:other-optimization}

It would be interesting to extend the analysis to modifications of the basic gradient descent optimization procedure. The implicit bias of different optimization methods has been studied by \citet{gunasekar2018characterizing} covering some instances of mirror descent, natural gradient descent, Adam, and steepest descent with respect to different potentials and norms. 
In particular, they show that the implicit bias of coordinate descent corresponds to the minimization of the 1-norm of the weights. 
It will be interesting to work out the explicit form of these descriptions in function space. 


\vskip 0.2in
\bibliography{main.bib}

\end{document}